\pdfoutput=1
\documentclass[11pt]{article}
\usepackage[OT1]{fontenc}
\usepackage[margin=1in]{geometry}
\usepackage{enumitem}
\usepackage{fullpage}
\usepackage{smile}

\newcommand{\mytag}[1]{\tag*{\color{blue} // #1}}

\newcommand{\iidfrom}{\overset{\iid}{\sim}}

\newlist{steps}{enumerate}{1}
\setlist[steps, 1]{label = Step \arabic*:}

\newcommand{\pmax}{\lor}
\newcommand{\pmin}{\land}
\newcommand{\warmup}{{\mathtt{warmup}}}
\newcommand{\snr}{\mathrm{SNR}}

\newcommand{\lbdaU}{{\lambda}_U}
\newcommand{\brown}{\color{brown}}

\newcommand{\Uthres}{{\mathtt{U\underline{~}thres}}}
\newcommand{\vzero}{\mathbf{0}}

\newcommand{\operYtoXbeta}{\mathcal{A}}
\newcommand{\operXtoYbeta}{\mathcal{B}}

\newcommand{\connectedcomponent}{\mathrm{CC}}

\DeclareMathOperator*{\Deg}{Deg}
\newcommand{\eigveq}{\simeq}

\usepackage{scalerel}
\newcommand{\ldeq}[1][1]{{=_{{\textstyle\mathstrut}\scaleto{\leqslant #1\mathstrut}{6.5pt}}}}
\DeclareMathOperator*{\efforder}{Ord}
\newcommand{\error}[1][1]{\mathrm{err}_{{\textstyle\mathstrut}\scaleto{> #1\mathstrut}{6.5pt}}}

\newcommand{\owntag}[1]{\stepcounter{equation}\tag{\theequation, #1}}

\newcommand{\herm}{\mathrm{He}}

\newlist{myenumi}{enumerate}{3}
\setlist[myenumi,1]{
    label=(\roman*),
    ref  =(\roman*), 
    }
\setlist[myenumi,2]{
    label=\textbf{(\alph*)},
    leftmargin=.75cm}
\setlist[myenumi,3]{label=\bfseries(\roman*),
                   ref  =\themyenumii\textbf{.(\roman*)}}

\crefname{myenumi}{}{}
\crefname{myenumii}{}{}
\crefname{myenumiii}{}{}

\newlist{factenum}{enumerate}{3}
\setlist[factenum,1]{
    label={(\textbf{\arabic*})},
    ref  =(\textbf{\arabic*}), 
    leftmargin=1.5cm}
\setlist[factenum,2]{
    label=\textbf{(\alph*)},
    ref  =\thefactenumi{(\alph*)},
    leftmargin=0cm}
\setlist[factenum,3]{label=\bfseries(\roman*),
                   ref  =\thefactenumii\textbf{.(\roman*)}}

\crefname{factenumi}{fact}{facts}
\crefname{factenumii}{fact}{facts}
\crefname{factenumiii}{fact}{facts}

\newcommand{\TIFwarmupi}{\hyperref[def:TIF-warmup]{TIF-$i$ }}
\newcommand{\TIFwarmupj}{\hyperref[def:TIF-warmup]{TIF-$j$ }}

\acrodef{icl}[ICL]{In-Context Learning}
\acrodef{msa}[\texttt{MS-Attn}]{Multi-head Softmax Attention}
\acrodef{psd}[PSD]{Positive Semi-definite}
\acrodef{svd}[SVD]{Singular Value Decomposition}
\acrodef{gf}[Gradient Flow]{Gradient Flow}
\acrodef{si}[Symmetric Initialization]{Symmetric Initialization}
\acrodef{asi}[Asymmetric Initialization]{Asymmetric Initialization}
\acrodef{sw}[Symmetric Weights]{Symmetric Weights}
\acrodef{asw}[Asymmetric Weights]{Asymmetric Weights}
\acrodef{mgf}[MGF]{Moment Generating Function}
\acrodef{cdf}[CDF]{Cumulative Distribution Function}
\acrodef{wth}{Within-Task Homogeneous}
\acrodef{mmse}[MMSE]{Minimum Mean Square Error}
\acrodef{snr}[SNR]{Signal-to-Noise Ratio}
\newcommand{\DC}{{Decomposability Condition}}
\newcommand{\ConvergenceICL}{{\text{GF-ICL}}}
\newcommand{\MultiheadICLLB}{{\text{LB-ICL}}}
\newcommand{\SingleheadICLOpt}{{\text{OptS-ICL}}}
\newcommand{\ConstructICL}{{\text{Const-ICL}}}
\def \ssa {{single-head softmax attention}}

\newcommand{\h}{{(h)}}

\newcommand{\hprime}{{(h')}}
\newcommand{\handhprime}{{(h, h')}}
\newcommand{\handzero}{{(h, 0)}}
\newcommand{\histar}{{(h_i^\star)}}

\newcommand{\intf}{{\mathtt{Intf}}}
\newcommand{\noise}{{\mathtt{Noise}}}
\newcommand{\signal}{{\mathtt{Signal}}}
\newcommand{\AXinterference}{A_{ \mathtt{Intf}}}
\newcommand{\AXnoise}{A_{\mathtt{Noise}}}
\newcommand{\AXsignal}{A_{\mathtt{Signal}}}
\newcommand{\BYinterference}{B_{ \mathtt{Intf}}}
\newcommand{\BYnoise}{B_{\mathtt{Noise}}}
\newcommand{\BYsignal}{B_{\mathtt{Signal}}}

\newcommand{\barq}{\cleverbar{q}}
\newcommand{\barcT}{\cleverbar{\cT}}
\newcommand{\baromega}{\cleverbar{\omega}}
\newcommand{\baralpha}{\cleverbar{\alpha}}
\newcommand{\bard}{\overline{d}}

\newcommand{\muih}{\mu_{i}^{\h}}
\newcommand{\muihprime}{\mu_{i}^{\hprime}}
\newcommand{\muistar}{\mu_{i}^{\star}}
\newcommand{\mujh}{\mu_{j}^{\h}}
\newcommand{\muh}{\mu^{\h}}
\newcommand{\muhprime}{\mu^{\hprime}}

\newcommand{\baromegaih}   {\baromega_{i}^{\h}}
\newcommand{\baromegaihprime}   {\baromega_{i}^{\hprime}}
\newcommand{\baromegaistar}{{\baromega_{i}^{\star}}}
\newcommand{\baromegajh}   {\baromega_{j}^{\h}}
\newcommand{\baromegah}   {\baromega^{\h}}

\newcommand{\hprimeandh}   {{(h', h)}}
\newcommand{\handh}  {{(h, h)}}
\newcommand{\tomegainorm} {\norm{\tilde\omega_i}_2}

\newcommand{\uistar}{u_i^{\star}}

\newcommand{\vecd}{\bm{d}}

\newcommand{\convergence}{\mathtt{Conv}}

\usepackage{xparse}
\NewDocumentCommand{\fracexproductL}{ O{\omega^\h} O{\omega^\h} }{\frac{\exp\bigl(\bigl\langle #1, #2 \bigr\rangle\bigr)}{L}}
\NewDocumentCommand{\exproduct}{ O{\omega^\h} O{\omega^\h} }{{\exp\bigl(\bigl\langle  #1, #2 \bigr\rangle\bigr)}}

\newcommand{\lowerbound}{\mathtt{lb}}
\newcommand{\binomial}{\mathrm{Binom}}
\newcommand{\bernoulli}{\mathrm{Bern}}
\newcommand{\simple}{\mathtt{sim}}

\newcommand{\de}{d_e} 
\newcommand{\rhoi}{\rho_i}
\newcommand{\rhoistar}{\rho_{i}^{\star}}
\newcommand{\checkxiiprime}{\check\xi^\prime_{i}}
\newcommand{\hatxiiprime}{\hat\xi^\prime_{i}}

\definecolor{brown}{rgb}{0.0, 0.0, 0.0}
\definecolor{teal}{rgb}{0.0, 0.0, 0.0}
\definecolor{violet}{rgb}{0.0, 0.0, 0.0}
\usepackage{caption}
\usepackage{floatrow}
\usepackage{subcaption}

\usepackage{lipsum}

\newcommand\blfootnote[1]{%
  \begingroup
  \renewcommand\thefootnote{}\footnote{#1}%
  \addtocounter{footnote}{-1}%
  \endgroup
}

\newcommand{\revise}[1]{{#1}}
\newcommand{\siyurevise}[1]{{#1}}
\newcommand{\HJrevise}[1]{{#1}}
\title{\bf Training Dynamics of Multi-Head Softmax Attention for In-Context Learning: Emergence, Convergence, and Optimality}

\author{
    Siyu Chen
\and 
    Heejune Sheen 
\and 
    Tianhao Wang
\and 
    Zhuoran Yang
\and 
{\small\textit{Department of Statistics and Data Science, Yale University}} 
\and
{
    \small\texttt{\{siyu.chen.sc3226, heejune.sheen, tianhao.wang, zhuoran.yang\}@yale.edu}
}
}

\date{}

\begin{document}

\maketitle


\begin{abstract}%
We study the dynamics of gradient flow for training a multi-head softmax attention model for in-context learning of multi-task linear regression. 
We establish the global   convergence  of gradient flow under suitable choices of initialization. 
In addition,  we prove that 
an interesting ``task allocation" phenomenon emerges during the gradient flow dynamics, where each attention head focuses on solving a single task of the multi-task model.
Specifically, we prove that the gradient flow dynamics can be split into three phases --- a \emph{warm-up} phase where the loss decreases rather slowly and the attention heads gradually build up their inclination towards individual tasks, an \emph{emergence} phase where each head selects a single task and the loss rapidly decreases, and a \emph{convergence} phase  where the attention parameters converge to a limit. 
Furthermore, we prove the optimality of gradient flow in the sense that the limiting model learned by gradient flow is on par with the best possible multi-head softmax attention model up to a constant factor. 
Our analysis also delineates a strict separation in terms of the prediction accuracy of ICL between single-head and multi-head attention models. 
The key technique for our convergence  analysis is to map  the gradient flow dynamics in the parameter space to a set of ordinary differential equations in the spectral domain, where the relative magnitudes of the semi-singular values of the attention weights determines task allocation.  
To our best knowledge, our work provides the first convergence result for the multi-head softmax attention model. 
\end{abstract}


\blfootnote{Accepted for presentation at the
Conference on Learning Theory (COLT) 2024}
\section{Introduction}
The transformer architecture \citep{vaswani2017attention} is the backbone of many foundational models in artificial intelligence (AI), demonstrating striking empirical success in domains including natural language processing~\citep{devlin2018bert}, computer vision~\citep{Dosovitskiy2020AnII}, and reinforcement learning~\citep{chen2021decision}.
At a high level, a transformer is a sequence-to-sequence model that processes an input sequence of tokens and produces an output sequence of tokens.
An autoregressive  transformer generates the output in an autoregressive manner, i.e., each new token is generated by taking   all previously generated  tokens as the input of the transformer model. 
Autoregressive transformer architecture has emerged as the mainstream paradigm of large language models  (LLMs), with examples such as GPT~\citep{radford2019language,brown2020language,achiam2023gpt}, Claude \citep{anthropic2023claude}, and Gemini \citep{team2023gemini}. 

These LLMs are trained on internet-scale data across a wide range of topics and languages. 
A distinguishing feature of these models is their ability to learn from a few demonstrations of a new task that does appear in the training data, a phenomenon known as \ac{icl}~\citep{brown2020language}.
That is, without updating the model parameters, by feeding a few input-output examples of a task and querying a new input, 
LLMs are able to give  the correct output. 
The \ac{icl} ability serves as the foundation  for building more sophisticated prompting methods that uses LLMs for solving complicated problems \citep{huang2022towards}.

While the \ac{icl} capability of transformers has been empirically demonstrated and applied, theoretical understanding of this phenomenon is still in its infancy. 
In particular, we lack a comprehensive understanding of a key component of the transformer architecture, the attention mechanism~\citep{vaswani2017attention}, and how it is related to the \ac{icl} ability.
In the attention mechanism, multiple attention heads \emph{attend} to tokens in the input sequence by assigning attention weights and aggregates the values according to the weights from all tokens to form the output features. 
These attention weights are probabilities given by a softmax function and thus the resulting model is referred to as \ac{msa}.
Existing works have investigated how single-head attention-based models can be trained for ICL \citep{zhang2023trained,huang2023context,kim2024transformers}, but the multi-head case is still under-explored, though it has been observed that multi-head attention outperforms single-head attention for ICL~\citep{cui2024superiority,xing2024benefits}.

In this work, 
we aim to make a step towards understanding the role \ac{msa} played in the emergence of \ac{icl} ability. 
To this end, we focus on a simple but fundamental setting where transformer is a one-layer multi-head softmax attention model trained for the \ac{icl} task of multi-task linear regression. 
In particular, the \ac{msa} model is asked to see $L$ covariate-response examples,  $\{ x_l, y_l \} _{l \in [L]}$, sampled from  a randomly sampled noisy multi-task linear regression problem, i.e., $y_l = G^\top x_l  + \varepsilon_l$, and predict the regression target $G^\top q$ of a new covariate $q$, referred to as a query. 
Here,  $G \in \RR^{d_y \times d} $ is a random matrix sampled randomly from a distribution and $\varepsilon_l$ is a random noise. 
Moreover, we further assume that $G$ admits a \emph{multi-task structure} in the sense that it can be transformed by two fixed orthogonal matrices $\varPhi$  and $\varPsi$ into a block diagonal matrix, where each block corresponds to the signal of  an individual task. 
In other words, upon orthogonal transformations, the multi-task linear regression is reduced to a set of independent linear models. 
We are considering sufficiently large $L$ and $d$ while the dimension-to-sequence length ratio $d/L = \Theta(1)$.
To perform ICL, 
the  \ac{msa} model is first trained on a variety of randomly sampled instances of multi-task ICL data and then evaluated on a random new instance.

To understand how the attention mechanism enables the ICL ability, we delve into the gradient flow  dynamics of training the \ac{msa} model at a population level. 
This corresponds to the setting where the number of instances of multi-task task linear regression goes to infinity. 
Under such a setting, we aim to address the following questions: 
\begin{enumerate}[
    label=\textbf{(\alph*)}, 
    leftmargin=1cm, 
    ]
    \item Does gradient flow dynamics of ICL on multi-task linear data converge to a limit? 
    \item Does the learned \ac{msa} model exhibit ICL ability with high prediction accuracy? 
    \item Does the multi-head structure offer any advantage over the single-head version? 
\end{enumerate}

We provide affirmative answers to all these three questions. 
In particular, using a symmetric initialization scheme where the key and query weights are the same, we prove that the gradient flow with respect to the population loss converges to a limit. 
In addition, we characterize the prediction error of the limiting \ac{msa} model  explicitly in terms of the signal-to-noise ratio of the linear model 
and the the ratio $d/L$ under certain condition, e.g., both $d$ and $L$ are large.
Moreover, such an error decays to zero when $d/L$ decreases.  
We also characterize the minimal error attained by the class of \ac{msa} models and prove that the model founded by gradient flow is on par with the best model up to a constant factor.
\HJrevise{Finally, we demonstrate the superiority of multi-head structure in two aspects. First, the multi-head structure exhibits a ``task allocation'' phenomenon, where each attention head focuses on solving an individual and non-overlapping task. 
Second, there exists a strict separation between single-head and multi-head models. 
In particular, the minimal prediction error achieved by the best single-head model is strictly larger than that of the \ac{msa} model learned by gradient flow by a factor of $H$, where $H$ represents the number of heads.}

\paragraph{Main Contributions.}   Our main contributions are summarized as follows:
\begin{enumerate}
\item First, we establish the convergence of gradient flow for a one-layer multi-head softmax attention model for ICL (\Cref{thm:convergence-multi-head-symmetric}). 
\revise{A key technical ingredient is to reduce the gradient flow dynamics in the parameter space to the corresponding spectral dynamics in the eigenspace of the data features, which allows a more tractable analysis (\Cref{lem:decomposability preserved}).}
Our analysis captures also fine-grained properties of the dynamics like the phase transition behavior that leads to the sudden emergence of ICL ability.

\item We prove that the convergence point of gradient flow \revise{admits the following properties: each attention head focuses on solving an individual task without cross-head interference and acts as the optimal single-head attention model for that task (\Cref{lem:msa icl loss-mainbody} and \Cref{thm:optimality of ssa}).} 
In addition, we show that the model found by gradient flow achieves minimal ICL loss up to a constant factor when the number of heads matches the number of tasks (\Cref{thm:optimality of equiangular}).
  
\item We show that the multi-head model strictly outperforms the single-head model by a factor of $H$ when the number of tasks matches the number of heads. 
\siyurevise{This is demonstrated by comparing the multi-head model found by the gradient flow to the optimal single-head model in terms of the ICL loss.}

\item As a byproduct, we identify an interesting \say{attention allocation} phenomenon for a single-head softmax attention head to optimally solve multiple tasks (\Cref{thm:optimality of ssa}).
As an extension, we make a comparison between linear and softmax attention in \Cref{sec:extension}. 
\HJrevise{Moreover, we show that the model learned by gradient flow generalizes to new data with sequence lengths differing from the training data, while we also characterize its limitation in solving nonlinear tasks.}
\end{enumerate}
To our best knowledge, our work is the first to provide a comprehensive analysis of the training dynamics of a single-layer multi-head softmax attention model \revise{and the properties of the model trained by gradient flow}.
Our analysis characterizes, in a unified manner, the rich interplay between different design aspects of the attention mechanism and ICL task-specific properties.

\subsection{Our Approach: Tracking the  Spectral Dynamics}
\begin{figure}[ht]
    \centering
        \begin{subfigure}[b]{0.475\textwidth}
            \includegraphics[width=\textwidth]{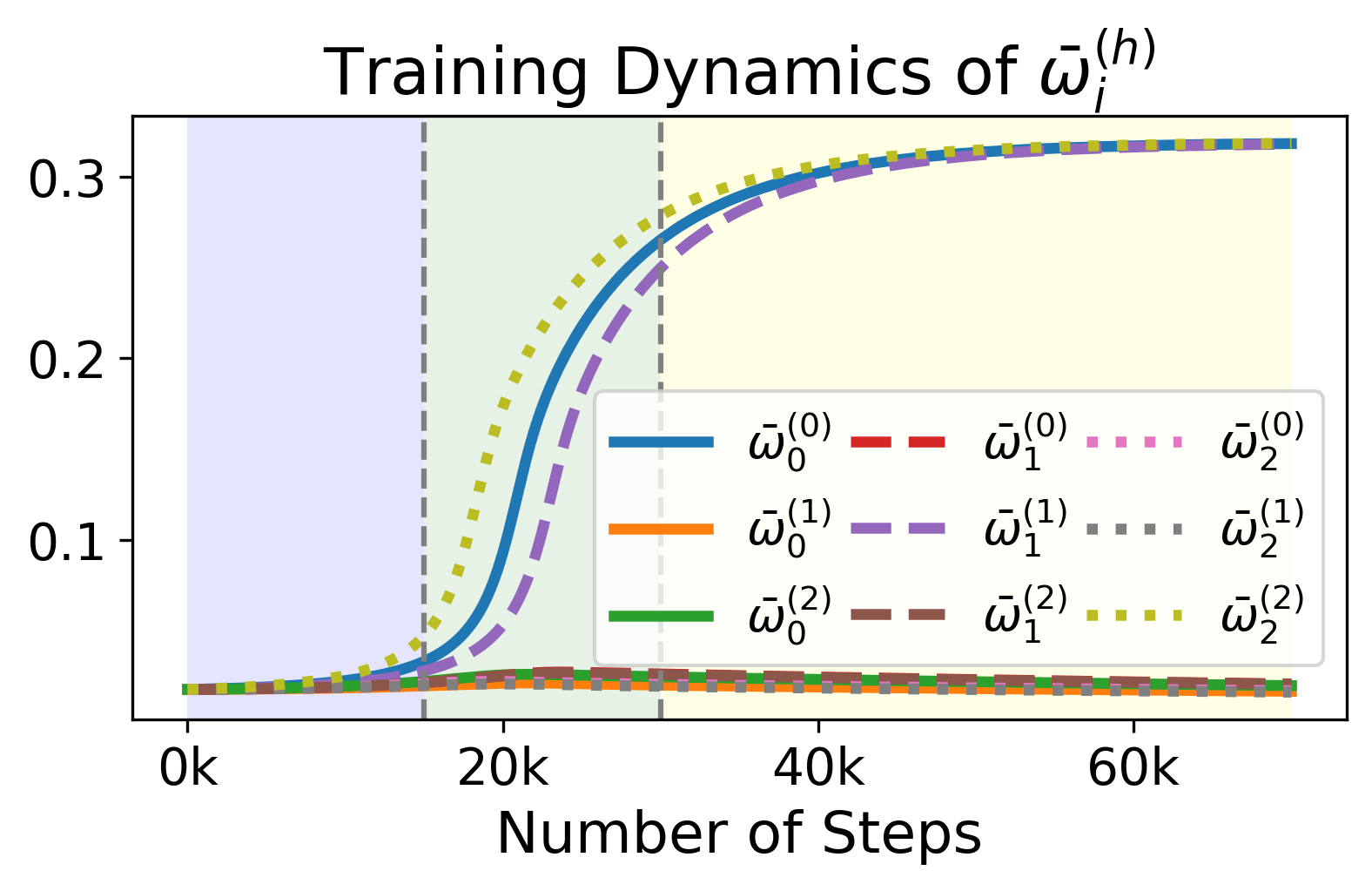}
        \end{subfigure}
        \begin{subfigure}[b]{0.475\textwidth}
            \includegraphics[width=\textwidth]{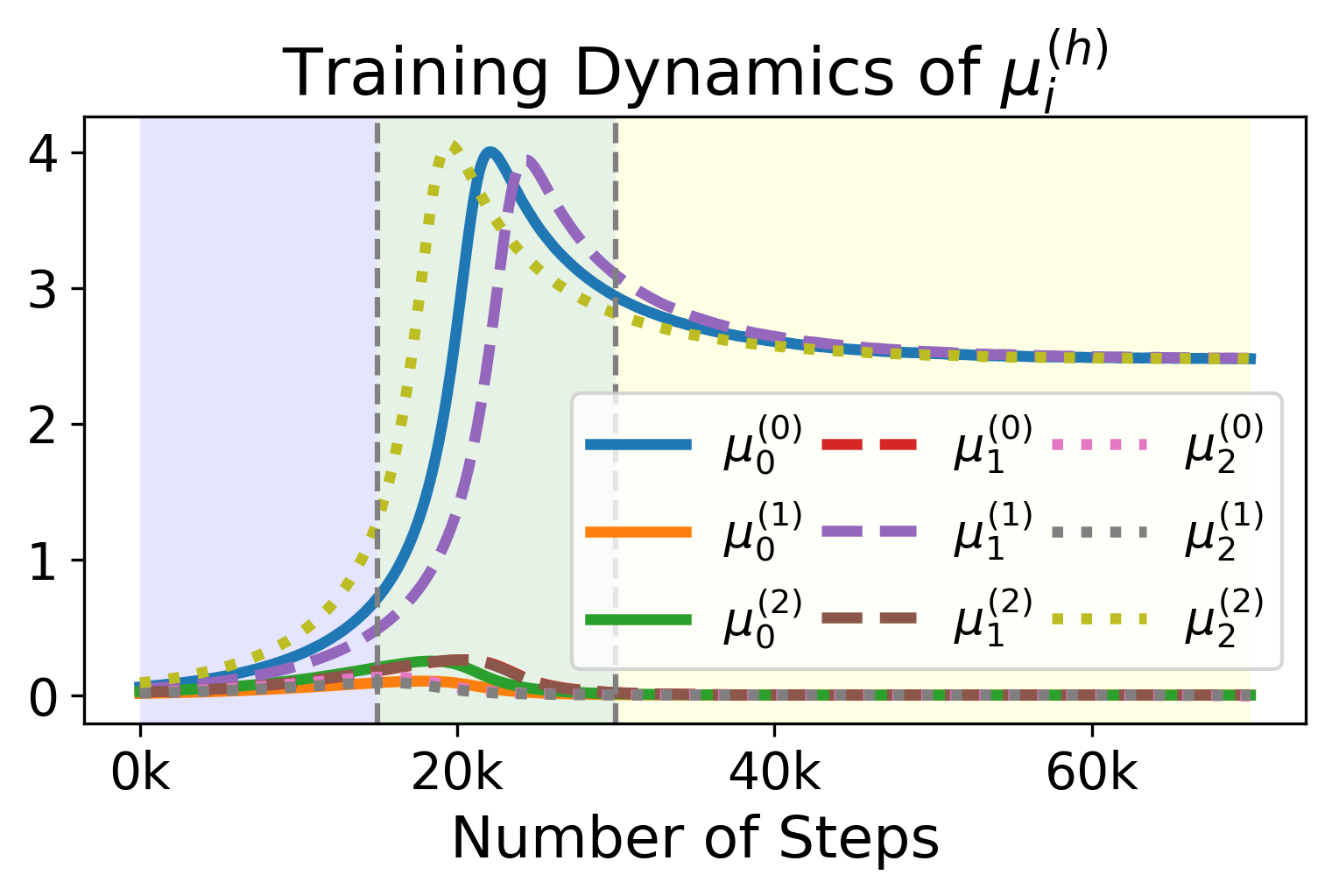}
        \end{subfigure}

        \begin{subfigure}[b]{0.475\textwidth}
            \includegraphics[width=\textwidth]{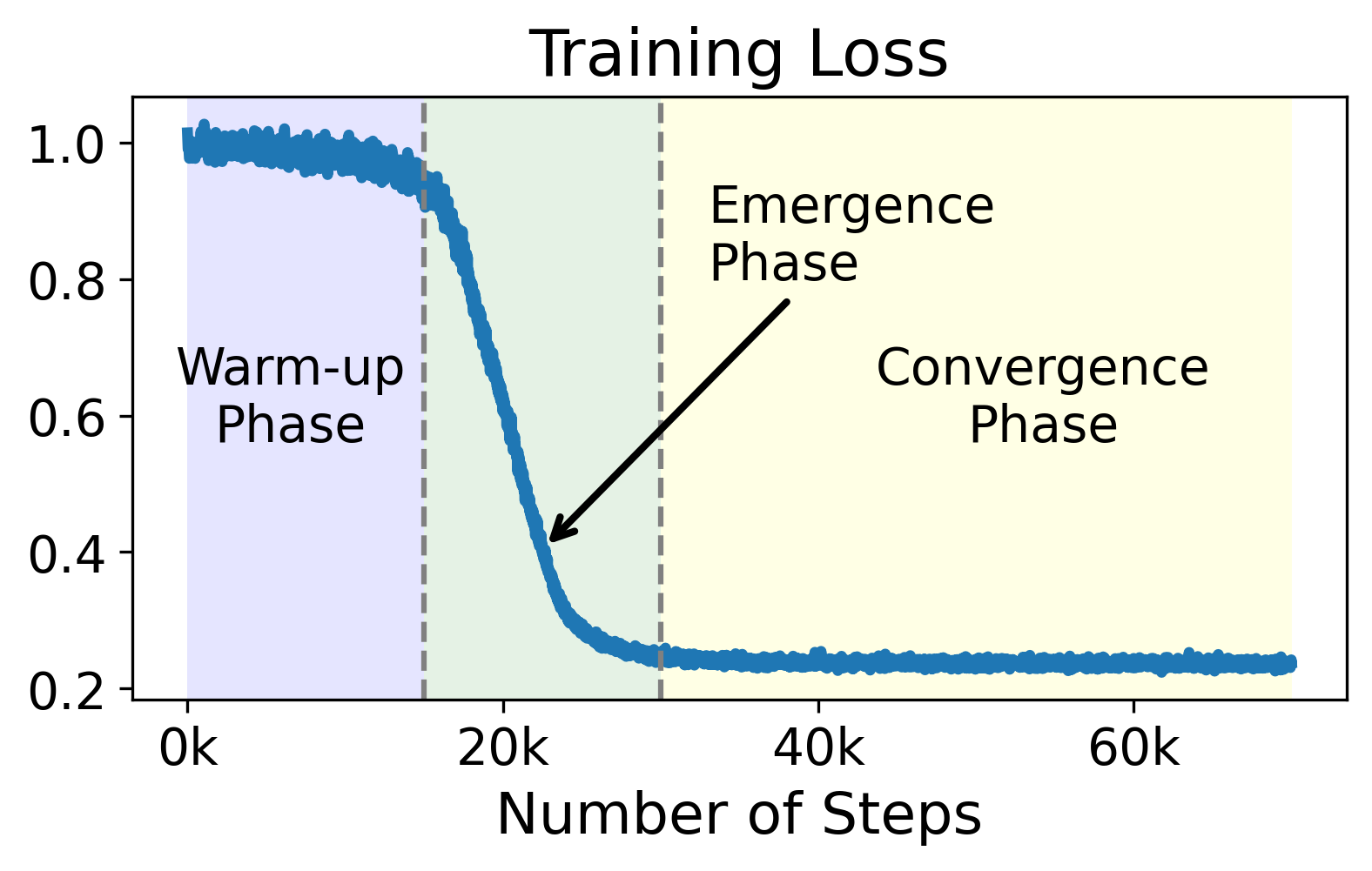}
        \end{subfigure}
        \begin{subfigure}[b]{0.475\textwidth}
            \includegraphics[width=\textwidth]{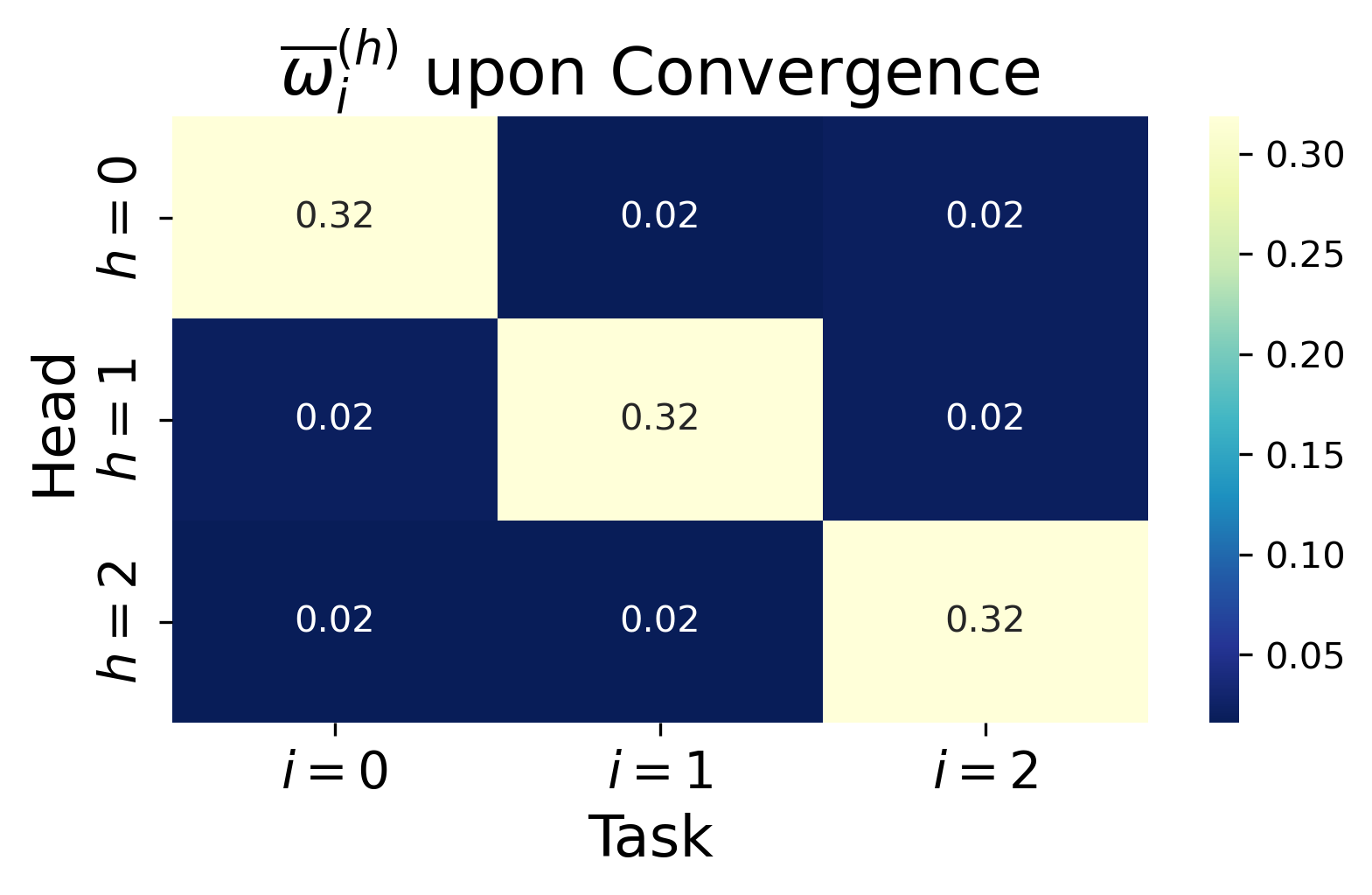}
        \end{subfigure}
    \captionsetup{width=0.95\textwidth}
    \caption{Illustration of the training loss and dynamics of the weights of the \ac{msa}. 
    The dynamics undergo three phases: warm-up, emergence, and convergence. 
    Here, we set $H=3$ heads, $I=3$ homogeneous tasks with signal strength $\lambda=1$ and no noise, dimension $d=30$ and context sequence length $L=100$.
    The top two plots show the dynamics of the eigenvalues of two combined weight matrices, $\baromega^\h$ for $W_X^\h$ and $\mu^\h$ for $U_Y^\h$ for all $h\in[H]$.
    The third plot shows the training loss and the last plot shows the value of $\baromegaih$ upon convergence.
    Here, the optimal head for task $i$ is $h_i^\star = i$ for $i\in[I]$ by initialization.
    }
    \label{fig:dyn}
\end{figure}

The analysis of the dynamics is challenging due to \revise{the symmetries} of the loss landscape (see related discussion in \S\ref{sec:related_work}), as well as the nonlinearity introduced by the softmax operation.
\begin{enumerate}
    \item We address the first challenge by properly aligning the  eigenspaces of the weight matrices with those of the context data features under the condition of Decomposable Weights  (\Cref{def:decomposability property}).
    \revise{This gives rise to a crucial simplification that allows us to track the spectral dynamics of the attention weights and output weights during the training process.}
    \item To address the second challenge, we perform careful moment analyses on the attention probability vectors produced by the softmax operation.
    \HJrevise{Based on this, we control the nonlinear effects and prune higher-order terms to extract the key factors driving the learning process. This is essential for both the dynamic analysis and the establishment of optimality bounds.}
\end{enumerate}
\revise{Combining these two ingredients paves the way for understanding the attention layer's behavior both during the training process and at the convergence point.}
The following informal theorem summarizes our main result on the dynamics of the multi-head attention model.
\begin{theorem}[Gradient flow for \ac{msa}, informal]\label[theorem]{thm:convergence-multi-head-symmetric_informal}
Under proper initialization (also satisfying \Cref{def:decomposability property}) of the gradient flow, assume that the number of heads is larger than the number of tasks, 
then the following holds:
\begin{myenumi}[
    ref = \Cref{thm:convergence-multi-head-symmetric}(\roman*),
    leftmargin=.5cm,
    align=left,
]
\item  \emph{(Warm-up)}
\HJrevise{There exists a threshold time $T_0 > 0$, before which the loss value decreases slowly, and each attention head gradually adjusts its weights to select one individual task.}

\item \emph{(Emergence)}
\HJrevise{For a brief period around $T_0$, the loss undergoes a sudden decrease, each attention head rapidly becomes focused on a single task, and cross-head interference vanishes.}

\item \emph{(Convergence)} 
\HJrevise{Finally, gradient flow converges to a point where each task is predominantly handled by a single attention head, in the sense that the output from this head dominates the output of the entire model for this task.}

\end{myenumi}
\end{theorem}


In complement to the dynamics analysis, we further study the optimality property of the convergence point of gradient flow.
\begin{myenumi}
    \item We first consider using a single-head softmax attention model to optimally solve multi-task linear regression.
    We show the lower bound of the optimal ICL loss and construct a group of attention weights to achieve the lower bound.
    When specialized to a single task, we directly conclude that the optimal attention weights are indeed the same as the convergence point of the gradient flow dynamics.
    The challenge here is the nonconvexity of the landscape induced by the softmax function, and our main approach is to find the optimal working regime of the softmax attention and derive a good nonlinear approximation to the attention's behavior in this regime.
    As a byproduct, we also show that the optimal attention weights are decomposable in the sense of \Cref{def:decomposability property}, which justifies our simplification of the gradient flow dynamics to the spectral dynamics.
    \item For the multi-head case, we consider a subclass of the \ac{msa} model bearing a certain symmetric structure called the equiangular weights (\Cref{def:equiangular}). Under task homogeneity, i.e., all tasks having the same \ac{snr} and dimension, we show using a similar technique that the convergence point of the gradient flow is indeed optimal. The analysis demonstrates an additional characterization of the cross-head interference compared with the single-head case.
\end{myenumi}

\paragraph{Organization of the paper.} 
The rest of the paper is organized as follows:
We finish the introduction by discussing related works.
In \Cref{sec:preliminary}, we introduce the one-layer multi-head softmax attention model and the ICL setting of multi-task linear regression.
\HJrevise{In \Cref{sec:dynamics_main}, we present gradient flow dynamics and characterize the emergence and convergence behavior of the attention weights during the training process.}
In \Cref{sec:optimality_main}, we investigate the optimality of the trained model. In \Cref{sec:extension}, we discuss extensions of our main results.
Detailed proofs and auxiliary results are deferred to the appendix.

\subsection{Related Work}
\label{sec:related_work}
We briefly review additional related work on ICL and learning properties of transformers in other settings.
First, as suggested by empirical evidence, various explicit weight constructions have been proposed to show that transformers can do ICL by emulating specific algorithms such as gradient descent and its variants \citep{akyurek2023what,von2023transformers,bai2023transformers,fu2023transformers}.
It is further theoretically verified via landscape analysis that the global minimizer (or certain critical points) of the ICL loss corresponds to certain adaptive version of gradient descent \citep{ahn2023transformers,mahankali2023one}.
Moreover, convergence analyses have also been proposed by \citet{zhang2023trained, huang2023context,kim2024transformers, nichani2024transformers}
under different model and data assumptions.
Nonetheless, these convergence analyses are tailored to single-head attention models and require simplifications of the model such as combining the key and query matrices into a single weight matrix.
Though not changing the model's expressiveness, such simplifications make the loss landscape more amenable to analysis by removing the rotational symmetry of the attention weights\footnote{Here we are refering to the fact that $K^\top Q = (O K)^\top (OQ)$ for any orthogonal matrix $O$.}.
Also note that the permutation symmetry between multiple heads further complicates the loss landscape, and so far there is no convergence analysis for multi-head attention models for ICL.

The statistical complexity of transformers for ICL of linear functions has been characterized by \citet{wu2023many}, and \citet{cheng2023transformers,guo2023transformers, collins2024context} investigated ICL beyond the simple linear functions.
In particular, \citet{collins2024context, edelman2024evolution, makkuva2024attention} studied the ICL capability of transformers when data is drawn from Markov chains.
There are also works explaining how transformers perform in-context learning from the Bayesian perspective \citep{xie2021explanation,Muller2021TransformersCD,zhang2022analysis,zhang2023and,ahuja2023context,jeon2024information}.
Transformers' capability of in-context decision making has been explored by \citet{lin2023transformers,sinii2023context}.
See also \citet{li2023transformers,dai2022can,raventos2023pretraining} for other related results on ICL.
Existing works have also studied the use of multiple attention heads from different perspectives \citep{an2020repulsive,mahdavi2023memorization}.
Beyond the case of ICL, there have been recent advances in understanding the learning properties of transformers in other settings \citep{edelman2022inductive,Li2023HowDT,jelassi2022vision,sanford2023representational,pmlr-v202-giannou23a,Liu2022TransformersLS,Tarzanagh2023TransformersAS,tarzanagh_max-margin_2023,tian2023joma,tian2023scan,song2023uncovering,deora2023optimization,chen2024provably}.
We do not discuss these works in detail here since they are out of the scope of the current paper.

\section{Preliminaries}\label[section]{sec:preliminary}
We first introduce the notation conventions used throughout the paper.
We use uppercase letters like $A$ and $B$ to denote matrices and lowercase letters like $u$ and $v$ for column vectors.
For any positive integer $n$, we denote $[n] := \{1, \ldots, n\}$.
We denote by $\vspan(M)$ the column space of matrix $M$.
For two matrices $A$ and $B$ (not necessarily squared), we write $A \eigveq B$ if they share the same singular vector spaces, i.e., $A = U \Sigma V^\top$ and $B = U \Lambda V^\top$ for some orthogonal matrices $U$ and $V$ and diagonal matrices $\Sigma$ and $\Lambda$.
For any vector $v\in\RR^{d}$, we define the softmax operation as $\softmax(v) := (e^{v_1}/\sum_{i=1}^de^{v_i}, \ldots, e^{v_d}/\sum_{i=1}^{d} e^{v_i})^\top$.
We denote by $\odot$ the element-wise Hadamard product.
Let $\NN$ be the set of non-negative integers.


Next, we introduce the one-layer multi-head softmax attention model in \Cref{sec:msa_def} and the multi-task linear regression problem for \ac{icl} in \Cref{sec:icl_def}. 

\subsection{One-Layer Multi-Head Softmax Attention Model}\label[section]{sec:msa_def}

We consider a \ac{msa} model with $H$ attention heads and context length $L$, defined as follows: 
Let $D$ be the input dimension, $d_y$ be the output dimension, and ${d_e}$ be the embedding dimension.
Throughout the paper, we assume $d_e \ge d$, where $d = D - d_y$.
For each attention head $h\in[H]$, let $O^{(h)} \in \RR^{d_y\times {d_e}}$ and $V^{(h)}, K^{(h)}, Q^{(h)} \in \RR^{{d_e}\times D}$ be the output projection, value, key, and query weights, respectively. 
We let $\Theta = \{O^\h, V^\h, K^\h, Q^\h\}_{h\in[H]}$ denote the parameters of the \ac{msa} model. 
Then given an input matrix $Z \in \RR^{D\times L}$ and a query vector $z_q\in\RR^D$, the output $\hat y_q$ of \ac{msa} is 
\begin{align}\label{eq: msa}
    \hat y_q = \mathtt{MS\text{-}Attn} ( Z, z_{q}; \Theta)  := \sum_{h=1}^H O^{(h)} V^{(h)} Z \cdot \softmax \left( \frac{Z^\top K^{\h^\top} Q^\h z_q }{ \sqrt{{d_e}}}\right) \in \RR^{d_y}.
\end{align}
Following from \eqref{eq: msa}, given the parameter $\Theta$, for  each head $h\in[H]$, we define the attention score vector $s^{(h)} \in \RR^L $ and the corresponding attention probability vector $p^{(h)}\in \RR^L$  as
\begin{align}\label{eq:attn_score_prob}
    s^{(h)}  =  \frac{Z^\top K^{\h^\top} Q^\h z_q}{\sqrt{{d_e}}} \in \RR^L, \qquad p^{(h)}  = \softmax(s^{(h)}) \in \RR^L.
\end{align}
Intuitively, $ s^{(h)}$ measures the similarity between the query vector and each column of $Z$, and $p^{(h)}$ is the probability distribution over $[L]$ that is induced by $s^{(h)}$.
Such a distribution is used to weight the value head $V^{(h)} Z \in \RR^{{d_e}\times L}$ to get a vector $V^{(h)} Z p^{(h)}$ that lies in the embedded space   $\RR^{d_e} $. The output $\hat y_q$ is obtained by mapping such an embedding to the output space $\RR^{d_y}$ by a linear map $O^{(h)}$. That is, we can equivalently write \eqref{eq: msa} as $\hat y_q = \sum_{h=1}^H O^{(h)} V^{(h)} Z p^{(h)}$.
To further simplify the notation, we define the combined output weight  $ U^{(h)} \in\RR^{d_y \times {D}} $ and attention weight $W^\h \in \RR^{D\times D}$: 
\begin{align}\label{eq:combined_weights}
    U^{(h)} := O^{(h)} V^{(h)}    ,\quad W^{(h)} :=  \frac{K^{(h)\top} Q^{(h)}}{\sqrt{d_e}}  \quad \text{for each } h\in[H].
\end{align}

Throughout the paper, we will consider the query vector $z_q$ with the last $d_y$ coordinates being zero, i.e., $z_q = [q^\top, 0, \dots, 0]^\top$ where $q\in\RR^{{d}}$.
The output $\hat y_q$ can be viewed as the lower right $d_y\times 1$ block of the output of the standard multi-head softmax attention architecture with input matrix $[Z, z_q] \in \RR^{D \times (L+1)}$.
We remark that the residual connection is omitted in \eqref{eq: msa} since it does not affect the output, and we do not consider any attention mask as there is only one single layer.

\subsection{In-Context Learning of Multi-Task Linear Regression}\label[section]{sec:icl_def}
We consider a data structure where each input token comprises a covariate with dimension ${d}$ and a label with dimension $d_y$ such that $D={d}+d_y$.
The inputted $Z$ and $z_q$ can be decomposed as
\begin{gather*}
    Z = \begin{bNiceMatrix}
        x_1 & x_2 & \Cdots & x_L \\
        y_1 & y_2 & \Cdots & y_L
    \end{bNiceMatrix}, \quad 
    z_q = \begin{bNiceMatrix}
        q \\
        0
    \end{bNiceMatrix}.
\end{gather*}
Such input format is standard in recent literature about Transformers and ICL \citep{garg2022can,akyurek2023what,von2023transformers,ahn2023transformers,zhang2023trained}, but we remark that here we allow $d_y > 1$, in contrast to the single-task case where $d_y=1$ considered in previous works.
In the sequel, we denote by $X=[x_1, \dots, x_L] \in \RR^{{d}\times L}$ and $Y=[y_1,\dots, y_L] \in \RR^{d_y\times L}$.
For the ICL, we consider a multi-task linear regression task described by the following data assumption. 

 \begin{assumption}[Multi-Task \ac{icl} Data]
    \label[assumption]{assump:data}
    Every in-context sample $(Z, z_q)$ is independently sampled according to the distribution described as follows:
    First, we assume the covariate vectors and the query vector to be distributed as
        $x_1, \ldots, x_L \overset{\iid}{\sim} \cN(0, I_{{d}}) \text{ and } q \sim \cN(0, I_{{d}})$.
    Next, for a random coefficient matrix $G\in\RR^{{d}\times d_y}$, each response vector is given by $y_l = G^\top x_l + \varepsilon_l$ for $l\in[L]$ where the noise $\varepsilon_1, \ldots, \varepsilon_L \overset{\iid}{\sim} \cN(0, \sigma^2 I_{d_y})$. 
    Note that the coefficient matrix $G$ is shared within the context sequence $Z$.
    Then the goal is to predict the true response vector $y_q = G^\top q$ given the context $Z$ and the query $z_q$.
    To this end, we assume $G$ to be independent of both the covariate vectors $\{x_l\}_{l=1}^L$ and the noise $\{\varepsilon_l\}_{l=1}^L$, and moreover, we assume that $G$ admits a multi-task structure:
    \begin{equation}\label{eq: decomposition of beta}
    \begin{gathered}
        G = d^{-1/2} \cdot \varPhi
        \cdot
        \begin{bNiceArray}{cccc}[margin,parallelize-diags=false]
            g_1 & 0       & \cdots & 0      
            \\
            0       & g_2 & \cdots & 0      \\
            \vdots  & \vdots  & \ddots & \vdots \\
            0       & 0       & \cdots & g_I \\
        \end{bNiceArray}
        \cdot \varPsi^\top,
    \end{gathered}
    \end{equation}
    where $\varPhi\in\RR^{d\times d}$ and $\varPsi\in\RR^{d_y\times d_y}$ are two fixed orthogonal matrices, and corresponding to each task $i\in[I]$, the random vector $g_i\in\RR^{d_i}$ with dimension $d_i$ satisfies $\EE[g_i] = 0$ and $\Cov[g_i] = \lambda_i I_{d_i}$ with $\lambda_i$ being the signal strength for each task $i\in[I]$. 
    Here, we restrict $I = d_y$ for simplicity.
    In addition, $d^{-1/2}$ is the normalization factor which ensures that $\EE[\norm{y_q}_2^2] = 1$.
\end{assumption}
The above decomposition of $G$ breaks down the ICL problem into $I$ components, and inferring the correct response requires simultaneously estimating all $I$ tasks.
Note that we allow $G$ to be supported within some subspace. 
Here we only consider cross-task differences while assuming homogeneous signal strength within each task for ease of analysis.
We define the single sample Signal-to-Noise Ratio (SNR) for task $i$ as $\snr_i \defeq \lambda_i d_i / (d \sigma^2)$, and also $\phi_i := 1 + \snr_i^{-1}$.
We define $\cJ_i = \{d_1+\dots+d_{i-1} +1, \cdots, d_1+\dots+d_{i}\}$ as the set of indices for task $i$.
Then $\{ \cJ_i \}_{i\in[I]}$ forms a partition of $[d]$.
Throughout the paper, we consider sufficiently large context length $L$ and dimension $d$, and we assume that the ratios $d/L = \Theta(1)$ and $d_i / d = \Theta(1)$ for all $i\in[I]$.

\begin{assumption}\label{assump:scale}
For a constant $C>0$, it holds that $d/L \in (C, 1/C)$ and $d_i / d > C$ for all $i\in[I]$.
\end{assumption}

Before we proceed, we first introduce $V_X^\h, K_X^\h, Q_X^\h\in\RR^{{d_e}\times {d}}$ and $V_Y^\h, K_Y^\h, Q_Y^\h\in\RR^{{d_e}\times d_y}$ as the splits of $V^\h, K^\h, Q^\h$ with respect to the $X$ and $Y$ parts as shown in \Cref{fig:msa}. 
We split the combined output weights by writing 
$U_X^\h = O^\h V_X^\h$ as the left $d_y \times {d}$ block and  $U_Y^\h=O^\h V_Y^\h$ as the right $d_y\times d_y$ block of $U^\h$,  respectively. 
We also split the combined attention weights by packing
$W_X^\h = {K_X^\h}^\top Q_X^\h/\sqrt {d_e}$
as the top left ${d}\times {d}$ block and $W_Y^\h = {K_Y^\h}^\top Q_X^\h / \sqrt{{d_e}}$ as the bottom left $d_y\times {d}$ block of $W^\h$. 
\begin{equation} \label{eq:define_U_W_main}
U^\h  = \begin{bNiceMatrix}
     
        U_X^\h &  U_Y^\h   
        \end{bNiceMatrix}, \qquad
W^\h  = \begin{bNiceMatrix}
     
        W_X^\h & \star \\  W_Y^\h & \star
\end{bNiceMatrix}.
\end{equation}
The splitting described above is visualized in \Cref{fig:msa}. 
\begin{figure}[t]
    \centering
    \includegraphics[width=0.9\textwidth]{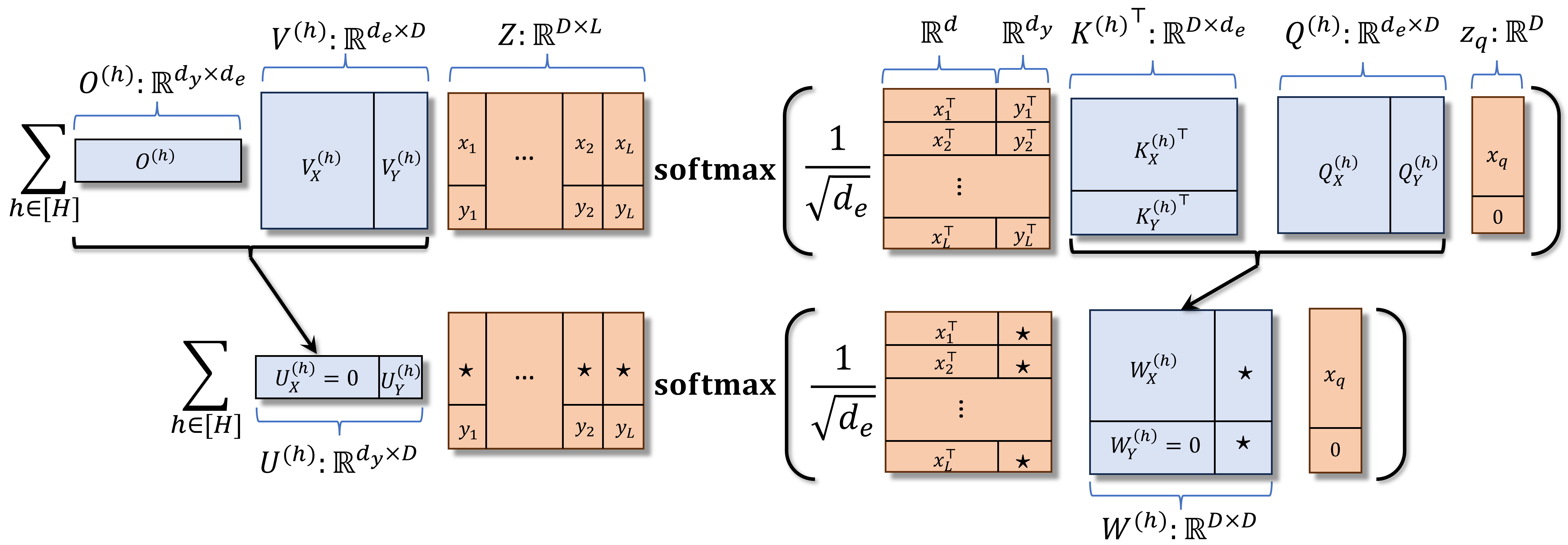}
    \captionsetup{width=0.95\textwidth}
    \caption{Illustration of the \ac{msa} with $H$ heads, embedding dimension ${d_e}$, and sequence length $L$. Weights $(V^\h, K^\h, Q^\h)$ can be split into the $X$ and $Y$ parts. The combined weights $(U^\h, W^\h)$ are also split into $(U_X^\h, U_Y^\h)$ and $(W_X^\h, W_Y^\h)$. Under the decomposability assumption ($W_Y^\h = 0$, $U_X^\h = 0$), the attention scores only depend on the $X$ part, and the output is an aggregation of the $Y$ part.
    Although we introduce the concept of combined weights, it is important to note that the gradient flow operates on the individual weights.
    }
    \label{fig:msa}
\end{figure}

\paragraph{Gradient Flow.} 
We consider a pretraining process with gradient flow on the parameters $\Theta = \{O^\h, V^\h, K^\h, Q^\h\}_{h\in[H]}$ to minimize the following population mean-squared error:
\begin{align}
     \cL (\Theta) = \EE [\ell(\hat y_q, y_q)], \quad\text{where}\quad \ell(\hat y_q, y_q) =   \norm{\hat y_q - y_q}_2^2 / 2,
    \label{eq: loss}
\end{align}
and the expectation is taken over the distribution over $(Z,z_q)$ in \Cref{assump:data}.
With the parameterization of the \ac{msa} in \eqref{eq: msa} and the loss function in \eqref{eq: loss}, one can derive the gradient flow dynamics for parameters $\Theta =\{O^\h, V^\h, K^\h, Q^\h\}_{h\in[H]}$ by moving $\Theta$ along the direction of the negative gradient  of the mean squared loss, i.e., $\partial_t \Theta = - \nabla_{\Theta} \cL(\Theta) $.  By a direct calculation, we have the following dynamics\footnote{Here we omit the time dependence of the weights for simplicity.}:
\begin{equation}
    \label{eq: gradient flow dynamics}
\begin{aligned}
\partial_t O^\h = - B^\h {V^\h}^\top, &\quad \partial_t V^\h = - {O^\h}^\top B^\h,\\
\partial_t K^\h = - d_e^{-1/2} Q^\h {A^\h}^\top, &\quad \partial_t Q^\h = - d_e^{-1/2} K^\h A^\h, 
\end{aligned}
\end{equation}
where $A^\h \in\RR^{{D}\times {D}}$ and $B^\h \in\RR^{d_y\times {D}}$ are defined respectively as
\begin{align}
    A^\h &= \EE\bigg[
        Z {P^\h }^\top Z^\top  {U^\h}^\top   \bigg(\sum_{h'=1}^H U^\hprime Z p^\hprime - y_q\bigg) z_q^\top 
        \bigg], \label{eq:define_A_h}
    \\
    B^\h &= \EE\bigg[\bigg(\sum\nolimits_{h'=1}^H U^\hprime Z p^\hprime - y_q\bigg) {p^\h}^\top  Z^\top
    \bigg],  \label{eq:define_B_h}
\end{align}
where we have an $L\times L$ matrix  $(P^\h)_{l, m} = \partial p^\h_l/\partial s^\h_m$ and in the case of softmax attention, it takes the form of
$P^\h = \diag(p^\h) - p^\h {p^\h}^\top$.
\section{Dynamics of Gradient Flow: Emergence and Convergence}\label[section]{sec:dynamics_main}




In this section, we present our main results on the gradient flow dynamics, stated previously in the informal \Cref{thm:convergence-multi-head-symmetric_informal}.
We first introduce in \Cref{sec:decomposable_weights} the specific set of weights that we will consider.
Then we describe the analysis and state the full results in \Cref{sec:convergence_analysis}.


\subsection{Decomposable Weights}\label[section]{sec:decomposable_weights}

We introduce the concept of decomposable weights, a key property that will be preserved along gradient flow and will be used for analyzing the induced dynamics in the eigenvalue space.
\begin{definition}
    \label{def:decomposability property}
    For orthogonal matrices $\varPhi\in\RR^{d \times d }$ and $\varPsi\in\RR^{d_y \times d_y }$, we say the weights of the \ac{msa} are decomposable with respect to $(\varPhi,\varPsi)$ if for all $h\in[H]$, the following three conditions hold:
    \begin{myenumi}[
        label=(\roman*), 
        ref=\Cref{def:decomposability property}(\roman*),
        leftmargin=*, 
        align=left, 
        topsep=-1ex,
        itemsep=-1ex
        ]
        \item \label{cond:U_X and W_Y are zero} \label{cond:orthogonal}
        \emph{(Orthogonality)} For the matrices $U^\h$ and $W^\h$ defined in \eqref{eq:define_U_W_main}, we have 
        $U_X^\h = W_Y^\h = 0$.
        Moreover, the column subspaces (of $\RR^{d_e}$) spanned by $K_X^\h$ and $K_Y^\h$ are orthogonal, and the column spaces spanned by $V_X^\h$ and $V_Y^\h$ are orthogonal. 
        That is, 
        $\vspan(K_X^\h) \perp \vspan(K_Y^\h)$, and $\vspan(V_X^\h) \perp \vspan(V_Y^\h)$.
        \item \label{cond:U_Y and W_X aligned}
        \emph{(Simultaneous Diagonolizability)} 
        We have 
         $K_X^\h \eigveq Q_X^\h$ with the same right eigenvector matrix $\varPhi$ and $V_Y^\h \eigveq O^{\h\top}$ with the same right eigenvector matrix $\varPsi$. 
         That is, there exists an orthogonal matrix 
        $M \in \RR^{d_e\times d_e}$ such that 
        $$
        M^\top K_X^\h \varPhi = \diag\bigl (\sigma(K_X^\h) \big ), \qquad M^\top Q_X^\h \varPhi = \diag \bigl (\sigma(Q_X^\h) \big ).
        $$ 
     Here  $\sigma(K_X^\h), \sigma(Q_X^\h) \in \RR^d$ are the $d$ semi-singular values\footnote{The difference between semi-singular values and standard singular values is that we allow semi-singular values to be negative.} of $K_X^\h$ and $Q_X^\h$ respectively, and we let $\diag  (\sigma(K_X^\h)   )$ to denote the diagonal matrix in $\RR^{d_e\times d}$ whose $(j,j)$-th entry is given by $\sigma(K_X^\h)_j$,  the $j$-th entry of $\sigma(K_X^\h)$, $\forall j\in[d]$. We define $\diag  (\sigma(Q_X^\h)   )$ similarly. 
Besides, there exists an orthogonal matrix $N \in  \RR^{d_e\times d_e}$ such that 
$$
N^\top V_Y^\h \varPsi = \diag\bigl (\sigma(V_Y^\h) \big ), \qquad N^\top {O^\h}^\top \varPsi = \diag \bigl (\sigma(O^\h) \big ),
$$
where $\sigma(V_Y^\h), \sigma(O^\h) \in \RR^{d_y}$ are the $d_y$ semi-singular values of $V_Y^\h$ and $O^\h$ respectively. 
Here $\diag  (\sigma(V_Y^\h)   )$ and $\diag  (\sigma(O^\h)   ) $ are   diagonal matrices in $\RR^{d_e\times d_y}$ defined similarly.

        \item \label{cond:task-wise homogeneous}
        \emph{(Task-wise Homogeneity)}  Note that
        $\sigma(K_X^\h)$ and 
        $\sigma(Q_X^\h)$ are vectors in $\RR^d$ and that $[d] $ is partitioned into $I$ disjoint sets $\cJ_1, \ldots, \cJ_I$.
        We require that the semi-singular values are the same within each task. That is, if 
        $j, j'\in \cJ_i$ for a same  task $i$, then we have 
        $$\sigma(Q_X^\h)_j =  \sigma(Q_X^\h) _{j'}, \qquad    \sigma(K_X^\h) _j =  \sigma(K_X^\h) _{j'}.$$

    \end{myenumi}    
\end{definition}


\Cref{cond:orthogonal}
asserts that $U_X^\h$ and $W_Y^\h$ are zero. By 
\eqref{eq: msa}  
and \eqref{eq:attn_score_prob}, since the attention scores are computed by $X^\top W_X q$ and the output is given by $ U_Y Y q$, this condition implies that the attention probability is only a function of the covariate and the query, and the output is just an aggregation of $Y$. 
This property further implies that the attention probability is independent of  the noise $\epsilon$ and the coefficient $G$, which enables us to simplify gradient flow dynamics. 
Moreover, \Cref{cond:orthogonal} 
requires that the column spaces spanned by $K_X^\h$ and $K_Y^\h$ are orthogonal. 
Similar property holds for $V_X^\h$ and $V_Y^\h$. 
As we will see in \Cref{sec:proof of decomposability preserved}, orthogonality of these subspaces ensures that $U_X^\h$ and $W_Y^\h$ remains zero through the  course of gradient flow dynamics.
 
In addition, 
\Cref{cond:U_Y and W_X aligned} suggests that the split weight matrices  are simultaneously diagonalizable   and  that the right eigenvectors are aligned with task-specific directions given by $(\varPhi, \varPsi)$. 
In particular, 
 let 
 \begin{align}\label{eq:define_eigenvalues_omega_mu} 
 \omega^\h = \sigma(K_X^\h) \odot \sigma(Q_X^\h), \qquad \mu^\h = \sigma(O^\h) \odot \sigma({V_Y^\h}),
\end{align} 
 which are vectors in $\RR^d$ and $\RR^{d_y}$ respectively.
 Then we have $W_X^\h = \varPhi \,\diag(\omega^\h) \varPhi^\top$ and 
$U_Y^\h = \varPsi \diag(\mu^\h) \varPsi^\top$. 
That is, $W_X^\h$ and $U_Y^\h$ are symmetric matrices and can be diagonalized by $\varPhi$ and $\varPsi$ respectively
, and the eigenvalues are given by $\omega^\h$ and $\mu^\h$ respectively. 

Furthermore, \Cref{cond:task-wise homogeneous} is a natural assumption given that within task $i$'s support $\cJ_i$, both $g_i$ and $X_{(i)}$ are isotropic random vectors under our data model. That is, the mean is zero and covariance is proportional to the identity matrix.
This condition implies  that $\omega^\h$   defined in \eqref{eq:define_eigenvalues_omega_mu} should also satisfy the same structure.
Here, $X_{(i)}$ is the submatrix of $X$ containing only the rows corresponding to task $i$'s support $\cJ_i$. 

Finally, we remark that the family of decomposable weights seems a bit restrictive. 
Nevertheless, we will justify this restriction by showing that the minimal ICL loss is approximately attained by decomposable weights in \Cref{sec:optimality_main} (see the discussion below \Cref{thm:optimality of ssa}).  
Notably, letting $Q_X^\h, K_X^\h, V_Y^\h, O^\h$ be diagonal matrices (but not necessarily squared) yields a special case of decomposable weights, and any decomposable weights can be transformed to this special case by rotation.
Another key feature of decomposable weights is that, when the initial weights are decomposable, such a structure is preserved through the course of gradient flow dynamics. This is shown in the lemma below.


\begin{lemma}[Preservation of Decomposibility along Gradient Flow]
    \label[lemma]{lem:decomposability preserved}
    Under \Cref{assump:data}, suppose the initialization of gradient flow is decomposable with respect to $(\varPhi,\varPsi)$ in the sense of \Cref{def:decomposability property}.
    Then the decomposability is preserved along the gradient flow trajectory.
    Moreover, the semi-singular values of the weight matrices  $O^\h, V_Y^\h, K_X^\h, Q_X^\h$ for $h\in[H]$ satisfy
    \begin{align*}
        \partial_t \sigma(K_X^\h) = - d_e^{-1/2} \sigma(Q_X^\h) \odot \sigma({A_{XX}^{\h\top}}), 
        &\quad \partial_t \sigma(Q_X^\h) = - d_e^{-1/2} \sigma(K_X^\h) \odot \sigma({A_{XX}^\h}), \\
        \partial_t \sigma(O^\h) = - \sigma(B_Y^\h) \odot \sigma({V_Y^{\h\top}}), &\quad \partial_t \sigma(V_Y^\h) = - \sigma({O^\h}^\top) \odot \sigma(B_Y^\h), 
    \end{align*}
    where $A_{XX}^\h$ is the top left $d\times d$ block of $A^\h$ and $B_Y^\h$ is the right $d_y\times d_y$ block of $B^\h$.
\end{lemma}
When the weights are initialized to be decomposable, the gradient flow dynamics become simpler as described in \Cref{lem:decomposability preserved}.
In particular, this lemma shows that the weight matrices remains decomposable with the same orthogonal matrices $(\varPhi,\varPsi)$.
As a result, we only need to track the evolution os the semi-singular values of the weight matrices,
which enables us to focus on the gradient flow  dynamics in the spectral domain.


\subsection{Analysis of Gradient Flow Dynamics}\label[section]{sec:convergence_analysis}


We proceed to give a high level overview of our analysis of the gradient flow dynamics. 
Here we focus on the dominant part of the gradient flow dynamics to illustrate the main idea, and the complete analysis involving the higher order terms is given in \Cref{sec:approximation_dynamics} and \Cref{sec:appendix_convergence_analysis}.

When \Cref{def:decomposability property} is satisfied, we define $\baromega^\h\in\RR^I$ to be the vector where each $\baromega_i^\h$ is the average value of the entries of $\omega^\h \in \RR^d$ corresponding to the $i$-th task in the decomposition of $G$. That is, $\baromega^\h_i$ is equal to any $\omega^\h_j$
with $j \in \cJ_i$.
If the weights are decomposable, then the output is only a function of the data, $\baromega:=\{\baromega^\h\}_{h\in[H]}$,  and $\mu:=\{\mu^\h\}_{h\in[H]}$.
Recall that $d_y = I$ and thus $\mu^\h \in \RR^{d_y}$ is also a vector in $\RR^I$. 
By \Cref{cond:task-wise homogeneous} and \Cref{lem:decomposability preserved}, it suffices to focus on $\baromegah$ and $\muh$ when the weights are initialized to be decomposable.

We assume that at initialization, $\sigma(Q_X^\h) = \sigma(K_X^\h) = \sqrt{\omega_0} \vone_{d}, \sigma(O^\h) = \sigma(V_Y^\h) = (\mu_0^\h)^{-1/2}, $ where $\omega_0>0$ is a sufficiently small constant and $\mu_0^\h$ is a vector of positive and sufficiently small entries.
We thoroughly analyze the signal and interference components of the gradient flow dynamics and observe three phases: {\color{blue}warm-up}, {\color{blue}emergence}, and {\color{blue}convergence}. 
Specifically, for each task $i$, we analyze the dynamics of the parameters of the optimal head, $(\baromegaistar, \muistar)$, as well as those of the non-optimal heads $(\baromegaih, \muih)$ for $h\neq h_i^\star$. 
Here the optimal head for task $i$ is defined as $h_i^\star = \argmax_{h \in [H]}\muih$ and we have the abbreviations $\baromegaistar \triangleq \baromega_i^{(
    h_i^\star)}$ and $\muistar \triangleq \mu^{(h_i^\star)}$.
Notice that $h_i^\star$ is actually a function of time 
$t$. As we will state in \Cref{thm:convergence-multi-head-symmetric}, in initialization, $h_i^\star$ is determined by the initial values of $\{\mu^\h\}_{h\in[H]}$, and the choice of $\mu_0^\h$ ensures that $h_i^\star$ is unique for each task $i$. 
Then, as we will see below  $h_i^\star$ remains the same throughout the dynamics. 
That is, the optimal head for each task is fixed during gradient flow  and only determined in initialization. 

In the following, we characterize the three phases of gradient flow dynamics in the spectral domain. For ease of exposition, we rescale the time by letting  $t \leftarrow 2 d t$ in the following discussion.


\paragraph{Phase I: Warm-up.}
During the warm-up phase, the optimal heads gradually dominate the non-optimal heads. 
\siyurevise{Recall that $\phi_i =1 + \snr_i^{-1} = 1 + \sigma^2 d / (\lambda_i d_i)$ where $\lambda_i$ is the signal strength for task $i$.}
We can prove that the dynamics of $\baromegaih$ and $\muih$ are approximately given by
\begin{align}\label{eq:approximated_dyn_main1}
    \partial_t \log \baromegaih
    \approx \frac{\lambda_i \muih  }{\sqrt{d_e}}, \quad 
    \partial_t \log \muih
    \approx  d_i  \lambda_i \baromega_i^\h  - \underbrace{\sum_{h'=1}^H \fracexproductL[\omega^\h][\omega^\hprime] d_i  \lambda_i \phi_i \mu_i^\hprime}_{\ds\text{Cross-head interference}},
\end{align}
where the second term in the dynamics of $\muih$ is the cross-head interference term because it contains the contribution of $\mu_i ^\hprime$ for $h'\neq h$.
When $t$ is small, $\omega ^\h $ is not far from the initial value $\omega_0 \cdot \vone_{d}$ for all $h \in [H]$.
Thus $\exp(\langle \omega^\h, \omega^\hprime\rangle ) $ in \eqref{eq:approximated_dyn_main1} is roughly of the same order for all $h, h'$. 
As a result, the cross-head interference term is almost the same for all heads as long as $\baromega^\h$ is small enough at initialization.
Then, by comparing heads $h$ with $h_i^\star$, we eliminate the interference term and get 
\begin{align} \label{eq:approximated_dyn_main2}
    \partial_t \log\bigg(\frac{\muistar}{\muih}\bigg) \approx \lambda_i d_i (\baromegaistar - \baromegaih),  \quad 
    \partial_t \log\bigg(\frac{\baromegaistar}{\baromegaih}\bigg) \approx \frac{\lambda_i}{\sqrt{d_e}} (\muistar - \muih).
\end{align}
Note that $\mu_i^\star $ is the largest among $\mu_i^\h$ for all $h \in [H]$. The second equation in \eqref{eq:approximated_dyn_main2} ensures that $\baromegaistar/\baromegaih$  increases in times. Since they are initialized with the same value $\omega_0$, \eqref{eq:approximated_dyn_main2} implies that $\baromegaistar$ will grow faster than $\baromegaih$ for all $i \in I$, i.e., $\baromegaistar = \max_{i\in [I]}\baromegaih$. 
By examining the two equations in  \eqref{eq:approximated_dyn_main2}, we see that the dominance of $\muistar$ over each other $\muh$ will induce the dominance of $\baromegaistar$ over each other $\baromegaih$ through the second equation, which then in-turn reinforces the dominance of $\muistar$ via the first equation.

In summary, during the warm up stage, both $\muistar$ and $\baromegaistar$ will increase quickly and keep their dominating role over other heads in terms of solving task $i$. 
This phase terminates when there exists an $i \in [I]$ such that $\baromegaistar$ is sufficiently large such that the interference terms in the dynamics of $\muih$ cannot be canceled out across different heads.

\paragraph{Phase II: Emergence.}
After the warm-up phase, the optimal head parameters $\muistar$ and $\baromegaistar$  undergo a rapid and substantial increase for each task $i$. 
As the optimal head $h_i^\star$  becomes increasingly dominant, $\{ \mu_i^\h , \baromega_i^\h\}_{h \neq h_i^* }$ for $i\neq h_i^\star$ will remain small.
As a result, 
we can  only focus on  the optimal head parameters and neglect the non-optimal head parameters, which yields the following simplified dynamics: 
\begin{align}\label{eq:approximated_dyn_main3}
    \begin{split}
    \partial_t \muistar 
    &\approx \lambda_i d_i \cdot \Big (- {\exp(d_i(\baromegaistar)^2)} \cdot {L}^{-1} \phi_i{\muistar}^2 + (1-\baromegaistar\muistar ) \baromegaistar \muistar \Big ),  \\
    \partial_t \baromegaistar &\approx \Bigl (1 - \big (1 + {\exp(d_i(\baromegaistar)^2\bigr )} \cdot {L}^{-1}  \cdot d_i\phi_i ) \cdot \muistar  \baromegaistar \Big ) \cdot d_e^{-1/2} \lambda_i\baromegaistar\muistar.
    \end{split}
\end{align}
To see that the dynamics in \eqref{eq:approximated_dyn_main3} captures the gradient flow dynamics, we plot the dynamics of $\muistar$ in \Cref{fig:dynamics_simu}-(a) for $\phi_i = \lambda_i = 1$, $d_i = 10$, $d_e = 30$, and $L = 100$. The curve aligns well with the gradient flow dynamics in the simulation experiments, which is plotted in \Cref{fig:dyn}.

To analyze this dynamics, we first note that because the first equation in \eqref{eq:approximated_dyn_main3} has a factor $d_i$ while the second equation is multiplied by $d_e^{-1/2}$,  we can conclude that $\muistar$ changes much faster than $\baromegaistar$.
As a result, we  leverage a two-timescale analysis by first pushing $\muistar$ to its limit and write the limiting value as a function of $\baromegaistar$. 
Then we can only focus on the dynamics of $\baromegaistar$ and analyze its behavior.
In particular, setting the right-hand side of the first equation in \eqref{eq:approximated_dyn_main3} to zero, we have 
\begin{align}
    \muistar \approx \frac{\baromegaistar}{\baromegaistar^2 + \phi_i\exp\bigl(d_i(\baromegaistar)^2\bigr) L^{-1}}. 
    \label{eq:muistar dynamics}
\end{align}
Then, plugging \eqref{eq:muistar dynamics} into the dynamics of $\baromegaistar$, we have 
\begin{align}
    \partial_t \baromegaistar \approx \frac{\left(1-d_i(\baromegaistar)^2 \right) \cdot d_e^{-1/2}\lambda_i } {2 
    + {\phi_i\exp(d_i(\baromegaistar)^2)}/{(L(\baromegaistar)^2)}
    + {L(\baromegaistar)^2}/{(\phi_i\exp(d_i(\baromegaistar)^2))}}. 
    \label{eq:emergence}
\end{align}
Setting the right-hand side of \eqref{eq:emergence} to zero yields a critical point $\baromegaistar = d_i^{-1/2}$. 
Thus, when $\baromegaistar < d_i^{-1/2}$, the right-hand side of \eqref{eq:emergence} is positive,
which indicates the growth of $\baromegaistar$ before its  value reaches $d_i^{-1/2}$.
To further understand the dynamics of $\baromegaistar$, it suffices to consider the behavior of the denominator: 
\[
    2 
    + \frac{\phi_i\exp\bigl(d_i(\baromegaistar)^2\bigr)}{L(\baromegaistar)^2}
    + \frac{L(\baromegaistar)^2}{\phi_i\exp(d_i(\baromegaistar)^2)}
\]
For small $\baromegaistar$, the second term is large and dominates, thus resulting in the slow growth of $\baromegaistar$ at the beginning.
However, as $\baromegaistar$ \siyurevise{grows larger (but still much less than $d_i^{-1/2}$)}, the value of the denominator quickly decreases, making the growth of $\baromegaistar$ fast and giving rise to the emergence phase. 
We also plot the dynamics of $\baromegaistar$ according to the apprximated dynamics in \eqref{eq:emergence} in \Cref{fig:dynamics_simu}-(b). The curve aligns well  \Cref{fig:dyn}, which justifies the validity of the approximation.


\paragraph{Phase III: Convergence.}
\siyurevise{As $\baromegaistar$ approaches $d_i^{-1/2}$, we see from \eqref{eq:emergence} that the growth of $\baromegaistar$ slows down, and eventually $\baromegaistar$ converges to $d_i^{-1/2}$.
By also invoking \eqref{eq:muistar dynamics}, we can show that eventually}
\begin{align*}
\baromegaistar \to \frac{1}{\sqrt{d_i}}, \quad \muistar \to \frac{\sqrt{d_i}}{1+ e  d_i\phi_i L^{-1}}.
\end{align*}
While for each non-optimal head, $\baromegaih$ does not change much from the initialization and $\muih$ converges to $0$. 

Collecting all the above components, we arrive at the main result on gradient flow dynamics presented in \Cref{thm:convergence-multi-head-symmetric}.
For a rigorous analysis, we pick constants $c$ and $\epsilon$ such that 
\begin{align}
    3 \le c = o(\sqrt{\log L}), \quad \frac 1 2 > \epsilon \ge \frac{1}{c} + \frac{3}{1 + \sqrt{1+c^2/2}}.
    \label{eq:epsilon}
\end{align}
Picking $\epsilon =1/4$ suffices, and a smaller value of $\epsilon$ requires larger $L$.


\begin{figure}[ht]
    \centering
    \begin{subfigure}[b]{0.32\textwidth}
        \includegraphics[width=\textwidth]{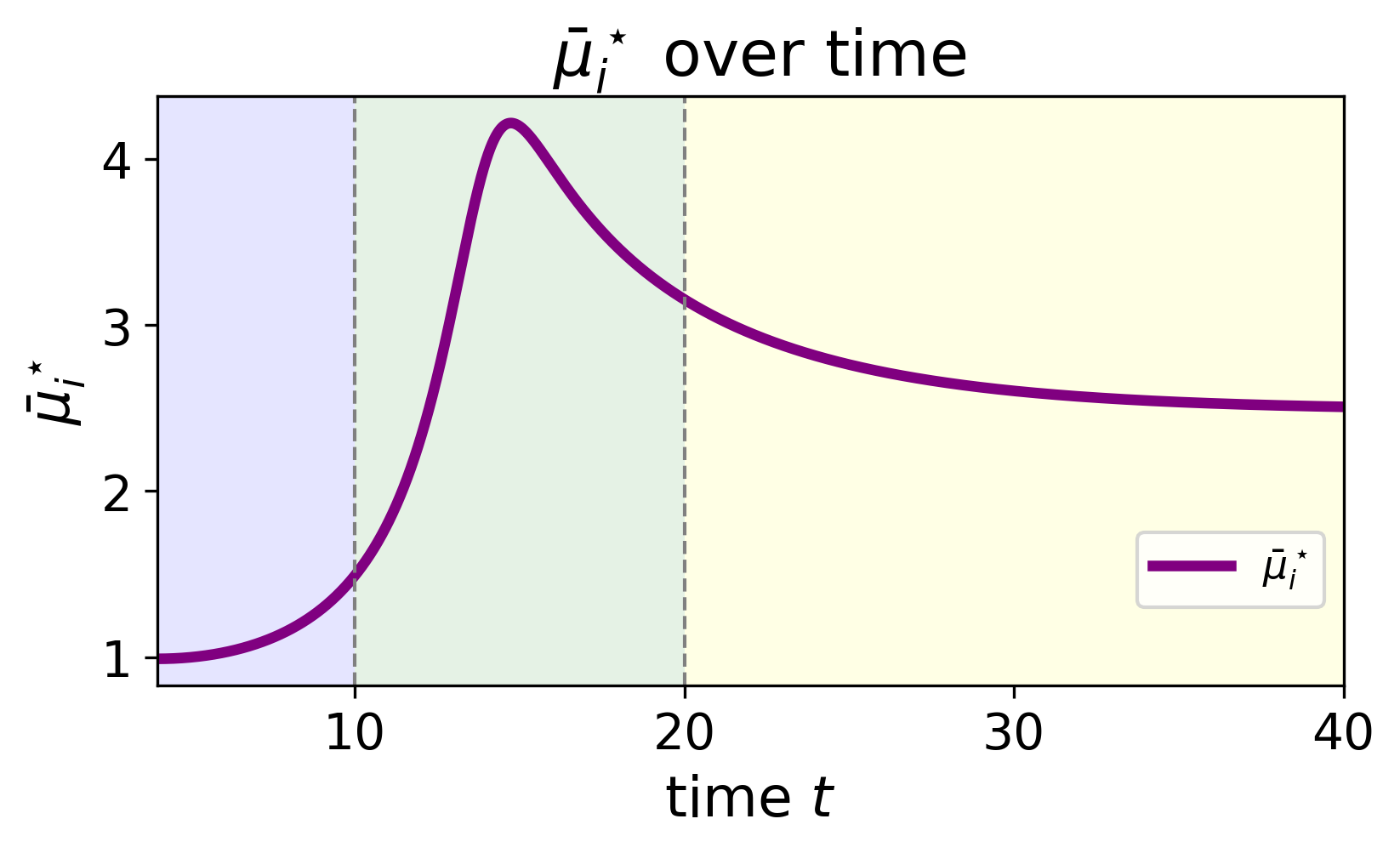}
        \caption{Dynamics of $\mu_i^\star$ in \eqref{eq:approximated_dyn_main3}}
        \label{fig:sub1}
    \end{subfigure}
    \hfill 
    \begin{subfigure}[b]{0.32\textwidth}
        \includegraphics[width=\textwidth]{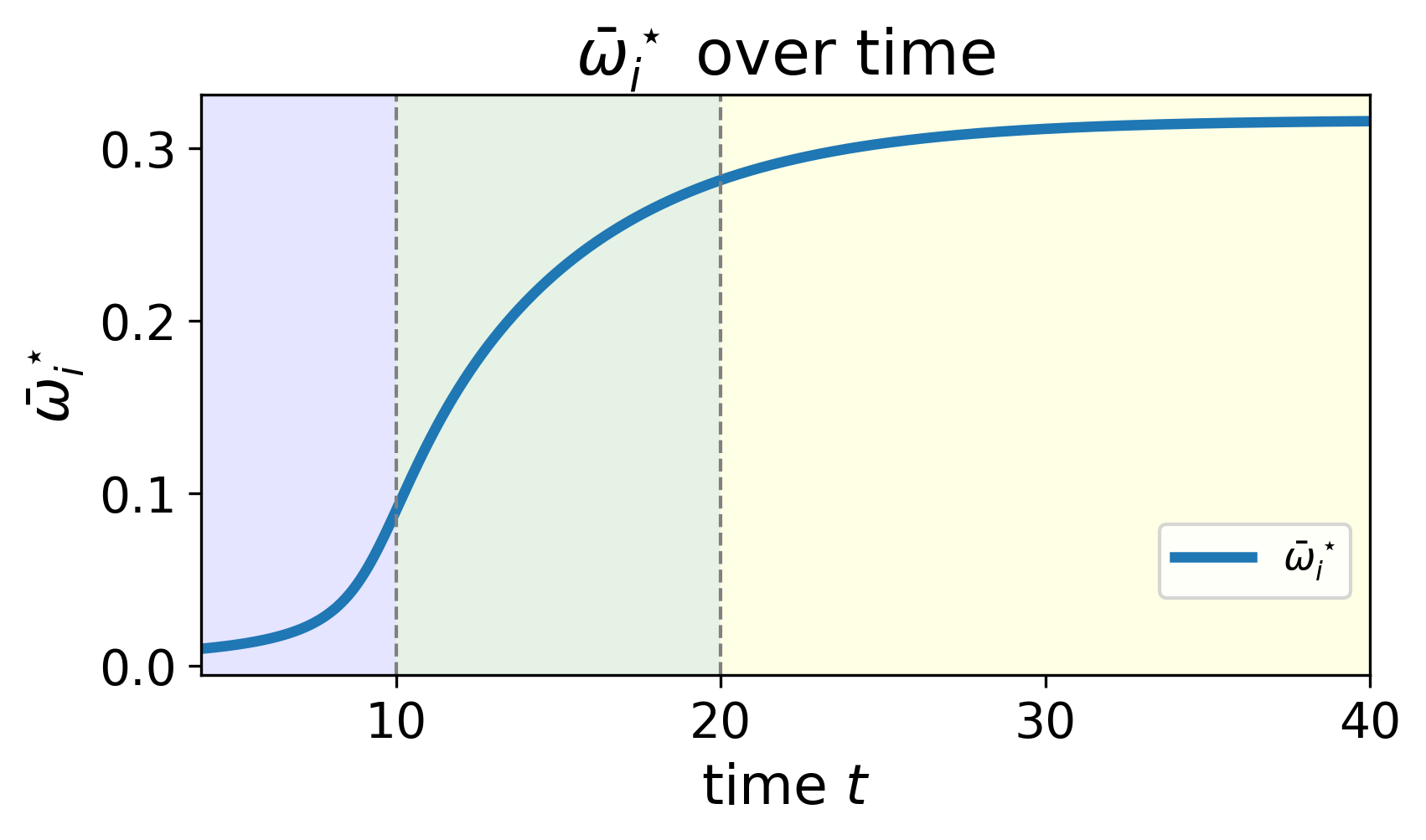}
        \caption{Dynamics of $\baromegaistar$ in \eqref{eq:emergence}}
        \label{fig:sub2}
    \end{subfigure}
    \hfill 
    \begin{subfigure}[b]{0.32\textwidth}
        \includegraphics[width=\textwidth]{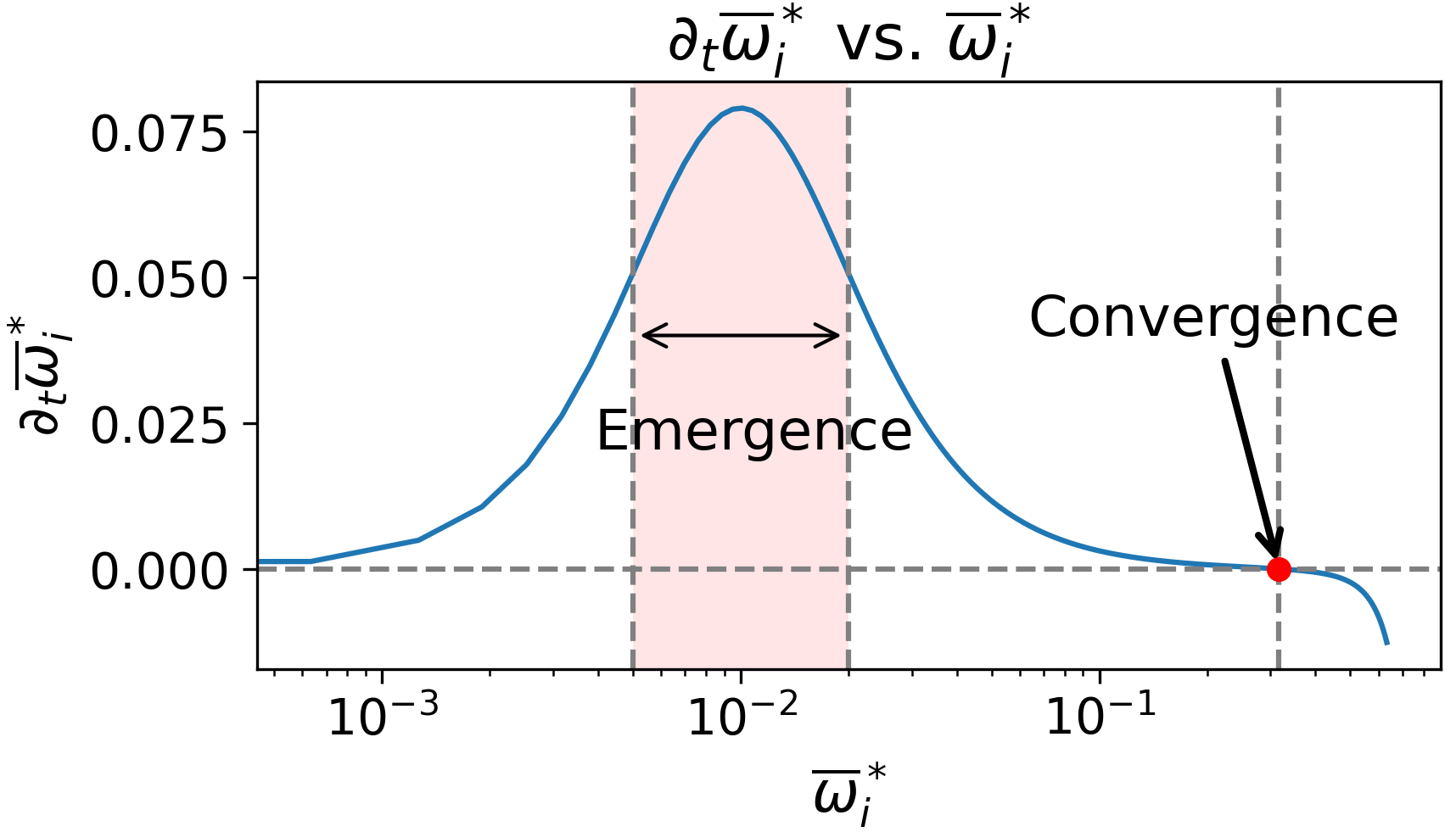}
        \caption{$\partial_t\baromegaistar$ as a function of $\baromegaistar$}
        \label{fig:sub3}
    \end{subfigure}
    \captionsetup{width=0.95\textwidth}
    \caption{
        Approximated dynamics of the gradient flow with $d_i = 10$, $d_e = 30$,  $L=100$, $\phi_i =1$ and $\lambda_i =1$.  
    In (a) we plot the dynamics of $\mu_i^\star$ in \eqref{eq:approximated_dyn_main3}, starting from $t =4$ and $(\baromegaistar, \muistar) = (0.01, 1)$.
    In (b) we plot the dynamics of $\baromegaistar$ in \eqref{eq:emergence}, starting from $t =4$ and $\baromegaistar = 0.01$. 
    We note that the approximated dynamics  given by \eqref{eq:approximated_dyn_main3} and \eqref{eq:emergence}
    are consistent with the numerical simulations in \Cref{fig:dyn} based on data. 
    In (c) we plot $\partial_t\baromegaistar$ as a function of $\baromegaistar$ with the same setup. When $\baromegaistar$ is small (left), the growth of $\partial_t\baromegaistar$ is slow. When $\baromegaistar$ reaches the middle region, the growth of $\baromegaistar$ accelerates. When $\baromegaistar$ grows to the right end, it converges to $d_i^{-1/2}$. }
    \label{fig:dynamics_simu}
\end{figure}

\begin{theorem}[Convergence of gradient flow for \ac{msa}]
    \label[theorem]{thm:convergence-multi-head-symmetric}
\revise{Under \Cref{assump:data} and \Cref{assump:scale}, let $L$ be sufficiently large such that \eqref{eq:epsilon} holds.}
Let $d_e \ge d$. 
Suppose the initialization is decomposable in the sense of \Cref{def:decomposability property} and moreover,
\begin{align*}
    \sigma(Q_X^\h) = \sigma(K_X^\h) = \sqrt{\omega_0} \vone_{d}, \quad 
    \sigma(O^\h) = \sigma(V_Y^\h) = \sqrt{\mu_0^\h}, 
\end{align*}
where $\omega_0$ and $\mu_0^\h\in \RR^{d_y}$ are positive and sufficiently small \siyurevise{in the sense of \Cref{assump:initialization}}.
Suppose $H \ge I$ 
and that at initialization, there exists a unique optimal head $h_i^\star$  for each task $i$ such that 
$
    {(\muistar(0) - \muih(0))}/{\muistar(0)} \ge c_1, \forall h\neq h_i^\star, \forall i\in[I],$
where $c_1>0$ is a positive constant. Let $\muistar$ and $\baromegaistar$ be abbreviations of $\mu_i^{(h_i^\star)}$, $\baromega_i^{\histar}$, respectively. 
We rescale $t \leftarrow 2 d t$.
Then the following holds:
\begin{myenumi}[
    ref = \Cref{thm:convergence-multi-head-symmetric}(\roman*),
    leftmargin=.5cm,
    align=left,
]
\item \label{result:optimal head takes all}
    \emph{(Optimal Head Takes All)} For each task $i\in[I]$, the optimal head $h_i^\star$ determined by initialization will remain optimal during the training process in the sense that $\muistar > \muih$ and $\baromegaistar > \baromegaih$ for all $h\neq h_i^\star$. 
    In particular, for any non-optimal head $h\neq h_i^\star$, we have \siyurevise{$\muih\le \muistar \exp(-\Omega(\sqrt{d_e}))$} and $\baromegaih \le O(\omega_0)$ upon convergence for all $i\in[I]$. 
\item \emph{(Emergence)}
    Fix a small constant $\alpha = \Theta(1) < 1$ and let $T_{0,i} = {\sqrt{d_e}\phi_i}/{(\lambda L \omega_0)}$
    be the threshold time for task $i$ to emerge.
    For time $t\le (1 - \alpha) T_{0,i}$, we have $\baromegaistar(t) \le \alpha^{-1} \omega_0$, \siyurevise{which means $\baromegaistar(t)$ is \say{fixed} around initialization}
    and the loss for task $i$ is larger than \siyurevise{$\Omega(\lambda_i d_i / d)$}, where $\lambda_i d_i / d$ is the value of the \ac{icl} loss if the attention output is $0$ for task $i$.
    If $t\ge (1 + \alpha) T_{0,i}$, then $\baromegaistar(t) \ge \sqrt{\alpha d_i^{-1}}$ and the loss for task $i$ is at most $ O(e^\alpha/\alpha)$ times the optimal loss for task $i$ in the single-head case given later by \eqref{eq:1head1task-optimal loss}:
    \[
        \cL_i^\star \defeq \frac{\lambda_i d_i}{d} \cdot \frac{e d_i \phi_i L^{-1}}{1 + e d_i \phi_i L^{-1}}. 
    \]
\item (\emph{Convergence}) \label{result:convergence}
    Consider the optimal head for task $i$. Let $\delta \in (0, 1)$.  
    With training time  
    \[\Omega\bigg( \max_{i\in[I]} \frac{\sqrt{d_e} \phi_i}{\lambda_i L \omega_0}  + \frac{\sqrt{d_e}}{\lambda_i\sqrt{d_i}} \log \frac{4}{\delta}\bigg) \le T \le \exp(O(\sqrt{d_e})),\]
    it holds that 
    \begin{align}
        \sup_{i\in[I]} \bigg|
                \frac{\baromegaistar(T)}{{d_i}^{-1/2}}  -  1
    \bigg| \le \delta, \quad\text{and}\quad 
            \sup_{i\in[I]} \bigg|
                \frac{\muistar(T)}{\sqrt{d_i}/(1 + e d_i \phi_i L^{-1})} - 1
            \bigg| \le O(\delta ),
        \label{eq:convergence-error}
    \end{align}
        which is the same as the global optimal solution for a single-head softmax attention trained only for task $i$ (see \eqref{eq:1head1task-optimal solution}).
\end{myenumi}
\end{theorem}




The upper bound of the convergence time is due to technical reasons where we need to control the cumulative effect of the approximation error in the gradient flow.
We compare our results to the existing literature regarding the training dynamics of attention-based models.
\citet{zhang2023trained} proved the convergence of gradient flow for one-layer single-head  linear attention, while our results are for multi-head softmax attention and we do not combine the key and query weight matrices into a single matrix.
Our results further provide a full characterization of the training dynamics from small initialization to convergence and demonstrate a phase transition in the emergence of the \ac{icl} capacity. 
\citet{huang2023context} was the first to study the training dynamics of softmax attention, but their analysis is limited to the single-head case with context features coming from an orthogonal dictionary. 
Moreover, the convergence point derived in \citet{huang2023context} is in sharp contrast to ours. 
As we will show in the following section, we have $\EE[\norm{p}_2^2] = O(L^{-1})$ upon convergence, while the convergence point in \citet{huang2023context} has $\EE[\norm{p}_2^2] =\Omega(1)$ as the tokens that are the same as the query will have dominating attention weights. 
The difference is due to the orthogonality between token features and not taking into account the effect of noise in their analysis. 
Indeed, we will show in the following section (\Cref{thm:optimality of ssa}) that having a dominating attention weight is suboptimal in the presence of noise.
Finally, \citet{deora2023optimization} also considered training a multi-head softmax attention model, but they focused on the classification setting and their analysis is limited to the regime of neural tangent kernel \citep{jacot2018neural} where the parameters do not evolve much during training.

We conclude this section with a discussion of the convergence point of the gradient flow.

\paragraph{Each Optimal Head Aligns with the Task's Subspace.}
By \Cref{result:convergence}, for the optimal head $h_i^\star$ corresponding to task $i$, the attention score is computed as $s = \baromegaistar (\varPhi X)_{(i)}^\top  (\varPhi q)_{(i)}$, where we use subscript \say{$(i)$} to denote the slice of a vector (matrix) corresponding to task $i$'s indices $\cJ_i$.
Intuitively, the attention score is obtained by (i) rotating the key and query vectors to align with the task's eigenvectors and (ii) computing the inner product between the aligned key and query vectors after projected onto the task's subspace.

\paragraph{What is the Model learned by Gradient Flow?}
In addition, it was previously found by \citet{ahn2023transformers, mahankali2023one} that the optimal weights of a one-layer linear attention model correspond to one step of preconditioned gradient descent.
While in our case, the model in \Cref{thm:optimality of ssa} found by gradient flow is analogous to emulating certain \emph{kernel regression} where the softmax attention aggregates similarity scores between the query and the keys across all tasks. 
Note that if $(\varPhi X)_{(i)} $ and $(\varPhi q)_{(i)}$ are of the same 2-norm, then the attention probability is given by $ p_l \propto \exp(- \frac{1}{2} \baromegaistar \norm{(\varPhi x_l)_{(i)} -  (\varPhi q)_{(i)}}_2^2)$, and the output for task $i$, $\hat y_i = \muistar Y p$, is the same as that of the kernel smoothing method based on the radial basis function kernel with bandwidth $(\baromegaistar)^{-1/2}$.



\section{Optimality of Convergence Point of Gradient Flow} \label[section]{sec:optimality_main}
Next, we discuss the optimality (in terms of the ICL loss value) of the model trained by gradient flow.
The main result is summarized in the following lemma.
Let us first understand the \ac{icl} loss for this solution. 
\begin{lemma}
[ICL Loss for \ac{msa}]
\label[lemma]{lem:msa icl loss-mainbody}
Under the setting of \Cref{thm:convergence-multi-head-symmetric}, the ICL loss value achieved by the convergence point of gradient flow is upper bounded by 
\begin{equation}
    \sum_{i=1}^I \frac{\lambda_i d_i }{d} \cdot \left(\frac{e d_i \phi_i L^{-1}}{1 + e d_i \phi_i L^{-1}} + O(L^{-(1-\epsilon)/2} + \delta + \omega_0^2 d) \right), 
    \owntag{\ConvergenceICL}
    \label{eq:icl-loss-convergence-point}
\end{equation}
where $\epsilon$ is defined in \eqref{eq:epsilon} and $\delta$ is the error for the convergence point in \eqref{eq:convergence-error}. 
\end{lemma}
\siyurevise{We defer readers to \Cref{lem:msa icl loss} and its follow-up proof for the detailed derivation.}
For sequence length $L$, the \ac{icl} loss scales with $O(d/L)$ and is simply a summation of individual loss across tasks, which matches our observation that each task $i\in[I]$ is completed by an individual head without cross-head interference.
Naturally, one may ask:
\begin{enumerate}[leftmargin=0.7cm]
    \item Given the fact that each task is handled by a unique head, 
    does this dominating head's behavior endure similarity to a \ssa? Does it achieve the optimal \ac{icl} loss on that \emph{individual} task?
    \item Does this model achieve the optimal \ac{icl} loss within the class of \ac{msa}?
\end{enumerate}
In pursuit of answers to the above question, we continue to study the optimality properties for both \ssa~and \ac{msa}.
\revise{In \S\ref{sec:optimality_single_head}, we investigate what are the optimal attention weights and loss value for a single-head softmax attention dealing with multiple tasks.
In \S\ref{sec:optimality_multi_head}, we study lower bound of the optimal \ac{icl} loss within the class of \emph{equiangular weights} under homogeneity assumption on the tasks, and show that the convergence point of the gradient flow matches the lower bound.
In \S\ref{sec:implicit_regularization}, we give a proof sketch to these optimality results with more insights into the softmax attention.
}

\subsection{The Single-Head Case: Allocation of Attention Budget}\label[section]{sec:optimality_single_head}

We first consider the single-head case where $H=1$, and the goal here is to characterize the minimal value of the ICL loss, as summarized in the following theorem.
Here we recall the definition $\phi_i = 1 + \sigma^2 d / (\lambda_i d_i)$.

\begin{theorem}[Optimal Loss for Single-Head Softmax Attention]
\label[theorem]{thm:optimality of ssa}
Suppose \Cref{assump:data} holds with $\varPhi$ and $\varPsi$ being identity matrices (which is without loss of generality).
Under \Cref{assump:scale}, suppose $L$ is sufficiently large such that \eqref{eq:epsilon} holds with constants $\epsilon$ and $c$.
For a single-head softmax attention, assume that $W_Y = 0$. 
Consider constant noise level $\sigma^2=\Theta(1)$. 
Suppose for each task $i\in[I]$, the signal strength satisfies either $\lambda_i = \Theta(1)$ or $\lambda_i = 0$ for any $i\in[I]$. 
The model with the weights 
    \begin{align*}
        W^\star = \begin{bNiceArray}{cccc|c}
            \sqrt{b_1^\star/d_i}\cdot I_{d_1} & 0 & \cdots & 0 & \Block{4-1}<\large>{\star}\\
            0 & \sqrt{b_2^\star/d_i}\cdot I_{d_2} & \cdots & 0 & \\
            \vdots & \vdots & \ddots & \vdots & \\
            0 & 0 & \cdots & \sqrt{b_I^\star/d_I} \cdot I_{d_I} & \\
            \hline 
            \Block{1-4}<\large>{0} & & & & \star
        \end{bNiceArray}, \quad 
        U^\star =\begin{bNiceArray}{c|cccc}
            0 & \Block{1-4}<\large>{0} & & & \\
            \hline 
            \Block{4-1}<\large>{0} & u_1^\star & 0 & \cdots & 0\\
             & 0 & u_2^\star & \cdots & 0\\
             & \vdots & \vdots & \ddots & \vdots\\
             & 0 & 0 & \cdots & u_I^\star
        \end{bNiceArray}
    \end{align*}
    approximately achieves the optimal loss in the sense that
    \[
        \cL(U^\star, W^\star) \le \inf_{U, W: W_Y = 0} \cL(U, W) +  O(L^{-(1-\epsilon)/2}).
        \]
    Moreover, $B^\star$ and $b^\star = (b_i^\star)_{i\in[I]}$ are obtained by solving the following optimization problem:
    \begin{align}
        \min_{c^{-2}2\log L\ge B \ge 0, \atop b\in \RR_+^{I},  \vone^\top b = B}  \cL_\simple(B, b)\defeq \sum_{i=1}^I \frac{\lambda_i d_i}{d} \bigg( 1 - \frac{b_i }{ b_i  +  d_i\phi_i \cdot \exp(B) \cdot L^{-1} } \bigg),  
        \owntag{\SingleheadICLOpt}
        \label{eq:water filling}
    \end{align}
    and $u_i^\star = {\sqrt{b_i^\star d_i}}/{( b_i^\star + d_i \phi_i  \exp(B^\star)/L)}$. 
    In addition, $|\cL(U^\star, W^\star) - \cL_\simple(B^\star, b^\star)| = O(L^{-(1-\epsilon)/2})$.
\end{theorem}
See \S\ref{sec:implicit_regularization} for a proof sketch and \S\ref{sec:global-optimality-single-head} for a detailed proof.
\revise{In a nutshell, \Cref{thm:optimality of ssa} shows that the optimal \ac{icl} loss and the corresponding weights for a \ssa~are related to the optimal value and solutions to an optimization problem.
Our proof is based on first lower bounding the \ac{icl} loss for any weights by the optimal value of \eqref{eq:water filling} and then show the constructed optimal solution achieves the lower bound.
Note that the softmax operation induces nonlinearity, and our proof technique is to identify the operating regimes of the softmax attention (\Cref{fig:regime} in \S\ref{sec:implicit_regularization}) and decide the optimal operating regime for the \ac{icl} task.
After that, we isolate the higher order nonlinear terms and approximate the lower order nonlinear term in the system to make rigorous nonlinear analysis tractable.} 
We also remark that setting $W_Y=0$ is a standard practice for analyzing attention-based models for \ac{icl} and similar assumptions are also employed in \citet{zhang2023trained,von2023transformers, ahn2023transformers,huang2023context}. 

\revise{Below we provide discussion on the insights and implications of \Cref{thm:optimality of ssa}.}

\paragraph{Optimal Solution is Decomposable.}
For general orthogonal matrices $\varPhi$ and $\varPsi$, the corresponding optimal weights can be constructed by applying the rotation to $W_X^\star$ and $U_Y^\star$ while keeping the remaining entries zero, i.e., the optimal weights $\widetilde W^\star$ and $\widetilde U^\star$ are given by
\[
    \widetilde W_X^\star = \varPhi^\top W_X^\star \varPhi, \quad 
    \widetilde W_Y^\star = 0, \quad 
    \widetilde U_Y^\star = \varPsi^\top U_Y^\star \varPsi, \quad 
    \tilde U_X = 0.
\]
To see the equivalence, one can view the rotated data $\tilde X = \varPhi^\top X$, $\tilde \varepsilon = \varPsi^\top \varepsilon$ and $\tilde q = \varPhi^\top q$ as the new input, and the loss value is preserved due to the rotational invariance of the input distribution (See \eqref{eq:rotation-optimality}).
More importantly, we remark that the optimal weights are decomposable since  $W_X^\star$ and $U_Y^\star$ are diagonal, and one can always construct the corresponding $Q_X^\star$, $K_X^\star$, $V_Y^\star$ and $O^\star$ (also subject to rotation for general $\varPhi$ and $\varPsi$) that share common eigenvector spaces to get the same model.
Moreover, it is clear that the optimal weights also have homogeneous entries within each task.
These facts justify the decomposability condition in \Cref{def:decomposability property}.

\paragraph{Allocation of Attention Budget.}
We interpret the sum of the squared eigenvalues, $\|\omega\|_2^2 = \|W_X\|_{\fro}^2$, as the total \emph{attention budget}.
Correspondingly, one can intuitively understand the sum of the squared eigenvalues within task $i$'s support, $\sum_{j\in\cJ_i}\omega_j^2$, as the \emph{allocation of attention budget} for task $i$.
Then in \eqref{eq:water filling}, $B$ and $b=\{b_i\}_{i\in[I]}$ are the attention budget and the allocation vector of the budget, respectively.
Fixing the attention budget $B$, the inner optimization problem over $b$ is simply trying to optimally allocate the budget according to the signal-to-noise ratios $\{\snr_i\}_{i\in[I]}$ and the task dimensions $\{d_i\}_{i\in[I]}$, which is a convex optimization problem.

\paragraph{Softmax Attention Works in the Exponential Regime.}
Note that the optimal attention budget scales as $B^\star=o(\log L)$ (see the justification of this in \Cref{fact:order of attention budget}).
As we will illustrate later in \S\ref{sec:implicit_regularization}, such an attention budget corresponds to the \emph{exponential regime}\footnote{See \Cref{fig:regime} for a graphic illustration of the \say{exponential} behavior of the softmax attention in this regime} of softmax attention. 
This is a consequence of the fact that the $\exp(B)$ in the denominator of \eqref{eq:water filling} penalizes the attention budget. 
In contrast, in the \emph{saturation regime} where the attention budget $B=\Omega(\log L)$, the attention probability vector concentrates on a few tokens, which is suboptimal in the presence of noise (\Cref{lemma:SS-Attn lb-2} and \Cref{lemma:SS-Attn lb}).

Given that the attention budget $B=o(\log L)$, the attention probability vector is delocalized\footnote{This is a direct result of \Cref{lemma: E[pq]-moment} applied to a single head which says $\EE[\norm{p}_2^2] \approx \exp(\norm{\omega}_2^2)/L =o(1)$} so that the attention is spread out to capture the information from similar tokens in regression tasks and average out the noise.
This is in sharp contrast to the solution found by \citet{huang2023context} where the attention probability is concentrated on the tokens that are the same as the query, which is because their context features are from an orthogonal dictionary and the data is noiseless.

\paragraph{\ac{msa} Acts as Multiple Optimal Single-Head Softmax Attention.}
Suppose an attention head focuses on a single task, say task $i$, and treats the other tasks as with zero signal strength, i.e. $\lambda_j = 0$ for any $j\neq i$. 
We apply \Cref{thm:optimality of ssa} to this head, and the optimal eigenvalues are 
\begin{align}
    \baromega_i = \frac{1}{\sqrt{d_i}}, \quad \mu_i = \frac{\sqrt{d_i}}{1 + e d_i \phi_i L^{-1}}, \quad, \baromega_j = \mu_j = 0, \quad \forall j \in [I]\backslash\{i\}.
    \label{eq:1head1task-optimal solution}
\end{align}  
This is obtained by letting $b_i = B$ for the inner optimization in \eqref{eq:water filling} and directly solving the outer optimization problem which gives optimal solution at $B^\star=1$. 
The optimal loss value (up to $O(L^{-(1-\epsilon)/2})$ error) for task $i$ is then given by 
\begin{align}
    \cL_i^\star \defeq \frac{\lambda_i d_i}{d} \cdot \frac{e d_i \phi_i L^{-1}}{1 + e d_i \phi_i L^{-1}} = O\bigg(\frac{d}{L}\bigg),
    \label{eq:1head1task-optimal loss}
\end{align} 
Comparing this to the multi-head model trained by gradient flow in \Cref{thm:convergence-multi-head-symmetric}, we discover that at the convergence point, the optimal head is acting as an optimal \ssa~for each task.
Is this optimal within the class of \ac{msa}? 
We will next investigate this question over a class of equiangular weights, which is a natural class to consider when the tasks are homogeneous in the sense that they share the same dimension and signal strength.




\subsection{The Multi-Head: Equiangular Lower Bound and Bayesian Risk}\label[section]{sec:optimality_multi_head}


We start by introducing the following class of weights, which we call \emph{equiangular weights}.
\begin{definition}[Equiangular Weights]
    \label[definition]{def:equiangular}
    For weights of a \ac{msa} that are decomposable in the sense of \Cref{def:decomposability property}, we say the corresponding task-aggregated eigenvalues $\{\baromega^\h\}_{h\in[H]}$  are \emph{equiangular} if there exist constants $a, b\in\RR$ such that $\norm{\baromega^\h}_2^2 = a$ and $\langle \baromega^\h, \baromega^\hprime \rangle = b$ for all distinct $h, h'\in[H]$.
\end{definition}

The equiangular property ensures the cosine similarity between any two task-aggregated weights to be the same, \revise{which reduces the degree of freedom in the system and it suffices to optimize over $a$ and $b$.}
We remark that the equiangular weights are a reasonable class to consider under the task homogeneity assumption that all tasks have the same dimension and signal strength.
The following theorem addresses the optimality among the class of \ac{msa} with equiangular weights.
\begin{theorem}[Lower Bound of \ac{msa} within Equiangular Weights]
    \label{thm:optimality of equiangular}
    Under \Cref{assump:data} and \Cref{assump:scale}, suppose $H\ge 2$ and $I$ tasks are homogeneous such that $d_i = d/I = \bard$ and $\lambda_i = \lambda$ for all $i\in[I]$.
    \revise{Suppose task coefficient $g_i \iidfrom \cN(0, \lambda / d \cdot I_{\bard})$ for all $i\in[I]$.}
Let $c$ and $\epsilon$ satisfy \eqref{eq:epsilon}, and consider the regime
where $\norm{\baromega^\h}_\infty \le \sqrt{2 \log L / 3 d c^2}, \norm{\mu^\h}_\infty \le L^{3/4 - \epsilon/2}$ for all $h\in[H]$.
Then for any \ac{msa} with equiangular weights in the sense of \Cref{def:equiangular}, the \ac{icl} loss is lower bounded by 
\begin{align}
    \max\bigg\{\underbrace{\frac{\lambda}{ \phi^{-1} d^{-1} L \cdot  (H-1) + 1 }}_{\ds\text{\emph{\MultiheadICLLB}}} + O(L^{-\epsilon/2}), \quad \cR_{\mathrm{Bayes}}\bigg\}    
    \label{eq:lower-bound-multi-head}
\end{align}
where $\cR_{\mathrm{Bayes}}$ is the Bayesian risk \siyurevise{for this \ac{icl} task} with prior $g_i \iidfrom \cN(0, \lambda / d \cdot I_{\bard})$ for all $i\in[I]$.
Moreover, by decomposing $\cR_{\mathrm{Bayes}} := \mathrm{Variance} + \mathrm{Bias}$, the two terms are given asymptotically by
\begin{equation}
    \begin{aligned}
        \mathrm{Variance} &= I \sigma^2 \cdot \frac{br + (1 + r) - \sqrt{(br - 1 + r)^2 + 4 b}}{2 \sqrt{(br- 1 + r)^2 + 4 b}}, \\
        \mathrm{Bias} &= \lambda \cdot \bigg(\frac{br (1+r) + (1 - r)^2 - |1-r|\sqrt{(br - 1 + r)^2 + 4 b r}}{2 \sqrt{(br- 1 + r)^2 + 4 b}} + \bigg(1 - \frac 1 r\bigg) \ind (r > 1)\bigg), 
    \end{aligned}
    \notag
\end{equation}
where $b = \sigma^2/\lambda$ and $r = \bard /L$.
\end{theorem}
See \S\ref{sec:implicit_regularization} for a proof sketch and \S\ref{sec:lowerbound_multihead} for a detailed proof.
We remark the norm constraints on the weights are without much loss of generality since the optimal weights for the \ssa~satisfy $\baromega = O({d}^{-1/2})$ and $\mu = O(\sqrt d)$ as we have shown in \Cref{thm:optimality of ssa}, which means the attention is working in the exponential regime as we will discuss in \S\ref{sec:implicit_regularization}.
As a corollary, we give an affirmative answer to the question of the optimality of the model found by gradient flow.
\begin{corollary}
    \label[corollary]{cor:equiangular better than symmetric}
    Under the setting of \Cref{thm:optimality of equiangular}, let $H=I$.
    The \ac{icl} loss of the model in \Cref{thm:convergence-multi-head-symmetric} achieves the lower bound of the \ac{icl} loss given by \Cref{thm:optimality of equiangular} up to a constant multiplicative factor.
\end{corollary}

\paragraph{Lower Bound is Achieved by the Convergence Point.}
Now we compare the upper bound of the \ac{icl} loss with the lower bound in \eqref{eq:lower-bound-multi-head} under different settings.
For simplicity, we consider the asymptotic regime ($d, L\rightarrow \infty$ and $d/L= \Theta(1)$) where we neglect the error term that only depends on $L$ and consider a small constant $\bard/L$.
The Bayesian risk (or \ac{mmse}) when ignoring higher order terms of $\bard/L$ is given by
\begin{equation}
    \mathrm{MMSE} = \frac{\sigma^2}{d^{-1} L + I^{-1}(\snr^{-1}-1)}. 
    \owntag{MMSE}
    \label{eq:asymptotic-bayesian-risk-limit}
\end{equation}
We consider four cases stratified by the relationship between $H$ and $I$, and compare in each case the lower bound and the upper bound on the \ac{icl} loss.  
The result is summarized in \Cref{tab:bounds_summary} and visualized in \Cref{fig:loss_comparison}.
\renewcommand{\arraystretch}{1.75}
\begin{table}[t] 
    \centering 
    \begin{tabular}{c|c|c|c} 
        \hline \hline
        Case & Lower Bound & Upper Bound & Matched?\\
        \hline 
        $H = I$ & $\text{\MultiheadICLLB}=\frac{\lambda}{(H-1)\phi^{-1}d^{-1}L + 1}$ & $\text{\ConvergenceICL}=\frac{\lambda}{H e^{-1}\phi^{-1} d^{-1}L + 1}$ & \checkmark\\ 
        \hline 
        $H > I$ & $\mathrm{MMSE} = \frac{\sigma^2}{d^{-1} L + I^{-1}(\snr^{-1}-1)}$ & $\text{\ConvergenceICL} = \frac{\sigma^2\cdot (1 + \snr)}{e^{-1} d^{-1} L + I^{-1}(1 + \snr^{-1})}$ & \checkmark (low SNR) \\ 
        \hline 
        $H < I$ & $\text{\MultiheadICLLB}=\frac{\lambda}{(H-1)\phi^{-1}d^{-1}L + 1}$ & $\ConstructICL=\frac{\lambda}{H\cdot e^{-1}\phi^{-1}d^{-1}L + 1}$ & \checkmark ($I=kH$)\\ 
        \hline 
        \hline
        $H=1$ & $\text{\SingleheadICLOpt} = \frac{\lambda}{e^{-1} \phi^{-1}d^{-1}L + 1}$ & same as the lower bound & \checkmark\\
        \hline \hline
    \end{tabular} 
    \caption{Summary of lower bounds and upper bounds for different cases. 
    Here, {\MultiheadICLLB} is the lower bound given by \eqref{eq:lower-bound-multi-head}, 
    {\ConvergenceICL} is the \ac{icl} loss of the convergence point of the dynamics in \eqref{eq:icl-loss-convergence-point} which serves as an upper bound, MMSE is the Bayesian risk, {\SingleheadICLOpt} is given by the optimal value to \eqref{eq:water filling} that is the optimal \ac{icl} loss for a single-head softmax attention.
    For the $H>I$ case, we rewrite the {\ConvergenceICL} in terms of the $\snr$ for better comparison.
    The term ``\ConstructICL'' refers to the \ac{icl} loss for the constructed solution under the condition $I=kH$ where each head handles $k$ different tasks optimally and treat other tasks as with zero signal in the sense of \Cref{thm:optimality of ssa}.
    } 
    \label{tab:bounds_summary} 
\end{table}
\renewcommand{\arraystretch}{1}
\begin{myenumi}
    \item If $H = I$, then the \ac{icl} loss at the convergence point achieves the lower bound up to a constant factor $e$, where we are exploiting the full power of the \ac{msa} to solve the tasks. 
    \item If $H \ge I$, in general the {\ConvergenceICL} is suboptimal compared with the Bayesian risk.
    However, in the low SNR regime, i.e.,  $\snr = \lambda/(I\sigma^2) = o(1)$, the \ac{icl} loss of the trained model also reaches the Bayesian risk up to a constant factor $e$. 
    This is in fact not surprising and is due to the softmax-induced regularization. 
    We will further illustrate this in \Cref{sec:implicit_regularization}.
    \item We do not consider the gradient flow for $H < I$. 
    However, it is not hard to construct a solution when $I = k H$ for some  positive integer $k$.
    In this case, we just let each head handle $k$ different tasks while treating the remaining as with zero signal, and the optimal weights are given by \Cref{thm:optimality of ssa} and the \ac{icl} loss is given by {\ConstructICL} shown in \Cref{tab:bounds_summary}.
    Clearly, the construction achieves the lower bound up to a constant factor $e$.
    \item When comparing the \ac{icl} loss upper bound of the \ac{msa} with $H$ heads (Case (i)$\sim$(iii)) to the optimal \ac{icl} loss for the single-head case ({\SingleheadICLOpt} by \Cref{thm:optimality of ssa}), 
    \emph{we clearly see that the \ac{msa} is at least $H\pmin I$ times more efficient than the \ssa~in terms of the \ac{icl} loss.}
\end{myenumi}
In summary, we have shown by cases the tightness of the lower bound in \Cref{thm:optimality of equiangular}.
In particular, the Bayesian risk lower bound is achieved by the convergence point of the \ac{msa} for $H\ge I$ in the lower SNR regime.

\begin{figure}[t]
    \begin{subfigure}{0.45\textwidth}
        \centering 
        \includegraphics[width=1\textwidth]{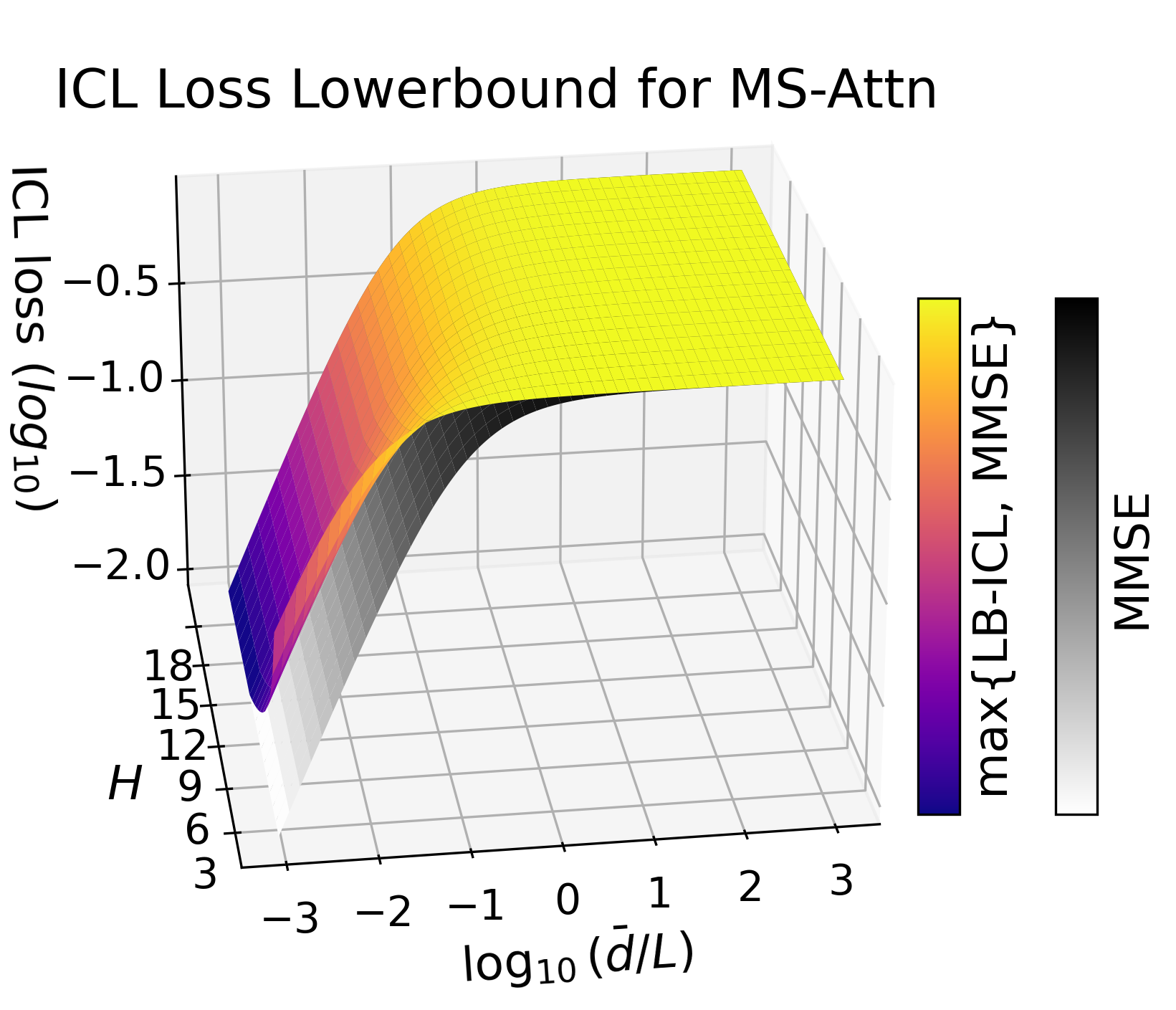}
        \caption{Visualization of the lower bound in \Cref{thm:optimality of equiangular}  and the Bayesian risk (MMSE).
        As further illustrated on the right, 
        for $H < I$, {\MultiheadICLLB} serves as the effective lower bound and for roughly $H> I$, the Bayesian risk serves as the effective lower bound. 
        }
    \end{subfigure}
    \hspace{10pt}
    \begin{subfigure}{0.35\textwidth}
        \centering 
        \includegraphics[width=0.9\textwidth]{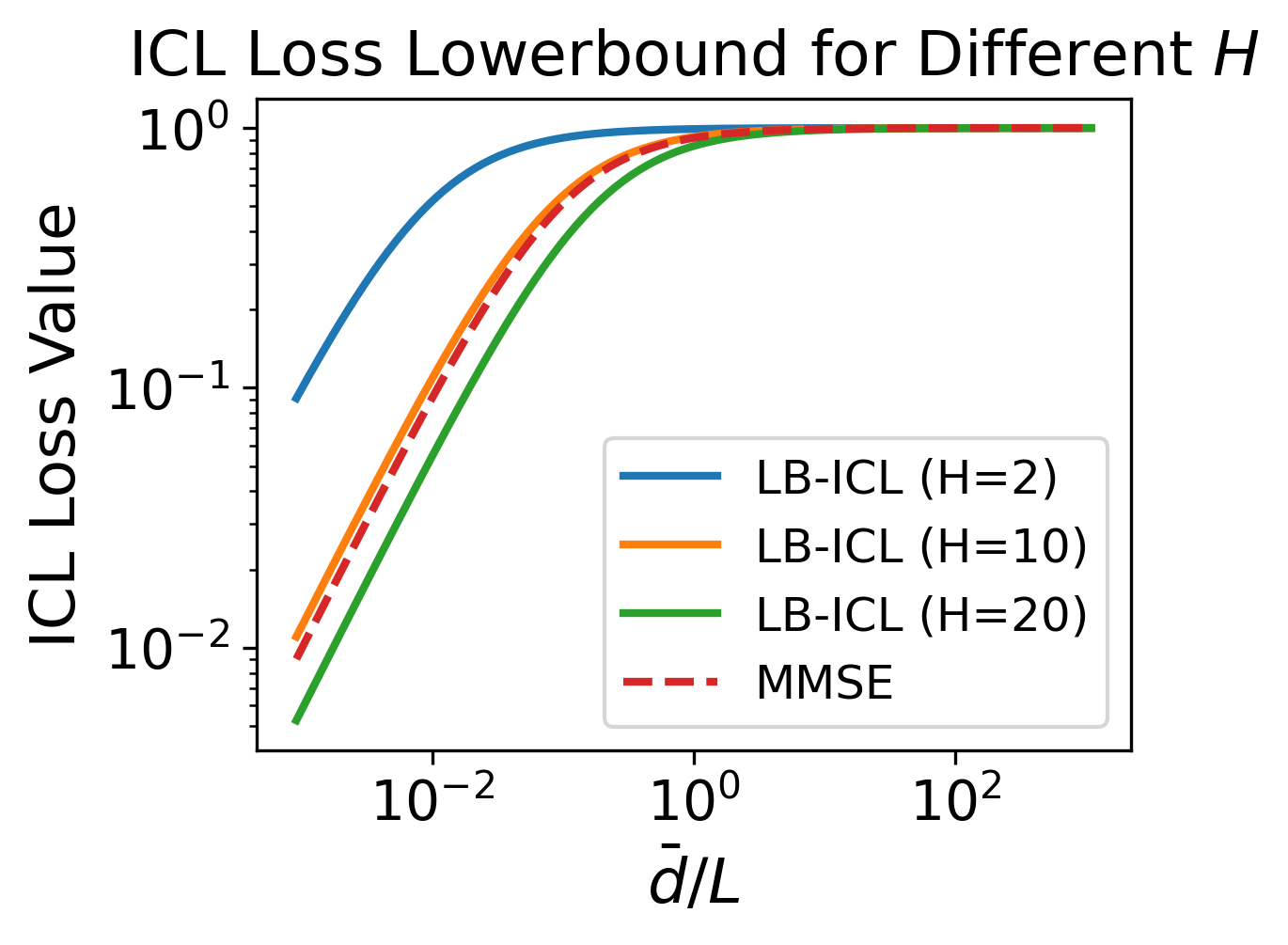}
        \includegraphics[width=0.9\textwidth]{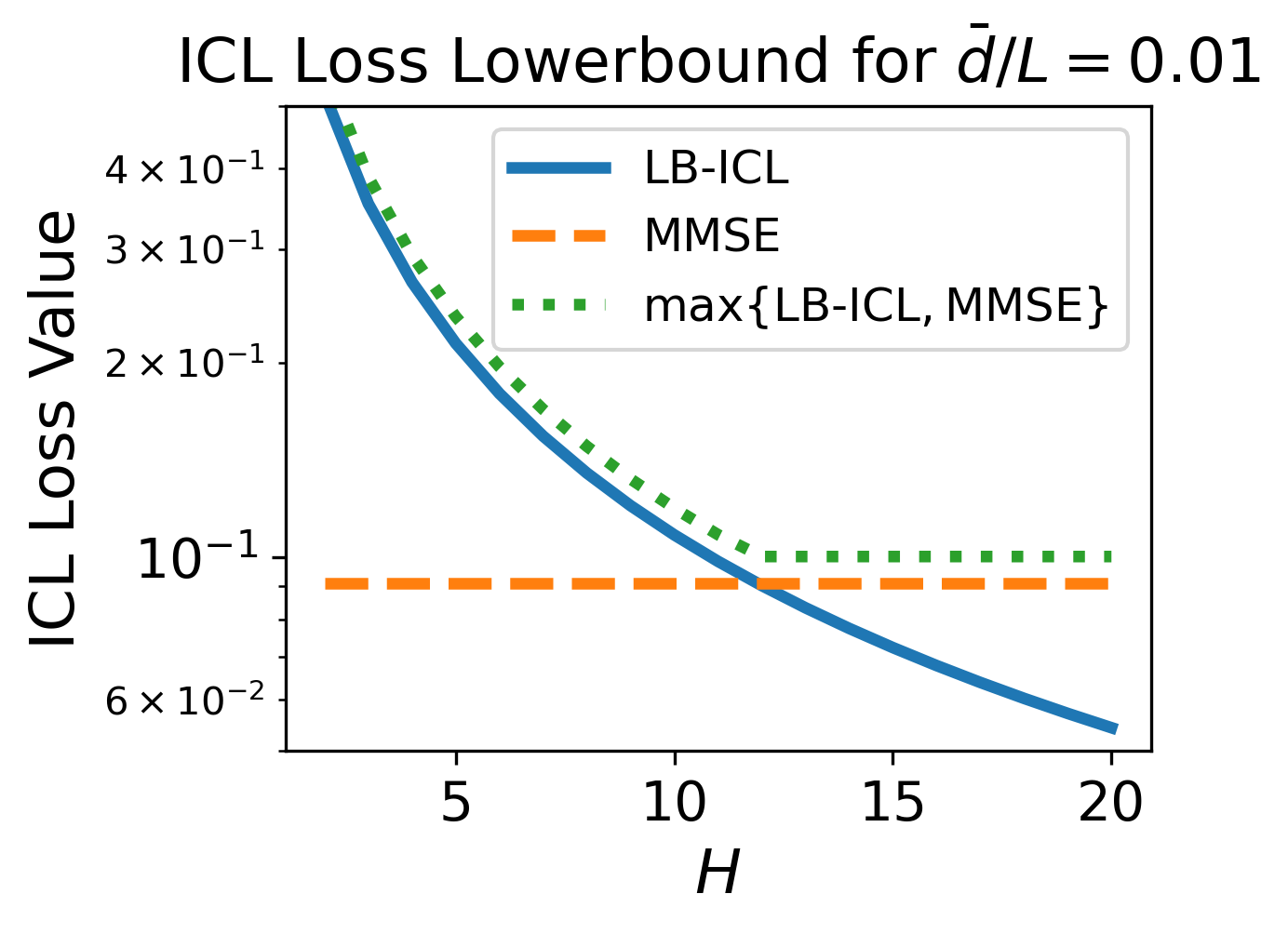}
        \caption{Comparison of {\MultiheadICLLB} and MMSE over (i) $\bard/L$ for different $H$ and (i) $H$ for $\bard/L = 0.01$.}
    \end{subfigure}
    \captionsetup{width=0.95\linewidth}
    \caption{Comparison between the lower bound {\MultiheadICLLB} and the Bayesian risk (MMSE) for different $H$ and $\bard/L$
    with $I=10$, $\lambda =\sigma^2=1$. Here, $\bard = d/I$.
    }
\end{figure}

\begin{figure}[t]
    \centering
    \includegraphics[width=0.45\textwidth]{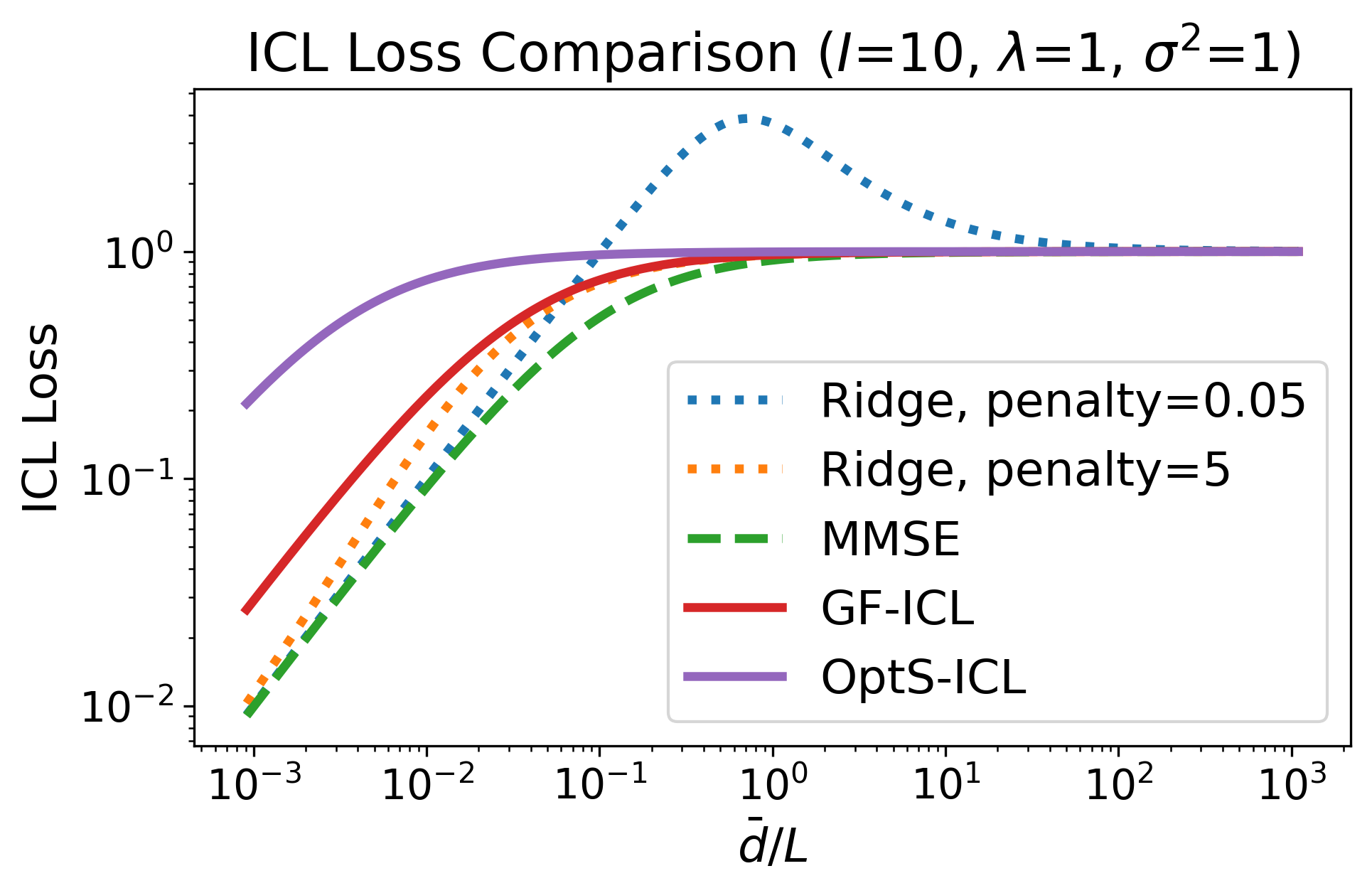}
    \captionsetup{width=0.95\linewidth}
    \caption{Comparison of the \ac{icl} loss between (i) ridge regression with regularization parameter $\mathrm{penalty} = 0.05$ and $\mathrm{penalty} = 5$; (ii) the Bayesian risk, i.e., MMSE; (iii) the \ac{icl} loss for the convergence point of the gradient flow \eqref{eq:icl-loss-convergence-point} with $H\ge I$; (iv) the \ac{icl} loss of the optimal single head attention \eqref{eq:water filling}. 
    Here, $\bard = d/I$.
    For large $\bard/L$, the {\ConvergenceICL} loss behaves similarly to ridge regression with large $\mathrm{penalty}=5$ and does not have the double descent phenomenon that is observed with insufficient regularization $\mathrm{penalty}=0.05$.
    This is due to the softmax-induced regularization as we will further illustrate in \S\ref{sec:implicit_regularization}.
    For small $\bard/L$, we see that \ac{msa} ({\ConvergenceICL}) has an $H\pmin I=10$ improvement over the single-head attention ({\SingleheadICLOpt}).
    We refer readers to \S\ref{sec:icl loss comparison} for more details.}
    \label{fig:loss_comparison}
\end{figure}

\subsection{Softmax-Induced Regularization and Proof Sketch for the Optimality Results}
\label[section]{sec:implicit_regularization}
Before presenting our proof sketch for the optimality results, let us first understand the loss decomposition of a single softmax transformer head.
\subsubsection{Loss Decomposition and Softmax-Induced Regularization}
We treat $W$ and $U$ as the combined weights of the single-head attention.
For ease of presentation, we only consider one task here while the multi-task case is deferred to the complete proof in \S\ref{sec:optimality_proof}. 
Hence, $d_y=1$ with $U_Y$ being just a scalar, $G = g_1 \in \RR^{d\times 1}$ and we have $\EE[G G^\top] = \lambda d^{-1} I_d$.
Suppose $W_Y=0$, then the ICL loss defined in \eqref{eq: loss} can be decomposed into
\begin{align*}
    \cL(U, W) = \underbrace{\EE 
    \big[\bigl\| G^\top q - U_Y G^\top X p\bigr\|_2^2\big]}_{\ds\text{Signal-induced~error}}
 + \underbrace{\EE\big[\bigl\| U_X X p \bigr\|_2^2\big]}_{\ds\text{Extra~error}} + \underbrace{\EE\big[\bigl\| U_Y \varepsilon p \bigr\|_2^2\big]}_{\ds\text{Noise-induced~variance}}.
\end{align*}
Note that the extra error term vanishes when setting $U_X=0$, and the remaining terms are not affected, thus justifying \Cref{cond:U_X and W_Y are zero}.
Expanding the signal-induced error, we have
\begin{align*}
    \EE\big[\bigl\| G^\top q - U_Y G^\top X p\bigr\|_2^2\big] = \lambda d^{-1} \cdot \EE\bigl[d - 2 U_Y \trace( X p q^\top 
    )  + U_Y^2 \trace( X p p^\top X^\top ) \bigr].
\end{align*}
Applying the Stein's lemma for the expectation term (\Cref{fact:p is a function of Wq's 2-norm}) yields
\begin{equation}
\begin{aligned}
    \EE[X p \given q] &=  W_X q (1 - \EE[\norm{p}_2^2\given q]), \\
    \EE[X p p^\top X^\top\given q] &= W_X q q^\top W_X^\top \EE[1 - \norm{p}_2^2 - 6 \norm{p}_3^3 + 6 \norm{p}_2^4 \given q] + I_{d} \EE[\norm{p}_2^2\given q].
\end{aligned}
\label{eq:stein-lemma-proofsketch}
\end{equation}
When $W_X$ has bounded operator norm (which will be justified later), it holds with high probability that $\norm{W_X q}_2^2 \le O(2\log L)$, and then we can ignore the higher order terms in the expectation (\Cref{lemma: higher-order-moment}).
Consequently, we have $\EE[X p \given q] \approx W_X q$ and
$
    \EE[X p p^\top X^\top\given q] \approx W_X q q^\top W_X^\top + {I_{d} \EE[\norm{p}_2^2\given q]}
$
up to $O(L^{-1})$ error. 
Applying these to simplify the signal-induced error term, we obtain
\begin{align}
    \cL(U, W) \approx \underbrace{\lambda d^{-1} \cdot \EE\bigl[
        \left\|q - U_Y W_X q\right\|_2^2\bigr]}_{\ds\text{Linearized bias}}  + \underbrace{(\lambda + \sigma^2) U_Y^2 \cdot \EE[\norm{p}_2^2]}_{\ds\text{Softmax-induced regularization}}. 
    \label{eq:loss-decomposition}
\end{align}
Here, the first term is the \emph{linearized bias}, where we get rid of the nonlinearity in the attention probability.
In particular, this term measures how close $U_Y W_X$ is to the identity.
The second term is a softmax-induced regularization induced by the \emph{concentration} of the attention probability.
\revise{With $\norm{p}_2^2$ becoming larger, the attention probability becomes more concentrated on a few input tokens, leading to a larger variance in the output.
In particular, when compared to standard linear regression's loss decomposition, a key difference is that the softmax-induced regularization term also depends on the signal strength $\lambda$ while the standard linear regression's regularization term only depends on the noise level $\sigma^2$.
Such a fact answers the question of why the \ac{msa}'s performance matches the MMSE only under the low SNR regime as shown \Cref{tab:bounds_summary}.
}
The key to our analysis is understanding the softmax-induced regularization term. 

\subsubsection{Exponential and Saturation Regime of Softmax Attention}
We characterize the behavior of the softmax-induced regularization in two regimes, which we call the \emph{exponential regime} and the \emph{saturation regime}, stratified by the scale of the attention budget $\norm{W_X}_{\fro}^2 = \EE[\norm{W_X q}_2^2]$.
In the sequel, we denote by $\EE[\norm{p}_2^2\given r] = \EE[\norm{p}_2^2\given \norm{W_X q}_2=r]$. 
Since $r$ is a random variable with $r^2$ concentrated around $\norm{W_X}_{\fro}^2$, we also characterize the two regimes by the scale of $r$.
\Cref{fig:regime} illustrates the behavior of the softmax-induced regularization term $\EE[\norm{p}_2^2\given r]$ as a function of $r$. 

\begin{figure}[t]
    \centering
    \includegraphics[width=0.45\textwidth]{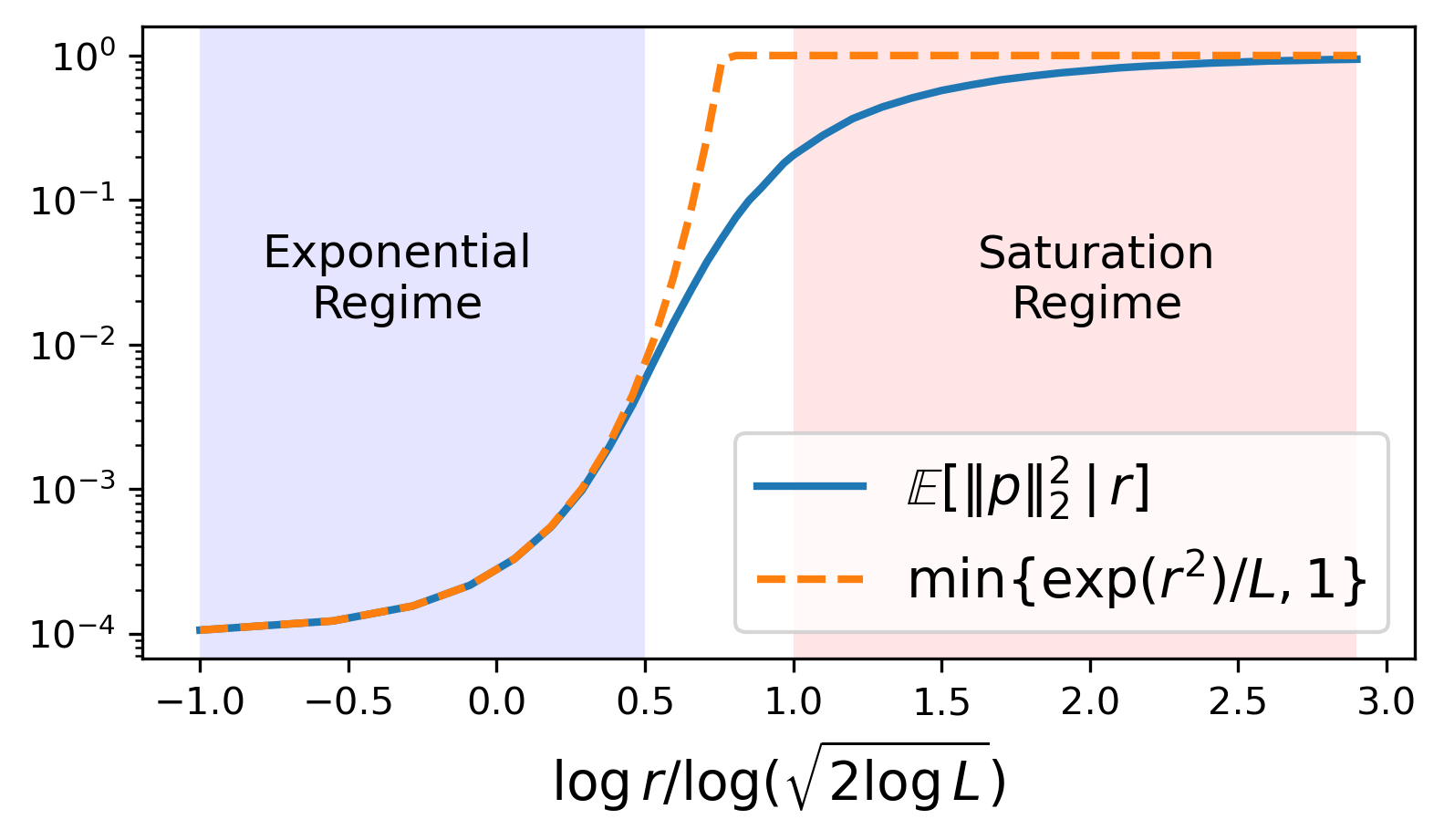}
    \includegraphics[width=0.45\textwidth]{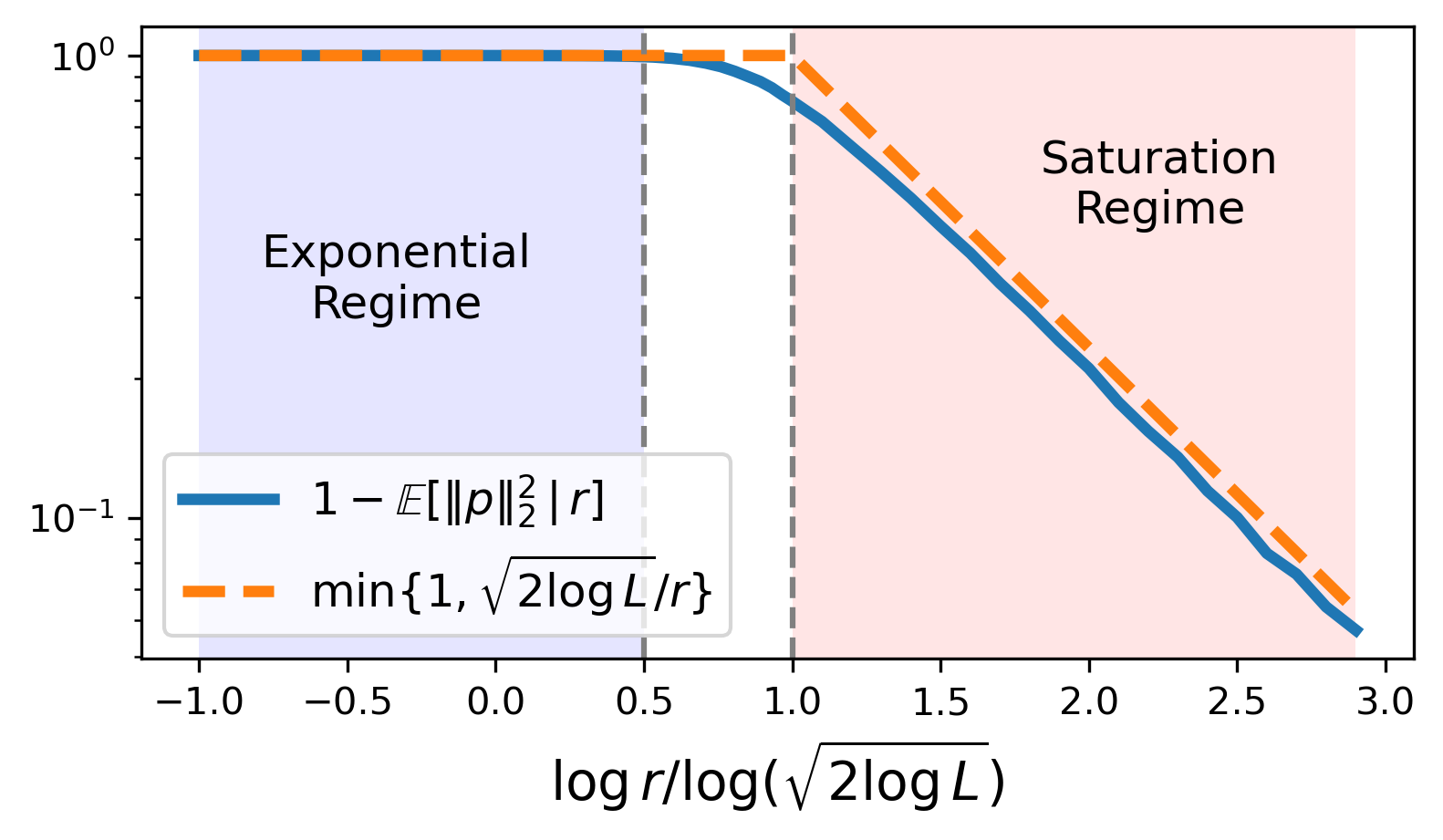}
    \captionsetup{width=0.95\textwidth}
    \caption{The behavior of $\EE[\norm{p}_2^2\given r]$ as a function of $r$ for $L=10^4$. The left figure shows the behavior of $\EE[\norm{p}_2^2\given r]$ and $\min\{\exp(r^2)/L, 1\}$ is a good approximation when $r = o(\sqrt{2 \log L})$ (\Cref{lemma: pseudo-dynamics}) in the \emph{exponential regime}. The right figure shows that $1 - \EE[\norm{p}_2^2\given r]$ for $r \gg \sqrt{2\log L}$ is upper bounded by $O(r^{-1+\epsilon})$ in the \emph{saturation regime}(\Cref{lem:p-moment-tail}).} 
    \label{fig:regime}
\end{figure}

\paragraph{Exponential Region ($r = o(\sqrt{2\log L})$, \Cref{lemma: pseudo-dynamics}).}
We name this regime the \emph{exponential regime} since in this regime, the softmax-induced regularization term in \eqref{eq:loss-decomposition} is approximately an exponential function of the attention budget.
This is also the regime where the optimal \ac{icl} rate is achieved (\Cref{thm:optimality of ssa}).
To derive the approximation form, we introduce an index $t\in [0, 1]$ and consider function $f(t, q) = \EE[\norm{\softmax(tX^\top W_X q)}_2^2 \given q]$. 
We then take the derivative $\frac{\partial f}{\partial t}$ to obtain
\begin{align*}
    \frac{\partial f}{\partial t} = \EE[\norm{p(t)}_2^2 - 4 \norm{p(t)}_3^3 + 3 \norm{p(t)}_2^4 \given q] \cdot \norm{W_X q}_2^2, 
\end{align*}
where $p(t) = \softmax(tX^\top W_X q)$.
This calculation follows from the Stein's lemma. 
By analyzing the tail of $\max_{l\in[L]} p_l(t)$, we can control the higher order moments of $p(t)$ and get \[\frac{\partial f}{\partial t} \approx \EE[\norm{p(t)}_2^2\given q] \cdot \norm{W_X q}_2^2 = f(t, q) \cdot \norm{W_X q}_2^2.\] 
Integrating this with respect to $t$ from $0$ to $1$ and by a careful analysis of the error along the integrating path, we conclude that $\EE[\norm{\softmax(tX^\top W_X q)}_2^2 \given q] = f(1, q) \approx \exp(\norm{W_X q}_2^2)/L$.
The analysis can also be extended to the multi-head case and the results are also useful for simplifying the dynamics in \S\ref{sec:appendix_dynamics_error}.

\paragraph{Saturation Regime ($r \gg \sqrt{2\log L}$, \Cref{lem:p-moment-tail}).}
The above approximation breaks down for large $r$. 
However, in this case we can utilize the fact that the attention probability is concentrated. 
Note that $\EE[\norm{p}_2^2\given r] = \EE[\norm{\softmax(r z)}_2^2]$, where $z\sim \cN(0, I_d)$.
It can be shown that
\begin{align}
    1 - \EE[\norm{p}_2^2\given r] \le 2 \EE[p_2 + p_3 + \cdots + p_L\given r] \le 2 \EE\bigg[\sum_{l=2}^L \exp(-r (z_1 - z_l))\bigg], 
    \label{eq:proof-sketch-1}
\end{align}
where $p_1, p_2, \ldots, p_L$ and $z_1, z_2, \ldots, z_L$ are the (descending) order statistics of $p$ and $z$ respectively.
The approach for dealing with \eqref{eq:proof-sketch-1} is to properly choose a threshold $\theta = \log(L r^\alpha)/r$ for some $\alpha > 1/2$ and then split the sum into two parts based on whether $z_1 - z_l > \theta$ or $z_1 - z_l \le \theta$. 
For the first part, we utilize the gap $\theta$ to upper bound the sum by $L \exp(-r \theta) = r^{-\alpha}$.
For the second part, we directly upper bound each term by $\exp(-r (z_1 - z_2))$ while controlling the total number of terms $N$.
This is essentially equivalent to bounding the tail of a binomial distribution with $L$ trials and success probability roughly $\PP_{v\sim \cN(0, 1)}(\sqrt{2\log L} - v \le \theta)$.
Finally, combining the moment generating function for $z_1 - z_2$ and also the tail bound for $N$, we can choose $\alpha$ to be some $1-\epsilon$ constant and obtain the desired result.

\paragraph{The Saturation Regime is Suboptimal for \ac{icl} with noise.}
\citet{huang2023context} studied the setting where the context features come from an orthogonal basis and the data is noiseless.
Their trained model that assigns $\Omega(1)$ probability to some tokens and falls into the saturation regime.
However, in the presence of noise, our analysis clearly shows the suboptimality of the saturation regime for \ac{icl} (\Cref{lemma:SS-Attn lb-2} and \Cref{lemma:SS-Attn lb}).
The proof is a combination of understanding the nonlinear \ac{icl} loss for different regions of the attention budget.
Notably, the proof of \Cref{lemma:SS-Attn lb-2} for the very extreme case of the attention budget is given by relating the lower bound of the \ac{icl} loss to a group of attention weights operating with \emph{normalized} query on the sphere.

\subsubsection{Proof Sketch for the Optimality Results}
Now we are ready to sketch the proof for the optimality results.

\paragraph{The Single-Head Case.}
Recall from \eqref{eq:loss-decomposition} and we consider the attention to operate in the exponential regime:
\begin{align}
    \cL(U, W) \approx \underbrace{\lambda d^{-1} \cdot \EE\bigl[
        \left\|I - U_Y W_X\right\|_{\fro}^2\bigr]}_{\ds\text{Linearized bias}}  + \underbrace{(\lambda + \sigma^2) \cdot U_Y^2 \cdot \frac{\exp(\norm{W_X}_{\fro}^2)}{L}}_{\ds\text{Softmax-induced regularization}}.
    \label{eq:loss-decomposition-2}
\end{align}
The softmax-induced regularization puts a quadratic penalty on the output projection weights $U_Y$ and an exponential penalty on the attention budget $\norm{W_X}_{\fro}^2$.
In \eqref{eq:loss-decomposition-2}, it is not hard to see that all attention budget should be concentrated only on the \emph{diagonal} entries of $W_X$ to minimize the loss.
For multi-task case, the same argument still holds but needs a more refined analysis, especially for dealing with nonlinearity that is hidden in \eqref{eq:loss-decomposition-2} for ease of presentation.
By treating $W_X = \sqrt{B/d} \cdot I_d$ and first optimizing over $U_Y$, we arrive at the an optimization problem over $B$:
\begin{align*}
    \min_{0\le B\le O(\log L)} \frac{\lambda}{d^{-1} L \phi^{-1} \exp(-B) B + 1}, 
\end{align*}
where $\phi = 1 + \snr^{-1} = 1 + \sigma^2 /\lambda$. 
The optimal is achieved at $B^\star = 1$ which gives the optimal loss value $\lambda/(e^{-1} \phi^{-1} d^{-1} L  + 1)$ and also the optimal solution. 
The results for the multi-task setting follow from similar ideas.

\paragraph{The Multi-Head Case.}
Here we need a more refined loss decomposition that isolates task-wise performance as follows:
\begin{gather*}
    \cL(\mu, \baromega) = \frac{\lambda}{I}\sum_{i=1}^I \cL_i(\mu, \baromega) + O(L^{-\epsilon/2}), \where \cL_i(\mu, \baromega) 
    \defeq 1 -  2 \mu_i^\top \baromega_i   +  \mu_i^\top (\baromega_i \baromega_i^\top + B) \mu_i, \\
    \baromega_i = (\baromega_i^{(1)}, \ldots, \baromega_i^{(H)})^\top, \quad \mu_i = (\mu_i^{(1)}, \ldots, \mu_i^{(H)})^\top, \quad B\in\RR^{H\times H} \:\text{with}\: (B_i)_{h h'} \defeq \frac{\phi_i\exp(\bard \langle \baromega^\h, \baromega^\hprime \rangle)}{L}.
\end{gather*}
Note that by the equiangular property, we can also decompose $B$ into the diagonal part and the off-diagonal part
\[
    B = (\tilde a - \tilde b) I + \tilde b E, 
\]
where $E$ is the matrix with all entries being $1$ and $\tilde a, \tilde b$ are the diagonal and off-diagonal entries of $B$ respectively.
Motivated by the structure of $B$, we can also project $\baromega_i$ into the parallel and orthogonal components with respect to $\vone_H$. 
By doing quadratic minimization over $\mu_i$, and a nonlinear minimization over $\tilde a, \tilde b$, we arrive at the lower bound for the equiangular weights.

\section{Extensions}
\label{sec:extension}
In this section, we discuss some extensions of the main results.
Let us first gain more understanding of the relationship between softmax attention and linear attention.

\subsection{Similarity between Softmax Attention and Linear Attention}
Here we consider the optimal single-head attention head and consider the regime $L \gg d$. 
Consider a query $q$ such that $\norm{W_X q}_2^2 = O(1)$ and we ignore the noise term.
We invoke \eqref{eq:stein-lemma-proofsketch} and also the approximation in the exponential regime (\Cref{lemma: pseudo-dynamics}, which applies regardless of the relative scale of $d$ and $L$ when $\norm{W_X q}_2^2 = O(1)$) to get 
\begin{align*}
    X p\approx \EE[X p \given q] =  W_X q (1 - \EE[\norm{p}_2^2\given q]) \approx W_X q \approx \frac{1}{L} X X^\top W_X q.
\end{align*}
Here the first approximation is due to the fact that the second moment is small by \eqref{eq:stein-lemma-proofsketch}, i.e.,
\[
    \trace(\EE[(X p - \EE[X p\given q]) (X p - \EE[X p\given q])^\top \given q]) =  d\cdot\EE[\norm{p}_2^2\given q] + O(L^{-1})  = O(d/L).
\]
Therefore, it follows that
\begin{align}
    \hat y_q = U_Y G^\top X p  \approx \frac 1 L U_Y Y X^\top W_X q = \frac 1 L U Z Z^\top W z_q, 
    \label{eq:linear attention approx}
\end{align}
which is equivalent to the linear attention model \citep{zhang2022analysis} but with a different scaling of the attention weights.
However, such approximation requires a large $L$.
For the setting of our main results where $d/L = \Theta(1)$, the approximation is not valid and the analysis has to go beyond the linear case.

\subsection{Length Generalization}
We also consider applying the trained model in \Cref{thm:convergence-multi-head-symmetric} with training data of length $L$ to a same \ac{icl} tasks with a new sequence of length $\tilde L$, and study the \ac{icl} loss for the new tasks. 
Specifically, the softmax attention model is trained in the same way as in \eqref{eq: gradient flow dynamics}, but when computing the ICL loss, we use a sequence of length $\tilde L$. 
That is, in \eqref{eq: loss}, $\hat y_q$ is computed according to \eqref{eq: msa} but with $Z$ consisting of $\tilde L$  covariate-response pairs from the multi-task linear model introduced in \Cref{assump:data}. We have the following upper bound on the ICL loss.

\begin{proposition}[Length Generalization for \ac{msa}]
\label{cor:length generalization}
For the convergence point of the dynamics described in \Cref{thm:convergence-multi-head-symmetric}, 
consider applying the trained model to a new \ac{icl} learning task with the same data structure under \Cref{assump:data}
but a new sequence length $\tilde L$. 
Then the corresponding loss value is upper bounded by
\begin{align}
        \sum_{i=1}^I \frac{\lambda_i d_i}{d}\cdot \left(\frac{e d_i \phi_i L^{-1}}{1 + e d_i \phi_i L^{-1}} +  O(L^{-(1-\epsilon)/2} + \omega_0^2 d + \delta) + O (\phi_i d_i |L^{-1} - \tilde L^{-1}|)^2\right). 
        \label{eq:length-generalization-1}
\end{align}
\end{proposition}
See \S\ref{sec:length-generalization} for a proof.
\Cref{cor:length generalization} shows that the \ac{icl} loss scales with $O(d/\tilde L) + O((d/(L\pmin \tilde L)^2))$. 
Here in \eqref{eq:length-generalization-1}, the first term is the intrinsic loss and the third term characterizes the error due to the mismatch in sequence length.
Notably, since $\baromegaistar$ is independent of the sequence length, we have the mismatch term coming only from the suboptimality of $\muistar$, which is actually a higher order term (for more details, see the quadratic error term in \Cref{lem:loss perturbation}).
Thus, we conclude that the model trained by gradient flow also generalizes well to new sequence length.
We also remark that such a nice property is due to the normalization of the softmax function, and the linear attention model does not have such a property as indicated by \eqref{eq:linear attention approx} where the weights should directly scale with $1/L$.

\subsection{Linear-to-Nonlinear Transfer}
We consider first training the model on linear tasks and then applying the trained model to a nonlinear task. 
We only consider the single-head case for simplicity.
Consider replacing $Y = G^\top X + \varepsilon$ with $Y = f(X) + \varepsilon$ where $f$ is a nonlinear function in the downstream \ac{icl} task (the training data is still linear).
Suppose $f$ has degree at most $D$ in the sense that $f$ is a linear combination of $d$-dimensional multivariate Hermite polynomials up to degree $D$, $\{\herm_\alpha \given \alpha\in\NN^d, \: |\alpha| \defeq \sum_{i=1}^d \alpha_i \le D\}$.
\begin{align*}
    f(x) = \sum_{|\alpha|\le D} \hat f_\alpha \herm_\alpha(x), \where \hat f_\alpha\in\RR.
\end{align*}
\begin{lemma}[Generalization of Single-Head Softmax Attention to Nonlinear Tasks]
\label{lem:nonlinear generalization}
    \revise{Let $L$ be sufficiently large such that \eqref{eq:epsilon} holds for constants $\epsilon$ and $c$.}
    Consider  \Cref{assump:data} on the data but with a nonlinear task $Y = f(X) + \varepsilon$ where $f$ has degree at most $D= O(1)$.
    Consider only a single head whose weights are given by the optimal \ssa~trained for this linear task, i.e., 
    \begin{align}
        W_X = d^{-1/2} \cdot I_d, \quad U_Y = \mu = \frac{\sqrt{d}}{1 + e d \phi L^{-1}}. 
        \label{eq:optimal single-head linear task}
    \end{align}
    Then with probability at least $1 - \exp(- (c^{-2} \log L - 1)^2 \pmin d/2)$, the average output of the model for a fixed query $q$ satisfies
    \begin{align}
        \bigg|\EE[\hat y_q \given q] - \mu \cdot \sum_{\alpha:|\alpha|\le D} \hat f_\alpha d^{-|\alpha|/2} \cdot q^\alpha \bigg| \le O(L^{-2(1-\epsilon)}), \label{eq:nonlinear generalization}
    \end{align}
    where we define $q^\alpha = \prod_{i=1}^d q_i^{\alpha_i}$.
\end{lemma}
See \S\ref{sec:nonlinear-generalization} for a proof.
\Cref{lem:nonlinear generalization} indicates that the optimal single-head softmax attention trained for a linear task does not generalize well to a nonlinear task.
On one hand, the coefficient is not well preserved due to the additional factor $d^{-|\alpha|/2}$, especially for high degree terms.
On the other hand, 
\[
    q^\alpha =  \prod_{i=1}^d \bigg(\sum_{k_i=0}^{\floor{\alpha_i /2}} \frac{\herm_{\alpha_i - 2k_i}(q_i)}{k_i ! 2^{k_i} (\alpha_i - 2 k_i)!}\bigg) = \herm_\alpha(q) + \sum_{k\in\NN^d, |k|\ge 1 \atop k_i\le \floor{\alpha_i/2},  \forall i\in[d]} \frac{\herm_{\alpha - 2k}(q)}{k! 2^k (\alpha - 2k)!},
\]
where $k! = \prod_{i=1}^d k_i!$ and $2^k = \prod_{i=1}^d 2^{k_i}$.
Such a fact indicates a \say{leakage} of the high degree terms to the low degree terms.
Notably, when $\alpha = 0$, we have output $\mu \hat f_0$, where we see that the current attention scales up the constant term up by $\mu = \Omega(\sqrt{d})$.
However, this can be remedied by also including a bias term to the output layer, i.e., $\hat y_q = U_Y Y p + b$ and also before inputting the sequence to the attention layer, i.e., $Y = Y - b'$.

At a high level, the failure of capturing higher order nonlinear effect is because the model \eqref{eq:optimal single-head linear task} trained on linear tasks only achieves the optimal variance and bias tradeoff for linear tasks, as $W_X = d^{-1/2} I_d$ \say{downscales} the attention score to have unit variance and $U_Y\approx \sqrt{d}$ \say{upscales} the aggregated output back to its original scale.
The \say{downscaling} seems to be too much for nonlinear tasks with higher degree terms.

\section{Conclusion and Future Work}\label[section]{sec:conclusion}
In this paper, we study the training dynamics of a single-layer multi-head softmax attention model for the in-context learning problem of multi-task linear regression.
By deriving the spectral dynamics of the attention weights induced by gradient flow, we provide a complete characterization of the training dynamics in terms of three phases: warm-up, emergence, and convergence.
\begin{myenumi}
    \item The warm-up phase shows that the optimal attention head for each task, which by initialization has small advantage, gradually wins the competition among heads under a strong cross-head interference. 
    \item The emergence phase demonstrates the sudden emergence of the \ac{icl} capability with a quick drop in the \ac{icl} loss.
    \item The convergence phase highlights a linear convergence rate to the optimal solution: each task is assigned a unique attention head which operates as the optimal single-head attention for that task.
\end{myenumi}
We also characterize the optimality of the convergence point of gradient flow.
Throughout the analysis, we identify two working regions of the softmax attention model and prove that for noisy tasks, the optimal weights lie in the \emph{exponential region} which demonstrates a unique tradeoff between bias and variance for the softmax attention model.

We also show that the optimal solution for a linear task trained with a fixed context sequence length generalizes well to different sequence length but poorly to nonlinear tasks. 
For future work, we believe it is important to further incorporate into the analysis other design factors in transformers, such as multi-layers, layer normalization, and positional encodings, etc.
It is also an interesting future direction to understand how softmax attention deals with nonlinear tasks, especially for the multi-head attention model.



\bibliographystyle{ims}
\bibliography{reference}

\newpage
\tableofcontents 
\appendix
\section{Additional Notations and Organization of the Appendices}



\paragraph{Addition Notations for Proofs.}
For a polynomial $f(x_1, \dots, x_n) = \prod_{i=1}^n x_i^{a_i}$ of $n$ variables where $a_1, \dots, a_n$ are positive integers,
we say that $f$ is \emph{even} if all $a_i$'s are even.
For any dimension $d$, we denote by $\vone_d$ the $d$-dimensional all-one vector.
We denote by $\Diag(M)$ the column vector created from the diagonal entries of $M$ and $\diag(v)$ the diagonal matrix created from the vector $v$.
We denote by $u\otimes v$ the outer product of vectors $u$ and $v$, and denote by $A \kron B$ the Kronecker product of matrices $A$ and $B$.
Following the convention for tensor inner product, we denote by $A \cdot B$ the inner product of tensors $A$ and $B$ such that $A\cdot B = \sum_k A_{:,k}B_{k,:}$, $A \cddot B$ the double inner product of tensors $A$ and $B$ such that $A\cddot B = \sum_{kl} A_{:, k,l}B_{l,k, :}$.
Here, the \say{$:$} in $A_{:, k}$ hides the indices of $A$ except for the last dimension(s).

\paragraph{Roadmap of the Appendices.} 
The remainder of the appendices is organized as follows:
\begin{enumerate}
\item \Cref{sec: p-moment} contains results on the (conditional) moments of the attention probability vector.
\item \Cref{sec:simplify_approximate dynamics} contains the simplification and approximation of the gradient flow dynamics under the \DC. In particular
\begin{enumerate}
    \item \Cref{sec:simplify_AB} provides simplifications of the gradient flow dynamics induced by the \DC.
    \item \Cref{sec:approximation_dynamics} contains the approximation of spectral dynamics induced by gradient flow.
    \item \Cref{sec:proof of decomposability preserved} provides the proof of \Cref{lem:decomposability preserved}, the preservation of the \DC~along gradient flow.
\end{enumerate}
\item \Cref{sec:appendix_convergence_analysis} contains the analyses for the dynamics.
\item \Cref{sec:optimality_proof} presents the optimality analysis of the convergence point of gradient flow.
\item \Cref{sec:generalization-proof} contains the proofs for the generalization results.
\item \Cref{sec:proof_aux_lemma} collects the proofs of auxiliary results used in the analysis.
\end{enumerate}

\newpage
\section{Characterizing Moments of the Attention Probability}\label[appendix]{sec: p-moment}

In this section, we characterize the correlation of the attention probability $\EE[(p^h)^\top p^{h'}\given q]$ by providing an upper and lower bound for this quantity.
The idea is to construct a gradient flow dynamics starting from zero for the attention weight $W^h$ and $W^{h'}$ and bounding the terms in the gradient flow dynamics.
In the sequel, we denote by $p$ and $\tilde p$ two attention probability vectors, and by $W$ and $\tilde W$ the corresponding weight matrices, and they are related by 
\begin{gather*}
    p = \softmax(X^\top W q), \quad \tilde p = \softmax(X^\top \tilde W q), \quad \text{where}\\
    X=\begin{bmatrix} x_1, \dots, x_L \end{bmatrix}\in \RR^{d \times L}, \quad x_l \iidfrom \cN(0, I_{d}), \forall l\in [L].
\end{gather*}
Here, $q \in \RR^d$ is the query vector and we have $q \sim \cN(0, I_d)$.
In addition, we define the kernel matrices $P\in \RR^{L\times L}$ and $\tilde P \in \RR^{L\times L}$ as
\begin{align*}
    P = \diag(p) - p p^\top, \quad \tilde P = \diag(\tilde p) - \tilde p \tilde p^\top.
\end{align*}

\subsection{$\EE[\norm{p}_2^2\given \norm{W q}_2=r]$ Is Monotone in $r$}
Let $r = \norm{Wq}_2$. 
Note that $\EE[\norm{p}_2^2\given q]$ is just a function of $r$ as we have shown in the previous section.\todo{explain this}
Thus, we also abbreviate $\EE[\norm{p}_2^2\given q]$ into $\EE[\norm{p}_2^2\given \norm{Wq}_2 =r]$. 
Note the there is still a gap for $r\in (c^{-1} \sqrt{2\log L}, O(\sqrt{2\log L}^3)$. 
To fully understand the behavior of the transformer so as to give the optimality result, we need to show that $\EE[\norm{p}_2^2\given \norm{Wq}_2 =r]$ is a growing function of $r$, which is presented in the following lemma. 
\begin{lemma}\label[lemma]{lem: p-moment-monotone}
Suppose $q,x_1,\ldots,x_L \iidfrom \cN(0, I_d)$, and denote $X = [x_1, \ldots, x_L]$.
For any fixed $W\in\RR^{d\times d}$, let $p=\softmax(X^\top W q)$, then $\EE[\norm{p}_2^2\given \norm{W q}_2=r]$ is monotonically increasing in $r$.
\end{lemma}
\begin{proof}[Proof of \Cref{lem: p-moment-monotone}]
Denote $w_l := x_l^\top W q / \norm{W q}_2$ for each $l\in[L]$.
Differentiating $\norm{p}_2^2$ with respect to $r:=\|Wq\|_2$ gives
\begin{align*}
    \frac{\rd}{\rd r} \norm{p}_2^2 
    &= \frac{\rd}{\rd r} \sum_{i=1}^L \left(\frac{\exp(r w_i)}{\sum_{j=1}^L \exp(r w_j)} \right)^2 \nend 
    & = \sum_{i=1}^L \frac{2 w_i \exp(2 r w_i) \sum_{j=1}^L \exp(r w_j) - \exp(2r w_i) 2\sum_{j=1}^L\exp(r w_j) x_j}{\left(\sum_{j=1}^L \exp(r w_j)\right)^3} \nend
    & =  2\cdot \frac{\sum_{i=1}^L w_i \exp(2 r w_i) \sum_{j=1}^L \exp(r w_j) - \sum_{i=1}^L \exp(2r w_i) \sum_{j=1}^L\exp(r w_j) w_j}{(\sum_{j=1}^L \exp(r w_j))^3}. 
\end{align*}
Further write $v_i := \exp(rw_i)$ for simplicity, then the numerator can be rewritten as
\begin{align*}
    \sum_{i,j=1}^L w_i v_i^2 v_j - \sum_{i,j=1}^L w_j v_i^2 v_j 
    &= \sum_{i,j=1}^L (w_i - w_j) v_i^2 v_j = \frac{1}{2} \sum_{i,j=1}^L (w_i - w_j) v_i^2 v_j + \frac{1}{2} \sum_{i,j=1}^L (w_j - w_i) v_j^2 v_i \nend
    &= \frac{1}{2} \sum_{i,j=1}^L (w_i - w_j) (v_i - v_j) v_i v_j \ge 0 
\end{align*}
where the inequality follows from the nonnegativity of $v_i$ and the observation that $v_i$ is a monotone function of $w_i$.
Therefore, we conclude that $\norm{p}_2^2$ is a monotonically increasing function of $r$, and hence is $\EE[\norm{p}_2^2\given \norm{Wq}_2 =r]$.
\end{proof}

\subsection{Approximating $\EE[p^\top \tilde p\given q]$ for small $\norm{Wq}_2$}
\label{sec: p-moment-small-norm}
\paragraph{Construct Pseudo Dynamics.}
We consider a constant-speed gradient flow on $W(t)$ and $\tilde W(t)$ within time interval $t\in [0, 1]$ with initial condition $W(0) = \tilde W(0) = 0$ and terminal condition $W(1) = W$ and $\tilde W(1) = \tilde W$.
The gradient flow for $\EE[p^\top \tilde p \given q]$ is given by
\begin{align*}
    \partial_t \EE[p^\top \tilde p\given q] 
    &= \sum_{l}\EE\left[\partial_t p_l \tilde p_l + p_l \partial_t \tilde p_l \biggiven q\right] \\
    &= \sum_{l m}\EE\left[
        P_{l m} \tilde p_l \cdot x_{m}^\top \partial_t W q  + p_l \tilde P_{l m} \cdot x_{m}^\top \partial_t \tilde W q \biggiven q
        \right] \\
    &= \sum_{l m} \EE\left[
        \nabla_{x_{m}}^\top (P_{lm} \tilde p_l) \partial_t W q +  \nabla_{x_{m}}^\top (\tilde P_{l m} p_l) \partial_t \tilde W q \biggiven q
        \right].
    \mytag{Stein's Lemma}
\end{align*}
Here, we use the Stein's Lemma to derive the last equality. Next, we expand the gradient and obtain
\begin{align}
    \partial_t \EE[p^\top \tilde p\given q] 
    &= \sum_{l m} \EE\left[
        P_{lm} (1 - 2 p_{m}) \tilde p_l \cdot q^\top W^\top \partial_t W q + P_{l m} \tilde P_{l m} q^\top \tilde W^\top \partial_t W q \biggiven q
    \right] \notag\\
    &\autoquad{2} + \sum_{l m} \EE\left[
        \tilde P_{l m} P_{l m} q^\top W^\top \partial_t \tilde W q + \tilde P_{l m} (1 - 2 \tilde p_{m}) p_l \cdot q^\top \tilde W^\top \partial_t \tilde W q \biggiven q
    \right] \notag\\
    &=  \EE\left[
        -2 \left(\tilde p^\top p^{\odot 2} - \norm{p}_2^2 p^\top \tilde p \right) \cdot q^\top W^\top \partial_t W q - 2 \left(p^\top \tilde p^{\odot 2} - \norm{\tilde p}_2^2 p^\top \tilde p \right) \cdot q^\top \tilde W^\top \partial_t \tilde W q \biggiven q
    \right] \notag\\
    &\autoquad{2} +  \EE\left[\left( p^\top \tilde p - p^\top \tilde p^{\odot 2} - \tilde p^\top p^{\odot 2} + (p^\top \tilde p)^2\right) q^\top (\tilde W^\top \partial_t W + \partial_t \tilde W^\top W) q \biggiven q \right] \notag\\
    & =  \EE\left[
        -  \left(\tilde p^\top p^{\odot 2} - \norm{p}_2^2 p^\top \tilde p \right) \cdot q^\top  \partial_t (W^\top W) q - \left(p^\top \tilde p^{\odot 2} - \norm{\tilde p}_2^2 p^\top \tilde p \right) \cdot q^\top  \partial_t (\tilde W^\top \tilde W) q \biggiven q
    \right] \notag\\
    &\autoquad{2} +  \EE\left[\left( {p^\top \tilde p} - p^\top \tilde p^{\odot 2} - \tilde p^\top p^{\odot 2} + (p^\top \tilde p)^2\right) q^\top \partial_t (\tilde W^\top W) q \biggiven q \right] \label{eq: p-moment-1}
\end{align}
Our next goal is to show that $p^\top \tilde p$ is the dominating term.
A naive lower bound for $p^\top \tilde p$ is given by $p^\top \tilde p \geq L^{-1}$.
So the claim holds if we could show that the remaining terms are of higher order $o(L^{-1})$.

\paragraph{Eliminate the Higher Order Terms.}
In the following, we consider bounding the higher order terms in \eqref{eq: p-moment-1}. We employ notations $p_t, \tilde p_t, W_t, {\tilde W}_t$ to denote the attention probability vectors and weight matrices at time $t$
We have $p_0 = \tilde p_0 = 0, W_0 = {\tilde W}_0 = 0$ at initialization and $p_1 = p, \tilde p_1 = \tilde p, W_1 = W, {\tilde W}_1 = \tilde W$ at terminal time.
We work in the region $\tau^2 \defeq \max\{\norm{W q }_2^2 , \norm{{\tilde W} q }_2^2 \}  \le c^{-2}\cdot 2\log L $ for some absolute constant $c$.
The result for representing the pseudo dynamics with only the dominating term is given by the following lemma. 
\begin{lemma}
    \label[lemma]{lemma: pseudo-dynamics}
    Fix constant $\epsilon \in (0, 1)$ and let $c>0$ be the solution to the fixed-point equation 
    \[\frac{1}{c} + \frac{3}{1 + \sqrt{1+c^2/2}} = \epsilon. \]
    If $\tau^2 \defeq \max\{\norm{W q }_2^2 , \norm{{\tilde W} q }_2^2 \} \le c^{-2}\cdot 2\log L$ holds, we have for $\EE[p^\top \tilde p\given q]$ that
    \begin{align*}
        \xi_L \defeq \left|\mathbb{E}[p^{\top} \tilde p \mid q]-\frac{\exp \big(q^{\top} \tilde W^{\top} W q \big) }{L}\right| = O(L^{-2(1-\epsilon)}). 
    \end{align*}
    Also, for the higher order terms, we also have
\begin{align*}
    &\max \left\{
    \EE[p_t^\top \tilde p_t^{\odot 2}\given q], \quad \EE[\tilde p_t^\top p_t^{\odot 2}\given q], \quad \EE[\norm{p_t}_2^2 p_t^\top \tilde p_t\given q], \quad \EE[\norm{\tilde p_t}_2^2 p_t^\top \tilde p_t\given q], \quad \EE[(p_t^\top \tilde p_t)^2\given q] 
\right\}
\nend 
    &\quad \le O(L^{-2(1-\epsilon)}).
\end{align*}
\end{lemma}

\begin{proof}[Proof of \Cref{lemma: pseudo-dynamics}]
The first step is to put an upper bound for the maximum of the softmax probability $p_l$. 
It suffices to look at the attention scores $s_l = x_l^\top W q$, which are $\iid$ distributed as $\cN(0, \norm{W q}_2^2)$, and the variance is bounded by $\norm{W q}_2^2 \le \tau^2$.
Using the Gaussian tail bound, we have
$$
\mathbb{P}\left(\max_{l\in [L]} s_{l}>\tau \sigma \cdot \sqrt{2 \log L}\right) \le L^{-\sigma^{2}+1}, \quad \forall \sigma>1. 
$$
Also, we are able to bound the number of negative attention scores by invoking the Hoeffding's inequality that
$$
\mathbb{P}\left(\sum_{m} \ind\left(s_{m}\ge 0\right)<\frac{L}{4}\right) \le \exp \left(-\frac{2(L / 4)^{2}}{L}\right)=\exp \left(-\frac{L}{8}\right).
$$
We define the event $\cE_1(W)$ for the softmax probability $p(t)$ over time $t\in [0, 1]$ as
\begin{align}
\cE_1(W) = \left\{\max_{l\in [L], \atop t\in[0, 1]} p_{l}(t) < \zeta(\tau, \sigma; L)\right\}, \quad\text{where}\quad 
\zeta(\tau, \sigma ; L)=\frac{\exp (\tau \sigma \sqrt{2 \log L})}{L / 4} \le 4 L^{\frac{2 \sigma}{c}-1}.    \label{def: cE_1}
\end{align}
The event $\cE(\tilde W)$ is defined similarly.
One can bound the probability of $\cE_1^c$ as
$$
\begin{aligned}
\mathbb{P}\left( \cE_{1}^c\right) & \le \mathbb{P}\left(\max_{l\in[L],\atop t\in[0, 1]} p_{l}(t)>\frac{\exp (\tau \sigma \sqrt{2 \log L})}{L / 4} \bigggiven \sum_{m} \ind\left(s_{m} \ge 0\right) \ge \frac{L}{4}\right) +\mathbb{P}\left( \sum_{m} \ind\left(s_{m} \ge 0\right) < \frac{L}{4} \right) \\
& \le \mathbb{P}\left(\max _{l\in[L]} s_{l}>\tau \sigma \sqrt{2 \log L}\right)+\exp \left(-\frac{L}{8}\right) \\
& \le L^{-\sigma^{2}+1}+\exp \left(-\frac{L}{8}\right), 
\end{aligned}
$$
where the second inequality holds since $\sum_{m}\exp(s_{m}) \ge \sum_{m} \ind(s_{m}>0) \ge L/4$ for the normalization of $p$.
As a result, we consider two disjoint events $\cE_1(W)\cap \cE_1(\tilde W)$, $\cE_1^c(W)\cup \cE_1^c(\tilde W)$ and obtain that
\begin{align}
& \max \left\{
    \EE[p_t^\top \tilde p_t^{\odot 2}\given q], \quad \EE[\tilde p_t^\top p_t^{\odot 2}\given q], \quad \EE[\norm{p_t}_2^2 p_t^\top \tilde p_t\given q], \quad \EE[\norm{\tilde p_t}_2^2 p_t^\top \tilde p_t\given q], \quad \EE[(p_t^\top \tilde p_t)^2\given q]
\right\} \notag\\
& \quad \le L \cdot \zeta(\tau, \sigma; L)^3 +  \PP(\cE_1^c(W)\cup \cE_1^c(\tilde W)) \notag\\
& \quad \le 64 L^{\frac{6 \sigma}{c}-2}+2  L^{-\sigma^{2}+1} + 2 \exp(-L/8). \label{eq: p-moment-2}
\end{align}
Here, the first inequality holds by noting that $p_l, \tilde p_l \le \zeta(\tau, \sigma; L)$ element-wise on $\cE_1(W)\cap \cE_1(\tilde W)$ and that each term in the maximum is bounded by $1$.
Now, we select out the dominating term $p^\top \tilde p$ in \eqref{eq: p-moment-1} and with a little abuse of notation, 
denote the remaining terms as $\error(t;q)$ given by
$$\error(t;q)=\partial_{t} \mathbb{E}[p_t^\top \tilde p_t\given q] - \EE[p_t^\top \tilde p_t\given q] q^\top \partial_t({\tilde W}_t^\top W_t) q. $$
Then, \eqref{eq: p-moment-2} implies that
$$
\begin{aligned}
  \left|\error(t;q)\right| &\le\left(64 L^{\frac{6 \sigma}{c}-2}+2  L^{-\sigma^{2}+1} + 2 \exp(-L/8)\right) \\
& \autoquad{2} \cdot\left(\left|q^{\top}\left(\partial_{t} (W_t^{\top} W_t )+\partial_{t} ({\tilde W}_t^{\top} {\tilde W}_t )\right) q\right|+2\left|q^{\top} \partial_{t} ({\tilde W}_t^{\top} W_t ) q\right|\right) \\
&  \le 4\left(64 L^{\frac{6 \sigma}{c}-2}+2  L^{-\sigma^{2}+1} + 2 \exp(-L/8)\right) \tau^{2},
\end{aligned}
$$
where the last inequality holds since we have constant speed. 
From the dynamics \eqref{eq: p-moment-1}, we have
$$
\partial_{t}\left(\exp  \big(-q^{\top} {\tilde W}_t^{\top} W_t q \big) \cdot \mathbb{E} [p_t^{\top} \tilde p_t\given q]\right)=\exp \big(-q^{\top} {\tilde W}_t^{\top} W_t q\big) \cdot \operatorname{err}(t;q).
$$
By integrating the above equation from $t = 0$ to $t = 1$, we have
\begin{align}
& \left|\mathbb{E}[p^{\top} \tilde p \mid q]-\frac{\exp \big(q^{\top} \tilde W^{\top} W q \big) }{L}\right| \nend
& \quad=\left|\int_{0}^{1} \exp \big(-q^{\top} \tilde W^{\top}_t W_t q\big) \cdot \operatorname{err}(t;q) \rd t \right| \cdot \exp \big(q^{\top} \tilde W^{\top} W q\big)\nend
& \quad \le 4\left(64 L^{\frac{6 \sigma}{c}-2}+2  L^{-\sigma^{2}+1} + 2 \exp(-L/8)\right) \tau^{2} \cdot \frac{\exp \big(q^{\top} \tilde W^{\top} W q \big) -  1}{q^{\top} \tilde W^{\top} W q} \nend
& \quad \le 4\left(64 L^{\frac{6 \sigma}{c}-2}+2  L^{-\sigma^{2}+1} + 2 \exp(-L/8)\right) \cdot \exp \left(\tau^{2}\right). 
\label{eq: p-moment-4}
\end{align}
We invoke the upper bound for $\tau^2$ to obtain the final result, 
\begin{align*}
    \left|\mathbb{E}[p^{\top} \tilde p \mid q]-\frac{\exp \big(q^{\top} \tilde W^{\top} W q \big) }{L}\right| \le 4\left(64 L^{ {2}{c^{-2}} + {6 \sigma}{c^{-1}}-2}+2  L^{-\sigma^{2}+1+{2}{c^{-2}}} + 2 L^{2 c^{-2} -L (\log L)^{-1} / 8 }
    \right).
\end{align*}
Here, we plug in $\sigma = c \cdot (1 + \sqrt{1 + c^2 /3})^{-1}$ which make the first and second terms equal in order. We also note that the third term is of higher order. Hence, for each fixed constant $\epsilon\in (0, 1)$, we can always find $c>0$ satisfying the fixed-point condition 
\begin{align*}
    \frac{1}{c} + \frac{3}{1 + \sqrt{1+c^2/2}} = \epsilon, 
\end{align*}
such that for all $\tau^2 \defeq \max\{\norm{W q }_2^2 , \norm{{\tilde W} q }_2^2 \} \le c^{-2}\cdot 2\log L$: 
\begin{align*}
    &\left|\mathbb{E}[p^{\top} \tilde p \mid q]-\frac{\exp \big(q^{\top} \tilde W^{\top} W q \big) }{L}\right|
    \le O( L^{2(c^{-2}-1) + 6\cdot (1+ \sqrt{1 + c^2/2})^{-1}}) = O(L^{-2(1-\epsilon)}).
\end{align*}
Notably, the error term in \eqref{eq: p-moment-2} is strictly less than the error term in \eqref{eq: p-moment-4}. Hence for the higher order terms, we also have
\begin{align*}
    &\max \left\{
    \EE[p_t^\top \tilde p_t^{\odot 2}\given q], \quad \EE[\tilde p_t^\top p_t^{\odot 2}\given q], \quad \EE[\norm{p_t}_2^2 p_t^\top \tilde p_t\given q], \quad \EE[\norm{\tilde p_t}_2^2 p_t^\top \tilde p_t\given q], \quad \EE[(p_t^\top \tilde p_t)^2\given q] 
\right\}
\nend 
    &\quad \le O(L^{-2(1-\epsilon)}).
\end{align*}
Hence, the proof is complete.
\end{proof}

\paragraph{Concentration in $q$.}
Next, we take expectation with respect to $q \sim \cN\left(0, I_{D}\right)$. It suffices to consider the expectation $\mathbb{E}\big[\exp (q^{\top} \tilde W^{\top} W q)\big]$ according to \Cref{lemma: pseudo-dynamics}. 
The following lemma characterizes the expectation term. 
\begin{lemma}
    \label[lemma]{lemma: E[pq]-moment}
    Consider that $W\in\RR^{D\times D}$ and $\tilde W\in\RR^{D\times D}$ share the same left eigenvectors ${w_1, \dots, w_D}$ and also the same right eigenvectors $u_1, \dots, u_D$, i.e., $W = \sum_{k=1}^d \omega_k w_k u_k^\top$ and $\tilde W = \sum_{k=1}^d \tilde\omega_k w_k u_k^\top$. Here, $d$ is the effective rank for $W$ and $\tilde W$. 
    Suppose 
    \begin{gather*}
        \max\left\{\norm{\tilde\omega}_\infty, \norm{\omega}_\infty\right\} \le   L^{-1/4} \cdot (\log L)^{-1/2} , \quad \max\{\norm{\omega}_2^2, \norm{\tilde\omega}_2^2\} \le \frac{2 \log L}{3 c^2},  \\
        \max\{\norm{\omega}_4^4, \norm{\tilde\omega}_4^4\} \le L^{-(1-\epsilon_0)}\cdot (\log L)^{-1}
    \end{gather*}
    where $c>0$ is a constant depending on a given constant $\epsilon\in(0, 1)$ via the fixed-point equation 
    \begin{align*}
        \frac{1}{c} + \frac{3}{1 + \sqrt{1+c^2/2}} = \epsilon, 
    \end{align*}
    and $\epsilon_0\in (0, 1)$ is a constant whose purpose is to make sure $\omega = d^{-1/2} \vone$ satisfies the forth moment condition and could also be chosen to be the same as $\epsilon$.
    It then holds that
    \begin{align*}
        \EE\left[\left( \mathbb{E}[p^{\top} \tilde p\given q]-\frac{\exp \left(\trace[W^\top \tilde W]\right) }{L} \right)^2 \right] \le O(L^{- (3 -\epsilon_0)}).
    \end{align*}
\end{lemma}

\begin{proof}[Proof of \Cref{lemma: E[pq]-moment}]
Under the condition that $\tilde W$ and $W$ have the same left and right eigenvectors, it suffices to study the quantity $\mathbb{E}\big[\exp (\sum_{k=1}^{d} \omega_k \tilde\omega_k v_{k}^{2})\big]$ where $v_k\iidfrom \cN(0, 1)$. Here, $d$ is the maximal rank of $W$ and $\tilde W$, and $\omega_k$ and $\tilde\omega_k$ are the eigenvalues of $W$ and $\tilde W$ within the effective dimensions, respectively. 
We also denote by $\omega$ and $\tilde\omega$ the vector of eigenvalues for $W$ and $\tilde W$ when there is no ambiguity.
Note that $v_k^2$ is a chi-square random variable.
We invoke the following tail bound for chi-square distribution, which is given by Lemma 1 in \citet{laurent2000adaptive}.
\begin{gather}
\mathbb{P}\left(\sum_{k=1}^{d} \omega_k \tilde\omega_k v_k^{2}-\sum_{k=1}^{d} \omega_k \tilde\omega_k \ge 2\left\|\omega \odot \tilde\omega\right\|_{2} \sqrt{x}+2\left\|\omega \odot \tilde\omega\right\|_{\infty} x\right) \le \exp (-x), 
\label{eq:p-moment-chi^2 right tail}
\\
\mathbb{P}\left(\sum_{k=1}^{d} \omega_k \tilde\omega_k v_k^{2}-\sum_{k=1}^{d} \omega_k \tilde\omega_k \le-2\left\|\omega \odot \tilde\omega\right\|_{2} \sqrt{x}\right) \le \exp (-x). \notag
\end{gather}
For our purpose, we take $x=4 \log L$, which gives
\begin{align}
\mathbb{P}\left(\left|\sum_{k=1}^{d} \omega_k \tilde\omega_k v_{k}^{2}-\sum_{k=1}^{d} \omega_k \tilde\omega_k\right| \ge 2\left\|\omega \odot \tilde\omega\right\|_{2} \sqrt{4 \log L}+8\left\|\omega \odot \tilde\omega\right\|_{\infty} \log L\right) \le \frac{2}{L^{4}}. \label{eq: p-moment-3}
\end{align}
Define the above event as $\cE_{2}^{c}(W, \tilde W)$. 
Before we proceed, one must be aware that to use \Cref{lemma: pseudo-dynamics} for approximating $\mathbb{E}\big[p^\top \tilde p\given q\big]$ with $\exp (q^{\top} \tilde W^{\top} W q)$, we need to show that event 
$$\cE_3(W)\defeq \left\{\|W q\|_{2}^{2} \le c^{-2} \cdot {2 \log L}\right\}$$ 
also holds with high probability.
To show this point, 
we have by \eqref{eq:p-moment-chi^2 right tail} where $\tilde\omega$ is replaced by $\omega$ and $x=4\log L$ that 
\begin{align}
    \mathbb{P}\left( \norm{W q}_2^2  \ge \norm{\omega}_2^2 + 4\left\|\omega\right\|_{4}^2 \sqrt{\log L}+8\left\|\omega \right\|_{\infty}^2 \log L \right) \le L^{-4}. 
    \label{eq: p-moment-fail prob E3}
\end{align}
Note that under the conditions
\begin{align}
    \max\{\norm{\omega}_4^4, \norm{\tilde\omega}_4^4\}\le \frac{ \log L}{36 c^4} ,\quad  
    \max\{\norm{\omega}_2^2, \norm{\tilde\omega}_2^2\} \le \frac{2 \log L}{3 c^2}, \quad
    \max\{\norm{\omega}_\infty , \norm{\tilde\omega}_\infty\} \le \frac{1}{\sqrt{12} c}, 
    \label{eq: E3 condition}
\end{align}
it follows that $\mathbb{P}(\cE_3^c(W)) = \mathbb{P}\left( \norm{W q}_2^2  \ge c^{-2} \cdot 2 \log L\right) \le L^{-4}$, and the same holds for $\cE_3^c(\tilde W)$.
Hence, we are able to control the difference between $\EE[p^\top \tilde p]$ and $\exp(\omega^\top \tilde\omega) / L$ by 
\begin{align}
&  \EE\left[\left( \mathbb{E}[p^{\top} \tilde p\given q]-\frac{\exp \left(\omega^{\top} \tilde\omega\right) }{L} \right)^2 \right] \notag\\
& \quad \le 
    \EE\left[\left( \mathbb{E}[p^{\top} \tilde p\given q]-\frac{\exp \left(\omega^{\top} \tilde\omega\right) }{L} \right)^2 \cdot \ind\left[\cE_2(W, \tilde W) \cap \cE_3(W) \cap \cE_3(\tilde W)\right] \right] \notag\\
    & \autoquad{3} 
    +\left( \PP\bigl(\cE_2^c(W, \tilde W)\bigr) + \PP\bigl(\cE_3^c(W)\bigr) + \PP\bigl(\cE_3^c(\tilde W)\bigr) \right) 
    \cdot 2\left(1+\frac{2\exp \left(\omega^{\top} \tilde\omega\right)}{L^2}\right) \notag\\
& \quad \le 
    2\cdot \frac{\exp \left(2\omega^{\top} \tilde\omega\right)}{L^2} \cdot\left(\exp \left(2\left\|\omega \odot \tilde\omega\right\|_{2} \sqrt{4 \log L}+8\left\|\omega \odot \tilde\omega\right\|_{\infty} \log L\right)-1\right)^2 + \xi_L^2 \notag\\
    & \autoquad{3}+\frac{6}{L^{4}} \cdot \left(1+\frac{\exp \left(2\omega^{\top} \tilde\omega\right)}{L^2}\right) , 
    \notag
\end{align}
where the second inequality holds by using the approximation result in \Cref{lemma: pseudo-dynamics} with $\xi_L = O(L^{-2(1-\epsilon)})$, and the last inequality holds by plugging in \eqref{eq: p-moment-3}. 
Under the condition 
$$\max\left\{\norm{\tilde\omega}_\infty, \norm{\omega}_\infty\right\} \le   L^{-1/4} \cdot (\log L)^{-1/2} , \quad \max\{\norm{\omega}_4^4, \norm{\tilde\omega}_4^4\} \le L^{-(1-\epsilon_0)}\cdot (\log L)^{-1}, $$
it is easy to show that
\begin{gather*}
    \exp \left(2\left\|\omega \odot \tilde\omega\right\|_{2} \sqrt{2 \log L}+4\left\|\omega \odot \tilde\omega\right\|_{\infty} \log L\right) - 1 \le  O(L^{-(1-\epsilon_0)/2} ), \\
    \frac{6}{L^{4}} \cdot \left(1+\frac{\exp \left(2\omega^{\top} \tilde\omega\right)}{L^2}\right) = O(L^{-4}), \quad \xi_L = O(L^{-2(1-\epsilon)}).
\end{gather*}
As a result, we conclude that
\begin{align*}
    \EE\left[\left( \mathbb{E}[p^{\top} \tilde p\given q]-\frac{\exp \left(\omega^{\top} \tilde\omega\right) }{L} \right)^2 \right] \le O(L^{- (3 -\epsilon_0)}), 
\end{align*}
which completes the proof.
\end{proof}

\subsection{Controlling the Higher Order Moments.}
\begin{lemma} \label{lemma: higher-order-moment}
    Under \Cref{def:graph-polynomial}, consider a $\cG$-induced polynomial $f_\cG(p, \tilde p) = \prod_{v\in\cV} \bigl( p_{v}^{a_v} \tilde p_{v}^{b_v} \bigr) \cdot \prod_{(v, v')\in\cE}\delta_{v v'}$ and let $ \efforder(\cG)\deleq\sum_{v\in\cV} (a_v + b_v) - |\connectedcomponent_{\ge 1}(\cG)|$ be the effective order of $\cG$. 
    Here, $p$ and $\tilde p$ are the attention probability vectors for $W$ and $\tilde W$, respectively.
    Suppose $\efforder(\cG)\ge \kappa = 2$.
    Fix any $\epsilon \in (0, 1)$ and let $c\ge \frac{2\Deg(\cG)\sqrt{2\kappa+1}}{\epsilon \kappa}$ where $\Deg(\cG) = \sum_{v\in\cV} (a_v + b_v)$. 
    If 
    \[\max\{\norm{\omega}_4^4, \norm{\tilde\omega}_4^4\}\le \frac{ \log L}{36 c^4} ,\:  
    \max\{\norm{\omega}_2^2, \norm{\tilde\omega}_2^2\} \le \frac{2 \log L}{3 c^2}, \:
    \max\{\norm{\omega}_\infty , \norm{\tilde\omega}_\infty\} \le \frac{1}{\sqrt{12} c}\] holds, we have
$$
\mathbb{E}\left[f_\cG(p, \tilde p)^{2}\right] \le O(L^{-4(1-\epsilon)}). 
$$
\end{lemma}
\begin{proof}[Proof of \Cref{lemma: higher-order-moment}]
Let $n= |\connectedcomponent_{\ge 1}(\cG)|$ and $\Deg(\cG) = \sum_{v\in\cV}^n (a_v + b_v)$. 
The argument is a combination of the general ideas in \Cref{lemma: pseudo-dynamics} and \Cref{lemma: E[pq]-moment}.
Let us first fix a query $q$. 
On the success of the event $\cE_{1}(W) \cap \cE_{1}(\tilde W)$ defined in \eqref{def: cE_1}, 
one has by definition that 
\[
f_\cG(p, \tilde p)^{2} \le\left(4 L^{\frac{2\sigma}{c}-1}\right)^{2m} \cdot L^{2 n} = 16^{\Deg(\cG)} \cdot L^{\frac{4\sigma}{c} \Deg(\cG)} \cdot L^{-2\efforder(\cG)}, 
\]
The failure probability for this joint event is at most
$
2\big(L^{-\sigma^{2}+1}+L^{-\frac{L}{8 \log L}}\big) 
$ as we have shown in the proof of \Cref{lemma: pseudo-dynamics} if $\cE_{3}(W) \cap \cE_{3}(\tilde W)$ holds. Here, $\cE_3(W)$ is defined by the success of $\|W q\|_{2}^{2} \le c^{-2} \cdot {2 \log L}$. 
Note that $|f_\cG(p, \tilde p)| \le 1$. 
Therefore, we have on the success of $\cE_{3}(W) \cap \cE_{3}(\tilde W)$ that
\begin{align*}
\mathbb{E}\left[f(p, \tilde p)^{2}\given q\right] 
&\le 16^{\Deg(\cG)} \cdot L^{\frac{4\sigma}{c} \Deg(\cG)} \cdot L^{-2\kappa} +2\left(L^{-\sigma^{2}+1}+L^{-\frac{L}{8 \log L}}\right) .
\end{align*}
We take $\sigma = \sqrt{2\kappa+1}$. 
By choosing a constant $c= (2\sigma \Deg(\cG)) /(\epsilon \cdot \kappa) $ according to $\epsilon$, the total degree $\Deg(\cG)$, and $\kappa$,  
we obtain that $\mathbb{E}\left[f(p, \tilde p)^{2}\given q\right] \le O(L^{-2 \kappa \cdot (1-\epsilon)})$ on the success of $\cE_{3}(W) \cap \cE_{3}(\tilde W)$.
Following the conditions in \eqref{eq: E3 condition} on the failure probability of $\cE_{3}(W) \cap \cE_{3}(\tilde W)$, we have
\begin{align*}
    \EE\left[f(p,\tilde p)^2\right] 
    & = \EE\left[  \mathbb{E}[f(p, \tilde p)^2\given q]  \right] \notag\\
    &\le \EE\left[
        \mathbb{E}[f(p, \tilde p)^2\given q] \cdot \ind\left[\cE_3(W) \cap \cE_3(\tilde W)\right] 
    \right] + \PP(\cE_3^c(W)) + \PP(\cE_3^c(\tilde W)) \notag\\
    &\le O(L^{-2 \kappa \cdot (1-\epsilon)}) + O(L^{- 4}) \le O(L^{- 4(1-\epsilon)}), 
\end{align*}
where in the first inequality we use the fact that $|f_\cG(p, \tilde p)| \le 1$ and in the last inequality we use the same conditions as \Cref{eq: E3 condition} in order to invoke \eqref{eq: p-moment-fail prob E3}.
Thus, we complete the proof.
\end{proof}
\subsection{Understanding $\EE[\norm{p}_2^2\given q]$ for large $\norm{Wq}_2$}
Previously, we have a thorough understanding of the behavior of $\EE[\norm{p}_2^2]$ as $\norm{Wq}_2^2\le c^{-2} \cdot 2 \log L$. 
We are still lacking understanding of what happens if $\norm{Wq}_2^2 \gg c^{-2} \cdot 2 \log L$. 
In this section, we give a characterization of the behavior of $\EE[\norm{p}_2^2\given q]$. 
Note that $\norm{p}_2^2 \le 1$. 
When $\norm{Wq}_2$ gets large, we anticipate $\norm{p}_2^2$ to get closer to $1$. 
Hence, we instead understand the behavior of $h(r) = 1 - \EE[\norm{p}_2^2\given q]$, where $r=\norm{Wq}_2$.
In the following, we provide an upper bound for $h(r)$. 
\begin{lemma}\label[lemma]{lem:p-moment-tail}
Define $h(r) = 1 - \EE[\norm{p}_2^2\given \norm{W q}_2 = r]$. Take a small constant $\epsilon\in (0, 1)$. 
For $r \ge O(\sqrt{2\log L}^{3/\epsilon})$,  it holds that
\[  
h(r) = \EE[1 - \norm{p}_2^2] \le O(r^{-(1-\epsilon)}).
\]
\end{lemma}
\begin{proof}[Proof of \Cref{lem:p-moment-tail}]
Let $p_1, \dots, p_L$ be the elements of $p$ in the descending order. 
In the following , we consider $r$ to be fixed and the expectation is taken only with respect to the context $X$ (or equivalently, attention probability $p$). 
We let $s_1, \dots, s_L$ be the attention score calculated by $x_1^\top W q, \dots, x_L^\top W q$, also in the same descending order and let $z_1, \dots, z_L$ be the projection of $x_1, \dots, x_L$ onto the direction of $W q$. 
We have
\begin{align*}
    h(r) \le \EE[1 - p_1^2] = \EE[(1-p_1)(1+p_1)]  \le 2 \EE[1 - p_1] = 2 \cdot \EE[p_2 +\dots + p_L]. 
\end{align*}
Also we note that 
\begin{align*}
    \frac{p_1}{p_l} = \exp(s_1 - s_l) = \exp( r (z_1 - z_l)) \Rightarrow p_l \le p_1 \exp\left( - r (z_1 - z_l)\right) \le \exp\left( - r (z_1 - z_l)\right). 
\end{align*}
Hence, 
\begin{align*}
    h(r) \le 2 \EE\left[\sum_{l=2}^L \exp(-r (z_1 - z_l))\right]. 
\end{align*}
We split the summation into two parts depending on whether $z_1 - z_j \le \log(L r^{\alpha})/ r$, where $\alpha > 1/2$ is a constant to be determined later. 
Thereby, 
\begin{align*}
    h(r) 
    &\le 2 \EE\bigg[ \exp(-r (z_1 - z_2)) + \sum_{j\ge 3: z_1 -  z_j \le \log (L r^\alpha)/r} \exp(-r (z_1 - z_j)) + L \exp\left( - r \cdot \log (L r^\alpha)/r\right)\bigg] \nend 
    & \le 2 \EE\left[(N+1) \exp(-r (z_1 - z_2)) \right] + \frac{2}{r^\alpha}, 
\end{align*}
where $N$ is the size of $\{j \in \{3, \dots, L\}: z_1 - z_j \le \log (L r^\alpha)/r\}$ and is also a random variable. 
Here, we remark that in the second inequality, we just lower bound $z_1 - z_j$ by the minimum $z_1 - z_2$. 
We pick some small constant $a\in(0, 1)$ and note that 
\begin{align}
    &2\EE\left[(N+1) \exp(-r (z_1 - z_2)) \right]  \nend
    &=  2\EE\left[(N+1) \exp(-r (z_1 - z_2)) \ind(z_1 > \sqrt{2 a \log L })\right] + 2\EE\left[(N+1) \exp(-r (z_1 - z_2)) \ind(z_1 \le \sqrt{2 a \log L })\right] \nend 
    &\le 2\EE\left[(N+1) \exp(-r (z_1 - z_2)) \ind(z_1 > \sqrt{2 a \log L })\right] + 2 L \PP(z_1 \le \sqrt{2 a\log L})\nend 
    &\le 2\EE\left[(N_1+1) \exp(-r (z_1 - z_2))\right] + 2 L \exp\left(  - \frac{1}{\sqrt{2\pi}} \cdot \left(\frac{1}{\sqrt{2a \log L}} - \frac{1}{\sqrt{2 a \log L}^3}\right) L^{1-a} \right), 
    \label{eq:p-moment-tail-1}
\end{align}
where we define $N_1$ as the size of the class $\{j\in\{3, \dots, L\}: \sqrt{2a \log L } - z_j \le \log (Lr^\alpha)/r\}$ and it is obvious that $N \le N_1$ under the condition $z_1 > \sqrt{2a\log L}$. 
Here, in the first inequality, we upper bound $N$ by $L$ and in the second inequality, we invoke the following fact: 
\begin{lemma}[Gaussian First Order Statistics]
Let $v$ be the first order statistics of $L$ $\iid$ standard Gaussian random variables. For any $a>0$:
\[
\PP(v \le \sqrt{2 a \log L}) \le \exp\left( \frac{1}{\sqrt{2\pi}} \cdot \left(\frac{1}{\sqrt{2a \log L}} - \frac{1}{\sqrt{2 a \log L}^3}\right) L^{1-a}
\right)
\]
\end{lemma}
We take $a =  1-\epsilon$, and to make the second term in \eqref{eq:p-moment-tail-1} smaller than $O(r^{-\alpha})$, we just need 
\[  
    \epsilon \ge \frac{\log\left(\sqrt{4\pi \log L} \log (L r^{\alpha})\right)}{\log L} \rightarrow \frac{\log (\log L\log r)}{\log L} = o(1), 
\]
where the right hand side is $o(1)$ as long as $r=o(\exp(L))$. 
To this end, we just set the first inequality to be equality for $\epsilon$. 
Now, for the first term, we have
\begin{align*}
    &2\EE\left[(N_1+1) \exp(-r (z_1 - z_2))\right] \nend
    &\quad = 2 \EE\left[\sharp\left\{ j=2,\dots, L: \:z_j \ge \left(\sqrt{(1-\epsilon)} - \frac{\log(L r^\alpha)}{r \sqrt{2\log L}}\right) \cdot  \sqrt{2\log L} \right\}  \exp(-r(z_1 - z_2)\right] \nend 
    &\quad\le  2 \EE\left[\sharp\left\{ j\in[L]: \:z_j \ge \left(\sqrt{(1-\epsilon)} - \frac{\log(L r^\alpha)}{r \sqrt{2\log L}}\right) \cdot  \sqrt{2\log L} \right\}  \exp(-r(z_1 - z_2)\right] \nend
    &\quad = 2 \EE\left[N_2\exp(-r(z_1 - z_2)\right], 
\end{align*}
where we define 
\[
N_2 \defeq \sharp\left\{ j\in[L]: \:z_j \ge \left(\sqrt{(1-\epsilon)} - \frac{\log(L r^\alpha)}{r \sqrt{2\log L}}\right) \cdot  \sqrt{2\log L} \right\}  
\]
We note that for a standard Gaussian random variable $z$, 
\begin{align*}
&\PP\left(z \ge \left(\sqrt{(1-\epsilon)} - \frac{\log(L r^\alpha)}{r \sqrt{2\log L}}\right) \cdot  \sqrt{2\log L}\right) \nend 
&\quad \le \frac{1}{\sqrt{2\pi}} \exp\left( - \left(\sqrt{(1-\epsilon)} - \frac{\log(L r^\alpha)}{r \sqrt{2\log L}}\right)^2  \cdot  \log L\right) \cdot \frac{1}{\sqrt{\log L}} \nend 
&\quad \le \frac{1}{\sqrt{2\pi \log L}} L^{-1 + \epsilon + \frac{2\log(L r^\alpha)}{r\sqrt{2\log L}}} = L^{-1} \cdot L^{\frac{\log\left(\sqrt{4\pi \log L} \log (L r^{\alpha})\right)}{\log L} + \frac{2\log(L r^\alpha)}{r\sqrt{2\log L}} - \frac{\log \sqrt{2\pi \log L}}{\log L}}. 
\end{align*}
where we invoke the upper bound $\PP(z\ge a) \le f(z) /a$ with $f(\cdot)$ being the density of $z$.
For simplicity, we consider function 
\[
\brown \varphi(r) \ge \frac{\log\left(\sqrt{4\pi \log L} \log (L r^{\alpha})\right)}{\log L} + \frac{2\log(L r^\alpha)}{r\sqrt{2\log L}} - \frac{\log \sqrt{2\pi \log L}}{\log L}, 
\]
where we note that $r \ll  \exp(L/\alpha)$ (the upper bound is because the first term in $\varphi(r)$ cannot be too large).
To this end, we invoke the tail bound for binomial distribution that if $N_2$ follows $\binomial(L, q)$, then 
\begin{align*}
    \PP(N_2 \ge m) \le \exp\left( - L \cD(m/L \| q)\right) , 
\end{align*}
where $\cD(q_1 \| q_2)$ is the KL divergence between $\bernoulli(q_1)$ and $\bernoulli(q_2)$. 
We plug in the expression for $q = L^{\varphi(r) - 1}$ and obtain that 
\begin{align*}
    \PP(N_2 \ge m) 
    &\le \exp\left( 
        -2 L \left( \frac{m}{L} \log \frac{m}{L^{\varphi(r)}} +  \left(1 - \frac{m}{L}\right) \log \left(1 + \frac{m/L - q}{1  - m/L}\right)\right) 
    \right)\nend 
    &\le \exp\left( 
        -2 \left( m\log \frac{m}{L^{\varphi(r)}} + {m - L^{\varphi(r)}}\right) 
    \right), 
\end{align*}
where in the second inequality we invoke the fact that $\log(1+x)\le x$. 
We pick $m = L^{\varphi(r) + \xi(r)}$, where $\xi(r)>0$ is a constant depending on $r$ to be determined later,  and obtain 
\begin{align*}
    \PP(N_2 \ge L^{\varphi(r) + \xi(r)}) \le \exp\left( -2 L^{\varphi(r) + \xi(r)} \cdot (1+\xi(r)\log L) + 2 L^{\varphi(r)}\right) \le L^{-2 L^{\varphi(r)+\xi(r)} \cdot \xi(r)}, 
\end{align*}
Here, the first and the third terms cancel and the second term gives the final upper bound.
To this end, we have for $2\EE[N_2 \exp(- r(z_1 - z_2))]$
that 
\begin{align*}
    2\EE[N_2 \exp(- r(z_1 - z_2))] 
    &\le 2\EE[N_2 \exp(- r(z_1 - z_2)) \ind(N_2 < L^{\varphi(r) + \xi(r)})] 
    \nend 
    &\qquad + 2 \EE[ N_2 \exp(-r(z_1 - z_2)) \ind(N_2 \ge L^{\varphi(r) + \xi(r)})] \nend 
    &\le 2 L^{\varphi(r)+\xi(r)} \EE[\exp(- r(z_1 - z_2))] + 2 L \cdot L^{-2 L^{\varphi(r)+\xi(r)} \xi(r)}. 
\end{align*}
Note that to make the second term less than $O(r^{-\alpha})$, a sufficient condition we need is 
\[
\brown \varphi(r) + \xi(r) - \frac{\log(\xi(r)^{-1})}{\log L}\ge \frac{\alpha \log r}{(\log L)^2}. 
\]
Therefore, all we need to do is controlling the \ac{mgf} for the difference between the Gaussian's first and second order statistics. 
Let $u, v$ be the second and the first order statistics of $z_1, \dots, z_L$. 
The joint distribution of $u, v$ is then given by 
\begin{align*}
    f(u, g) = \frac{L(L-1)}{2\pi} \exp\left(-\frac{u^2 + v^2}{2}\right) \Phi(u)^{n-2} \ind (u\le v), 
\end{align*}
where $\Phi$ is the standard Gaussian \ac{cdf}. 
The \ac{mgf} is then given by
\begin{align*}
    M(r) &= \int_{-\infty}^{+\infty}\int_{-\infty}^{+\infty} \frac{L(L-1)}{2\pi} \exp\left(-\frac{u^2 + v^2}{2}\right) \Phi(u)^{n-2} \ind (u\le v) \exp\left( - r (v - u)\right) \rd u\rd v \nend 
    & = \int_{-\infty}^{+\infty} \frac{L(L-1)}{\sqrt{2\pi}} \exp\left(-\frac{u^2}{2}\right) \Phi(u)^{n-2} \exp\left(r u\right) \cdot \int_{u}^{+\infty}\frac{1}{\sqrt{2\pi}} \exp\left(-\frac{v^2}{2} - r v\right)\rd v  \cdot \rd u \nend
    &= \int_{-\infty}^{+\infty} \frac{L(L-1)}{\sqrt{2\pi}} \exp\left(-\frac{u^2}{2}\right) \Phi(u)^{n-2} \exp\left(\frac{r^2}{2} + r u\right) \cdot \left(1 - \Phi(u + r)\right) \cdot \rd u. 
\end{align*}
We note that the marginal for $u$ is (we just need to plug $r = 0$ into the above equation, and after marginalizing $v$ the remaining expression is just $f(u)$)
\begin{align*}
    f(u) = \frac{L(L-1)}{\sqrt{2\pi}} \exp\left(-\frac{u^2}{2}\right) \Phi(u)^{n-2} \cdot \left(1 - \Phi(u)\right). 
\end{align*}
Notably, in the above calculation we fix $u$ and only marginalize $v$. Thus, a byproduct is the following fact:
\begin{align*}
    \EE\left[ \exp\left(-r (v - u)\right) \given u \right] = \exp\left(\frac{r^2}{2} + r u\right) \cdot \frac{\PP(z\ge u+ r)}{\PP(z\ge u)} \le 1. 
\end{align*}
Therefore, the \ac{mgf} can also be rewritten as 
\begin{align*}
    M(r) &= \EE_u \left[ \exp\left(\frac{r^2}{2} + r u\right) \cdot \frac{\PP(z\ge u+ r)}{\PP(z\ge u)}\right] \nend 
    & = \EE\left[
        \exp\left(\frac{r^2}{2} + r u\right)\frac{\PP(z\ge u+ r)}{\PP(z\ge u)} \ind\left(u\in (\sqrt{2(1-\tilde \epsilon)\log L}, \sqrt{2(1+b)\log L}) \right)
    \right] \nend 
    &\qquad + \EE\left[
        \exp\left(\frac{r^2}{2} + r u\right)\frac{\PP(z\ge u+ r)}{\PP(z\ge u)} \ind\left(u\notin (\sqrt{2(1-\tilde \epsilon)\log L}, \sqrt{2(1+b)\log L}) \right)
    \right] \nend 
    & = \EE\left[
        \exp\left(\frac{r^2}{2} + r u\right)\frac{\PP(z\ge u+ r)}{\PP(z\ge u)} \ind\left(u\in (\sqrt{2(1-\tilde \epsilon)\log L}, \sqrt{2(1+b)\log L}) \right)
    \right] \nend 
    &\qquad +  \PP\left(u\notin (\sqrt{2(1-\tilde \epsilon)\log L}, \sqrt{2(1+b)\log L}) \right) 
\end{align*}
where we pick $b \in (0, 1)$ and $\tilde \epsilon \in (O((\log L)^{-1}), 1)$ as small constant to be determined later. 
For the second order statistics of $L$ $\iid$ Gaussian random variables, we have 
\begin{fact}[Gaussian Second Order Statistics]
    Let $v, u$ be the first and the second order statistics of $L$ $\iid$ standard Gaussian random variables. 
    Also, let $z$ be another standard Gaussian random variable. 
    Then for any $\epsilon \in (O((\log L)^{-1}), 1)$:
    \begin{align*}
        \PP(u\ge \sqrt{2(1+\epsilon) \log L}) \le \PP(v\ge \sqrt{2(1+\epsilon )L }) \le L^{-\epsilon},
    \end{align*}
    and 
    \begin{align*}
        \PP(u\le \sqrt{2(1-\epsilon)\log L}) 
        &= \PP(z\le \sqrt{2(1-\epsilon)\log L})^{L-1} \cdot \PP(z\ge \sqrt{2(1-\epsilon)\log L}) \cdot \binom{L}{1} \nend 
        &\qquad + \PP(z\le \sqrt{2(1-\epsilon)\log L})^L \nend 
        & \le (L + 1) \cdot  \PP(z\ge \sqrt{2(1-\epsilon)\log L})^{L - 1} \nend 
        &\le (L + 1) \cdot \exp\left( - \frac{(L-1)^{\epsilon}}{\sqrt{2\pi}} \cdot \left(\frac{1}{\sqrt{2(1-\epsilon)\log L}} - \frac{1}{\sqrt{2(1-\epsilon)\log L}^3}\right)\right) \nend 
        &\le (L + 1) \cdot \exp\left( - \frac{(L-1)^{\epsilon}}{\sqrt{4\pi \log L}} \right), 
    \end{align*}
    where the last inequality holds if $\epsilon\ge O(1/ 2\log L)$. 
\end{fact}
Using the above fact, we deduce that 
\begin{align*}
    M(r) 
    &\le L^{-b} + (L + 1) \cdot \exp\left( - \frac{(L-1)^{\epsilon}}{\sqrt{4\pi \log L}} \right)  + \max_{u\in \cU(\tilde\epsilon, b)}\: \exp\left(\frac{r^2}{2} + r u\right) \cdot \frac{\PP(z\ge u+ r)}{\PP(z\ge u)}, 
\end{align*}
where for our convenience, we define $\cU(\tilde\epsilon, b) = (\sqrt{2(1-\tilde \epsilon)\log L}, \sqrt{2(1+b)\log L})$.
Now, we invoke the tail bound for Gaussian distribution and obtain that 
\begin{align}
    \exp\left(\frac{r^2}{2} + r u\right) \cdot \frac{\PP(z\ge u+ r)}{\PP(z\ge u)}
    &\le \exp\left(\frac{r^2}{2} + r u\right)\cdot \frac{f(u+r)}{f(u)} \cdot \frac{(u+r)^{-1}}{u^{-1} - u^{-3}}  = \frac{u}{u+r} \cdot \frac{1}{1 - u^{-2}}. 
\end{align}
As the last piece of the jigsaw, we conclude that 
\begin{align*}
    &2\EE[N_2 \exp(- r(z_1 - z_2))] \nend
    &\quad \le 2 L^{\varphi(r)+\xi(r)} \EE[\exp(- r(z_1 - z_2))] + 2 L \cdot L^{-2 L^{\varphi(r)+\xi(r)} \xi(r)} \nend 
    &\quad \le 2 L^{\varphi(r)+\xi(r)} \cdot \left(L^{-b} + (L + 1) \cdot \exp\left( - \frac{(L-1)^{\epsilon}}{\sqrt{4\pi \log L}} \right)  + \max_{u\in \cU(\tilde\epsilon, b)}\: \frac{u}{u+r} \cdot \frac{1}{1 - u^{-2}}\right)  \nend 
    &\qqquad + 2 L \cdot L^{-2 L^{\varphi(r)+\xi(r)} \xi(r)}. 
\end{align*}
Now, we pick $\beta > 0$ and let 
\[  
    L^{-b} =
    (L + 1) \cdot \exp\left( - \frac{(L-1)^{\tilde\epsilon}}{\sqrt{4\pi \log L}} \right) = r^{-\beta}, 
\]
which implies 
\[  
    b = \frac{\beta \log r}{\log L}, \quad \tilde\epsilon = \frac{\log\left(\log \frac{L+1}{r^\beta} \cdot \sqrt{4\pi \log L}\right)}{\log (L-1)} = o(1), 
\]
where the solution also satisfies the contraint $\tilde \epsilon > O((\log L)^{-1})$ so long as $\brown r \ll (L+1)^{\beta^{-1}}$
and we have 
\begin{align*}
    &\max_{u\in \cU(\tilde\epsilon, b)}\: \frac{u}{u+r} \cdot \frac{1}{1 - u^{-2}} \nend
    &\quad \le \frac{1}{ 1 + r/\sqrt{2(1+b)\log L}} \cdot \frac{1}{1-(2(1-\tilde\epsilon)\log L)^{-2}} = \frac{1 + O((\log L)^{-2})}{1 + r/\sqrt{2(\log L+\beta \log r )}}. 
\end{align*}
Therefore, we have 
\begin{align}
    &2\EE[N_2 \exp(- r(z_1 - z_2))] \nend
    &\quad \le 2 L^{\varphi(r)+\xi(r)} \cdot \left(\frac{2}{r^\beta} + \frac{1 + O((\log L)^{-2})}{1 + r/\sqrt{2(\log L+\beta \log r )}} \right) + 2 L^{1 -2 L^{\varphi(r)+\xi(r)} \xi(r)}.
    \label{eq:p-moment-tail-3}
\end{align}
To this end, it is clear that the right choice of $\beta$ should be $1$. 
Here, we recall the definition for $\varphi(r)$ that
\[
    \varphi(r) \ge \frac{\log\left(\sqrt{4\pi \log L} \log (L r^{\alpha})\right)}{\log L} + \frac{2\log(L r^\alpha)}{r\sqrt{2\log L}} - \frac{\log \sqrt{2\pi \log L}}{\log L}, 
\]
and the condition for picking $\xi(r)$ is 
\[
    \varphi(r) + \xi(r) - \frac{\log(\xi(r)^{-1})}{\log L}\ge \frac{\alpha \log r}{(\log L)^2}. 
\]
We redefine $\tilde\varphi(r) = \log L /\log r \cdot \varphi(r) $ and $\tilde\xi(r) = \log L /\log r \cdot \xi(r) $. 
Then we conclude that 
\begin{gather*}
    \tilde \varphi(r) \ge \frac{\log\left(\sqrt{4\pi \log L} \log (L r^{\alpha})\right)}{\log r} + {\frac{\log(L r^\alpha) \sqrt{2\log L}}{r \log r } } - \frac{\log \sqrt{2\pi \log L}}{\log r}, \nend 
    \tilde \varphi(r) + \tilde\xi(r) - \frac{\log(\tilde \xi(r)^{-1} \cdot \log L/\log r)}{\log r} \ge \frac{\alpha }{\log L}
\end{gather*}
To make the third term in \eqref{eq:p-moment-tail-3} $L^{1 -2 L^{\varphi(r)+\xi(r)} \xi(r)}$ also less than $O(r^{-\alpha})$, a sufficient condition is that 
\begin{align*}
    \frac{\log L}{\log r} -  2 r^{\tilde\varphi(r) + \tilde\xi(r)} \tilde\xi(r) \le -\alpha
\end{align*}
To this end, we can see directly that by choosing $\tilde \varphi(r)$ and $\tilde\xi(r)$ to be small constant, say $\epsilon/3$, all the inequalities are satisfied for $r \ge O((\log L)^{3/2})$ (in order to make $\frac{\log(L r^\alpha) \sqrt{2\log L}}{r \log r }$ of the scale $o(1)$). 
In order to make \eqref{eq:p-moment-tail-3} of order $O(r^{-\alpha})$ with $\alpha=1-\epsilon$, we note that the second term dominates and we just need to ensure that 
\[
    r^{-1 + 2\epsilon/3} \sqrt{2\log L} \le r^{-1+\epsilon} \Rightarrow r \ge \sqrt{2\log L}^{3/\epsilon}
    \]
and $\alpha = 1-\epsilon$. 
Therefore, we wrap our proof by concluding that for $r \ge O(\sqrt{2\log L}^{3/\epsilon})$, 
\[  
h(r) = \EE[1 - \norm{p}_2^2] \le O(r^{-(1-\epsilon)}).
\]
\end{proof}

\newpage
\section{Simplification and Approximation of the Gradient Flow Dynamics}
\label[section]{sec:simplify_approximate dynamics}
In this section, we first simplify the gradient flow dynamics of the \ac{msa} to the eigenvalue space under the Decomposability Condition introduced in \Cref{def:decomposability property}. 
Then we approximate the spectral gradient flow dynamics using the results from \Cref{sec: p-moment}.
As a byproduct of the simplification of the gradient flow dynamics, we further prove that the Decomposability Condition is preserved along the gradient flow.

\subsection{Simplification Induced by the Decomposability Condition} \label[appendix]{sec:simplify_AB}
In this section, we aim to show that the gradient flow dynamics in  \eqref{eq: gradient flow dynamics} admit a simplified form under the Decomposability Condition introduced in \Cref{def:decomposability property}. 

Recall that given the parameters $\Theta = \{ O^{(h)}, V^{(h)}, K^{(h)}, Q^{(h) } \}_{h\in [H]}$, the output of the  the \ac{msa} function, is given by \eqref{eq: msa}. 
The gradient flow for minimizing the mean-squared loss $L(\Theta)$ in \eqref{eq: loss} is introduced in \eqref{eq: gradient flow dynamics}, which involves matrices $A^{(h)} \in \RR^{{D} \times {D}}$ and $B^{(h)} \in \RR^{{d_y} \times D} $ defined in \eqref{eq:define_A_h} and \eqref{eq:define_B_h}, respectively. Here $D = d + d_y$ where $d$ and the $d_y$ are the dimensions of the covariate and output, respectively.
Note that these matrices 
depends on the attention score $s^{(h)}$ and attention probability $p^{(h)}$ defined in \eqref{eq:attn_score_prob}, and output weight matrix $U^{(h)}$ given in \eqref{eq:combined_weights}. 
Thus, both $A^{(h)}$ and $B^{(h)}$
are (perhaps complicated) functions of the parameter $\Theta$. 
In the following, we aim to prove that when $\Theta$ satisfies the Decomposability Condition, $A^{(h)}$ and $B^{(h)}$ can be diagonalized using the orthogonal matrices specified by the Decomposability Condition.
The main result can be summarized in the following proposition.

\begin{proposition} [Simplification of $A^{(h)}$ and $B^{(h)}$] 
    \label[proposition]{prop:simplify_AB}
    Let $\Theta = \{ O^{(h)}, V^{(h)}, K^{(h)}, Q^{(h) } \}_{h\in [H]}$  be a set of parameters and let $A^{(h)}$ and $B^{(h)}$ be functions of $\Theta$ defined according to \eqref{eq:define_A_h} and \eqref{eq:define_B_h}, respectively.
    We use subscripts $X$ and $Y$ to indicate the indices that corresponds to the covariate or the response, i.e., $X$ indicate $\{1,\ldots,d\}$ and $Y$ indicates $\{ d+1, \ldots, D\}$. 
    Then, when $\Theta$  satisfy the Decomposability Condition in \Cref{def:decomposability property}, 
    $A^{(h)}$ and $B^{(h)}$ can be written as 
    \begin{align}\label{eq:block_matrix_zero}
        A^{(h)} & = \begin{bNiceMatrix}
            A_{XX}^{(h)} & A_{XY}^{(h)}\\
            A_{YX}^{(h)} &A_{YX}^{(h)}
        \end{bNiceMatrix}  = 
        \begin{bNiceMatrix}
            A_{XX}^{(h)} & \vzero\\
            \vzero & \vzero
        \end{bNiceMatrix}, \qquad  B^{(h)} = \begin{bNiceMatrix}
            B_{X}^{(h)} & 
            B_{Y}^{(h)}\\
        \end{bNiceMatrix}
        = \begin{bNiceMatrix}
            \vzero & 
            B_{Y}^{(h)}\\
        \end{bNiceMatrix} .
    \end{align}
     That is, only the top-left block of $A^{(h)}$ and the right block of $B^{(h)}$ are non-zero.
     Moreover, these two nonzero blocks can be diagonalized by orthogonal matrices $\varPhi$ and $\varPsi$, respectively, where $\varPhi$ and $\varPsi$ are introduced by the Decomposability Condition in \Cref{def:decomposability property}.
     That is, 
     $\varPhi^\top   A_{XX}^{(h)} \varPhi $ and $\varPsi^\top B_{Y}^{(h)} \varPsi$ are diagonal matrices.
     \label[proposition]{fact:diagonality of A}
\end{proposition} 

In the rest of this section, we provide a proof of \Cref{prop:simplify_AB}. Our proof is split into two parts. First, we use the Decomposability Condition to prove the results in \eqref{eq:block_matrix_zero}. 
Then we use the fact that the covariates are Gaussian distributed to show that the non-zero blocks can be diagonalized by the orthogonal matrices $\varPhi$ and $\varPsi$.

\subsubsection{Simplification of $A^\h$ and $B^\h$ under the Decomposability Condition}
\begin{proof}(\emph{Proof of Equation \eqref{eq:block_matrix_zero} in \Cref{prop:simplify_AB}})
We first simplify the expression for $A^\h$ and $B^\h$ in \eqref{eq:define_A_h} and \eqref{eq:define_B_h} under the Decomposability Condition introduced in \Cref{def:decomposability property}, which gives us \eqref{eq:block_matrix_zero}.
\paragraph{Simplification of $A^\h$.}
    By the definition of $A^{(h)}$ in \eqref{eq:define_A_h},  we have
    \begin{align}\label{eq:expand_A_h}
        A^\h &= \underbrace{-\EE\left[Z {P^\h}^\top Z^\top {U^\h}^\top y_q z_q^\top\right]}_{\ds \deleq A^\handzero} + \sum_{h'=1}^H \underbrace{\EE\left[Z {P^\h}^\top Z^\top {U^\h}^\top  U^\hprime Z p^\hprime  z_q^\top\right]}_{\ds \deleq A^\handhprime}
        .
    \end{align}
Recall that $s^\h$ and $p^\h$ in \eqref{eq:attn_score_prob} are the attention score and probability vectors, and $P^\h =  \diag(p^\h) - p^\h {p^\h}^\top$.   
Also recall the definition of $U^\h$ and $W^\h$ in \eqref{eq:combined_weights}. 
By \Cref{cond:U_X and W_Y are zero}, 
we have $W_Y^\h = 0$,
where $W_{Y}^\h $ is the  bottom left block of $W^\h$.
As a result, since the last $d_y$ indices of $z_q$ is zero, we have 
$s^\h =  Z^\top  W^\h  z_q   =  X ^\top W_X ^\h q$.
This implies that $p^\h$ depends only on the covariates $X$ and the query $q$, and thus $p^h, G, \varepsilon$ are independent of each other.

Now we write $Y = G^\top  X +\varepsilon  $ where $\varepsilon \in \RR^{d_y \times L}$ stands for the noise terms in $Y$. Hence, we can write $Z$ as in the following form: 
\begin{align}\label{eq:Z_block_mat}
Z = \begin{bNiceMatrix}
    X \\ Y 
\end{bNiceMatrix}
= \begin{bNiceMatrix}
    I_{d} \\ {G}^\top
\end{bNiceMatrix} X + \begin{bNiceMatrix}
    0 \\ \varepsilon
\end{bNiceMatrix}.
\end{align}
Substituting $Z$ using \eqref{eq:Z_block_mat} into \eqref{eq:expand_A_h}, we can write  $A^\handzero$ as 
\begin{align*}
    A^\handzero &= - \EE\left[
        \left(\begin{bNiceMatrix}
            I_{d} \\ {G}^\top
        \end{bNiceMatrix} X + \begin{bNiceMatrix}
            0 \\ \varepsilon
        \end{bNiceMatrix}\right)
        {P^\h}^\top \left(X^\top 
        \begin{bNiceMatrix}
            I_{d} & {G}
        \end{bNiceMatrix} + \begin{bNiceMatrix}
            0 & \varepsilon^\top
        \end{bNiceMatrix}
        \right)
         {U^\h}^\top {G}^\top q z_q^\top
    \right] \\
    & = -\begin{bNiceMatrix}
        \EE\Bigl[
            X {P^\h}^\top X^\top {G} {U_Y^\h}^\top {G}^\top q q^\top
        \Bigr] & {\vzero}\\
        \EE\Bigl[
            {G}^\top X {P^\h}^\top X^\top {U_X^\h}^\top {G}^\top q q^\top 
        \Bigr] &   {\vzero}
    \end{bNiceMatrix} 
    = \begin{bNiceMatrix}
        -\EE\Bigl[
            X {P^\h}^\top X^\top {G} {U_Y^\h}^\top {G}^\top q q^\top
        \Bigr] & \vzero\\
        \vzero & \vzero
    \end{bNiceMatrix}.
\end{align*}
where in the second equality, 
we use the fact that $G$ and $\varepsilon$ are independent and mean zero, and thus only the terms with second-order moments of $G$ and $\varepsilon$ are p.
Moreover, 
in the last inequality, we note that submatrix $U_X^\h$ is zero by \Cref{cond:U_X and W_Y are zero}.
Thus, we show that $A^\handzero$ is a block-diagonal matrix with only the top left block being nonzero.
The derivation of  such simplification is based on (i) the independence between $p^\h$, $G$ and $\varepsilon$ and (ii) the fact that $U_X^\h$ is zero by \Cref{cond:U_X and W_Y are zero}. 
Here (i) is implied by the fact that $W_Y^\h = 0$, given also by \Cref{cond:U_X and W_Y are zero}. We will use the same argument to simplify $A^\handhprime$ in \eqref{eq:expand_A_h} and $B^\h$ in \eqref{eq:define_B_h}.

For $A^\handhprime$, by direct calculation, we have
\begin{align*}
    A^\handhprime &= \EE\left[
        \left(\begin{bNiceMatrix}
            I_{d} \\ {G}^\top
        \end{bNiceMatrix}
        X + 
        \begin{bNiceMatrix}
            0 \\ \varepsilon
        \end{bNiceMatrix}\right)
        {P^\h}^\top 
        \left(\begin{bNiceMatrix}
            I_{d} \\ {G}^\top
        \end{bNiceMatrix}
        X + 
        \begin{bNiceMatrix}
            0 \\ \varepsilon
        \end{bNiceMatrix}\right)^\top {U^\h}^\top U^\hprime \left(\begin{bNiceMatrix}
            I_{d} \\ {G}^\top
        \end{bNiceMatrix}
        X + 
        \begin{bNiceMatrix}
            0 \\ \varepsilon
        \end{bNiceMatrix}\right) p^\hprime z_q^\top
    \right] \\
    &= 
    \begin{bmatrix}
        \EE\Bigl[X {P^\h}^\top X^\top {U_X^\h}^\top U_X^\hprime X p^\hprime z_q^\top
        + X {P^\h}^\top X^\top {G}  {U_Y^\h}^\top U_Y^\hprime {G}^\top X p^\hprime z_q^\top\Bigr]
            \\
        \EE\Bigl[{G}^\top X {P^\h}^\top X^\top {G} {U_Y^\h}^\top U_X^\hprime X p^\hprime z_q^\top + {G}^\top X {P^\h}^\top X^\top {U_X^\h}^\top U_Y^\hprime {G}^\top X
        p^\hprime z_q^\top\Bigr]
    \end{bmatrix}
     \\
    & \autoquad{2} +
    \begin{bNiceMatrix} 
        \EE\Bigl[X {P^\h}^\top  \varepsilon^\top
        {U_Y^\h}^\top U_Y^\hprime\varepsilon
        p^\hprime z_q^\top \Bigr] 
        \\
        \EE\Bigl[\varepsilon {P^\h}^\top \varepsilon^\top {U_Y^\h}^\top U_X^\hprime X p^\hprime z_q^\top + \varepsilon {P^\h}^\top X^\top {U_X^\h}^\top U_Y^\hprime \varepsilon
        p^\hprime z_q^\top\Bigr]
    \end{bNiceMatrix}.
\end{align*}
Then, using the fact that $p^\h$, $G$, and $\varepsilon$ are independent,  we have
\begin{align*}
    A^\handhprime     
    & = \begin{bNiceMatrix}
        \EE\Bigl[
            X {P^\h}^\top X^\top {G}  {U_Y^\h}^\top U_Y^\hprime {G}^\top X p^\hprime q^\top
            + X {P^\h}^\top  \varepsilon^\top
            {U_Y^\h}^\top U_Y^\hprime\varepsilon
            p^\hprime q^\top
        \Bigr] & \vzero\\
        \vzero & \vzero
    \end{bNiceMatrix}.
\end{align*}

Furthermore,  we define matrices $\AXsignal^\h$, $\AXinterference^\handhprime$, and $\AXnoise^\handhprime$ as follows:
\begin{align}
    \AXsignal^\h &  \deleq -\EE\Bigl[
        X {P^\h}^\top X^\top {G} {U_Y^\h}^\top {G}^\top q q^\top
    \Bigr], \label{eq:define_A_signal}\\
    \AXinterference^\handhprime & \deleq \EE\Bigl[
        X {P^\h}^\top X^\top {G}  {U_Y^\h}^\top U_Y^\hprime {G}^\top X p^\hprime q^\top\Bigr], \label{eq:define_A_intf}\\
    \AXnoise^\handhprime & \deleq \EE\Bigl[
        X {P^\h}^\top  \varepsilon^\top
        {U_Y^\h}^\top U_Y^\hprime\varepsilon
        p^\hprime q^\top
    \Bigr]. \label{eq:define_A_noise}
\end{align}
Here $\AXsignal^\h$ is the signal part, which has a negative sign and leads the gradient flow dynamics to a desired solution. 
Moreover, $\AXinterference^\handhprime$ represents the cross-head interference as it involves both $h$-th and $h'$-th head, and $\AXnoise^\handhprime$ is the noise part as it involves $\varepsilon$. 

Therefore, we prove that only the top left block $A_{XX}^\h$ of $A^\h$ is nonzero, and $A_{XX}^\h$ can be written as 
\begin{align}\label{eq:Axx_final}
    A_{XX}^\h = \AXsignal^\h + \sum_{h'=1}^H \AXinterference^\handhprime + \sum_{h'=1}^H \AXnoise^\handhprime.
\end{align}

\paragraph{Simplification of $B^\h$.}
By the definition of $B^\h$ in \eqref{eq:define_B_h},
we can write $B^\h$ as
\begin{align*}
B^\h &=\! \EE\left[
         \sum_{h'=1}^H U^\hprime Z p^\hprime  {p^\h}^\top \!\! Z^\top
    \right] - \EE\left[
        y_q  {p^\h}^\top \!\! Z^\top
   \right]  .
\end{align*}
Note that $y_q = G^\top q$  and $Y = G^\top X + \varepsilon$. 
Writing in a block matrix form, we have 
\begin{align*}
    B^h &= \sum_{h'=1}^H 
    \begin{bmatrix} U_X^\hprime & U_Y^\hprime\end{bmatrix}\cdot  \begin{bNiceMatrix}
        \EE\bigl[X p^\hprime{p^\h}^\top X^\top \bigr]
        & 
        \EE\bigl[X p^\hprime{p^\h}^\top (X^\top {G} + \varepsilon^\top)\bigr] \notag\\
        \EE\bigl[({G}^\top X + \varepsilon) p^\hprime{p^\h}^\top X^\top\bigr] 
        & 
        \EE\bigl[({G}^\top X + \varepsilon) p^\hprime{p^\h}^\top (X^\top {G} + \varepsilon^\top)\bigr]
    \end{bNiceMatrix} \notag\\
    &\qquad
    - \EE\left[
        {G}^\top q {p^\h}^\top \begin{bNiceMatrix} X^\top  & X^\top {G} + \varepsilon^\top \end{bNiceMatrix}
    \right].
\end{align*}
Then using the independence between $\varepsilon$, $p^\h$, and ${G}$ implied by \Cref{cond:U_X and W_Y are zero} and the fact that $U_X^{(h')} = 0$ for all $h' \in [H]$, we further have 
\begin{align*}
 B^\h & = \sum_{h'=1}^H \begin{bNiceMatrix}
        \vzero & 
        \EE\left[U_Y^\hprime {G}^\top X p^\hprime{p^\h}^\top X^\top {G} + U_Y^\hprime \varepsilon p^\hprime{p^\h}^\top \varepsilon^\top \right] 
    \end{bNiceMatrix} 
    -\begin{bmatrix} \vzero & \EE\left[{G}^\top q {p^\h}^\top X^\top{G}\right] \end{bmatrix} . 
\end{align*}
Now  we define matrices $\BYsignal^\h$, $\BYinterference^\handhprime$, and $\BYnoise^\handhprime$ as follows:
\begin{align}
    \BYsignal^\h \deleq -\EE\left[{G}^\top  q {p^\h}^\top X^\top{G}\right]&, \quad 
    \BYinterference^\handhprime \deleq \EE\left[U_Y^\hprime {G}^\top X p^\hprime{p^\h}^\top X^\top {G}\right], \label{eq:define_B_signal}
    \\
    \BYnoise^\handhprime & \deleq \EE\left[U_Y^\hprime \varepsilon p^\hprime{p^\h}^\top \varepsilon^\top \right], \label{eq:define_B_noise}
\end{align}
and $B^\h$ only has the right block ($Y$ part)  being nonzero, which is given by
\begin{align}\label{eq:by_as_sum}
    B_Y^\h = \BYsignal^\h + \sum_{h'=1}^H \BYinterference^\handhprime + \sum_{h'=1}^H \BYnoise^\handhprime.
\end{align}
In conclusion, we have proved \eqref{eq:block_matrix_zero} in \Cref{prop:simplify_AB}, and the nonzero blocks of $A^\h$ and $B^\h$ are given by \eqref{eq:define_A_signal}--\eqref{eq:define_A_noise},  \eqref{eq:define_B_signal}, and \eqref{eq:define_B_noise}, respectively.
\end{proof}

\subsubsection{Simplification of $A_{XX}^\h$ and $B_{Y}^\h$ with Gaussian Covariate}
\label[section]{sec:simplify AB with Gaussian covariate}

It remains to prove that the nonzero blocks of $A^\h$ and $B^\h$ can be diagonalized by the orthogonal matrices $\varPhi$ and $\varPsi$ introduced by the Decomposability Condition in \Cref{def:decomposability property}. We will prove this fact leveraging the fact that the covariate distribution is Gaussian. 
In the following, for ease of presentation, we introduce some additional notation that will only be used inside this section. 

\paragraph{Additional Notation.}
Consider two heads $ h, h'\in [H]$,  
We locally drop the superscript $(h), (h')$ and write $(\cdot)^\h$ as $(\cdot)$ and $(\cdot)^\hprime$ as $\tilde{(\cdot)}$.
For example, under this notation, we write $A^\h$ as $A$ and $A^\hprime$ as $\tilde A$. 
Similarly, the attention scores $s^\h$ and $s^\hprime$ are written as $s$ and $\tilde s$, respectively. 
Since $W_Y$ and $\tilde W_Y$ are zero matrices, we have $s = X^\top W_X q $ and $\tilde s = X^\top \tilde W_X q$. 
Note that $s, \tilde s \in \RR^{L}$ and we let $s_{l}$ and $\tilde s_{l}$ denote the $l$-th entry of $s$ and $\tilde s$, respectively.
That is, $s_{l} = x_{l}^\top W_X q$ and $\tilde s_{l} = x_{l}^\top \tilde W_X q$.
For any $l \in [L]$, we define an operator  $ \cT_l $ as follows:
\begin{align}\label{eq:define_T_l}
    \cT_l \circ f \deleq \left(W_X q \cdot  \frac{\partial}{\partial {s_l}}  + \tilde W_X q \cdot \frac{\partial}{\partial {\tilde s_l}} \right) \circ f = W_X q \cdot \frac{\partial f}{\partial {s_l}}  + \tilde W_X q \cdot \frac{\partial f}{\partial {\tilde s_l}}, 
\end{align}
where $f$ is any  real-valued function of the $l$-th attention scores $s_l$ and $\tilde s_l$.
In particular, $W_X \in \RR^{d\times d}$, $q \in \RR^d$, and $\partial f/\partial {s_l} \in \RR$.   
Note that after applying the operator to $f$, the output is a $d$-dimensional vector-valued function of $s_l$ and $\tilde s_l$.
The intuition of the operator $\cT_l$ is as follows. When $f$ is a function of the attention scores $s$ and $\tilde s$,  it is a function of $X$, i.e., $\{ x_l\}_{l\in [L]}$. 
For any fixed $l$, $x_l$ appears in $s$ and $\tilde s$ only through $s_l$ and $\tilde s_l$.
We can prove that $\nabla_{x_{l}} f = \cT_{l} \circ f$ by using the chain rule. 
Furthermore, for regularity, we  assume that $f$ is arbitrary-order differentiable and the derivatives have bounded expectations with respect to $x_l\iidfrom \mathcal{N}(0, I_{d})$.
When it is clear from the context, we write the partial derivatives $\partial /\partial_{s_l} $ and $\partial /\partial_{\tilde s_{l}} $ as $\partial_l$ and $\tilde\partial_l$, respectively.

Additionally, we will encounter the composition of operators $\{\cT_l\}$. 
To this end, we extend the definition in \eqref{eq:define_T_l} to vector-valued functions. Specifically, let $F = \{ F_{i} \}_{i\in [m]} : \RR^{L} \times \RR^{L} \rightarrow \RR^{m}$ be a mapping that maps $s$ and $\tilde s$ to $\RR^{m}$. 
Then we define $\cT_{l} \circ F$ as a matrix in $\RR^{   d \times m}$ where the $i$-th row is given by $\cT_{l} \circ F_{i}$. Then the composition     $(\cT_{l} \circ(  \cT_{m} \circ f) )$  can be viewed as a  matrix-valued function taking values in $\RR^{d\times d}$, and  $(\cT_{l} \circ(  \cT_{m} \circ (\cT_n \circ  f) ) )$ can be viewed as a function taking values in third-order tensor space $\RR^{d\times d\times d}$. 
To simplify the notation, we denote them by $(\cT_{l} \otimes    \cT_{m}) \circ f  $ and $(\cT_{l} \otimes   \cT_{m} \otimes \cT_n  ) \circ  f$, respectively.


Meanwhile, note that 
when the parameter $\Theta$ satisfies the Decomposability Condition,
$W_X$ and $\tilde W_X$ can be diagonalized by $\Phi$. 
That is, we have 
 $\tilde W_X = \varPhi \diag(\tilde\omega) \varPhi^\top$ and $W_X = \varPhi \diag (\omega) \varPhi^\top$. 
 We let  $\cleverbar q = \varPhi^\top q $ denote the rotated query vector and define a operator $\cleverbar\cT_l$ by letting  
\begin{align}\label{eq:def_barT_oper}
    \cleverbar\cT_l \circ f \deleq \left(\diag(\omega) \cdot  \partial_l  + \diag(\tilde\omega)  \cdot  \tilde\partial_l \right) \circ f  = \diag(\omega) \cdot \partial_l f + \diag(\tilde\omega) \cdot \tilde\partial_l f,
\end{align}
for any function $f$ of $(s, \tilde s)$. 
It is clear that $\cleverbar\cT_l \circ f$ gives a diagonal matrix as $\diag(\omega)$ is a diagonal matrix and $\partial _{l} f$ is a scalar.
Moreover, $\cT_l $ and $\cleverbar\cT_l $ are connected via the property that  
\begin{align}\label{eq:relation_T_Tbar}
 \cT_l \circ f = \varPhi \diag(\omega) \varPhi^\top  q \cdot \partial_{l} f + 
 \varPhi \diag(\tilde \omega) \varPhi^\top  q \cdot \tilde\partial_{l} f =  
 \varPhi (\cleverbar\cT_l \circ f) \cleverbar q. \end{align}
We will also encounter the compositions of operators $\{\cleverbar\cT_l\}$, which takes a rather simple form.   
Specifically, we have 
\begin{align*}
& (\cleverbar\cT_{l} \circ(  \cleverbar\cT_{m} \circ f) ) =  (\cleverbar\cT_{m} \circ(  \cleverbar\cT_{l}  \circ f))  \notag \\
& \qquad = 
\diag(\omega^{\odot 2}) \cdot \partial_{l} \partial_{m} f + \diag(\tilde\omega^{\odot 2}) \cdot \tilde\partial_{l} \tilde\partial_{m} f + 
\diag(\omega\odot  \tilde\omega) \cdot ( \partial_{l} \tilde \partial_{m} f   +  \tilde\partial_{l}  \partial_{m} f) .
\end{align*}
Note that $(\cleverbar\cT_{l} \circ(  \cleverbar\cT_{m} \circ f) )$ is a diagonal matrix. 

In addition, in the following, to simplify $A^\h$ and $B^\h$ using the Gaussian variate distribution, 
we will recursively apply the Stein's Lemma, will yields expectations of tensor-valued random variables. 
For ease of presentation, we introduce the following notation for handling third-order tensors. 
In particular, suppose $f = \{ f_{lmn}\}_{l, m, n \in [L]}$ is a third-order tensor-value function of $(s, \tilde s)$ indexed by $(l, m, n) \in [L]\times [L] \times [L]$, we define 
\begin{align}\label{eq:define_bracket}
    \dsbr{f_{lmn}} \deleq \sum_{l', m', n' \in [L]} \EE[f_{l' m'n'}\given q],
\end{align} 
where the expectation is taken with respect to the randomness of covariate $X$. 
That is, $\dsbr{f_{lmn}}$ denotes the sum of entries of the expectation of the tensor $\{ f_{lmn}\} $. Here we hide the dependency on $q$ for simplicity.
Note that the conditional expectation in \eqref{eq:define_bracket}  takes the same value   if we instead condition on $\cleverbar q$. 
Besides, we can similarly define $\dsbr{f_{lm}}$ and $\dsbr{f_{l}}$ for second-order tensor-value function and vector-value function of $(s, \tilde s)$, respectively.
Furthermore, for a third-order tensor $M$, we let $M^{\top(i_1 i_2 i_3)}$ denote the $(i_1, i_2, i_3)$ transpose of $M$.
For instance, $M^{\top(132)}$ represents the transpose of $M$ by swapping the second and the third dimension.

\paragraph{Gaussian Moments.} A nice property of the Gaussian distribution is the Stein's Lemma, which states that $\EE[ x  g(x)] = \EE[  \nabla  g(x)]$ for any function $g$ under some regularity condition. Here the expectation is taken with respect to $x \sim N(0, I_{d})$. 
Note that $A_{XX}^\h$ and $B_Y^\h$ are functions $X$. We will leverage the Stein's Lemma to explicitly compute the expectations with respect to $X$ in $A_{XX}^\h$ and $B_Y^\h$. 
In particular, we will resort to the following lemma.

\begin{lemma}  \label[lemma]{fact:expectation terms}
    Suppose
    for any $l\in [L]$,
    the $l$-th covariate $x_l$ in the \ac{icl} context $Z$ is drawn from a standard normal distribution, i.e., $x_l\iidfrom \mathcal{N}(0, I_{d})$. 
    In addition, let  $ \{ f_{lmn}\}_{l, m, n \in [L]}$ be a third-order tensor-value function of $(s, \tilde s)$ indexed by $(l, m, n) \in [L]\times [L] \times [L]$. 
    Similarly, let $\{ f_{lm}\}_{l, m \in [L]}$ be a second-order tensor-value function of $(s, \tilde s)$ indexed by $(l, m) \in [L]\times [L]$, and let $\{ f_{l}\}_{l \in [L]}$ be a vector-value function of $(s, \tilde s)$ indexed by $l \in [L]$.
    Then, we have the following results. 
    \begin{myenumi}[
        label=(\roman*),
        ref=\Cref{fact:expectation terms}(\roman*),
    ]\abovedisplayskip=-1\baselineskip
     \belowdisplayskip=-1\baselineskip
        \item [(i)]  \label{fact:expectation-order3}
        \begin{flalign*}
            \dsbr{f_{lmn} \cdot x_l\otimes x_m\otimes x_n }
            &=  \dsbr{ \cT_l\circ(f_{lmm}) \otimes I_{d}} 
            + \dsbr{ I_{d} \otimes \cT_l\circ(f_{mml}) } 
             &\nend 
            &\qquad + \dsbr{ I_{d} \otimes \cT_l\circ(f_{mlm})}^{\top(132)}
            + \dsbr{(\cT_l \otimes \cT_m \otimes \cT_n ) \circ (f_{lmn})}. &
        \end{flalign*}
        \item [(ii)] \label{fact:expectation-order2}
        $
            \dsbr{f_{lm} \cdot x_l \otimes x_m} = \dsbr{(\cT_l\otimes \cT_m)\circ f_{lm}} + \dsbr{f_{ll} \cdot I_{d}}
        $. 
        \item[(iii)] \label{fact:expectation-order1}
        $
            \dsbr{f_l x_l} = \dsbr{\cT_l \circ f_l}.
        $
    \end{myenumi}
    \normalsize
\end{lemma}
\begin{proof} (\emph{Proof of \Cref{fact:expectation terms}})
See \S\ref{proof:fact:expectation terms} for a detailed proof.
\end{proof}

Now we are ready to simplify $A^\h$ and $B^\h$ by showing that the nonzero blocks of $A^\h$ and $B^\h$ can be diagonalized by the orthogonal matrices $\varPhi$ and $\varPsi$, respectively.

\subsubsection*{Simplification of $A_{XX}^\h$ with Gaussian Covariate}
Recall that we show in \eqref{eq:Axx_final} that $A_{XX}^\h$ is equal to a sum of three terms. 
We first consider $\AXsignal^\h$ defined  in \eqref{eq:define_A_signal}. 
By the second identity of  Lemma \ref{fact:expectation terms},  we have
\begin{align*}
    \AXsignal^\h &= -\EE\Bigl[
        X {P}^\top X^\top {G} U_Y^\top {G}^\top q q^\top
    \Bigr] \nend 
    &= -\EE\left[
        \left(\dsbr{(\cT_l\otimes \cT_m)\circ P_{ml}} + \dsbr{P_{ll} \cdot I_{d}}\right)
        {G} {U_Y ^\top} {G}^\top q q^\top 
    \right] ,
    \mytag{\Cref{fact:expectation-order2}}
\end{align*}
where the expectation in the second equality is taken with respect to the randomness of $G$ and $q$. 
Here we omit the superscript $(h)$ for simplicity.
Besides,  we write  $$\dsbr{(\cT_l\otimes \cT_m)\circ P_{ml}} = \sum_{l, m}  \EE[ (\cT_l\otimes \cT_m)\circ P_{ml} \given q ], \qquad \dsbr{P_{ll} \cdot I_{d}} = \sum_{l}  \EE[  P_{ll} \cdot I_{d} \given q ].$$  
By the definition of $\cT_l$ in \eqref{eq:define_T_l} and noticing that $P_{ml}$ is a function of $s$ only, we have 
\begin{align*}
(\cT_l\otimes \cT_m)\circ P_{ml}  & = W_X q q^\top W_X ^\top \cdot \partial _{l} \partial_m P_{ml} = \varPhi \,\diag(\omega) \,\varPhi^\top  q  q^\top \,\varPhi \,\diag(\omega) \varPhi^\top  \cdot \partial _{l} \partial_m P_{ml}  \\
&  = \varPhi \,\diag(\omega) \,\cleverbar q \, \cleverbar q^\top \,\diag(\omega) \,\varPhi^\top \cdot \partial _{l} \partial_m P_{ml} \\
& =  \varPhi \bigl (\cleverbar \cT_l \circ ( \cleverbar q\,\cleverbar q^\top \cleverbar\cT_m \circ P_{ml})) \varPhi^\top \defeq  \varPhi\,  (\cleverbar \cT_l \, \cleverbar q \,\cleverbar q^\top \cleverbar\cT_m) \circ P_{ml}  \,  \varPhi^\top,
\end{align*}
where in the first equality we use the fact that $\tilde \partial _{l} P_{ml} = 0$ since $P_{ml}$ does not depend on $\tilde s$. 
In the third equality, we use the fact that $\cleverbar q = \varPhi^\top q $. 
In the fourth equality  we apply the definitions of $\cleverbar\cT_l$ and $\cleverbar\cT_m$. Moreover, we let $(\cleverbar \cT_l \, \cleverbar q \,\cleverbar q^\top \cleverbar\cT_m) \circ f$ denote 
the operator composition $\cleverbar \cT_l \circ ( \cleverbar q\,\cleverbar q^\top \cleverbar\cT_m \circ f) $ for any function $f$.

Furthermore, notice that both $\dsbr{(\cT_l\otimes \cT_m)\circ P_{ml}} $ and $ \dsbr{P_{ll} \cdot I_{d}}$ are functions of $q$ only, and thus are independent of $G$. Thus, we can also take conditional expectation of $G U_Y^\top G^\top $ given $q$ in $\AXsignal^\h$.  We introduce the following lemma that computes such an expectation.

\begin{lemma}
    \label[lemma]{fact:beta M beta}
    Consider ${G} = 1/ \sqrt{d}\cdot  \varPhi \cleverbar{G} \varPsi^\top$ with $\varPhi\in \OO^{d}$, $\varPsi\in \OO^{d_y}$, and $ \cleverbar{G}$ is a random matrix in $\RR^{d\times d_y}$. Let   $M\in\RR^{d_y\times d_y}$ and $N\in\RR^{d\times d}$ be two fixed matrices.
    Suppose the $d\times d_y$ random matrix $\cleverbar{G}$ has independent  entries with $\EE[\cleverbar{G}]=0$ and $\EE[\cleverbar{G}^{\odot 2}]=\varLambda$. It then holds that
    \begin{align*}
        \EE[{G} M {G}^\top] =  1/d \cdot \varPhi \: \diag\big (\varLambda \: \Diag(\varPsi^\top M \varPsi) \big) \varPhi^\top, \: 
        \EE[{G}^\top N {G}] = 1/d \cdot  \varPsi \diag\big( \varLambda^\top \: \Diag(\varPhi^\top N \varPhi)  \big) \varPsi^\top.
    \end{align*}
    Here $\Diag(\cdot)$ denotes the vector that consists the digonal entries of a matrix, and $\diag(\cdot)$ denotes the diagonal matrix created from a vector. 
\end{lemma}

This lemma can be computed by direct computation and we omit the proof for brevity. 
Besides, to simplify the notation, 
we define 
\[
    \operYtoXbeta\circ M \deleq 1/ d \cdot  \diag\big(\varLambda \: \Diag(\varPsi^\top M \varPsi) \big), \quad 
    \operXtoYbeta\circ N \deleq 1/ d \cdot  \diag\big( \varLambda^\top \: \Diag(\varPhi^\top N \varPhi) \big).
\] 
as operators that map a matrix $M\in\RR^{d_y\times d_y}$ or $N\in\RR^{d\times d}$ to another diagonal matrix.
In particular, if $M=\varPsi \cleverbar M \varPsi$ takes $\varPsi$ as its eigenvector matrix and $\cleverbar M$ as a diagonal matrix, then $\operYtoXbeta\circ M = \diag(\varLambda \vec(\cleverbar M))/d$.
By this lemma, we have 
$
\EE [{G} U_Y^\top {G}^\top \given q ] =  \operYtoXbeta\circ U_Y^\top. 
$
Thus, we  further have 
\begin{align}\label{eq:AXsignal_diag}
    \AXsignal^\h
    &= -\EE\left[
        \left( \varPhi\, \dsbr{(\cleverbar \cT_l \, \cleverbar q \,\cleverbar q^\top \cleverbar\cT_m ) \circ P_{ml}}  \varPhi^\top + \dsbr{P_{ll}  \cdot I_{d}} \right)
        \varPhi \bigl(\operYtoXbeta\circ U_Y^\top\bigr) \varPhi^\top q  q^\top 
    \right] \nend 
    &= -\varPhi\, \EE\left[
        \left(\dsbr{(\cleverbar \cT_l \, \cleverbar q \,\cleverbar q \cleverbar\cT_m ) \circ P_{ml}} + \dsbr{P_{ll}  \cdot I_{d}} \right)
        (\operYtoXbeta\circ U_Y^\top)\, \cleverbar q \, \cleverbar q^\top 
    \right] \varPhi^\top.  
\end{align}
Note that $   (\cleverbar \cT_l \, \cleverbar q \,\cleverbar q \cleverbar\cT_m) \circ P_{ml}   $, $P_{ll}  \cdot I_{d}$, and $\operYtoXbeta\circ U_Y^\top$ are all diagonal matrices.  Hence, $\varPhi^\top \AXsignal^\h \varPhi $ is equal to the expectation of the product of a diagonal matrix and a rank-one matrix $\cleverbar q \, \cleverbar q^\top$, where the diagonal matrix is a function of $\cleverbar q$.

We now introduce a 
key observation, which is a direct result of the inner product structure of the attention scores and the rotation invariance of the distribution of the covariate $x_l$. 
\begin{lemma}
    \label[lemma]{fact:p is a function of Wq's 2-norm}
    Consider $g(s, \tilde s)$ as an arbitrary function of the attention scores $s$ and $\tilde s$.
    Then,  $\EE[g(s, \tilde s)\given q]$, viewed as a function of $\cleverbar q$, is a function of only $\langle \omega^{\odot 2}, \cleverbar q^{\odot 2} \rangle$, $ \langle \tilde\omega^{\odot 2}, \cleverbar q^{\odot 2}\rangle$, and $\langle \omega \odot \tilde\omega, \cleverbar q^{\odot 2}\rangle$. 
    Here $ \omega$ and $\tilde\omega$ are the eigenvalues of $W_X$ and $\tilde W_X$, respectively.  Here $\omega^{\odot 2}$ denotes the elementwise square of $\omega$, i.e., $\omega^{\odot 2}$ is a vector in $\RR^{d}$ with the $i$-th entry being $\omega_i^2$.
\end{lemma}
\begin{proof}(\emph{Proof of \Cref{fact:p is a function of Wq's 2-norm}})
    See \S\ref{proof:fact:p is a function of Wq's 2-norm} for a detailed proof.
\end{proof}

We remark that \Cref{fact:p is a function of Wq's 2-norm} shows that when $f$ is a  function that simultaneously depends on two attention scores $s$ and $\tilde s$, the conditional expectation $\EE[f(s, \tilde s)\given q]$ is a function of only the second-order moments of $\cleverbar q$, i.e., $\{\cleverbar q_j^2 \}_{j\in[d]}$. 
This fact is a result of  the rotational invariance of the covariate distribution. 
This lemma directly implies that $\varPhi^\top \AXsignal^\h \varPhi $ is a diagonal matrix, which is shown in the following lemma.

\begin{lemma} \label[lemma]{lemma_diagonal_qqt}
    Let $g(\barq) = \prod_{i=1}^d \barq_i^{c_i}$   be a polynomial of $\barq$ with $c_i \in \NN$. 
    Then, if $\sum_{i=1}^d c_i$ is odd,
    we have $\EE[g(\barq)\given \barq^{\odot 2}]=0$. 
    In particular, 
    $\EE [ \barq \barq^\top\given \barq^{\odot 2}]$ is a diagonal matrix.
\end{lemma}
\begin{proof} (\emph{Proof of \Cref{lemma_diagonal_qqt}})
Note that when conditioned on $\barq^{\odot 2}$, each entries of $\barq$ are $\iid$ and takes values $-|\barq_i|$ and $|\barq_i|$ with equal probability.
Hence, by symmetry, we have $\EE[g(\barq)\given \barq^{\odot 2}]=0$ if $\sum_{i=1}^d c_i$ is odd.
Therefore, we prove this lemma. 
\end{proof}

Applying \eqref{eq:AXsignal_diag} and \Cref{lemma_diagonal_qqt} to $\AXsignal^\h$, we have
\begin{align*}
    \AXsignal^\h 
    & = -\varPhi\, \EE\left[
        \left(\dsbr{(\cleverbar \cT_l \, \cleverbar q \,\cleverbar q^\top \cleverbar\cT_m ) \circ P_{ml}} + \dsbr{P_{ll}  \cdot I_{d}} \right)
        (\operYtoXbeta\circ U_Y^\top)\, \cleverbar q \, \cleverbar q^\top 
    \right] \varPhi^\top \nend 
    & = -\varPhi\, \EE\left[\EE\left[
         \dsbr{(\cleverbar \cT_l \, \cleverbar q \,\cleverbar q^\top \cleverbar\cT_m ) \circ P_{ml} \cdot (\operYtoXbeta\circ U_Y^\top)\, \cleverbar q \, \cleverbar q^\top}  + \dsbr{P_{ll}  \cdot I_{d}} (\operYtoXbeta\circ U_Y^\top)\, \cleverbar q \, \cleverbar q^\top 
        \Biggiven \barq^{\odot 2}
    \right] \right] \varPhi^\top. 
\end{align*}
Note that the second term is a diagonal matrix by directly applying \Cref{lemma_diagonal_qqt}.
We note that the first term is also a diagonal matrix since $\cleverbar q^\top \barcT_m  \circ P_{ml} (\operYtoXbeta \circ U_Y^\top) \barq = \langle \Diag(\barcT_m^\top \circ P_{ml} (\operYtoXbeta \circ U_Y^\top)), \barq^{\odot 2}\rangle$ 
is a real number and only depends on $\barq^{\odot 2}$.
The remaining term $\EE[\barq \barq^\top\given \barq^{\odot 2}]$ clearly gives us a diagonal matrix.
Hence, we claim that $\AXsignal^\h$ is diagonal.
We will leverage the same technique to simplify $\AXinterference^\handhprime$ and $\AXnoise^\handhprime$, respectively.

For $\AXinterference^\handhprime$, by the first  identity of   \Cref{fact:expectation terms},  we have
\begin{align*}
    \AXinterference^\handhprime &= \EE\Bigl[
        X {P^\h}^\top X^\top {G}  {U_Y^\h}^\top U_Y^\hprime {G}^\top X p^\hprime q^\top\Bigr] \nend 
        &= \EE\Bigl[  \Bigl(\bigl(
            \dsbr{\cT_l \circ (P_{ml}\tilde p_m)} \otimes I_{d}  +  I_{d} \otimes \dsbr{\cT_l \circ (P_{mm}\tilde p_l)} + (I_{d}\otimes \dsbr{\cT_l \circ (P_{lm}\tilde p_m)})^{\top(132)}   \nend 
            &\qqquad   + \dsbr{(\cT_l \otimes \cT_m \otimes \cT_n ) \circ (P_{ml}\tilde p_n)}
        \bigr) \cddot \bigl({G} \tilde U_Y^\top U_Y {G}^\top\bigr)\Bigr) \otimes q
        \Bigr].
\end{align*}
Recall that ``$\cddot$'' denotes the double inner product of two tensors. 
In particular, we have the following relationships for the double inner product. 

\begin{lemma} \label[lemma] {lemma:tensor_double_inner}
    Let $N, M\in\RR^{d\times d}$ be two matrices and  $u, v, w, x\in\RR^{d }$ be four vectors. 
    Then we have
    \begin{align*}
        v \otimes N \cddot M \otimes u  = vu^\top \cdot \trace(NM),&  \qquad (I_{d} \otimes v) \cddot M = (v^\top M)^\top, \\
        (I_{d}\otimes v)^{\top(132)} \cddot M = M v, & \qquad  u\otimes v \otimes w \cddot M \otimes x = (w^\top M v) u x^\top .
    \end{align*}
\end{lemma}

Combining the first identity in  \Cref {lemma:tensor_double_inner}  and  \Cref{fact:beta M beta}  to the expression of $\AXinterference^\handhprime$, we have
\begin{align} \label{eq:decompose_double_inner1}
  &  \EE \big [  \dsbr{\cT_l \circ (P_{ml}\tilde p_m)} \otimes I_{d} \cddot \bigl({G} \tilde U_Y^\top U_Y {G}^\top \bigr) \otimes q  \bigr ]  \notag\\
   &\qquad   =   \EE \big [ \dsbr{\cT_l \circ (P_{ml}\tilde p_m)} q^\top \cdot \trace\bigl(\varPhi \operYtoXbeta \circ(\tilde U_Y^\top U_Y) \varPhi^\top\bigr) \bigr ]  \notag \\
   &\qquad  =  \varPhi \, \EE \big [ \dsbr{\cleverbar\cT_l \circ (P_{ml}\tilde p_m)} \,\cleverbar q\, \cleverbar q^\top \trace\bigl( \operYtoXbeta \circ(\tilde U_Y^\top U_Y) \bigr) \bigr ] \varPhi^\top,  
\end{align}
where $\cleverbar q = \varPhi^\top q $. 
Besides, 
using the second identity in  \Cref {lemma:tensor_double_inner}  and  \Cref{fact:beta M beta}, we have 
\begin{align} \label{eq:decompose_double_inner2}
    &  \EE \big [  I_{d} \otimes \dsbr{\cT_l \circ (P_{mm}\tilde p_l)} \cddot \bigl({G} \tilde U_Y^\top U_Y {G}^\top \bigr) \otimes q  \bigr ]  \notag \\
    &\qquad   =   \EE \big [ \varPhi \operYtoXbeta \circ(\tilde U_Y^\top U_Y) \varPhi^\top \dsbr{\cT_l \circ (P_{mm}\tilde p_l)} q^\top  \bigr ]  \notag \\
    &\qquad  =  \varPhi \,\EE \big [  \operYtoXbeta \circ(\tilde U_Y^\top U_Y)  \dsbr{\cleverbar \cT_l \circ (P_{mm}\tilde p_l)}\cleverbar q  \cleverbar q^\top  \bigr ]\varPhi   , 
\end{align} 
where the second equality follows from \eqref{eq:relation_T_Tbar}.
Note   that $\dsbr{\cT_l \circ (P_{mm}\tilde p_l)}$ is a vector in $\RR^{d}$ and $\operYtoXbeta\circ(\tilde U_Y^\top U_Y)$ is a diagonal matrix in $\RR^{d\times d}$. 
Moreover,  by the third identity in  \Cref {lemma:tensor_double_inner}  and  \Cref{fact:beta M beta}, we have
\begin{align} \label{eq:decompose_double_inner3}
    &  \EE \big [  I_{d} \otimes \dsbr{\cT_l \circ (P_{lm}\tilde p_m)}^{\top(132)} \cddot \bigl({G} \tilde U_Y^\top U_Y {G}^\top \bigr) \otimes q  \bigr ]  \notag\\
    &\qquad   =   \EE \big [ \varPhi \operYtoXbeta \circ(\tilde U_Y^\top U_Y) \varPhi^\top \dsbr{\cT_l \circ (P_{lm}\tilde p_m)} q^\top  \bigr ]  \notag \\
    &\qquad  =  \varPhi \,\EE \big [  \operYtoXbeta \circ(\tilde U_Y^\top U_Y)  \dsbr{\cleverbar \cT_l \circ (P_{lm}\tilde p_m)}\cleverbar q  \cleverbar q^\top  \bigr ]\varPhi   , 
\end{align}
where the second equality follows from \eqref{eq:relation_T_Tbar}.
Finally, note that by the definition of $ \cT_l$ in \eqref{eq:define_T_l},
  $ \cT_l \otimes \cT_m \otimes \cT_n \circ (P_{ml}\tilde p_n)$ is a rank-one third-order tensor. By 
\Cref{fact:beta M beta} and the last identity in \Cref {lemma:tensor_double_inner}, we have
\begin{align}\label{eq:decompose_double_inner4}
    &  \EE \big [  \dsbr{\cT_l \otimes \cT_m \otimes \cT_n \circ (P_{ml}\tilde p_n)} \cddot \bigl({G} \tilde U_Y^\top U_Y {G}^\top \bigr) \otimes q  \bigr ]  \notag\\
    &\qquad   =   \EE \Big [  \dsbr{\bigl( \cT_l   \cT_n^\top \bigl(\operYtoXbeta \circ(\tilde U_Y^\top U_Y) \bigr)  \cT_m \bigr) \circ (P_{ml}\tilde p_n)}    q^\top   \Bigr ]  \notag \\
    &\qquad  =  \varPhi \,\EE \Big [  \dsbr{\bigl(\cleverbar\cT_l \cleverbar q \cleverbar q^\top \cleverbar\cT_n \bigl(\operYtoXbeta \circ(\tilde U_Y^\top U_Y) \bigr) \cleverbar\cT_m \bigr) \circ (P_{ml}\tilde p_n)} \cleverbar q \cleverbar q^\top   \Bigr ] \varPhi   ,
\end{align}
where   the second equality is obtained by applying    \eqref{eq:relation_T_Tbar} for three times.

Thus, combining \eqref{eq:decompose_double_inner1}--\eqref{eq:decompose_double_inner4},  we have for $\AXinterference^\handhprime$ that
\begin{align*}
    \AXinterference^\handhprime
    &
    =\varPhi \, \EE\left[ 
        \dsbr{\cleverbar\cT_l \circ (P_{ml}\tilde p_m)} \,\cleverbar q\, \cleverbar q^\top \trace\bigl( \operYtoXbeta \circ(\tilde U_Y^\top U_Y) \bigr) 
        +  \operYtoXbeta \circ(\tilde U_Y^\top U_Y) \dsbr{\cleverbar\cT_l \circ (P_{mm}\tilde p_l)} \,\cleverbar q\, \cleverbar q^\top \right.\nend    
    &\qquad  \left.
        + \operYtoXbeta \circ(\tilde U_Y^\top U_Y) \dsbr{\cleverbar\cT_l \circ (P_{lm}\tilde p_m)}  \cleverbar q  \cleverbar q^\top 
        + \dsbr{\bigl(\cleverbar\cT_l \cleverbar q \cleverbar q^\top \cleverbar\cT_n \bigl(\operYtoXbeta \circ(\tilde U_Y^\top U_Y) \bigr) \cleverbar\cT_m \bigr) \circ (P_{ml}\tilde p_n)} \cleverbar q \cleverbar q^\top 
    \right] \varPhi^\top.
\end{align*}
Here, we treat  each $\cT_l$ as a vector and $\cleverbar\cT_l$ as a diagonal matrix.
Note that $\operYtoXbeta\circ(\tilde U_Y^\top U_Y)$ is a diagonal matrix. 
Thus, the terms with $\dsbr{\cdot }$ in \eqref{eq:decompose_double_inner1}, \eqref{eq:decompose_double_inner2}, and \eqref{eq:decompose_double_inner3} are all diagonal matrices.
Moreover, they are functions of $s$ and $\tilde s$. Thus, for these three terms, by  \Cref{fact:p is a function of Wq's 2-norm} and \Cref{lemma_diagonal_qqt}, we can first take a conditional expectation with respect to $\cleverbar q^{\odot 2}$ and prove that they are all diagonal matrices. 
Moreover, in \eqref{eq:decompose_double_inner4}, the term 
$\cleverbar\cT_n \bigl(\operYtoXbeta \circ(\tilde U_Y^\top U_Y) \bigr) \cleverbar\cT_m  $ is a diagonal matrix as it is the product of three diagonal matrices.
As a result, 
${\textstyle \dsbr{ (\cleverbar\cT_l \cleverbar q \cleverbar q^\top \cleverbar\cT_n (\operYtoXbeta \circ(\tilde U_Y^\top U_Y)  ) \cleverbar\cT_m  ) \circ (P_{ml}\tilde p_n)} \cleverbar q \cleverbar q^\top}$ involves fourth-order moments of $\cleverbar q$.  By applying  \Cref{fact:p is a function of Wq's 2-norm}  and   \Cref{lemma_diagonal_qqt}, and taking conditional expectation given $\cleverbar q^{\odot 2}$, we can similarly argue that this term is a diagonal matrix.
Therefore, we conclude that $\AXinterference^\handhprime$ is a diagonal matrix.

Finally, for $\AXnoise^\handhprime$, by the last   identity of   \Cref{fact:expectation terms} and the first identity of \Cref{lemma:tensor_double_inner},  we have
\begin{align*}
    \AXnoise^\handhprime 
    &= \EE\Bigl[
        X P^\top  \varepsilon^\top
        U_Y^\top \tilde U_Y\varepsilon
        \tilde p q^\top
    \Bigr] = \EE\left[
        \left(\dsbr{x_l \otimes \varepsilon_m \otimes \varepsilon_n \cdot (P_{ml}\tilde p_n)} \cddot 
        \tilde U_Y U_Y^\top\right)  q^\top
    \right] \nend
    & = \sigma^2 \EE\Bigl[\trace\bigl(\tilde U_Y U_Y^\top\bigr)  \cdot 
        \dsbr{ (P_{ml}\tilde p_m) \cdot x_{l} } \cdot  
        q^\top  
    \Bigr]   = \sigma^2 \EE\Bigl[\trace\bigl(\tilde U_Y U_Y^\top\bigr)  \cdot 
        \dsbr{\cT_l \circ (P_{ml}\tilde p_m)} 
        q^\top
    \Bigr] 
    \\
    &= \sigma^2 \cdot \varPhi \, \EE\Bigl[\trace(\tilde U_Y U_Y^\top) \cdot  \dsbr{\cleverbar\cT_l \circ (P_{ml}\tilde p_m)} 
         \,\cleverbar q\,\cleverbar q^\top
    \Bigr] \varPhi^\top, 
\end{align*}
where the second  equality follows from direct computation; in the  third equality, 
we use the fact that $\varepsilon_l \iidfrom \mathcal{N}(0, \sigma^2 I_{d_y})$ and the fact that $I \cddot M = \trace(M)$; in the fourth equality we apply \Cref{fact:expectation terms}; the last equality follows from \eqref{eq:relation_T_Tbar}.
Similarly, applying  \Cref{fact:p is a function of Wq's 2-norm}  and and \Cref{lemma_diagonal_qqt}, we prove that $\AXnoise^\handhprime$ is a diagonal matrix.

In summary, we have for $A_{XX}^\h$ that
\begin{align*}
    \varPhi^\top \!\! A_{XX}^{(h)} \varPhi 
    &= \underbrace{-\EE\left[
        \left(\dsbr{(\cleverbar \cT_l  \cleverbar q  \cleverbar q^\top \cleverbar\cT_m^\top) \circ P_{ml}} + \dsbr{P_{ll} I_{d}}\right)
        (\operYtoXbeta\circ U_Y^\top)  \cleverbar q   \cleverbar q^\top 
    \right]}_{\dr Signal}  \notag \\
    & \qquad+  \underbrace{\sum_{h'=1}^H \sigma^2 \EE\Bigl[\trace(\tilde U_Y U_Y^\top) \dsbr{\cleverbar\cT_l \circ (P_{ml}\tilde p_m)} 
     \cleverbar q \cleverbar q^\top
    \Bigr]}_{\dr Noise}\nend 
    &\qquad + \sum_{h'=1}^H \EE\Bigl[ 
        \dsbr{\cleverbar\cT_l \circ (P_{ml}\tilde p_m)} \,\cleverbar q\, \cleverbar q^\top \trace\bigl( \operYtoXbeta \circ(\tilde U_Y^\top U_Y) \bigr) 
        +  \operYtoXbeta \circ(\tilde U_Y^\top U_Y) \dsbr{\cleverbar\cT_l \circ (P_{mm}\tilde p_l)} \,\cleverbar q\, \cleverbar q^\top \nend    
    &\autoquad{2} 
        \underbrace{\quad + \operYtoXbeta \circ(\tilde U_Y^\top U_Y) \dsbr{\cleverbar\cT_l \circ (P_{lm}\tilde p_m)} \cleverbar q \cleverbar q^\top 
        + \dsbr{\bigl(\cleverbar\cT_l   \cleverbar q \cleverbar q^\top \cleverbar\cT_n^\top \bigl(\operYtoXbeta \circ(\tilde U_Y^\top U_Y) \bigr) \cleverbar\cT_m \bigr) \circ (P_{ml}\tilde p_n)} \cleverbar q \cleverbar q^\top 
    \Bigr]}_{\dr Interference}.
\end{align*}
{\color{blue} We prove that, when $\Theta$ satisfies Decomposability Condition, 
only the top left block of $A^\h$, $A_{XX}^\h$, is non-zero.
Moreover, 
  $\varPhi^\top A_{XX}^\h \varPhi$ is  a diagonal matrix.}

\subsubsection*{Simplification of $B_Y^\h$ with Gaussian Covariate.}
\label{sec:Stein B_Y}
Recall that in \eqref{eq:by_as_sum} we show that $B_Y^\h$ can be decomposed into three terms: $\BYsignal^\h$, $\BYinterference^\handhprime$, and $\BYnoise^\handhprime$, which are defined in \eqref{eq:define_B_signal}  and \eqref{eq:define_B_signal}. 
In the following, we prove that each term can be diagonalized by $\varPsi$.

For $\BYsignal^\h$, by the last   identity of   \Cref{fact:expectation terms} and \Cref{fact:beta M beta}, we have
\begin{align*}
    \BYsignal^\h &  = -\EE\left[{G}^\top q p^\top X^\top{G}\right] = -\varPsi \EE\left[
        \operXtoYbeta \circ \left(q \dsbr{\cT_l \circ p_l}^\top\right)
    \right] \varPsi^\top  \notag \\
    & 
    = -\varPsi \EE\left[
        \operXtoYbeta \circ \left(\varPhi \, \cleverbar q\, \cleverbar q^\top \dsbr{\cleverbar\cT_l \circ p_l}^\top \varPhi^\top \right)
    \right] \varPsi^\top.
\end{align*}
where the   last equality holds by \eqref{eq:relation_T_Tbar} and the definition $\cleverbar q = \varPhi^\top q $.
Note that the operator $\operXtoYbeta $ always produces a diagonal matrix, we prove that 
$\varPsi ^\top \BYsignal^\h \varPsi  $ is a diagonal matrix.
For $ \BYinterference^\handhprime$, by the second  identity of   \Cref{fact:expectation terms} and \Cref{fact:beta M beta},  we have
\begin{align*}
    \BYinterference^\handhprime 
    &= \EE\left[\tilde U_Y {G}^\top X \tilde p p^\top X^\top {G}\right] \nend 
    &= \varPsi \EE\left[
        \diag(\tilde\mu) \operXtoYbeta \circ \left(
            \dsbr{(\cT_l \cT_m^\top) \circ (\tilde p_l p_m)} 
            + \dsbr{\tilde p_l p_l} \cdot I_{d}
        \right)
    \right] \varPsi^\top 
    \mytag{\Cref{fact:expectation-order2}}
    \nend 
    & = \varPsi \EE\left[
        \diag(\tilde\mu) \operXtoYbeta \circ \left( \varPhi \left(
            \dsbr{(\cleverbar\cT_l \, \cleverbar q \, \cleverbar q^\top \cleverbar\cT_m ) \circ (\tilde p_l p_m)} 
            + \dsbr{\tilde p_l p_l} \cdot I_{d} \right) \varPhi^\top
        \right)
    \right] \varPsi^\top. 
\end{align*}
Here, the second equality also holds by \Cref{cond:U_Y and W_X aligned} that $\tilde U_Y = \varPsi \diag(\tilde\mu) \varPsi^\top$ and the last equality follows from \eqref{eq:relation_T_Tbar}.
Similarly, thanks to  the operator $\operXtoYbeta$, $\varPsi ^\top \BYinterference^\handhprime \varPsi  $ is a diagonal matrix.
Finally, for $\BYnoise^\handhprime$, we have
\begin{align*}
    \BYnoise^\handhprime &= \EE\left[\tilde U_Y \varepsilon \, \tilde p \, p^\top \varepsilon^\top \right]
    = \sigma^2 \cdot \EE\left[\tilde U_Y \dsbr{I_{d_y} \cdot \tilde p_l p_l} \right] = \sigma^2 \varPsi \EE\left[\diag(\tilde\mu) \dsbr{\tilde p_l p_l} \right] \varPsi^\top,
\end{align*}
where we use  the fact that $\varepsilon_l \iidfrom \mathcal{N}(0, \sigma^2 I_{d_y})$. This matrix is also digonalizable by $\varPsi$ as $\dsbr{\tilde p_l p_l}$ is a scalar.

In summary, we have for $B_Y^\h$ that
\begin{align*}
    \varPsi^\top B_Y^\h \varPsi 
    &= \underbrace{-\EE\left[
        \operXtoYbeta \circ \left(\varPhi \, \cleverbar q\, \cleverbar q^\top \dsbr{\cleverbar\cT_l \circ p_l}^\top \varPhi^\top \right)
    \right]}_{\dr Signal} 
    + \underbrace{\sum_{h'=1}^H \sigma^2  \EE\left[\diag(\tilde\mu) \dsbr{\tilde p_l p_l}\right]}_{\dr Noise}  \nend 
    &\qquad + \underbrace{\sum_{h'=1}^H \EE\left[
        \diag(\tilde\mu) \operXtoYbeta \circ \left( \varPhi \left(
            \dsbr{(\cleverbar\cT_l \, \cleverbar q \, \cleverbar q^\top \cleverbar\cT_m^\top) \circ (\tilde p_l p_m)} 
            + \dsbr{\tilde p_l p_l} \cdot I_{d} \right) \varPhi^\top
        \right)
    \right]}_{\dr Interference}.
\end{align*}
{\color{blue} We prove that, when $\Theta$ satisfies Decomposability Condition, 
only the right block of $B^\h$, $B_Y^\h$, is non-zero.
Moreover, 
$\varPsi^\top B_Y^\h \varPsi$ is  a diagonal matrix.}  
Therefore, we conclude the proof of \Cref{prop:simplify_AB}.

\subsection{Approximation of the Spectral Dynamics}\label{sec:approximation_dynamics}
\paragraph{Additional Notations.}
We follow the same notations as in \Cref{sec:simplify AB with Gaussian covariate}. 
Besides, we denote by $\delta_{ij}$ the Kronecker delta function such that $\delta_{ij} = 1$ if $i=j$ and $0$ otherwise.
Furthermore, we will also encounter graph-theoretic notions. 
For a graph $\cG$, we denote by $\connectedcomponent(\cG)$ the set of connected components of $\cG$.
We denote by $R(v)$ the set of nodes that are reachable from node $v$ in $\cG$.
Note that $R(v)$ is also the element of $\connectedcomponent(\cG)$ that contains $v$.

\subsubsection{Low-Effective Order Approximation to the Derivatives of Softmax Attention Probability}
\label{sec:low-effective-order approximation}
In order to give a rigorous calculation of $A_{XX}^\h$ and $B_Y^\h$, we need to compute the derivatives of the attention probability $p_l$ and $\tilde p_l$ with respect to the attention scores $s$ and $\tilde s$ as implied by the operator $\cleverbar\cT_l$ that 
\[
    \cleverbar\cT_l \circ f \deleq \left(\diag(\omega)  \partial_l  + \diag(\tilde\omega) \tilde\partial_l \right) \circ f.
\]
For simplicity, we replace $\diag(\omega)$ and $\diag(\tilde\omega)$ with $w$ and $\tilde w$ as placeholders, respectively.
With such a replacement,  we have 
$\cleverbar\cT_l \circ f = (w\partial_l  + \tilde w\tilde\partial_l  )f   $.
Also, when considering $\cT_l \circ f$, we just need to replace $w$ and $\tilde w$ with $W q $ and $\tilde W q$ by multiplying back the rotation matrix $\varPhi$ and $\barq = \varPhi^\top q$ according to \eqref{eq:relation_T_Tbar}.

Note that $P = \diag(p) - p p^\top$ for softmax attention probability $p$, which is a symmetric matrix. 
For calculation of $A_{XX}^\h$, we need to compute the following derivatives:
\begin{gather*}
    \dsbr{w^{\otimes 2} \partial_l \partial_m P_{ml}}, 
    \:
    \dsbr{(w\partial_l + \tilde w\tilde\partial_l) P_{mm}\tilde p_l}, \: 
    \dsbr{(w\partial_l + \tilde w\tilde\partial_l) P_{lm}\tilde p_m}, \\
    \dsbr{(w\partial_l + \tilde w\tilde\partial_l)\otimes (w\partial_m + \tilde w\tilde\partial_m)\otimes (w\partial_n + \tilde w\tilde\partial_n) P_{ml}\tilde p_n}.
\end{gather*}
For $B_Y^\h$, we need to compute the following derivatives:
\begin{gather*}
    \dsbr{w\partial_l  p_l}, \:
    \dsbr{(w\partial_l + \tilde w\tilde\partial_l)\otimes (w\partial_m + \tilde w\tilde\partial_m) \tilde p_l p_m}. 
\end{gather*}
We have the following fact about the derivatives of the attention probabilities. 
\begin{lemma}
    \label[lemma]{fact:derivatives}
    We have the following identities involving partial  derivatives of the attention probability $p$ with respect to the attention scores $s$:
    \begin{myenumi}[
        label=(\roman*),
        ref=\Cref{fact:derivatives}(\roman*),
        leftmargin=1.5em
    ]
        \item \label{deriv: 1}
        $\partial_l p_m = P_{ml} = \delta_{lm} p_l - p_l p_m$;
        \item \label{deriv: 2}
        $\partial_{lm}^2 p_n = \partial_l P_{nm} = \delta_{lmn} p_l - \delta_{mn} p_l p_m - \delta_{nl} p_m p_n - \delta_{lm} p_n p_l + 2 p_l p_m p_n$;
        \item \label{deriv: 3}
        $\sum_{l} \partial_l P_{nl} = \sum_{l} \partial_n P_{ll}= \sum_{l} 2 p_n p_l^2 - 2 p_n^2 \delta_{ln} = 2 p_n (\norm{p}_2^2 - p_n)$;
        \item \label{deriv: 4}
        $\sum_{l} \partial_{lm} P_{nl} = \partial_m \left(\sum_{l} \partial_{l} P_{nl}\right) = 2(-2 p_n^2 \delta_{lmn} + p_n p_l^2 \delta_{mn} + 2 p_n p_l^2 \delta_{lm} + 2 p_n^2 p_m \delta_{ln} - 3p_m p_n p_l^2) $;
    \end{myenumi}
\end{lemma}

\paragraph{Calculations for $B_Y^\h$.}
We give a rigorous calculation of the terms for $B_Y^\h$ in the following. By \Cref{deriv: 1}, we have
\begin{align*}
    \dsbr{\partial_l p_l} = \dsbr{P_{ll}} = 1 - \norm{p}_2^2.
\end{align*}
Furthermore,  by \Cref{deriv: 2}, we have
\begin{align*}
    &\bigdsbr{(w\partial_l + \tilde w\tilde\partial_l)\otimes (w\partial_m + \tilde w\tilde\partial_m) \tilde p_l p_m} \nend 
    &\quad = \Bigl\llbracket w^{\otimes 2} \tilde p_l (\delta_{lm} p_m -  p_l p_m - 2\delta_{lm} p_m^2 + 2 p_l p_m^2) + \tilde w^{\otimes 2} p_l (\delta_{lm} \tilde p_m -  \tilde p_l \tilde p_m - 2\delta_{lm} \tilde p_m^2 + 2 \tilde p_l \tilde p_m^2) \nend 
    &\qqquad + \tilde w \otimes w (p_m - p_m^2)(\tilde p_l - \tilde p_l^2) + w \otimes\tilde w (\delta_{lm} p_l - p_l p_m) (\delta_{lm} \tilde p_l - \tilde p_l \tilde p_m) \Bigr\rrbracket \nend 
    &\quad = \EE\Bigl[ 2 w^{\otimes 2} (-\tilde p^\top p^{\odot 2} + \tilde p^\top p \norm{p}_2^2) + 2 \tilde w^{\otimes 2} (-p^\top \tilde p^{\odot 2} + p^\top \tilde p \norm{\tilde p}_2^2) \nend 
    &\qqquad + \tilde w \otimes w (1 - \norm{p}_2^2)(1 - \norm{\tilde p}_2^2) + w\otimes \tilde w (p^\top \tilde p - p^\top \tilde p^{\odot 2} - \tilde p^\top p^{\odot 2} + (\tilde p^\top p)^2 ) \biggiven q \Bigr].
\end{align*}
The following fact summarizes the above calculations in the original matrix form (substituting the placeholders $w$ and $\tilde w$ with $W q$ and $\tilde W q$, respectively).
\begin{lemma}
    \label[lemma]{fact:derivatives for B terms}
    Suppose $q$ is the query vector and $W, \tilde W$ are two attention weight matrices such that $p = \softmax(X^\top W q)$ and $\tilde p = \softmax(X^\top \tilde W q)$. 
    Suppose that each column of $X$ is $\iid$ drawn from $\cN(0, I_{d_x})$. 
    Then  we have by \Cref{fact:expectation terms} that
    \begin{align*}
        & \EE[X p\given q]  =  \dsbr{\cT_l \circ p_l } = W q\cdot \EE\left[ (1 -\norm{p}_2^2)\given q\right], \\
        & \EE[X p \tilde p X^\top \given q]   = 
        \dsbr{p_l \tilde p_l \cdot I_{d_x} } + \dsbr{(\cT_l\otimes \cT_m)\circ (p_l \tilde p_m)}
        \nend
        &=\EE[p^\top \tilde p\given q] \cdot I_{d_x} + \EE\Bigl[ 2 W q q^\top W^\top (-\tilde p^\top p^{\odot 2} + \tilde p^\top p \norm{p}_2^2) + 2 \tilde W q q^\top \tilde W^\top(-p^\top \tilde p^{\odot 2} + p^\top \tilde p \norm{\tilde p}_2^2) \nend 
        &\qquad+ \tilde W q q^\top W^\top (1 - \norm{p}_2^2)(1 - \norm{\tilde p}_2^2) + W q q^\top \tilde W^\top (p^\top \tilde p - p^\top \tilde p^{\odot 2} - \tilde p^\top p^{\odot 2} + (\tilde p^\top p)^2 ) \Biggiven q \Bigr].
    \end{align*}
    In particular, if $W = \tilde W$, then we have a simpler form:
    \begin{align*}
        \EE[X p p X^\top \given q] = \EE[ W q q^\top W^\top (1 - \norm{p}_2^2 - 6 \norm{p}_3^3 + 6 \norm{p}_2^4)] + I_{d_x} \EE[\norm{p}_2^2]. 
    \end{align*}
\end{lemma}

\paragraph{Calculations for $A_{XX}^\h  $.}
For $A_{XX}^\h$, expanding all these derivatives is tedious and unnecessary. We only need to keep track of those terms that contribute significantly to the final results and view the rest as higher-order noise. 
This necessitates  developing a systematic way to evaluate the importance of each term in the expansion. To this end, in the following, we relate the expansion of partial derivates of a polynomial function of attention probabilities $p$ and $\tilde p$ to a graph structure where the nodes are in $\cV = [L]$ and the edges are in $\cE = [L] \times [L]$.
We introduce the notion of graph-induced polynomials and effective order as follows. 
\begin{definition}[Graph-Induced Polynomial and Effective Order]
    \label[definition]{def:graph-polynomial}
    Consider a weighted graph $\cG=(\cV, \cE, a, b)$ where $\cV =  [L]$ is a set of vertices, $\cE = [L] \times [L]$ is a set of undirected edges (we allow self-loop), and $a = \{a_v\}_{v\in\cV}, b = \{b_v\}_{v\in\cV}$ with $(a_v, b_v)\in \NN$ are the nonnegative integer weights. 
    For any subgraph $\cG'$ of $\cG$, define the total weights of $\cG'$ as $W(\cG') \deleq \sum_{v\in\cV'} (a_v + b_v)$.
    We have the following definitions: 
    \begin{myenumi}
        \item The polynomial induced by graph $\cG$ with variables $p=\{p_v\}_{v\in\cV}$, $\tilde p=\{\tilde p_v\}_{v\in\cV}$ is defined as $f_\cG(p, \tilde p) = \prod_{v\in\cV} \bigl( p_{v}^{a_v} \tilde p_{v}^{b_v} \bigr) \cdot \prod_{(v, v')\in\cE}\delta_{v v'}$.
        \item Let $\connectedcomponent_{\ge n}(\cG) \deleq \{c\in \connectedcomponent(\cG)\given W(c) \ge n\}$ be the set of connected components of $\cG$ with total weights no less than $n$. 
        \item We define the \emph{effective order} of $\cG$ as $ \efforder(\cG)\deleq\sum_{v\in\cV} (a_v + b_v) - |\connectedcomponent_{\ge 1}(\cG)|$.
        \item We say that two graphs $\cG$ and $\cG'$ are \emph{equivalent} if $\cV=\cV'$, $a=a'$, $b=b'$ and $\connectedcomponent_{\ge 1}(\cG) = \connectedcomponent_{\ge 1}(\cG')$.
    \end{myenumi}
\end{definition}
\Cref{def:graph-polynomial} draws a parallel between the weighted graph structure to polynomials of this form $\prod_{v\in\cV} \bigl( p_{v}^{a_v} \tilde p_{v}^{b_v} \bigr) \cdot \prod_{(v, v')\in\cE}\delta_{v v'}$. 
We remark that the \emph{effective order} is critical for our analysis since terms with higher effective order turn out to be higher-ordered smaller in scale. 
The following fact shows that the polynomial and the underlying graph structure are isomorphic up to the class of \emph{equivalent graphs}.
\begin{lemma}
    \label[lemma]{fact:equivalent graphs}
    For any two graphs $\cG$ and $\cG'$, $f_\cG(p, \tilde p) = f_{\cG'}(p, \tilde p)$ if and only if $\cG$ and $\cG'$ are equivalent. 
    Moreover, any two equivalent graphs have the same effective order.
\end{lemma}
\begin{proof}(\emph{Proof Sketch of \Cref{fact:equivalent graphs}})
    It is easy to check that $f_\cG(p, \tilde p) = f_{\cG'}(p, \tilde p)$ if and only if $\cG$ and $\cG'$ have the same set of vertices, same weights on the vertices, and the same set of connected components, which is exactly what it means by saying that $\cG$ and $\cG'$ are equivalent.
    The second claim on the effective order follows directly from the definition that the effective order is only a function of the weights on the vertices and the set of connected components.
\end{proof}
To this end, we can also define the effective order of a polynomial as the effective order of \emph{any} graph that induces this polynomial.
We next introduce a \emph{partial derivative} operator on both the graph and the polynomial.
\begin{definition}[Derivative Operator]
    \label{def:derivative operator}
    Consider a polynomial $f_\cG(p, \tilde p) = \prod_{v\in\cV} \bigl( p_{v}^{a_v} \tilde p_{v}^{b_v} \bigr) \cdot \prod_{(v, v')\in\cE}\delta_{v v'}$. 
    Define \emph{partial derivative} operator $\partial_u$ for any $u\in\cV$ as 
    \begin{align}
        \partial_u f_{\cG}(p, \tilde p) 
        &\deleq \sum_{s\in\cV} a_s p_s^{a_s -1} (\delta_{us} p_s - p_s p_u) \tilde p_s^{b_s} \cdot \prod_{v\in\cV\setminus\{s\}} \bigl( p_{v}^{a_v} \tilde p_{v}^{b_v} \bigr) \cdot \prod_{(v, v')\in\cE}\delta_{v v'} \nend
        &= \underbrace{\sum_{s\in\cV} a_s \delta_{us} \cdot \prod_{v\in\cV} \bigl( p_{v}^{a_v} \tilde p_{v}^{b_v} \bigr) \cdot \!\!\!\! \prod_{(v, v')\in\cE}\delta_{v v'}}_{\dr\text{Adding an edge $(u, s)$}} - \underbrace{\left(\sum_{s\in\cV} a_s\right) \cdot p_u \cdot \prod_{v\in\cV} \bigl( p_{v}^{a_v} \tilde p_{v}^{b_v} \bigr) \cdot \!\!\!\! \prod_{(v, v')\in\cE}\delta_{v v'}}_{\dr\text{Increasing $a_u$ by $1$}}.
        \label{eq:partial derivative operator}
    \end{align}
    Moreover, $\tilde\partial_u$ is defined similarly.
\end{definition}
Note that the definition of the partial derivative operator is consistent with the definition of the derivative of the attention probability $p$ in \Cref{fact:derivatives}.
When combined with the underlying graph structure, we have an intuitive interpretation of the partial derivative operator.
The partial derivative operator $\partial_u$ produces at most $|\cV|+1$ terms in the summation, where each of the first $|\cV|$ term adds a new edge $(u, s)$ for some $s \in \cV$ to the graph $\cG$ and the last term adds $1$ to weight $a_u$ of vertex $u$ in $\cG$.
The following fact comes directly from the observation that both adding a new edge or adding $1$ to the weight does not decrease the effective order.
\begin{lemma}
    \label{fact:derivative does not decrease effective order}
    Consider a $\cG$-induced polynomial $f_\cG(p, \tilde p) = \prod_{v\in\cV} \bigl( p_{v}^{a_v} \tilde p_{v}^{b_v} \bigr) \cdot \prod_{(v, v')\in\cE}\delta_{v v'}$. Applying $\partial_v$ or $\tilde\partial_v$ yields in a summation of polynomials where each has an effective order no less than $\efforder(\cG)$.
\end{lemma}
\begin{proof}(\emph{Proof Sketch of \Cref{fact:derivative does not decrease effective order}})
    The proof follows directly from the definition of the partial derivative operator in \Cref{def:derivative operator} and that adding a new edge or increasing the weight of a vertex does not decrease the effective order.
\end{proof}
Therefore, when applying a sequence of partial derivative operators and want to keep track of only the low-effective-order terms, we can safely ignore higher-order terms after each application of the partial derivative operator.
To this end, we use notation $\sum_i f_{\cG_i}(p, \tilde p) \ldeq[k] \sum_{j} f_{\cG_j'}(p, \tilde p)$ to denote the \emph{low-effective-order equivalence} between two expressions if $\sum_i f_{\cG_i}(p, \tilde p) \ind(\efforder(\cG_i)\le k) = \sum_{j} f_{\cG_j'}(p, \tilde p)\ind(\efforder(\cG_j')\le k)$. 
The following fact helps us simplify the calculation under this low-effective-order equivalence.
\begin{lemma}
    \label[lemma]{fact:low-effective-order equivalence}
    We denote by $R(v)$ the set of nodes that are reachable from node $v$ in $\cG$.
    Consider a $\cG$-induced polynomial $f_\cG(p, \tilde p) = \prod_{v\in\cV} \bigl( p_{v}^{a_v} \tilde p_{v}^{b_v} \bigr) \cdot \prod_{(v, v')\in\cE}\delta_{v v'}$. 
    Let $k$ be the effective order of $f_\cG(p, \tilde p)$.
    Consider two cases: (i) if $W(R(u)) = 0$, we have
\begin{align*}
    \partial_u f_{\cG}(p, \tilde p) 
        &\ldeq[k] \sum_{s\in\cV} a_s \delta_{us} \cdot \prod_{v\in\cV} \bigl( p_{v}^{a_v} \tilde p_{v}^{b_v} \bigr) \cdot \!\!\!\!\prod_{(v, v')\in\cE}\delta_{v v'} - \left(\sum_{s\in\cV} a_s\right) \cdot p_u \cdot \prod_{v\in\cV} \bigl( p_{v}^{a_v} \tilde p_{v}^{b_v} \bigr) \cdot \!\!\!\! \prod_{(v, v')\in\cE}\delta_{v v'}, 
\end{align*}
and (ii) if $W(R(u)) \ge 1$, we only need to keep 
\begin{align}
    \partial_u f_{\cG}(p, \tilde p) \ldeq[k] \left(\sum_{s\in R(u)} a_s\right) \cdot \prod_{v\in\cV} \bigl( p_{v}^{a_v} \tilde p_{v}^{b_v} \bigr) \cdot \!\!\!\!\prod_{(v, v')\in\cE}\delta_{v v'}, 
    \label{eq:low-effective-order calculation-2}
\end{align}
where $\delta_{us}$ is not needed since $s\in R(u)$ and deleting the edge $(u, s)$ will not change the connected component structure.
    Similar result holds for $\tilde\partial_u f_{\cG}(p, \tilde p)$.
\end{lemma}
\begin{proof}(\emph{Proof of \Cref{fact:low-effective-order equivalence}})
   See Appendix \ref{proof:flow-effective-order equivalence} for a detailed proof.
\end{proof}
We now characterize each term in the decomposition of $A_{XX}^\h$. By \Cref{deriv: 4}, we have
\begin{align*}
    \dsbr{w^{\otimes 2} \partial_l \partial_m P_{ml}}
    &= 2 w^{\otimes 2} \dsbr{-2 p_m^2 \delta_{lm}  + 2 p_m p_l^2 \delta_{lm} + 2 p_m^3 \delta_{lm} + p_m p_l^2 - 3p_m^2 p_l^2} \nend 
    &= 2 w^{\otimes 2}  \EE\left[ 
        - \norm{p}_2^2 + 4\norm{p}_3^3 - 3 \norm{p}_2^4 \biggiven q
    \right]. 
\end{align*}
For $\bigdsbr{(w\partial_l + \tilde w\tilde\partial_l) P_{mm}\tilde p_l}$, we have by \Cref{deriv: 3} and \Cref{deriv: 1} that
\begin{align*}
    \bigdsbr{(w\partial_l + \tilde w\tilde\partial_l) P_{mm}\tilde p_l}
    &= w\dsbr{(- 2p_l^2 \delta_{lm} + 2 p_l p_m^2) \tilde p_l} + \tilde w \dsbr{
        (p_m -p_m^2) (\tilde p_l - \tilde p_l^2)
    } 
    \nend 
    &= w \EE\bigl[ \tilde p^\top \left( -2 p^{\odot 2}  + 2 p \norm{p}_2^2\right) \given q \bigr]
    + \tilde w \EE\bigl[(1 - \norm{p}_2^2)(1 - \norm{\tilde p}_2^2)\given q \bigr]
    , 
\end{align*}
where the last equality holds by noting that $\sum_l \delta_{lm} \tilde p_l -\tilde p_l \tilde p_m = 0$. 
For $\bigdsbr{(w\partial_l + \tilde w\tilde\partial_l) P_{lm}\tilde p_m}$, we have also by \Cref{deriv: 3} and \Cref{deriv: 1} that
\begin{align*}
    \bigdsbr{(w\partial_l + \tilde w\tilde\partial_l) P_{lm}\tilde p_m} 
    &= w \bigdsbr{(2 p_m p_l^2 - 2 p_m^2 \delta_{lm}) \tilde p_m } + \tilde w \dsbr{(\delta_{lm}p_l - p_l p_m)(\delta_{lm}\tilde p_l - \tilde p_l \tilde p_m)} \nend 
    &= w \EE\bigl[ \tilde p^\top \left(- 2 p^{\odot 2} + 2 p \norm{p}_2^2\right) \given q \bigr] + \tilde w \EE\bigl[
        \tilde p^\top p - p^\top \tilde p^{\odot 2} - \tilde p^\top p^{\odot 2} + (\tilde p^\top p)^2 \given q
    \bigr]. 
\end{align*}
Lastly, we have 
\begin{align*}
    &\dsbr{(w\partial_l + \tilde w\tilde\partial_l)\otimes (w\partial_m + \tilde w\tilde\partial_m)\otimes (w\partial_n + \tilde w\tilde\partial_n) P_{ml}\tilde p_n} \nend 
    &\quad = \dsbr{(w\partial_m + \tilde w\tilde\partial_m)\otimes (w\partial_n + \tilde w\tilde\partial_n)\otimes  \left(w \left(2 p_m p_l^2 - 2 p_m^2 \delta_{lm}\right)\tilde p_n + \tilde w P_{ml} \tilde P_{nl}\right)}^{\top(312)} \!\!\!\!
    \mytag{\Cref{deriv: 3}}
    \nend 
    &\quad = \dsbr{(w\partial_m + \tilde w\tilde\partial_m)\otimes (w\partial_n + \tilde w\tilde\partial_n)\otimes  \left(w \left(2 p_m p_l^2 - 2 p_m^2 \delta_{lm}\right)\tilde p_n + \tilde w (\delta_{lm} p_l - p_l p_m ) (\delta_{ln} \tilde p_l - \tilde p_l \tilde p_n)\right)}^{\top(312)} \!\!\!\!\!\!\!\!\!\!. 
\end{align*}
Note that all the terms in $\left(w \left(2 p_m p_l^2 - 2 p_m^2 \delta_{lm}\right)\tilde p_n + \tilde w (\delta_{lm} p_l -  p_l p_m ) (\delta_{ln} \tilde p_l - \tilde p_l \tilde p_n)\right)$ are of effective order $1$.
Now, we apply $(w\partial_m + \tilde w\tilde\partial_m)$ to these terms and only keep track of those terms of effective order $1$.
Using \Cref{fact:low-effective-order equivalence}, we have
\begin{align*}
    &\dsbr{(w\partial_m + \tilde w\tilde\partial_m)\left(w \left(2 p_m p_l^2 \tilde p_n - 2 p_m^2 \tilde p_n \delta_{lm}\right) + \tilde w (\delta_{lmn } p_l\tilde p_l - \delta_{lm} p_l \tilde p_l \tilde p_n - \delta_{ln} p_l p_m \tilde p_l + p_l p_m \tilde p_l \tilde p_n) \right)}^{\top}\nend 
    &\quad \ldeq \Bigl\llbracket w^{\otimes 2} (2 p_m p_l^2\tilde p_n  - 4 p_m^2\tilde p_n \delta_{lm})+ \tilde w \otimes w (\delta_{lmn } p_l\tilde p_l - \delta_{lm} p_l \tilde p_l \tilde p_n - \delta_{ln} p_l p_m \tilde p_l + p_l p_m \tilde p_l \tilde p_n) \nend 
    &\qqquad + \tilde w^{\otimes 2} (\delta_{lmn} p_l \tilde p_l - \delta_{lm} p_l \tilde p_l \tilde p_n ) \Bigr\rrbracket.
\end{align*}
In the above calculation, we note that $W(R(m)) \ge 1$ holds for all terms. Hence, we apply \eqref{eq:low-effective-order calculation-2} to these terms for both $\partial_m$ and $\tilde\partial_m$. 
For $\left(2 p_m p_l^2 \tilde p_n - 2 p_m^2 \tilde p_n \delta_{lm}\right)$, only terms with $\partial_m$ survives since $b_m = 0$ for $\tilde p$. 
For $\left(\delta_{lmn } p_l\tilde p_l - \delta_{lm} p_l \tilde p_l \tilde p_n - \delta_{ln} p_l p_m \tilde p_l + p_l p_m \tilde p_l \tilde p_n\right)$, when applying $\partial_m$, all the terms remain unchanged since the $\sum_{s\in R(m)} a_s = 1$ and when applying $\tilde\partial_m$, we have $\tilde\partial_m(- \delta_{ln} p_l p_m \tilde p_l + p_l p_m \tilde p_l \tilde p_n) = 0$. 
Next, we apply $(w\partial_n + \tilde w\tilde\partial_n)$ to the above expressions, 
\begin{align*}
    &\Bigl\llbracket (w\partial_n + \tilde w\tilde\partial_n) \Bigl(w^{\otimes 2} (2 p_m p_l^2\tilde p_n  - 4 p_m^2\tilde p_n \delta_{lm})+ \tilde w \otimes w (\delta_{lmn } p_l\tilde p_l - \delta_{lm} p_l \tilde p_l \tilde p_n - \delta_{ln} p_l p_m \tilde p_l + p_l p_m \tilde p_l \tilde p_n) \nend 
    &\quad + \tilde w^{\otimes 2} (\delta_{lmn} p_l \tilde p_l - \delta_{lm} p_l \tilde p_l \tilde p_n )\Bigr) \Bigr\rrbracket^{\top(231)} \nend 
    &\quad \ldeq \Bigl\llbracket w^{\otimes 2}\otimes \tilde w (2 p_m p_l^2\tilde p_n  - 4 p_m^2\tilde p_n \delta_{lm})+ \tilde w \otimes w \otimes \tilde w (\delta_{lmn } p_l\tilde p_l - \delta_{lm} p_l \tilde p_l \tilde p_n - \delta_{ln} p_l p_m \tilde p_l + p_l p_m \tilde p_l \tilde p_n) \nend 
    &\qqquad 
    + \tilde w \otimes w^{\otimes 2} (\delta_{lmn} p_l \tilde p_l - \delta_{ln} p_l p_m \tilde p_l)
    + \tilde w^{\otimes 3} (\delta_{lmn} p_l \tilde p_l - \delta_{lm} p_l \tilde p_l \tilde p_n ) + \tilde w^{\otimes 2} \otimes w \delta_{lmn} p_l \tilde p_l  \Bigr\rrbracket \nend 
    &\quad = \dsbr{w^{\otimes 2}\otimes \tilde w (2 p_m p_l^2\tilde p_n  - 4 p_m^2\tilde p_n \delta_{lm}) + \tilde w^{\otimes 2} \otimes w \delta_{lmn} p_l \tilde p_l} \nend 
    &\quad = \EE\bigl[-2 w^{\otimes 2}\otimes \tilde w  \norm{p}_2^2 + \tilde w^{\otimes 2} \otimes w p^\top \tilde p\given q\bigr]. 
\end{align*}
In the first equality, we again have $W(R(n)) \ge 1$ for all terms and we only use \eqref{eq:low-effective-order calculation-2} for calculation.
This time, applying $\partial_n$ to $(2 p_m p_l^2\tilde p_n  - 4 p_m^2\tilde p_n \delta_{lm})$ gives $0$ since $\sum_{s\in R(n)} a_s = 0$.
Applying $\tilde\partial_n$ to the second part $(\delta_{lmn } p_l\tilde p_l - \delta_{lm} p_l \tilde p_l \tilde p_n - \delta_{ln} p_l p_m \tilde p_l + p_l p_m \tilde p_l \tilde p_n)$ does not change the form since $\sum_{s\in R(n)} b_s = 1$.
Applying $\partial_n$ to $(\delta_{lmn } p_l\tilde p_l - \delta_{lm} p_l \tilde p_l \tilde p_n - \delta_{ln} p_l p_m \tilde p_l + p_l p_m \tilde p_l \tilde p_n)$ only gives us the first and the third terms. 
Applying $\tilde\partial_n$ to $(\delta_{lmn } p_l\tilde p_l - \delta_{lm} p_l \tilde p_l \tilde p_n)$ does not change the form since $\sum_{s\in R(n)} b_s = 1$ and applying $\partial_n$ to $(\delta_{lmn } p_l\tilde p_l - \delta_{lm} p_l \tilde p_l \tilde p_n)$ gives us $\delta_{lmn} p_l \tilde p_l$. 
Moreover, for the second equality, we have the 2nd-4th terms zero by noting that $\sum_l p_l = 1$ by the normalization of the attention probability.

We summarize the above results into the following table, where the highest degree $(\sum_{v\in\cV} a_v + b_v)$ is upper bounded by counting the total number of $p$, $\tilde p$ and $\barcT$ in the expression. Here, $P_{lm}=\delta_{lm}p_l - p_l p_m$ is counted as $2$ and each $\barcT$ raises the degree at most $1$ by \Cref{fact:low-effective-order equivalence}.
\begin{table}[ht]
    \centering
    \begin{tabular}{l@{\hspace{5pt}}l@{\hspace{-5pt}}c}
        \hline
        \hline
        \textbf{Objects} & \textbf{Expressions} & \textbf{Degree}\\
        \hline
        $\bigdsbr{\cleverbar\cT_l \circ p_l}$ & $= \EE\left[w\cdot (1-\norm{p}_2^2) \given q\right]$ & $2$ \\
        $\bigdsbr{(\cleverbar\cT_l\otimes \cleverbar\cT_m) \circ (\tilde p_l p_m)}$ & $\ldeq \EE\left[\tilde w \otimes w\cdot (1 - \norm{p}_2^2 - \norm{\tilde p}_2^2) + w\otimes \tilde w \cdot p^\top \tilde p \given q\right]$ & $4$\\
        \hline\\[-1em]
        $\dsbr{(\cleverbar\cT_l\otimes\cleverbar\cT_m) \circ P_{ml}}$ & $= \EE\left[w^{\otimes 2}  
        \cdot 2(- \norm{p}_2^2 + 4\norm{p}_3^3 - 3 \norm{p}_2^4) \given q\right]$ & $4$\\
        $\bigdsbr{\cleverbar\cT_l \circ (P_{mm}\tilde p_l)}$ & $=\EE\left[w \cdot 2\tilde p^\top \left( - p^{\odot 2}  + p \norm{p}_2^2\right) + \tilde w \cdot (1 - \norm{p}_2^2)(1 - \norm{\tilde p}_2^2)\given q\right]$ & $4$\\
        $\bigdsbr{\cleverbar\cT_l \circ (P_{lm}\tilde p_m)}$ & $\ldeq \EE\left[w \cdot 2\tilde p^\top \left(- p^{\odot 2} + p \norm{p}_2^2\right) + \tilde w \cdot \tilde p^\top p \given q\right]$ & $4$\\
        $\bigdsbr{(\cleverbar\cT_l\otimes \cleverbar\cT_m \otimes \cleverbar\cT_n) \circ (P_{ml}\tilde p_n)}$ & $\ldeq \EE\left[ w^{\otimes 2}\otimes \tilde w  \cdot (-2\norm{p}_2^2) + \tilde w^{\otimes 2} \otimes w \cdot p^\top \tilde p \given q\right]$ & $6$\\
        \hline
        \hline
    \end{tabular}
    \caption{Expressions for Terms in the Decomposition of $A_{XX}^\h$ and $B_Y^\h$ and Their Highest Degrees.}
    \label{tab:p-moments in A and B}
\end{table}

In the sequel, we denote by $\alpha_\signal^\h, \alpha_\noise^\h, \alpha_\intf^\h$ the vector of eigenvalues of $\AXsignal^\h, \AXnoise^\h, \AXinterference^\h$ respectively, and $\beta_\signal^\h, \beta_\noise^\h, \beta_\intf^\h$ the vector of eigenvalues of $\BYsignal^\h, \BYnoise^\h, \BYinterference^\h$ respectively.
Combining the above results with the fact that $\varPhi^\top A_{XX}^\h \varPhi$ is diagonal in \Cref{fact:diagonality of A}, we have for each terms in $\varPhi^\top A_{XX}^\h \varPhi$ that
\begin{align*}
    \alpha_\signal^\h &= - \frac{1}{d_x}\EE\left[
        \bigl(- 2\norm{p^\h}_2^2 + \error \bigr) \cdot \bigl\langle \omega^\h \odot (\Lambda \mu^\h), \barq^{\odot 2} \bigr\rangle  \cdot \omega^\h \odot \barq^{\odot 2}
    \right] \nend 
    &\qquad - \frac{1}{d_x} \EE\left[
        \bigl(1 - \norm{p^\h}_2^2\bigr) \cdot (\Lambda \mu^\h) \odot \barq^{\odot 2}
    \right],
\end{align*}
\begin{align*}
    &\alpha_\noise^\h  = \sum_{h'=1}^H \sigma^2 \EE\left[
        \bigl\langle \mu^\h, \mu^\hprime \bigr\rangle \cdot \left( \error \cdot \omega^\h + \bigl({p^\h}^\top \!\! p^\hprime +\error \bigr)\cdot  \omega^\hprime \right)  \odot \barq^{\odot 2}\right], \nend 
    &\alpha_\intf^\h = \sum_{h'=1}^H \frac{1}{d_x} \EE\left[
        \Lambda \bigl(\mu^\h \odot \mu^\hprime\bigr)  \odot \left(\error \omega^\h + \bigl(1 + \error[0] \bigr) \omega^\hprime \right) \odot \barq^{\odot 2}
        \right] \nend 
    &\quad+\sum_{h'=1}^H \frac{1}{d_x} \EE\left[
        \left(\vone_{d_x}^\top \Lambda \bigl(\mu^\h \odot \mu^\hprime\bigr)\vone_{d_x} +  \Lambda \bigl(\mu^\h \odot \mu^\hprime\bigr)\right) \odot \left(\error \cdot \omega^\h +\bigl({p^\h}^\top p^\hprime + \error\bigr)\cdot  \omega^\hprime \right)  \odot \barq^{\odot 2}
    \right] \nend 
    &\quad + \sum_{h'=1}^H \frac{1}{d_x} \EE\left[
        \bigl(-2\norm{p^\h}_2^2 + \error \bigr)\cdot \bigl\langle \Lambda \bigl(\mu^\h \odot \mu^\hprime\bigr) \odot \omega^\h \odot \omega^\hprime, \barq^{\odot 2} \bigr\rangle \cdot \omega^\h \odot \barq^{\odot 2}\right] \nend 
    &\quad + \sum_{h'=1}^H \frac{1}{d_x} \EE\left[
        \bigl({p^\h}^\top p^\hprime + \error \bigr) \cdot \bigl\langle \Lambda \bigl(\mu^\h \odot \mu^\hprime\bigr) \odot \omega^\h \odot \omega^\hprime, \barq^{\odot 2} \bigr\rangle \cdot \omega^\hprime \odot \barq^{\odot 2}\right] \nend 
    &\quad + \sum_{(\omega_1, \omega_2, \omega_3)\in \Omega}\sum_{h'=1}^H \frac{1}{d_x} \EE\left[
        \error[1] \cdot \bigl\langle \Lambda \bigl(\mu^\h \odot \mu^\hprime\bigr) \odot \omega_1 \odot \omega_2, \barq^{\odot 2} \bigr\rangle \cdot \omega_3 \odot \barq^{\odot 2}\right], 
\end{align*}
where we use $\error[k]$ to hide terms that are functions of $(p^\h, p^\hprime)$ and are of effective order higher than $k$. 
In addition, we denote by $\Omega = \{\omega^\h, \omega^\hprime\}^3 \setminus \{(\omega^\h, \omega^\hprime, \omega^\h), (\omega^\h, \omega^\hprime, \omega^\hprime)\}$.
Here, we also use the fact that $\operYtoXbeta \circ (\tilde U_Y^\top U_Y) = \diag(\Lambda (\tilde\mu\odot \mu))$ and $ \operYtoXbeta \circ (U_Y^\top) = \diag(\Lambda\mu)$ due to the fact that $\tilde U_Y$ and $U_Y$ are simultaneously diagonalizable by $\varPsi$. 
Similarly, we have for each terms in $\varPsi^\top B_Y^\h \varPsi$ that
\begin{align*}
    &\beta^\h = -\frac{1}{d_x} \EE\left[
        \bigl(1 - \norm{p^\h}_2^2 \bigr) \cdot \Lambda^\top (\omega^\h \odot \barq^{\odot 2}) \right] + \sum_{h'=1}^H \sigma^2 \EE\left[{p^\h}^\top p^\hprime \cdot \mu^\hprime\right] \nend 
    &\qqquad + \sum_{h'=1}^H \frac{1}{d_x} \EE\left[
        \mu^\hprime \odot \Lambda^\top \left(\left(
            \omega^\hprime \odot \omega^\h (1 + \error[0]) + {\omega^\h}^{\odot 2} \error + {\omega^\hprime}^{\odot 2} \error
        \right) \odot \barq^{\odot 2} \right)
    \right] \nend 
    &\qqquad + \sum_{h'=1}^H \frac{1}{d_x} \EE\left[{p^\h}^\top p^\hprime \cdot \mu^\hprime \odot \bigl(\Lambda^\top \vone_{d_y}\bigr) \right]. 
\end{align*}

\subsubsection{Error Analysis for the Approximation of $A_{XX}^\h$ and $B_Y^\h$}
\label{sec:appendix_dynamics_error}
In order to derive an approximation of $A_{XX}^\h$ and $B_Y^\h$, we propose the following conditions on $\mu^\h$ and $\omega^\h$.
\begin{condition}
    \label[condition]{cond:bounded weights}
    For fixed $\epsilon\in(0, 1)$ and $\epsilon_0\in(0, 1)$, we consider the following regime for $\{\omega^\h\}_{h\in[H]} $:
    \begin{gather*}
        \norm{\omega^\h}_\infty \le   L^{-1/4} \cdot (\log L)^{-1/2} , \quad \norm{\omega^\h}_2^2 \le \frac{2 \log L}{3 c^2}, \quad
        \norm{\omega^\h}_4^4 \le L^{-(1-\epsilon_0)}\cdot (\log L)^{-1}, 
    \end{gather*}
    where $c\in\RR_+$ is the minimal constant satisfying 
    \begin{align*}
        \max\left\{\frac{1}{c} + \frac{3}{1 + \sqrt{1+c^2/2}}, \frac{12\sqrt{5}}{2c} \right\} \le \epsilon.
    \end{align*}
    We also consider the following regime for $\{\mu^\h\}_{h\in[H]}$:
    \begin{align*}
        \muihprime \baromegaihprime \ge 0, \quad
        \sum_{h'=1}^H  \muihprime \baromegaihprime \le O(1),  \quad 
        \bigl|\muih\bigr| \le \sqrt{2 L \phi_i^{-1}} , \quad \forall i\in [d_y], h\in[H], 
    \end{align*}
    where $\phi_i = 1 + \snr_i^{-1} = 1 + (\sigma^2 d) / (d_i \lambda_i)$. 
\end{condition}
The above conditions on $\omega^\h$ is motivated by \Cref{lemma: E[pq]-moment} and \Cref{lemma: higher-order-moment} where in \Cref{lemma: higher-order-moment} we plug in $\Deg(\cG)\le 6$ and $\kappa=2$ in the condition $\frac{2\Deg(\cG)\sqrt{2\kappa+1}}{\epsilon \kappa} \le c$. The fact $\Deg(\cG)\le 6$ is according to the highest degree in \Cref{tab:p-moments in A and B}. 
Under \Cref{cond:bounded weights}, we conclude from \Cref{lemma: E[pq]-moment} that
\begin{align}
    \EE\left[\left( \mathbb{E}[{p^\h}^{\top} p^\hprime\given q]-\frac{\exp\bigl(\bigl\langle \omega^\h, \omega^\hprime \bigr\rangle\bigr) }{L} \right)^2 \right] \le O(L^{- (3 -\epsilon_0)}).
    \label{eq:E[p tildep]-approx}
\end{align}
And if $\efforder(\cG) \ge 2$ and $\Deg(\cG) \le 6$  for some graph $\cG$, then the $\cG$-induced polynomial $f_\cG(p^\h, p^\hprime)$ satisfies
\begin{align}
\EE\left[f_\cG(p^\h, p^\hprime)^{2}\right] \le O(L^{-4\cdot (1-\epsilon)})
\label{eq:E[f]-moment ub}
\end{align}
by \Cref{lemma: higher-order-moment}.

To use the previous results, we need to decouple the randomness in $f_\cG(p^\h, p^\hprime)$ and some polynomials of $\barq$. 
\paragraph{Notations.}
In the sequel, we use $f_{=1}$ to denote terms with effective order equal to $1$  and $f_{>1}$ to denote terms with effective order greater than or equal to $2$ where we hide the dependence on $(p^\h, p^\hprime)$ in the notation.
For our need, we consider two kinds of polynomials in $\barq$: (i) $g(\barq) = \barq^{\odot 2}$ and (ii) $g(\barq) = \langle v, \barq^{\odot 2}\rangle\cdot \barq^{\odot 2}$ where $v\in\RR_+^{d}$ is a constant vector. 
We abbreviate $g(\barq)$ as $g$ in the sequel.
We drop the superscript $(\cdot)^\h$ and replace $(\cdot)^\hprime$ by $\tilde{(\cdot)}$ for simplicity.
We denote by $v_i$ the $i$-th entry of $v$ and $v_{-i}$ the vector obtained from $v$ with the $i$-th entry removed.

\begin{lemma}
    \label[lemma]{lem:decoupling error}
    Suppose $v\in\RR_+^{d}$ or $v\in\RR_{-}^{d}$.  
    Under \Cref{cond:bounded weights}, we have 
    \begin{gather*}
        \max\left\{\left| \frac{\EE\left[ (p^\top \tilde p + f_{>1}) \cdot \langle v, \barq^{\odot 2} \rangle \cdot \barq_i^2\right]}{\exproduct[\omega][\tilde\omega]\cdot L^{-1} \cdot (\langle v\rangle + 2 v_i) } - 1  \right|, \quad \left| \frac{\EE\left[ (p^\top \tilde p + f_{>1}) \cdot \barq_i^2\right]}{\exproduct[\omega][\tilde\omega]\cdot L^{-1} } - 1  \right|  \right\}\le O(L^{-(1-\epsilon_0)/2}), \\
        \max\left\{ \left|\frac{\EE\left[  f_{>1} \cdot \langle v, \barq^{\odot 2} \rangle \cdot \barq_i^2\right]}{ (\langle v\rangle + 2 v_i) }\right| , \quad \left|{\EE\left[  f_{>1} \cdot \barq_i^2\right]}{ }  \right|\right\}\le O( L^{-2(1-\epsilon)}). 
    \end{gather*}
\end{lemma}
\begin{proof}(\emph{Proof of \Cref{lem:decoupling error}})
    See Appendix \ref{sec:proof_decoupling error} for a detailed proof.
\end{proof}

\paragraph{Notations.}
Recall that 
\begin{align*}
    \Lambda = \diag\bigl(\lambda_1 \vone_{d_1}, \dots, \lambda_I \vone_{d_I}, \vzero_{d_{I+1}}, \dots, \vzero_{d_{d_y}}\bigr),
\end{align*}
where $\vone_{d_i}$ is a $d_i$-dimensional all-one vector. 
For the eigenvalues $\mu^\h$ and $\omega^\h$, we denote by $\muih$ the $i$-th entry of $\mu^\h$, and $\baromegaih$ the average of the $i$-th block of $\omega^\h$, where the blocks are defined by the row partition of $\Lambda$ for $i \in [d_y]$. 
Moreover, if $I<d_y$, we assign nominal tasks indexed by $j=I+1, \dots, d_y$ with $\lambda_j = 0$, dimension $d_j$ such that $\sum_{j=I+1}^{d_y} d_j = d - \sum_{i=1}^I d_i$, and define $\mujh$ and $\baromegajh$ for $j\in\{I+1, \dots, d_y\}$ accordingly.
We let $\vecd = (d_1, \dots, d_{d_y})$ be the vector containing the dimensions for each corpus. 
We denote by $\baralpha_i^\h$ the average of $\alpha^\h$ within the $i$th block and $\baralpha = (\baralpha_1, \dots, \baralpha_{d_y})$ the vector containing the averages for each corpus.
Equipped with \Cref{lem:decoupling error}, we have for each terms in $\varPhi^\top A_{XX}^\h \varPhi$ that
\begin{align}
    &d \baralpha_\intf^\h 
    = \sum_{h'=1}^H \lambda \odot \mu^\h \odot \mu^\hprime \odot  \baromega^\hprime \odot \left(1\pm L^{-1}\right) 
    \label{eq:alpha intf-1}
    \\
    & + \sum_{h'=1}^H   \lambda \odot \mu^\h \odot \mu^\hprime \odot   \baromega^\h \odot \left(\pm L^{-2(1-\epsilon)}\right)
    \label{eq:alpha intf-2}
    \\ 
    &+ \sum_{h'=1}^H  \fracexproductL[\omega^\h][\omega^\hprime]  \bigl\langle  \vecd \odot \lambda, \mu^\h \odot \mu^\hprime \bigr\rangle  \baromega^\hprime \odot \left(1\pm (d_{\min}^{-1} + L^{-\frac{1-\epsilon_0}{2}}) \right)
    \label{eq:alpha intf-3}
    \\
    &+ \sum_{h'=1}^H   \bigl\langle \vecd \odot  \lambda, \mu^\h \odot \mu^\hprime \bigr\rangle \baromega^\h  \odot \left(\pm L^{-2(1-\epsilon)}\right)
    \label{eq:alpha intf-4}
    \\
    & - \sum_{h'=1}^H  2\fracexproductL[\omega^\h][\omega^\h] \bigl\langle \vecd \odot \lambda , \baromega^\h \odot \baromega^\hprime \odot \mu^\h \odot \mu^\hprime\bigr\rangle \baromega^\h \odot \left(1\pm \bigl(d_{\min}^{-1}+L^{-\frac{1-\epsilon_0}{2}} \bigr)\right)
    \label{eq:alpha intf-5}
    \\
    &+ \sum_{h'=1}^H  \fracexproductL[\omega^\h][\omega^\hprime] \bigl\langle \vecd \odot \lambda , \baromega^\h \odot \baromega^\hprime \odot \mu^\h \odot \mu^\hprime\bigr\rangle \baromega^\hprime \odot \left(1\pm \bigl(d_{\min}^{-1}+L^{-\frac{1-\epsilon_0}{2}} \bigr)\right)
    \label{eq:alpha intf-6}
    \\
    &\quad + \sum_{(\baromega_1, \baromega_2, \baromega_3)\in \Omega}  \sum_{h'=1}^H \bigl\langle \vecd \odot \lambda , \bigl|\baromega_1 \odot \baromega_2 \odot \mu^\h \odot \mu^\hprime \bigr|\bigr\rangle \baromega_3 \odot \left(\pm   L^{-2(1-\epsilon)}  \right), 
    \label{eq:alpha intf-7}
\end{align}
\begin{align*}
    d \baralpha_\signal^\h &=  2\fracexproductL[\omega^\h][\omega^\h] \bigl\langle \vecd \odot \lambda , \baromega^\h \odot \mu^\h \bigr\rangle \baromega^\h \odot \left(1\pm (d_{\min}^{-1} + L^{-\frac{1-\epsilon_0}{2}})\right) - \lambda \odot \mu^\h \odot (1 \pm L^{-1}), \nend 
    d \baralpha_\noise^\h 
    &=  \sum_{h'=1}^H  \sigma^2 d\fracexproductL[\omega^\h][\omega^\hprime] \bigl\langle \mu^\h, \mu^\hprime \bigr\rangle \baromega^\hprime \odot \left(1\pm L^{-\frac{1-\epsilon_0}{2}}\right) \nend 
    &\qquad 
    +  \sum_{h'=1}^H \sigma^2 d \bigl\langle \mu^\h, \mu^\hprime \bigr\rangle \baromega^\h \odot \left(\pm L^{-2(1-\epsilon)} \right), 
\end{align*}
where we have 
$$\Omega = \left(\{\baromega^\h, \baromega^\hprime\}^{\otimes 2} \setminus \{(\baromega^\h, \baromega^\hprime)\}\right) \otimes \{\baromega^\h, \baromega^\hprime\}, $$
as all the possible combinations of $(\baromega^\h, \baromega^\hprime)$ except for the combinations that already appear in \eqref{eq:alpha intf-5} and \eqref{eq:alpha intf-6}.
Here, we define $d_{\min} = \min_{i\in[I]} d_i$ and the additional error $d_{\min}^{-1}$ comes from the approximation of $\langle v \rangle + 2 v_i$ in \Cref{lem:decoupling error} with $\langle v\rangle$ since there are at least $d_{\min}$ terms equal to $v_i$ in the summation $\langle v \rangle$ and each coordinate of $v$ is either non-negative or non-positive.
Here, $x = a \pm b$ means that the upper bound for $x$ is $a+O(b)$ and the lower bound is $a-O(b)$.
In addition, we are able to invoke \Cref{lem:decoupling error} for \eqref{eq:alpha intf-5}, \eqref{eq:alpha intf-6} and the first term in $d \baralpha_\sigma^\h$ since $\baromega^\h\odot\baromega^\hprime\odot \mu^\h\odot \mu^\hprime$ and $\mu^\h \odot \baromega^\h$ have non-negative entries according to the condition $\muih \baromegaih\ge 0$ in \Cref{cond:bounded weights}.
Due to the same reason, we include an absolute value in \eqref{eq:alpha intf-7} to ensure that the entries of $\baromega_1 \odot \baromega_2 \odot \mu^\h \odot \mu^\hprime$ are non-negative in order to invoke \Cref{lem:decoupling error}, where we treat \eqref{eq:alpha intf-7} as an error term. Similarly, 
for   $\varPsi^\top B_Y^\h \varPsi$ we have
\begin{align}
    d \beta^\h 
    &= -  \vecd \odot \lambda \odot \baromega^\h \odot \left(1 \pm L^{-1}\right)
    +  \sum_{h'=1}^H  \sigma^2 d\fracexproductL[\omega^\h][\omega^\hprime] \mu^\hprime \odot \left(1\pm L^{-\frac{1-\epsilon_0}{2}}\right)
    \label{eq:beta signal and noise}
    \\
    &\qquad +  \sum_{h'=1}^H  \vecd \odot \lambda \odot   \baromega^\h \odot \baromega^\hprime \odot \mu^\hprime \odot \left(1\pm L^{-1}\right) 
    \label{eq:beta intf-1}
    \\
    &\qquad + \sum_{h'=1}^H  \vecd \odot \lambda \odot  \left({\baromega^\h}^{\odot 2} + {\baromega^\hprime}^{\odot 2}
    \right)  \odot \mu^\hprime \odot \left(\pm L^{-2(1-\epsilon)}\right)
    \label{eq:beta intf-2}
    \\
    &\qquad + \sum_{h'=1}^H  \fracexproductL[\omega^\h][\omega^\hprime] \vecd \odot \lambda \odot \mu^\hprime \odot \left(1 \pm L^{-\frac{1-\epsilon_0}{2}}\right)
    \label{eq:beta intf-3}
    , 
\end{align}

\subsubsection{Further Approximation for Symmetric Weights}
In this part, we consider the case where both $W_X^\h$ and $U_Y^\h$ are \ac{psd} matrices, which means that $\muih \ge 0$ and $\baromegaih \ge 0$ for all $i \in [d_y]$ and $h \in [H]$.
We first simplify $\baralpha_\intf^\h$. 
We consider the second order terms which scales as $L^{-2(1-\epsilon)}$. 
For \eqref{eq:alpha intf-7}, we note that 
the possible combinations for $(\baromega_1, \baromega_2)$ are $(\baromega^\h, \baromega^\h)$, $(\baromega^\hprime, \baromega^\hprime)$, and $(\baromega^\hprime, \baromega^\h)$ while $\baromega_3$ can be either $\baromega^\h$ or $\baromega^\hprime$.
If $(\baromega_1, \baromega_2) = (\baromega^\hprime, \baromega^\h)$, 
we can get rid of the absolute value in \eqref{eq:alpha intf-7}.
To this end, we notice that the form of this term is then no different from either \eqref{eq:alpha intf-5} or \eqref{eq:alpha intf-6} depending on the choice of $\baromega_3$, while both \eqref{eq:alpha intf-5} and \eqref{eq:alpha intf-6} can only be larger than $L^{-1}$.  
Hence, we can add an additional error term $L^{-(1-2\epsilon)}$ to both \eqref{eq:alpha intf-5} and \eqref{eq:alpha intf-6} and remove the case $(\baromega_1, \baromega_2) = (\baromega^\hprime, \baromega^\h)$ from \eqref{eq:alpha intf-7}.

We further simplify \eqref{eq:alpha intf-7} if both $\mu$ and $\baromega$ are non-negative. 
For the remaining cases, we note that $\baromega_1 = \baromega_2 $ and $\baromega_1^{\odot 2}$ has each element of order $o(1)$ according to \Cref{cond:bounded weights}.
If  $\baromega_3 = \baromega^\h$, we have \eqref{eq:alpha intf-7} bounded by \eqref{eq:alpha intf-4} element-wise and if $\baromega_3 = \baromega^\hprime$, we have \eqref{eq:alpha intf-7} bounded by \eqref{eq:alpha intf-3} times $L^{-(1- 2\epsilon)}$ element-wise.
For the second term \eqref{eq:alpha intf-2}, we also have by the non-negativity of $\mu$ that it is bounded by \eqref{eq:alpha intf-4} element-wise.
Hence, we conclude that only \eqref{eq:alpha intf-4} survives as the second order term in $\baralpha_\intf^\h$ if both $\mu$ and $\baromega$ are non-negative.

Next, we consider the first order terms which scales as $L^{-1}$.
By noting that $\mu^\hprime \odot \baromega^\hprime < O(1)$ element-wise according to \Cref{cond:bounded weights}, we have that \eqref{eq:alpha intf-5} is upper bounded by the first term in $d \baralpha_\signal^\h$ up to some constant.
By noting that $\|\baromega^\h\|_\infty < O(L^{-(1-\epsilon_0)/4})$ according to \Cref{cond:bounded weights}, we have that \eqref{eq:alpha intf-6} is upper bounded by \eqref{eq:alpha intf-3} times $L^{-(1-\epsilon_0)/2}$.
In summary, we have for $\baralpha^\h$ that 
\begin{align}
    &d \baralpha^\h  
    =- \left(1\pm L^{-1}\right)\odot  \lambda \odot \mu^\h 
    + \sum_{h'=1}^H \left(1\pm L^{-1}\right)\odot \lambda \odot \mu^\h \odot \mu^\hprime \odot  \baromega^\hprime 
    \label{eq:alpha intf-1234}
    \\
    & \quad + \sum_{h'=1}^H  \fracexproductL[\omega^\h][\omega^\hprime]  \bigl\langle  \vecd \odot \lambda \odot \phi, \mu^\h \odot \mu^\hprime \bigr\rangle  \baromega^\hprime \odot \left(1\pm (d_{\min}^{-1} + L^{-\frac{1-\epsilon_0}{2}} + L^{-(1-2\epsilon)}) \right) \nend 
    &\quad + \fracexproductL[\omega^\h][\omega^\h] \bigl\langle \vecd \odot \lambda , \baromega^\h \odot \mu^\h \bigr\rangle \baromega^\h \odot \left(\pm 1\right)
    + \sum_{h'=1}^H   \bigl\langle \vecd \odot  \lambda \odot \phi, \mu^\h \odot \mu^\hprime \bigr\rangle  \baromega^\h \odot \left(\pm L^{-2(1-\epsilon)}\right). \notag
\end{align}
For $\beta^\h$, we note that the second term \eqref{eq:beta intf-2} is of order $L^{-2(1-\epsilon)}$. 
Note that $\norm{\baromega^\h}_\infty \le O(L^{-(1-\epsilon_0)/4})$ according to \Cref{cond:bounded weights}. 
Hence, we have \eqref{eq:beta intf-3} bounded by \eqref{eq:beta intf-3} times $L^{-(1-\epsilon_0)/2}$.
In summary, we have for $\beta^\h$ that
\begin{align}
    d \beta^\h 
    &= - \left(1 \pm L^{-1}\right) \odot \vecd \odot \lambda \odot \baromega^\h +  \sum_{h'=1}^H \left(1\pm L^{-1}\right) \odot \vecd \odot \lambda \odot   \baromega^\h \odot \baromega^\hprime \odot \mu^\hprime  
    \nend
    & \qquad + \sum_{h'=1}^H \fracexproductL[\omega^\h][\omega^\hprime] \vecd \odot \lambda \odot \phi \odot \mu^\hprime \odot \left(1 \pm L^{-\frac{1-\epsilon_0}{2}}\right) 
    .
    \label{eq:beta intf-123}
\end{align}

\newpage 
\subsection{Preservation of The Decomposability Condition along Gradient Flow}
\label[appendix]{sec:proof of decomposability preserved}

In the previous section, we have shown that the decomposability of the weights implies that both $A^\h$ and $B^\h$ have only one nonzero block, and these submatrices can be diagonalized by $\varPhi$ and $\varPsi$.
In the following, we prove that the Decomposability Condition is preserved during the dynamics, given that the initialization of gradient flow satisfies the Decomposability Condition.

\begin{proof}(\emph{Proof of \Cref{lem:decomposability preserved}})
Below we verify the preservation of the conditions in \Cref{def:decomposability property} one by one.
\paragraph{Preservation of $U_X^\h = 0$ and $W_Y^\h = 0$.}
We first show the easy part in \Cref{cond:U_X and W_Y are zero} that $U_X^\h = 0$ and $W_Y^\h = 0$ during the dynamics. 
Recall from \Cref{sec:simplify_AB} that $A^\h$ only has the left-top block $A_{XX}^\h\in\RR^{d\times d}$ being non-zero. 
Therefore, the time-derivative of $W_Y^\h={K_Y^\h}^\top Q_X^\h$ satisfies
\begin{align*}
    \partial_t W_Y^\h 
    &= \partial_t {K_Y^\h}^\top Q_X^\h + {K_Y^\h}^\top \partial_t Q_X^\h \nend 
    &= - d_e^{-1/2} \bigg(\begin{bmatrix} {A_{YX}^\h} & {A_{YY}^\h}\end{bmatrix} {Q^\h}^\top Q_X^\h + {K_Y^\h}^\top K^\h \begin{bmatrix}A_{XX}^\h \\ {A_{YX}^\h}\end{bmatrix} \bigg)\nend 
    &=- d_e^{-1/2} {K_Y^\h}^\top K_X^\h A_{XX}^\h = 0, 
\end{align*}
where the last equality holds by \Cref{cond:orthogonal} that $\vspan(K_X^\h) \perp \vspan(K_Y^\h)$.
Also, we note that $B^\h$ only has the right block $B_{Y}^\h\in\RR^{d_y\times d_y}$ being non-zero.
Thus the time-derivative of $U_X^\h = O^\h V_X^\h$ satisfies
\begin{align*}
    \partial_t U_X^\h 
    &= \partial_t O^\h V_X^\h + O^\h \partial_t V_X^\h \nend
    &= - \begin{bmatrix} {B_X^\h} & B_Y^\h\end{bmatrix} {V^\h}^\top V_X^\h - O^\h {O^\h}^\top {B_X^\h} \\ 
    &= - B_Y^\h {V_Y^\h}^\top V_X^\h = 0, 
\end{align*}
where the last equality holds by \Cref{cond:orthogonal} that $\vspan(V_X^\h) \perp \vspan(V_Y^\h)$.
Hence, we conclude that $U_X^\h \equiv 0$ and $W_Y^\h \equiv 0$ along the gradient flow trajectory if they are initialized to be zero.

\paragraph{Preservation of Subspace Orthogonality.}
We next show that \Cref{cond:orthogonal} is preserved, for which it suffices to show that $\partial_t ({K_Y^\h}^\top K_X^\h) \equiv 0$ and $\partial_t ({V_X^\h}^\top V_Y^\h) \equiv 0$.
The time derivatives of $K_X^\h$ and $K_Y^\h$ are given by
\begin{align*}
    \partial_t K_X^\h = - d_e^{-1/2} Q^\h \begin{bmatrix}
        {A_{XX}^\h}^{\!\!\!\top} \\ {\color{violet}{A_{XY}^\h}^{\!\!\!\top}}
    \end{bmatrix}
    = - d_e^{-1/2} Q_X^\h {A_{XX}^\h}^{\!\!\!\top}, \quad
    \partial_t K_Y^\h = - d_e^{-1/2} Q^\h \begin{bmatrix}
        {A_{YX}^\h}^{\!\!\!\top} \\ 
        {{A_{YY}^\h}^{\!\!\!\top}}
    \end{bmatrix} = 0.
\end{align*}
These further imply
\begin{align*}
    \partial_t ({K_Y^\h}^\top K_X^\h) = {K_Y^\h}^\top \partial_t K_X^\h = - d_e^{-1/2} {K_Y^\h}^\top Q_X^\h {A_{XX}^\h}^{\!\!\!\top} = - d_e^{-1/2} {W_Y^\h} {A_{XX}^\h}^{\!\!\!\top} = 0.
\end{align*}
Similarly, we have $\partial_t V_X^\h = - {O^\h}^\top {B_X^\h} = 0, \partial_t V_Y^\h = - {O^\h}^\top B_Y^\h$, and thus
\begin{align*}
    \partial_t({V_X^\h}^\top V_Y^\h) = {V_X^\h}^\top \partial_t V_Y^\h = - {V_X^\h}^\top {O^\h}^\top B_Y^\h = - {{U_X^\h}^\top} B_Y^\h = 0.
\end{align*}
In the above two derivations, we have used \Cref{cond:U_X and W_Y are zero} that $U_X^\h = 0$ and $W_Y^\h = 0$.

\paragraph{Preservation of \Cref{cond:U_Y and W_X aligned} via diagonality of $\varPhi^\top A_{XX}^\h \varPhi$ and $\varPsi^\top B_Y^\h \varPsi$.}
Under the common singular vector space condition in \Cref{cond:U_Y and W_X aligned}, 
let $\varUpsilon^\h = \begin{bmatrix} \upsilon_1^\h, \dots, \upsilon_{d}^\h\end{bmatrix}\in\RR^{{d_e}\times d}$ and $\varPhi^\h\in\RR^{d \times d}$ be the common left- and right-singular vector matrix for both $Q_X^\h$ and $K_X^\h$, and $\varTheta^\h = \begin{bmatrix} \theta_1^\h, \dots, \theta_{d_y}^\h\end{bmatrix} \in\RR^{{d_e}\times d_y}$ and $\varPsi^\h\in\RR^{d_y\times d_y}$ be the common left- and right-singular vector matrix for both $V_Y^\h$ and ${O^\h}^\top$.
Simply put, we can decompose $Q_X^\h, K_X^h, V_Y^\h$ and $O^\h$ into
\begin{align*}
    Q_X^\h = \varUpsilon^\h \diag(\sigma(Q_X^\h)) \varPhi^\top, \quad K_X^\h = \varUpsilon^\h \diag(\sigma(K_X^\h)) \varPhi^\top, \\
    V_Y^\h = \varTheta^\h \diag(\sigma(V_Y^\h)) \varPsi^\top, \quad O^\h = \varPsi \diag(\sigma(O^\h)) \varTheta^\h, 
\end{align*}
where $(\sigma(Q_X^\h), \sigma(K_X^\h)) \in\RR^{d}$ and $(\sigma(V_Y^\h), \sigma(O^\h))\in\RR^{d_y}$ are the semi-singular values\footnote{We use the phrase \say{semi-singular} since $\sigma(Q_X^\h), \sigma(K_X^\h), \sigma(V_Y^\h)$ and $\sigma(O^\h)$ are not necessarily non-negative.}.
Note that we also have $\sigma(Q_X^\h) \odot \sigma(K_X^\h) = \omega^\h$ and $\sigma(V_Y^\h) \odot \sigma(O^\h) = \mu^\h$.
Our goal is to show that the singular vector spaces remain unchanged for all these four matrices during the dynamics.
With the observations that only $A_{XX}^\h$ and $B_Y^\h$ are non-zero, we have the following dynamics as we have derived in the previous step:
\begin{align*}
    \partial_t K_X^\h &= - d_e^{-1/2} Q_X^\h {A_{XX}^\h}^{\!\!\!\top} = - d_e^{-1/2} \varUpsilon^\h \diag(\sigma(Q_X^\h)) \bigl(\varPhi^\top {A_{XX}^\h} \varPhi\bigr)^\top \varPhi^\top, 
    \\ 
    \partial_t Q_X^\h &= - d_e^{-1/2} K_X^\h {A_{XX}^\h} = - d_e^{-1/2} \varUpsilon^\h \diag(\sigma(K_X^\h)) \bigl(\varPhi^\top {A_{XX}^\h} \varPhi\bigr) \varPhi^\top ,
    \\
    \partial_t O^\h &= - B_Y^\h {V_Y^\h}^\top = - \varPsi \bigl(\varPsi^\top B_Y^\h \varPsi\bigr) \diag(\sigma(V_Y^\h)) {\varTheta^\h}^\top, 
    \\ 
    \partial_t V_Y^\h &= - {O^\h}^\top B_Y^\h = - \varTheta^\h \diag(\sigma(O^\h)) \bigl(\varPsi^\top B_Y^\h \varPsi\bigr) {\varPsi}^\top.
\end{align*}
To this end, it suffices to verify that both $\varPhi^\top A_{XX}^\h \varPhi$ and $\varPsi^\top B_Y^\h \varPsi$ are diagonal matrices in order to show that the singular vector spaces remain unchanged.
The diagonality of $\varPsi^\top B_Y^\h \varPsi$ and $\varPhi^\top A_{XX}^\h \varPhi$ is shown by \Cref{prop:simplify_AB}.
Hence, we conclude that the singular vector spaces remain unchanged during the dynamics.


\paragraph{Preservation of \Cref{cond:task-wise homogeneous}.}
Note that the previously, \Cref{cond:orthogonal} and \Cref{cond:U_Y and W_X aligned} are self-preserved during the dynamics. 
We next show that \Cref{cond:task-wise homogeneous} is also preserved during the dynamics if the current state satisfies this condition.
Let $\cJ_i$ be the index set of the support of the $i$-th task as we have defined in the main text.
Given that $\sigma(Q_X^\h)_m = \sigma(Q_X^\h)_n$ for any $m, n\in \cJ_i$, and the same for $K_X^\h$, it suffices to check that the diagonal entries of $\varPhi^\top A_{XX}^\h \varPhi$ are the same for any $m, n\in \cJ_i$.
We say a $\RR^{d\times \star}$ matrix (or vector) is \ac{wth} if the $m$-th and the $n$-th rows (or entries) are always the same.
Note that operators $\barcT_l$ (by the \ac{wth} of $\omega^\h$) and $\cA$ (by the \ac{wth} property of $\Lambda= \EE[G^{\odot 2}]$)
are both \ac{wth}.
Therefore, it suffices to check that both
\begin{align*}
    \EE[f(s, \tilde s) \barq \barq^\top], \quad \EE\bigg[f(s, \tilde s)\sum_{j\in\cJ_i}\barq_j^2 \barq \barq^\top\bigg]
\end{align*}
have \ac{wth} \emph{diagonal} entries for any function $f$ that only depends on $s$ and $\tilde s$ and any $j\in[d]$.
These formulas capture all the terms in the decomposition of $\varPhi^\top A_{XX}^\h \varPhi$. 
The first term captures the signal term, the noise term, and the first three interference terms in the decomposition of $\varPhi^\top A_{XX}^\h \varPhi$. 
The last term captures the last interference term in the decomposition of $\varPhi^\top A_{XX}^\h \varPhi$.
Notably, by \Cref{fact:p is a function of Wq's 2-norm}, $f(s, \tilde s)$ is a function of only $\langle \omega^{\odot 2}, \cleverbar q^{\odot 2} \rangle$, $ \langle \tilde\omega^{\odot 2}, \cleverbar q^{\odot 2}\rangle$, and $\langle \omega \odot \tilde\omega, \cleverbar q^{\odot 2}\rangle$. 
For now, we can rewrite 
$\langle \omega^{\odot 2}, \barq^{\odot 2}\rangle = \sum_{i=1}^I d_i^{-1}\sum_{k\in\cJ_i} \omega_k^2 \cdot \sum_{m\in\cJ_i} \barq_m^{2}$ and the same for the remaining terms by the \ac{wth} property of $\omega$ or $\tilde\omega$. 
Let $\barq_{(i)}$ be the slice of $\barq$ that contains only the entries in $\cJ_i$.
Hence, $f(s, \tilde s)$ is just a function of $\{\norm{\barq_{(i)}}_2\}_{i\in[I]}$. 
Given the 2-norm of $\barq_{(i)}$, the posterior distribution of $\barq_{(i)}$ is shuffling-invariant.
Therefore, for the first term, we directly conclude that $\EE[f(s, \tilde s) \barq \barq^\top]$ has \ac{wth} diagonal entries.
For the second term, we combine $f(s, \tilde s)$ and $\sum_{j\in\cJ_i}\barq_j^2 \barq \barq^\top$ to form a new function that only depends on $\{\norm{\barq_{(i)}}_2\}_{i\in[I]}$.
Thus, we also conclude that $\EE\left[f(s, \tilde s)\sum_{j\in\cJ_i}\barq_j^2 \barq \barq^\top\right]$ has \ac{wth} diagonal entries.
Thus, we conclude that \Cref{cond:task-wise homogeneous} is also self-preserved during the dynamics.
\end{proof}

\newpage 
\section{Analysis of the Spectral Gradient Flow}\label[appendix]{sec:appendix_convergence_analysis}

In this section, we analyze the dynamics and convergence of gradient flow. For the analysis we separate the dynamics into two stages: the warm-up stage and the growth stage. The warm-up stage is further divided into five steps, with each step being studied and summarized afterward. The growth stage consists of the emergence and convergence phases. The dynamics paths are summarized in \Cref{sec:dynamics_path}, while the observations regarding emergence and convergence in the growth stage are summarized in \Cref{sec:emergence_and_convergence}.

Upon success of \Cref{lem:decomposability preserved}, we have 
\begin{align*}
    \partial_t \sigma(K_X^\h) = - \sqrt {d_e}^{-1} \sigma(Q_X^\h) \odot \sigma({A_{XX}^\h}^\top), 
    &\quad \partial_t \sigma(Q_X^\h) = - \sqrt {d_e}^{-1} \sigma(K_X^\h) \odot \sigma({A_{XX}^\h}), \\
    \partial_t \sigma(O^\h) = - \sigma(B_Y^\h) \odot \sigma({V_Y^\h}^\top), &\quad \partial_t \sigma(V_Y^\h) = - \sigma({O^\h}^\top) \odot \sigma(B_Y^\h), 
\end{align*}
Under the \ac{sw} initialization, we always have $\sigma(Q_X^\h) = \sigma(K_X^\h) = \sqrt{\omega^\h}$ and $\sigma(O^\h) = \sigma(V_Y^\h)=\sqrt{\mu^\h}$. 
Thus, we can further reduce the dynamics to
the combined dynamics of $\mu^\h$ and $\omega^\h$: 
\begin{align*}
    \partial_t \omega^\h &= - \sqrt {d_e}^{-1}
    \left(\sigma(Q_X^\h)^{\odot 2} +  \sigma(K_X^\h)^{\odot 2}\right) \odot \sigma({A_{XX}^\h}) &&= - 2\sqrt {d_e}^{-1} \omega^\h \odot \alpha^\h, \\
    \partial_t \mu^\h &= - \left(\sigma(O^\h)^{\odot 2} +  \sigma(V_Y^\h)^{\odot 2}\right) \odot \sigma(B_Y^\h) &&= - 2 \mu^\h \odot \beta^\h.
\end{align*}
\textbf{In the following, we rescale the time as $t \leftarrow 2 d t$.} 
Note that $\baromega^\h\in\RR^{d_y}$ is just a collection of the unique task-wise values in $\omega^\h$.
Recall the simplification of $\alpha^\h$ and $\beta^\h$ in \eqref{eq:alpha intf-1234} and \eqref{eq:beta intf-123}, which gives 
\begin{align}
    &\sqrt{\de}\partial_t \baromega^\h
    = \left(1\pm L^{-1}\right)\odot  \lambda \odot \mu^\h \odot \baromega^\h
    - \sum_{h'=1}^H \left(1\pm L^{-1}\right)\odot \lambda \odot \mu^\h \odot \mu^\hprime \odot  \baromega^\hprime \odot \baromega^\h \nend 
    &  \quad - \sum_{h'=1}^H  \fracexproductL[\omega^\h][\omega^\hprime]  \bigl\langle  \vecd \odot \lambda \odot \phi, \mu^\h \odot \mu^\hprime \bigr\rangle  \baromega^\hprime\odot \baromega^\h \odot \left(1\pm (d_{\min}^{-1} + L^{-\frac{1-\epsilon_0}{2}} + L^{-(1-2\epsilon)}) \right) \nend 
    &\quad - \fracexproductL[\omega^\h][\omega^\h] \bigl\langle \vecd \odot \lambda , \baromega^\h \odot \mu^\h \bigr\rangle (\baromega^\h)^{\odot 2} \odot \left(\pm 1\right) \nend
    &\quad - \sum_{h'=1}^H   \bigl\langle \vecd \odot  \lambda \odot \phi, \mu^\h \odot \mu^\hprime \bigr\rangle  (\baromega^\h)^{\odot 2} \odot \left(\pm L^{-2(1-\epsilon)}\right).
    \label{eq:dot baromegah}
\end{align}
Here, we denote by $\vecd$ the vector of task dimensions, i.e., $\vecd = (d_1, \ldots, d_I)^\top$.
Also, we have
\begin{align}
    \partial_t \mu^\h
    &=  \left(1 \pm L^{-1}\right) \odot \vecd \odot \lambda \odot \baromega^\h \odot \mu^\h -  \sum_{h'=1}^H \left(1\pm L^{-1}\right) \odot \vecd \odot \lambda \odot   \baromega^\h \odot \baromega^\hprime \odot \mu^\hprime  \odot \mu^\h
    \nend
    & \qquad - \sum_{h'=1}^H \fracexproductL[\omega^\h][\omega^\hprime] \vecd \odot \lambda \odot \phi \odot \mu^\hprime \odot \mu^\h \odot \left(1 \pm L^{-\frac{1-\epsilon_0}{2}}\right). 
    \label{eq:dot muh}
\end{align}
For simplicity, we let 
\begin{align*}
    \xi \triangleq L^{-1}, \quad \zeta \triangleq d_{\min}^{-1} + L^{-\frac{1-\epsilon_0}{2}} + L^{-(1-2\epsilon)},  \quad \eta \triangleq  L^{-2(1-\epsilon)}.
\end{align*}
We consider $\cI$ to be the set of tasks with $\lambda_i = \Theta(1)$ and $\cI_c$ to be the set of tasks with $\lambda_i = 0$, which we call the effective and nominal tasks, respectively.
This is a generalization of the setting in \Cref{thm:convergence-multi-head-symmetric}where all tasks are effective.
We will study the more general setting in the dynamics' analysis.

In the following, we will study the behavior of the dynamics in \eqref{eq:dot baromegah} and \eqref{eq:dot muh}.
We assume $\omega^\h$ to have the same initialization for each coordinate and have
$\baromega^\h_i(0) = \omega_0$ for each $h\in[H]$ and $i\in\cI$.
For $\mu^\h$, we don't restrict all eigenvalues to be the same. 
The following is the assumption for the initialization of $\mu^\h$.
We also let 
\begin{align*}
    \phi_i = 1 + \frac{\sigma^2 d }{d_i\lambda_i} = 1 + \snr_i^{-1}, 
\end{align*}
where $\snr_i = d_i\lambda_i/(\sigma^2 d )$ is the signal-to-noise ratio for the $i$-th task.

\begin{assumption}[Initialization]\label[assumption]{assump:initialization}
We assume $\omega^\h$ to have the same initialization for each coordinate with $\baromega^\h_i = \omega_0 $ where $\color {brown} \max \{\omega_0 ^{2} d, HL\omega_0 ^2\} \ll 1 $ for all $i\in [d_y]$. For each task $i\in \cI $, we assume there exist a unique \say {optimal head} $h_i^\star$ such that for any $h\neq h^\star_i$, $\mu^\h_i(0) <  \mu^{(h_i^\star)}_i(0) $, and  a minimal marginal difference $\epsilon \in (0, 1)$ such that 
\begin {align*} \frac {\mu^{(h_i^\star)}_i(0) -\mu^\h_i(0)}{\mu^{(h_i^\star)}_i(0)} \ge \epsilon = \Theta (1), \quad \forall h\neq h_i^\star, \quad \forall i\in \cI .
\end {align*}
We consider $h^\star_i\neq h^\star_j$ for $i\neq j$, which means that each task has a unique optimal head and no two tasks share the same optimal head. 
We assume that the initialization $\mu_i^\star(0)$ satisfies the following condition, \[ \brown \langle \vecd, \lambda \odot \phi \rangle  \lambda_i^{-1} \cdot 50 H L \omega_0 ^{3} \epsilon ^{-1} \ll \mu_i^\star(0) \ll L\omega_0 / (2H\phi_i), \quad \forall i\in \cI . \] 
For the nominal task, we assume $\color {brown}\mu^\h_j(0) \ll \min _{i\in \cI }(L\phi_i^{-1} \pmin d_i) \sqrt {|\cI _c|H}^{-1} \cdot \omega_0 $, $\color {brown}\mu^\h_j(0)^2 \ll \snr_i \cdot (1 \pmin L/(d_i\phi_i)) /(H |\cI _c| \zeta )$ for all $h\in [H]$ and $j\in \cI _c$. 
Moreover, we assume that $\omega_0 $ satisfies \begin{gather*} 
    \color {teal}\omega_0 \ll \max _{i\in \cI } \frac {\lambda_i}{ \phi_i \langle \vecd, \lambda \odot \phi \rangle \zeta }. \\ \color {teal} \frac {2C}{\epsilon }\cdot \left (\xi + 80HL\omega_0 ^2 + \frac {\langle \vecd, \lambda \odot \phi \rangle }{\lambda_i} \cdot 100 H^2 \omega_0 ^{2} \phi_i\right ) \ll \frac {\lambda_i\phi_i^{-1}}{\max _{k\in \cI } \lambda_k\phi_k^{-1}}. 
\end{gather*}
\end{assumption}

When it is clear from the context, we use $\mu_i^\star \equiv \mu^{(h_i^\star)}_i$ and $\baromega^\star_i \equiv \baromega^{(h_i^\star)}_i$.
During our analysis, we keep track of the following quantities, 
\begin{align*}
    \rho^\h_i \triangleq \frac{\mu^\h_i}{L\baromega^\star_i}, \quad \rho_i \triangleq \sum_{h=1}^H \rho^\h_i.
\end{align*}
We denote by $\rho^\star_i \equiv \rho^{(h^\star_i)}_i$.
The dynamics of $\rho^\h_i$ is given by
\begin{align}
    \partial_t \rho^\h_i = \frac{\partial_t\mu^\h_i - \partial_t \baromega^\star_i / \baromega^\star_i \cdot \mu^\h_i}{L \baromega^\star_i} = \frac{\partial_t \log \mu^\h_i - \partial_t \log{ \baromega^\star_i} }{L  \baromega^\star_i} \cdot \mu^\h_i.
    \label{eq:dot eta}
\end{align}

\subsection{Warm-up Stage}
We first give a rigorous definition of the warm-up stage.
\begin{definition}[Warm-up Stage]
    \label[definition]{def:warmup stage}
    Let $\color{brown} c >1$ be a constant.
    Let $T_{\warmup}$ be the first time that at least one of the following conditions is violated:
    \begin{myenumi}[
    label={\textbf{(A\arabic*)}}, 
    ref  = Condition {(A\arabic*)},
    ]
        \item \label{itm:warmup-lbdaU separation} ${(\muistar - \muih)}/{\muistar} \ge \epsilon/c$ holds for all $h\neq h_i^*$ and $i\in\cI$.
        \item \label{itm:warmup-lbdaW ub} $\baromega_k^\h \le 4\omega_0$ holds for all $h \in [H]$ and $k\in[d_y]$.
        \item \label{itm:warmup-lbdaU ub} $\muih \le 5L\omega_0$ holds for all $h \in [H]$ and $i\in \cI$. $\mujh \le \mujh(0)$ holds for all $h \in [H]$ and $j\in \cI_c$.
        \item \label{itm:warmup-lbdaW* grows faster} $\baromegaistar / \baromegaih \ge 1$ holds for all $h \neq h_i^*$ and $i\in\cI$.
        \item \label{itm:warmup-lbdaW*/lbdaU* ub} $\baromegaistar/\muistar \le \omega_0/\muistar(0)$  and $i\in\cI$.
        \item \label{itm:warmup-lbdaU*/lbdaW* ub} $\muistar/\baromegaistar \le 2 L \phi_i^{-1}$  and $i\in\cI$.
    \end{myenumi}
\end{definition}

By definition, the warm-up stage corresponds to the stage where $\baromega_k^\h$ is within constant factors of $\omega_0$ for all $k \in [d_y]$ and $h \in [H]$. 
Note that \Cref{itm:warmup-lbdaW ub} and \Cref{itm:warmup-lbdaU ub} already implies all the conditions in \Cref{cond:bounded weights}.
Thus, we can further simplify the dynamics based on \eqref{eq:dot baromegah} and \eqref{eq:dot muh} in the warm-up stage. 
To further decouple the dynamics for each task, we invoke the following definition.
\begin{definition}[Task-interference-free Dynamics for the Warm-up Stage]
    \label[definition]{def:TIF-warmup}
    Fix a effective or nominal task $i\in[d_y]$.
    We say that a dynamics on $\mu_i(t)$ and $\baromega_i(t)$  is a \textbf{T}ask-\textbf{I}nterference-\textbf{F}ree dynamics for task $i$ (\TIFwarmupi dynamics for short) in the warm-up stage if:
    \begin{itemize}
        \item[(\romannumeral 1)] The dynamics for $\mu_i(t)$ and $\baromega_i(t)$ are given by \eqref{eq:dot muh} and \eqref{eq:dot baromegah} respectively.
        \item[(\romannumeral 2)] The initializations of $\mu_i(0)$ and $\baromega_i(0)$ satisfy the conditions in \Cref{assump:initialization}. 
        \item[(\romannumeral 3)] For any other task $k\in[d_y]\backslash\{i\}$, $\mu_k(t)$ and $\baromega_k(t)$  satisfy \Cref{itm:warmup-lbdaU separation}-\Cref{itm:warmup-lbdaU*/lbdaW* ub} for task $k$ at any time.
    \end{itemize}
    We define $T_\warmup^i$ be the smallest time that the \TIFwarmupi dynamics violate at least one of the \Cref{itm:warmup-lbdaU separation}-\Cref{itm:warmup-lbdaU*/lbdaW* ub} for task $i$.
\end{definition}
Equipped with the conditions in \Cref{def:warmup stage} and the definition of the \TIFwarmupi dynamics for the warm-up stage, we can first simplify the dynamics of $\baromegaih$ and $\muih$.
When not specified, we use $i$ to denote an effective task and $j$ to denote a nominal task in the sequel.
\paragraph{Simplification of $\partial_t \muih$ and $\partial_t \mujh$.}
Recall by \eqref{eq:dot muh}, we have for the \TIFwarmupi dynamics of $\muih$ that
\begin{align}
    \partial_t \muih 
    &=  \left(1 \pm \xi\right)  d_i \lambda_i  \baromega_i^\h \mu_i^\h -  \sum_{h'=1}^H \left(1\pm \xi\right)  d_i \lambda_i    \baromega_i^\h  \baromega_i^\hprime \mu_i^\hprime  \mu_i^\h
    \nend
    & \qquad - \sum_{h'=1}^H \fracexproductL[\omega^\h][\omega^\hprime] d_i  \lambda_i \phi_i \mu_i^\hprime \mu_i^\h \left(1 \pm \zeta\right). 
\end{align}
where the ratio between the second and first terms is bounded by
\begin{align*}
    &\frac{\sum_{h'=1}^H \left(1\pm \xi\right)  d_i \lambda_i    \baromega_i^\h  \baromega_i^\hprime \mu_i^\hprime  \mu_i^\h}{\left(1 \pm \xi\right)  d_i \lambda_i  \baromega_i^\h \mu_i^\h} \le 2\sum_{h'=1}^H  \baromega_i^\hprime \mu_i^\hprime  \le 40 H \omega_0^2 L, 
\end{align*}
where \Cref{itm:warmup-lbdaW ub} and \Cref{itm:warmup-lbdaU ub} are used for upper bounding $\baromega_i^\hprime \mu_i^\hprime$ in the last inequality.
Also for the third term, we have $\exp\bigl(\langle \omega^\h,  \omega^\hprime \rangle\bigr) \le \exp(16 d \omega_0^2) \le 1 + 16e d \omega_0^2$ by \Cref{itm:warmup-lbdaW ub} and $\exp(x) \leq 1+ex$ for $x \le 1$.
Based on these observations, we define quantities $\tilde \xi$ and $\tilde \zeta$ as
\begin{align*}
    \color{brown} \tilde{\xi}=\xi+80 {HL \omega_0 ^{2 }} \ll 1, \qquad \tilde{\zeta}=\zeta + 32 e \omega_0^{2} d \ll 1, 
\end{align*}
and we can simplify the dynamics of $\muih$ as
\begin{align*}
    \partial_t \muih &= \lambda_i d_i \muih \left(-(1 \pm \tilde \zeta)\cdot \sum_{h'=1}^H \frac{\phi_i}{L}   \muihprime  +  (1 \pm \tilde \xi) \cdot \baromegaih \right).
\end{align*}
As a result, we also have for $\partial_{t}(\log \muistar-\log \muih) $ that
\begin{align*}
    \partial_{t}\left(\log \muistar -\log \muih \right) 
    &= \lambda_i d_i\left((1 \pm \tilde{\xi}) \baromegaistar-(1 \pm \tilde{\zeta}) \phi_i \sum_{h'=1}^{H} \frac{\muihprime}{L}\right) \\
    & \autoquad{2}- \lambda_i d_i\left((1 \pm \tilde{\xi}) \baromegaih-(1 \pm \tilde{\zeta}) \phi_i \sum_{h=1}^{H} \frac{\muihprime}{L}\right) \\
    & = \lambda_i d_i\left(\baromegaistar-\baromegaih \pm \tilde{\xi}\left(\baromegaistar+\baromegaih\right) \pm 2 \tilde{\zeta} \phi_i \sum_{h'=1}^{H} \frac{\muihprime}{L}\right).
\end{align*}
For the nominal task $j\in\cI_c$, we have for $\partial_t \mujh$ that
\begin{align}
    \label{eq:dot lbdaU nominal}
    \partial_t \mujh &= - (1 \pm \tilde \zeta)\cdot \sum_{h'=1}^H \frac{\sigma^2 d}{L} \cdot \mu_j^\hprime \mujh .
\end{align}
Note that this dynamics holds at any time since we are just using $\lambda_j = 0$ for nominal tasks.
\paragraph{Simplification of $\partial_t \baromegaih$ and $\partial_t \baromegajh$.}
Recall by \eqref{eq:dot baromegah}, we have for the dynamics of $\baromegaih$ that
\begin{align*}
    \sqrt d_e \cdot \partial_t \baromegaih
    &= - (1 \pm \zeta)\cdot \sum_{h'=1}^H 
        \fracexproductL[\omega^\h][\omega^\hprime]
        \cdot \bigl\langle \vecd\odot \lambda \odot \phi, \mu^\h \odot \mu^\hprime \bigr\rangle 
        \cdot\baromegaihprime \baromegaih
        \notag
        \\
    & \qquad - (1\pm \xi)\lambda_i \cdot \sum_{h'=1}^H  \muihprime \muih\baromegaihprime \baromegaih
    + (1\pm \xi)\lambda_i \cdot \muih \baromegaih\notag\\
    &\qquad (\pm1) \cdot \fracexproductL[\omega^\h][\omega^\h]
        \cdot \bigl\langle \vecd \odot \lambda, \baromegah \odot \mu^\h \bigr\rangle \cdot (\baromegaih)^2 
        \notag
        \\
    &\qquad \ (\pm\eta) \cdot \sum_{h'=1}^H \bigl\langle \vecd\odot \lambda \odot \phi, \mu^\h \odot \muhprime \bigr\rangle 
    \cdot (\baromegaih)^2  .
\end{align*}
Here, we can upper bound the inner product $\bigl\langle \vecd\odot \lambda \odot \phi, \mu^\h \odot \mu^\hprime \bigr\rangle$ and $\bigl\langle \vecd \odot \lambda, \baromegah \odot \mu^\h \bigr\rangle$ using \Cref{itm:warmup-lbdaW ub} and \Cref{itm:warmup-lbdaU ub} as
\begin{gather*}
    \bigl\langle \vecd\odot \lambda \odot \phi, \mu^\h \odot \mu^\hprime \bigr\rangle 
    \le \langle \vecd, \lambda \odot \phi \rangle 25 L^2 \omega_0^2, \\
    \bigl\langle \vecd \odot \lambda, \baromegah \odot \mu^\h \bigr\rangle \le \langle \vecd, \lambda \rangle 20 L \omega_0^2.
\end{gather*}
Thus, following the definition of $\tilde\zeta$ and $\tilde\xi$, 
we have for the dynamics of $\baromegaih$ that
\begin{align}
    \sqrt d_e \cdot \partial_t \baromegaih
    & =-\left(1 \pm \tilde{\zeta} \pm  H L \eta \pm  L^{-1}\right) \cdot\langle \vecd, \lambda \odot \phi \rangle 25 L \omega_0^{2} \sum_{h'=1}^{H} \baromegaihprime \baromegaih \notag\\
    & \qquad +(1 \pm \tilde{\xi}) \lambda_i \muih \baromegaih \notag\\
    & = \pm \frac{\langle \vecd, \lambda \odot \phi \rangle}{\lambda_i} \cdot \left(1 \pm \tilde{\zeta} \pm  H L \eta \pm  L^{-1}\right)\cdot 25 H L \omega_0^{2} \cdot \lambda_i \baromegaistar \baromegaih \notag\\
    &\qquad +(1 \pm \tilde{\xi}) \lambda_i  \muih\baromegaih\notag\\
    & = \pm \frac{\langle \vecd, \lambda \odot \phi \rangle}{\lambda_i} \cdot  50 H L \omega_0^{2} \cdot \lambda_i \baromegaistar \baromegaih+(1 \pm \tilde{\xi}) \lambda_i \muih \baromegaih. 
    \label{eq:dot lambda_W-warmup-1}
\end{align}
where we use \Cref{itm:warmup-lbdaW* grows faster} that $\baromegaihprime\le \baromegaistar$ in the second equality and incorporate the fact that ${\color{teal} 1 + \tilde{\zeta} + H L \eta + L^{-1} \le 2}$ in the last inequality.
Note that it is unclear which term dominates in \eqref{eq:dot lambda_W-warmup-1} since the ratio $\muih/\baromegaistar$ is unknown.
However, for $h = h_i^*$, we have a clear picture that

\begin{align*}
    \sqrt{d_e} \cdot \partial_t \baromegaistar 
    &=  \left(1 \pm \tilde{\xi} \pm \frac{\langle \vecd, \lambda \odot \phi \rangle}{\lambda_i} \cdot 50 H L \omega_0^{2} \cdot \frac{\baromegaistar}{\muistar}\right) \cdot \lambda_i \muistar \baromegaistar \nend
    &= \left(1\pm \check{\xi}_i(t)\right)  \cdot \lambda_i \muistar \baromegaistar,
\end{align*}
where we define $\check{\xi}_i(t)$ as
\begin{align*}
    \color{teal} \check{\xi}_i(t) = \tilde{\xi} + \frac{\langle \vecd, \lambda \odot \phi \rangle}{\lambda_i} \cdot 50 H L \omega_0^{2} \cdot \frac{\baromegaistar}{\muistar} \bigggiven_{t} \le \tilde{\xi} + \frac{\langle \vecd, \lambda \odot \phi \rangle}{\lambda_i} \cdot \frac{50 H L \omega_0^{3}}{\muistar(0)} \deleq \check{\xi}_i(0) \ll 1.
\end{align*}
Here, the upper bound is given by \Cref{itm:warmup-lbdaW*/lbdaU* ub} that $\baromegaistar/\muistar \le \omega_0/\muistar(0)$.
Similarly, we study the dynamics of $\log \baromegaistar - \log \baromegaih$ for $h \neq h_i^*$ as
\begin{align}
 &\sqrt{d_e} \cdot \partial_{t}\left(\log \baromegaistar-\log \baromegaih\right) \nend
 &\quad = \pm \frac{\langle \vecd, \lambda \odot \phi \rangle}{\lambda_i} \cdot 50 H L \omega_0^{2} \cdot \lambda_i \baromegaistar+(1 \pm \tilde{\xi}) \lambda_i \muistar \nend
& \autoquad{3} \pm \frac{\langle \vecd, \lambda \odot \phi \rangle}{\lambda_i} \cdot 50 H L \omega_0^{2} \cdot \lambda_i \baromegaistar-(1 \pm \tilde{\xi}) \lambda_i \muih 
\label{eq:dot log lbdaW diff-warmup-1}
\\
&\quad= \left(1 \pm \frac{\langle \vecd, \lambda \odot \phi \rangle}{\lambda_i} \cdot 100 H L \omega_0^{2} \cdot \frac{\baromegaistar}{\muistar} \cdot \frac{\muistar}{\muistar-\muih}
\pm \tilde{\xi} \cdot \frac{\muistar+\muih}{\muistar-\muih}\right) \notag\\
&\autoquad{3}  \cdot \lambda_i \cdot \left(\muistar-\muih\right) \notag\\
&\quad= \left(1 \pm \hat\xi_i(t) \right) \cdot \lambda_i \left(\muistar-\muih\right), \notag
\end{align}
where we define $\hat\xi_i(t)$ as 
\begin{align*}
    \color{teal}
    \hat\xi_i(t) = \frac{2\muistar}{\muistar-\muih} \check\xi_i(t) \le 2c\epsilon^{-1} \check\xi_i(t) \le 2c\epsilon^{-1} \check\xi_i(0) \deleq \hat\xi_i(0) \ll 1.
\end{align*}
Here, we incorporate \Cref{itm:warmup-lbdaU separation} to upper bound $\muistar/(\muistar-\muih)$ by $c\epsilon^{-1}$.
We are also interested in the dynamics of $\baromegaistar - \baromegaih$ for $h \neq h_i^*$, which can be lower bounded as
\begin{align*}
\sqrt{d_e} \cdot \partial_{t}\left(\baromegaistar-\baromegaih\right)  
&= \pm \frac{\langle \vecd, \lambda \odot \phi \rangle}{\lambda_i} \cdot 50 H L \omega_0^{2} \cdot \lambda_i\left(\baromegaistar\right)^{2}+(1 \pm \tilde{\xi}) \lambda_i \muistar \baromegaistar \\
& \qquad \pm \frac{\langle \vecd, \lambda \odot \phi \rangle}{\lambda_i} \cdot 50 H L \omega_0^{2} \cdot \lambda_i \baromegaistar \baromegaih-(1 \pm \tilde{\xi}) \lambda_i \muih {\color{magenta}\baromegaih} \\
& \ge - \frac{\langle \vecd, \lambda \odot \phi \rangle}{\lambda_i} \cdot 50 H L \omega_0^{2} \cdot \lambda_i\left(\baromegaistar\right)^{2}+(1 - \tilde{\xi}) \lambda_i \muistar \baromegaistar \\
& \qquad - \frac{\langle \vecd, \lambda \odot \phi \rangle}{\lambda_i} \cdot 50 H L \omega_0^{2} \cdot \lambda_i \baromegaistar \baromegaistar-(1 + \tilde{\xi}) \lambda_i \muih {\color{magenta}\baromegaistar} \\
& =\left(1 - \hat{\xi}_i(t)\right) \lambda_i \baromegaistar\left(\muistar-\muih\right), 
\end{align*}
where the inequality holds by invoking \Cref{itm:warmup-lbdaW* grows faster} that $\baromegaistar \ge \baromegaih$ and the last equality is given by a simple comparison to the form of \eqref{eq:dot log lbdaW diff-warmup-1}.

For the nominal task $j\in\cI_c$, we have for $\partial_t \baromegajh$ that 
\begin{align*}
    \sqrt{d_e} \partial_t \baromegajh
    &\le (\pm 1)\cdot \frac{\exp\bigl(\langle\omega^\h, \omega^\h\rangle\bigr)}{L} 
        \cdot \bigl\langle \vecd \odot \lambda, \baromegah \odot \mu^\h \bigr\rangle \cdot (\baromegajh)^2 
        \notag
        \\
    &\qquad + (\pm\eta) \cdot \sum_{h'=1}^H \bigl\langle \vecd\odot \lambda \odot \phi, \mu^\h \odot \mu^\hprime \bigr\rangle 
    \cdot (\baromegajh)^2  \\
    &\le \left(C L^{-1} \langle \vecd, \lambda \rangle 20 L \omega_0^2 + C\langle \vecd, \lambda \odot \phi \rangle 25 L^2 \omega_0^2 \eta \right) \cdot (\baromegajh)^2 \\
    &\le C\left(20 +  25 L \zeta \right) \cdot \langle \vecd, \lambda \odot \phi \rangle\cdot  \omega_0^2 \cdot (\baromegajh)^2 \\
    &\le \underbrace{26CL\zeta\cdot \langle \vecd, \lambda \odot \phi \rangle\cdot  \omega_0^2}_{\ds \color{teal}\ll 1} \cdot (\baromegajh)^2, 
\end{align*}
where in the last second inequality, we use the fact that $\zeta \ge L \eta \ge L^{-1}$.

\paragraph{Simplification of $\partial_t \rhoistar$ and $\partial_t \rhoi$.}
Recall from \eqref{eq:dot eta} that  the dynamics of $\rhoistar$ satisfy 
\begin{align*}
    \partial_t \rhoistar 
    &= \frac{\partial_t \log \muistar - \partial_t \log{\baromegaistar} }{L \baromegaistar} \cdot \muistar \nend 
    & = \frac{\lambda_i d_i \left(- (1 \pm \tilde \zeta)\cdot \sum_{h'=1}^H \frac{\phi_i}{L}   \muihprime  +  (1 \pm \tilde \xi) \cdot \baromegaistar \right) - (1\pm \check{\xi}_i(t)) \cdot  \sqrt{d_e}^{-1}\lambda_i \muistar }{L \baromegaistar} \cdot \muistar \nend
    & = \frac{\lambda_id_i \muistar}{L} \cdot\left(
        -(1 \pm \tilde \zeta)\cdot \phi_i \rhoi +  (1 \pm \tilde \xi) - \frac{(1\pm \check{\xi}_i(t))}{\sqrt{d_e}d_i} \cdot \frac{\muistar}{\baromegaistar }
    \right).
\end{align*}
Here, we invoke the ratio argument in \Cref{itm:warmup-lbdaU*/lbdaW* ub} that $\muistar/\baromegaistar \le 2L \phi_i^{-1}$ and ${\color{teal}\check\xi_i(t)\le \check\xi_i(0) \ll 1}$ to upper bound the last term by $\color{teal} \varsigma_i\deleq {L }/{(\sqrt{d_e}d_i \phi_i)}\ll 1 $.
Hence, we have for the dynamics of $\rhoistar$ that
\begin{align*}
    \partial_t \rhoistar 
    &= L^{-1} \lambda_id_i \muistar  \cdot\left(
        -(1 \pm \tilde \zeta)\cdot \phi_i \rhoi +  (1 \pm \tilde \xi \pm 4\varsigma_i)
    \right).
\end{align*}
Similarly, we have for the dynamics of $\rhoi$ that
\begin{align*}
    \partial_t \rhoi &=  \sum_{h=1}^H \frac{\lambda_i d_i \left(- (1 \pm \tilde \zeta)\cdot \sum_{h'=1}^H \frac{\phi_i}{L}   \muihprime  +  (1 \pm \tilde \xi) \cdot \baromegaih \right) - (1\pm \check{\xi}_i(t)) \cdot  \sqrt{d_e}^{-1}\lambda_i \muistar }{L \baromegaistar} \cdot \muih \nend
    & \le \lambda_i d_i \left(- (1 \pm \tilde \zeta)\cdot \phi_i \baromegaistar \rhoi^2 +  (1 \pm \tilde \xi) \cdot \frac{\sum_{h=1}^H \baromegaih \muih}{L\baromegaistar}  \right) \nend 
    & \le \lambda_i d_i \rhoi \baromegaistar \left(- (1 \pm \tilde \zeta)\cdot \phi_i  \rhoi +  (1 \pm \tilde \xi) \right).
\end{align*}
where in the second inequality, we invoke \Cref{itm:warmup-lbdaW* grows faster} that $\baromegaih\le \baromegaistar$.

As a summary, we have the following dynamics for the warm-up stage:
\begin{center}
    \begin{table}[h]
    \begin{tabular}{l@{\hspace{1cm}} l}
        \hline
        \hline
        \textbf{Terms} & \textbf{Simplified Dynamics without Cross-task Interference} \\
        \hline
        $\partial_t \muih$ & $ \lambda_i d_i \muih \bigl(-(1 \pm \tilde \zeta)\cdot \phi_i L^{-1}\sum_{h'=1}^H    \muihprime  +  (1 \pm \tilde \xi) \cdot \baromegaih \bigr)$ \\
        $\partial_t \baromegaistar$ & $(1\pm \check{\xi}_i(t)) \cdot  \sqrt{d_e}^{-1}\lambda_i \muistar \baromegaistar$ \\
        $\partial_t (\log \muistar - \log \muih)$ & $\lambda_i d_i\bigl(\baromegaistar-\baromegaih \pm \tilde{\xi}\bigl(\baromegaistar+\baromegaih\bigr) \pm 2 \tilde{\zeta} \phi_i \sum_{h'=1}^{H} \frac{\muihprime}{L}\bigr)$ \\
        $\partial_t (\log \baromegaistar - \log \baromegaih)$ & $(1 \pm \hat\xi_i(t) ) \cdot \sqrt{d_e}^{-1} \lambda_i \bigl(\muistar-\muih\bigr)$ \\
        $\partial_t (\baromegaistar - \baromegaih)$ & $\ge (1 - \hat{\xi}_i(t)) \cdot \sqrt{d_e}^{-1} \lambda_i \baromegaistar\bigl(\muistar-\muih\bigr)$ \\
        $\partial_t \rhoistar$ & $ L^{-1} \lambda_id_i \muistar  \cdot\bigl(
            -(1 \pm \tilde \zeta)\cdot \phi_i \rhoi +  (1 \pm \tilde \xi \pm 4\varsigma_i)
        \bigr)$ \\
        $\partial_t \rhoi$ & $ \le \lambda_i d_i \rhoi \baromegaistar \left(- (1 \pm \tilde \zeta)\cdot \phi_i  \rhoi +  (1 \pm \tilde \xi) \right)$ \\
        $\partial_t\mujh, j\in\cI_c$ & $ - \mujh (1 \pm \tilde \zeta)\cdot \sum_{h'=1}^H \frac{\sigma^2 d d_y}{L} \cdot \muhprime_j$ \\
        $\partial_t \baromegajh, j\in\cI_c$ & $\le 26CL\zeta \sqrt{d_e}^{-1}\cdot \langle \vecd, \lambda \odot \phi \rangle\cdot  \omega_0^2$ \\
        \hline
    \end{tabular}
    \caption{Simplified \TIFwarmupi Dynamics for the Warm-up Stage before $T_\warmup^i$}
    \label{tab:warmup simplified dynamics}
    \end{table}
\end{center}

Under the simplified dynamics in \Cref{tab:warmup simplified dynamics}, one could notice that tasks are decoupled in the sense that the above dynamics are \say{independent} for each task $i \in \cI$ while only cross-head interference exists.
Based on this observation, we can further split the dynamics for each task into 4 stages. 
\begin{definition}[Split of the Effective \TIFwarmupi Dynamics for the Warm-up Stage]
    \label[definition]{def:warmup stage-split}
    Consider a specific task $i \in \cI$ with the corresponding \TIFwarmupi dynamics.
Define $\vartheta \deleq \frac{\min_{i\in\cI}\phi_i\lambda_i^{-1}}{\max_{j\in\cI}\phi_j\lambda_j^{-1}}\le 1$ as the ratio between the smallest and largest $\phi_i\lambda_i^{-1}$ within the effective task set $\cI$.
Let $\alpha_i, \beta_i$, and $\gamma_i$ be three small constants for each $i \in \cI$ such that satisfy
\begin{gather*}
    \color{brown} \max \left\{  \frac{\varsigma_i}{\vartheta} \left(\log \left(\frac{2L  \omega_0}{\phi_i\muistar(0)}\right) \pmax H\right),  
        \tilde\xi + 2\tilde\zeta + 4\varsigma_i, 
        \frac{18e  (\tilde \xi+\tilde\zeta)}{(1-c^{-1})\epsilon} \cdot \log \left(\frac{2CL  \omega_0}{\phi_i\muistar(0)}\right) \right\}\ll \alpha_i \ll 1, \\
    \color{brown} \max\left\{
        \frac{c H (\tilde\zeta +\tilde\xi)}{\epsilon \vartheta}, 
        \frac{c(\tilde\zeta + \tilde\xi)^2 H}{(1-c^{-1})\epsilon^2 \varsigma_i}
    \right\}\ll\beta_i \ll 1,  \\
    \color{brown}  H\cdot \exp\left(-\frac{\epsilon \vartheta^2}{64c H \varsigma_i}\right) \ll \gamma_i \ll \min\left\{\alpha_i,  \frac{H\lambda_i\phi_i^{-1}}{\max_{k\in\cI} \lambda_k\phi_k^{-1}}\right\}, 
\end{gather*}
We say that the \TIFwarmupi dynamics with $i\in\cI$ fall into one of the following five steps if both (\romannumeral 1) \Cref{itm:warmup-lbdaU separation}-\Cref{itm:warmup-lbdaU*/lbdaW* ub} for task $i$ and (\romannumeral 2) additional condition(s) for that corresponding step listed in the following are satisfied:
\begin{myenumi}[
    label = {\textbf{Step \arabic*}.},
    ref   = {Cond of Step \arabic*},
    leftmargin=2cm
    ]
    \item \label{itm:warmup-step1} $\rho_i\phi_i \le 1-\alpha_i$.
    \item The following conditions hold:
        \begin{myenumi}
            \item \label{itm:warmup-step2-eta} $\rhoistar\phi_i\ge (1-\alpha_i)/H$ and the dynamics have been through Step 1.
            \item \label{itm:warmup-step2-lbdaW-diff} There exists at least one $h \in [H]\backslash\{h_i^*\}$ such that $(\baromegaistar - \baromegaih) / \baromegaistar < (3\tilde\zeta + 2\tilde\xi) \beta_i^{-1}$.
        \end{myenumi}
    \item The following conditions hold:
        \begin{myenumi}
            \item \label{itm:warmup-step3-lbdaW-diff} $(\baromegaistar - \baromegaih) / \baromegaistar \ge (3\tilde\zeta + 2\tilde\xi) \beta_i^{-1}$ for any $h \in [H]\backslash\{h_i^*\}$.
            \item \label{itm:warmup-step3-lbdaU-dominance} There exists at least one $h \in [H]\backslash\{h_i^*\}$ such that $\muistar / \muih < (H-1)/\gamma_i$. 
        \end{myenumi}
    \item the following conditions hold:
    \begin{myenumi}
        \item \label{itm:warmup-step4 lbdaU ratio}
        $\muistar / \muih \ge (H-1)/\gamma_i$ for all $h \in [H]\backslash\{h_i^*\}$.
        \item \label{itm:warmup-step4 eta*}
        $\rhoistar \phi_i < 1-\alpha_i$.
    \end{myenumi}
    \item \label{itm:warmup-step5} $\rhoistar\phi_i \ge 1-\alpha_i$.
\end{myenumi}
\end{definition}
We observe that there is no overlap between these five steps by definition.


Under the simplified dynamics, one can notice that the conditions in \Cref{def:warmup stage} are related in the sense that some of these conditions are critical while others are just a byproduct of the dynamics.
In the following proposition, we present such relations and show how the \say{critical} conditions imply the others.
\begin{lemma}[Critical Conditions]
    \label[lemma]{fact:critical conditions}
    For any $t \in (0, \infty)$, we have the following facts for the \TIFwarmupi dynamics where $i\in\cI$:
    \begin{itemize}
        \item[(\romannumeral 1)] If \Cref{itm:warmup-lbdaU separation}-\Cref{itm:warmup-lbdaU*/lbdaW* ub} hold for $\tau\in[0, t)$, then \Cref{itm:warmup-lbdaW* grows faster}-\Cref{itm:warmup-lbdaU*/lbdaW* ub} always hold with some marginal gap in the inequalities at time $t$.;
        \item[(\romannumeral 2)] If \Cref{itm:warmup-lbdaW ub} holds with some marginal gap at $t$ and \Cref{itm:warmup-lbdaU separation} also holds at $t$, then \Cref{itm:warmup-lbdaW ub} and \Cref{itm:warmup-lbdaU ub} hold with some marginal gap in the inequalities for any $\tau\in (0, t]$;
    \end{itemize}
    In particular, by scrutinizing the simplified dynamics in \Cref{tab:warmup simplified dynamics}, we have $
        \rho_i\phi_i \le 1 + \tilde\xi + 2\tilde\zeta
    $ and that both $\baromegaistar$ and $\baromegaistar/\baromegaih$ are increasing for any $h \in [H]\backslash\{h_i^*\}$.
    For the nominal task $j\in\cI_c$, we have $\mujh$ nonincreasing for any $h \in [H]$, which holds not only for the warm-up stage but also for the entire training process.
\end{lemma}
\begin{proof}
    It is obvious that $\baromegaistar$ is increasing by just the dynamics in \Cref{tab:warmup simplified dynamics}: $\partial_t \baromegaistar = (1\pm \check{\xi}_i(t)) \cdot  \sqrt{d_e}^{-1}\lambda_i \muistar \baromegaistar >0$.
    It is also direct that $\mujh$ is nonincreasing for any $h \in [H]$ by the dynamics in \Cref{tab:warmup simplified dynamics} for the nominal task $j\in\cI_c$. 
    For \Cref{itm:warmup-lbdaW* grows faster}, \Cref{itm:warmup-lbdaW*/lbdaU* ub}  and \Cref{itm:warmup-lbdaU*/lbdaW* ub}, we aim to show that the dynamics for $0 < \tau < t$ already implies some marginal gap in these conditions for time $t$.
    
    For \Cref{itm:warmup-lbdaW* grows faster}, we have for $0 < \tau < t$ that
    \begin{align}
        \partial_t (\log \baromegaistar - \log \baromegaih) \biggiven_{\tau} 
        &= \left(1 \pm \hat\xi_i(\tau) \right) \cdot \lambda_i \left(\muistar-\muih\right)\biggiven_{\tau} \notag \\
        & > \left(1 \pm \hat\xi_i(0) \right) \cdot \lambda_i \epsilon c^{-1} \muistar(\tau) > 0, 
        \ \mytag{by \ \Cref{itm:warmup-lbdaU separation}}
    \end{align}
    we directly conclude that $\baromegaistar(t) > \baromegaih(t)$  holds for all $h\neq h_i^*$ by integrating the above inequality from $0$ to $t$ and noting that $\baromegaistar(0) = \baromegaih(0) =\omega_0$.
    This also verifies the claim that $\baromegaistar/\baromegaih$ is increasing for any $h \neq h_i^*$.
    
    For \Cref{itm:warmup-lbdaW*/lbdaU* ub}, 
    recall that we have for $\partial_t \rhoistar$ that
\begin{align*}
    \partial_t \rhoistar
    &= L^{-1} \lambda_id_i \muistar  \cdot\left(
        -(1 \pm \tilde \zeta)\cdot \phi_i \rho_i +  (1 \pm \tilde \xi \pm 4\varsigma_i)\right) \notag\\
    &\ge L^{-1} \lambda_id_i \muistar  \cdot\left(
        -(1 \pm \tilde \zeta)\cdot \phi_i H \rhoistar +  (1 \pm \tilde \xi \pm 4\varsigma_i)\right) \mytag{\Cref{itm:warmup-lbdaU separation}} \\
    &\ge (1 + \tilde \zeta)^{-1} L^{-1} \lambda_id_i \muistar  \cdot\Bigl(
        -\phi_i H \rhoistar +  \bigl(1 - \underbrace{(\tilde \xi + 4\varsigma_i + 2\tilde\zeta)}_{\ds \ll \alpha_i}\bigr)\Bigr).
\end{align*}
    As long as $\rhoistar < (1-\alpha_i) \phi_i^{-1} H^{-1}$, we have $\partial_t \rhoistar \ge 0$, which can also be translated as whenever $\muistar / \baromegaistar < (1-\alpha_i) \phi_i^{-1} H^{-1} L$, we will have the ratio $\muistar / \baromegaistar$ increasing.
    Note that at initialization, we have $\muistar(0) / \baromegaistar(0) = \muistar(0) \omega_0^{-1} \le L / (2 H \phi_i) < (1-\alpha_i) \phi_i^{-1} H^{-1} L$ by \Cref{assump:initialization}.
    Hence, the ratio $\muistar / \baromegaistar$ will keep increasing until it reaches $(1-\alpha_i) \phi_i^{-1} H^{-1} L$, and then remains above this threshold. 
    Hence, we have for any $t \in (0, T_\warmup^i]$ that
    \begin{align*}
        \frac{\muistar(t)}{\baromegaistar(t)} 
        &\ge \min\left\{
            \frac{\muistar(t)}{\baromegaistar(t)}, 
            (1-\alpha_i) \phi_i^{-1} H^{-1} L
        \right\} \nend
        &= \begin{cases}
            \frac{\muistar(t)}{\baromegaistar(t)} & \text{before } \frac{\muistar(t)}{ \baromegaistar(t)} \text{ reaches } (1-\alpha_i) \phi_i^{-1} H^{-1} L, \\
            (1-\alpha_i) \phi_i^{-1} H^{-1} L & \text{otherwise}.
        \end{cases}
    \end{align*}
    In both case, we have $\frac{\muistar(t)}{\baromegaistar(t)} > \frac{\muistar(0)}{\baromegaistar(0)}$ holding strictly, which verifies \Cref{itm:warmup-lbdaW*/lbdaU* ub}.

    Next, for \Cref{itm:warmup-lbdaU*/lbdaW* ub}, we upper bound the gradient for $\rho_i$ as
    \begin{align*}
        \partial_t \rho_i 
        &\le L^{-1} \lambda_id_i \muistar  \cdot\left(
            -(1 \pm \tilde \zeta)\cdot \phi_i \rho_i +  (1 \pm \tilde \xi)\right) \\
        &\le (1 + \tilde \zeta) L^{-1} \lambda_id_i \muistar  \cdot\left(
            - \phi_i \rho_i +  (1 + \tilde \xi + 2 \tilde\zeta)\right).
    \end{align*}
    Thus, we conclude that if we initialize with $\rho_i\phi_i < 1$, $\rho_i\phi_i$ will not exceed $1 + \tilde\xi + 2\tilde\zeta$ during the warm-up stage.
    As a result, we will always have 
    $\muistar \le \rho_i L \baromegaistar \le (1 + \tilde\xi + 2\tilde\zeta) L \baromegaistar \phi_i^{-1}< 2 L \baromegaistar\phi_i^{-1}$ for $t \in (0, T_\warmup^i]$, which directly justifies that \Cref{itm:warmup-lbdaU*/lbdaW* ub} holds with some marginal gap.

    Lastly, for showing that 
    \Cref{itm:warmup-lbdaU separation} and \Cref{itm:warmup-lbdaW ub} at time $t$ implies \Cref{itm:warmup-lbdaU ub} and \Cref{itm:warmup-lbdaW ub} for any $\tau\in (0, t]$, we just invoke the previous bound $\muistar/\baromegaistar \le (1 + \tilde\xi + 2\tilde\zeta) L < 1.25 L$. 
    Thus, if $\baromegaistar(t) < 4\omega_0$ strictly holds under \Cref{itm:warmup-lbdaW ub}, we have that 
    \[
        \muistar(\tau) < 1.25 L \baromegaistar(\tau) \le 1.25 L \baromegaistar(t) \le 5L\omega_0, \quad \forall \tau \in (0, t].
    \]
    Here the second inequality holds by the monotonicity of $\baromegaistar$.
    As a result, we conclude that 
    \begin{align*}
        \muih(\tau) \le \muistar(\tau) < 5L\omega_0, \quad \baromegaih(\tau) \le \baromegaistar(\tau) < \baromegaistar(t) < 4\omega_0, \quad \forall \tau \in (0, t], h \in [H],
    \end{align*}
    where for the non-optimal heads, we just invoke \Cref{itm:warmup-lbdaU separation} to have $\muih \le \muistar$ and the increasing in $\baromegaistar / \baromegaih$ that $\baromegaistar / \baromegaih \given_\tau \ge \baromegaistar / \baromegaih \given_0 = 1$.
\end{proof}

The message of \Cref{fact:critical conditions} is that we can just focus on verifying \Cref{itm:warmup-lbdaU separation} and \Cref{itm:warmup-lbdaW ub} at one time stamp $t$ as the critical conditions for the warm-up stage to hold for any $\tau \in (0, t]$.
In other words, if all the conditions in \Cref{def:warmup stage} are satisfied for time $\tau \in (0, t)$, and the critical conditions \Cref{itm:warmup-lbdaU separation} and \Cref{itm:warmup-lbdaW ub} are still satisfied with some marginal gap at time $t$, then all the conditions in \Cref{def:warmup stage} are also satisfied at time $t$ with some marginal gap, which implies that for a sufficiently small time interval $\Delta t$, we are still in the warm-up stage during $[t, t+\Delta t)$.
Next, we study each of the 5 steps in \Cref{def:warmup stage-split} for an effective task $i\in\cI$.

\subsubsection{Step 1: Growth of $\rho_i$}
In the first step of the warm-up stage, we show that $\rho_i\phi_i$ increases to $1-\alpha_i$.
\paragraph{Dynamics enter Step 1 at the beginning of the Warm-up Stage.}
    We first show that the warm-up stage starts with Step 1, i.e., at initialization, all the conditions in \Cref{def:warmup stage} for task $i$ are satisfied.
    We have by \Cref{assump:initialization} that $\brown {(\muistar - \muih)}/{\muistar} \ge \epsilon > \epsilon / c$ for all $h \neq h_i^*$, which implies that \Cref{itm:warmup-lbdaU separation} holds.
    \Cref{itm:warmup-lbdaU ub}, \Cref{itm:warmup-lbdaU*/lbdaW* ub} and \Cref{itm:warmup-step1} also hold by condition  $\brown \muih(0) \le \muistar(0) \le L\omega_0 / (2H\phi_i)$ in \Cref{assump:initialization}.
    The remaining conditions are automatically satisfied by the initialization of $\baromegaih = \omega_0$.

    \paragraph{Upper Bounding the Duration of Step 1.}
    Let $t_1^i$ denote the time when the dynamics for task $i$ exits Step 1.
    Using the dynamics for $\muistar$ and $\baromegaistar$ in \Cref{tab:warmup simplified dynamics}, we have for $t \in (0, t_1^i)$ that
    \begin{align}
        \frac{\partial_t \muistar}{\partial_t\baromegaistar} &= \frac{\lambda_i d_i \muistar \cdot \left(
            \baromegaistar (1\pm \tilde\xi) - (1 \pm \tilde\zeta)\cdot \sum_{h'=1}^H \phi_i \frac{\muihprime}{L} \right)}{\sqrt{d_e}^{-1} \lambda_i \muistar \baromegaistar \cdot \left(1 \pm \check\xi_i(0) \right)} \nend
        & = \frac{d_i \sqrt{d_e}\cdot \left(
            (1 \pm \tilde\xi) - (1 \pm \tilde\zeta)\cdot \phi_i \rho_i \right) }{1 \pm \check\xi_i(0)}
            \ge \frac{d_i \sqrt{d_e}\cdot (\alpha_i - \tilde\xi -\tilde\zeta)}{2},
            \label{eq:dot lbdaU*/dot lbdaW*}
    \end{align}
    where in the last inequality we use the fact that $\rho_i\phi_i \le 1 - \alpha_i$ by \Cref{itm:warmup-step1} and the small error condition $\check\xi_i(0) \ll 1$.
    Thus, we have for $t_1^i$ that 
\begin{align*}
    (\baromegaistar(t_1^i) - \omega_0) \cdot \frac{d_i \sqrt{d_e}\cdot (\alpha_i - \tilde\xi -\tilde\zeta)}{2} \le \muistar(t_1^i) - \muistar(0) \le \sum_{h'=1}^H \muihprime(t_1^i) \le
    L \phi_i^{-1}\baromegaistar,
\end{align*}
where the last inequality holds by using \Cref{itm:warmup-step1}. Defining $\tilde\alpha_i = \alpha_i -\tilde\xi - \tilde\zeta$ and rearranging the above inequality, we have 
\begin{align}
    \baromegaistar(t_1^i) \le \omega_0\cdot \left(1 - \frac{2L}{d_i\sqrt{d_e}  \phi_i(\alpha_i - \tilde\xi - \tilde\zeta)}\right)^{-1} = \omega_0 \cdot \left(1 - \frac{2 \varsigma_i}{ \tilde\alpha_i}\right)^{-1} \brown = (1 + o(1)) \omega_0.
    \label{eq:lbdaW* upper bound-warmup-stage1}
\end{align}
For this upper bound to make sense, we require that 
$\brown \alpha_i  \gg \tilde\xi + \tilde\zeta + 2\varsigma_i$.
Consequently, we have for $\muistar(t_1^i)$ that
\begin{align}
    \muistar(t_1^i) \le \sum_{h'=1}^H \muihprime(t_1^i) \le L \rho_i(t_1^i) \baromegaistar(t_1^i) \le L \phi_i^{-1} \omega_0 \cdot \left(1 - \frac{2 \varsigma_i}{ \tilde\alpha_i}\right)^{-1}. 
    \label{eq:lbdaU* upper bound-warmup-stage1}
\end{align}
We can then upper bound the duration of Step 1 by studying the dynamics for $\baromegaistar$ as
\begin{align*}
    \partial_t \log\muistar = \lambda_i d_i \baromegaistar \cdot \left(-(1 \pm \tilde \zeta)\cdot \phi_i \rho_i  +  (1 \pm \tilde \xi)  \right)
    \ge \lambda_id_i \omega_0 \tilde\alpha_i, 
\end{align*}
where in the last inequality, we use the fact that $\partial_t\baromegaistar >0$ to establish the monotonicity of $\baromegaistar$ and lower bound $\baromegaistar$ by its initialization value $\omega_0$.
Moreover, the $\alpha - \tilde\xi - \tilde\zeta$ term comes from the same argument as in \eqref{eq:dot lbdaU*/dot lbdaW*}.
As a result, we have for $t \in (0, t_1^i]$ that
$\muistar(t)/\muistar(0) \ge \exp\left(\lambda_id_i \omega_0 \tilde\alpha_i t\right)$. 
Combining this lower bound with the upper bound in \eqref{eq:lbdaU* upper bound-warmup-stage1}, we have for $t \in (0, t_1^i]$ that
\begin{align*}
    \exp\left(\lambda_id_i \omega_0 \tilde\alpha_i t_1^i \right) \le \frac{\muistar(t_1^i)}{\muistar(0)} \le \frac{L \phi_i^{-1} \omega_0}{\muistar(0)} \cdot \left(1 - \frac{2 \varsigma_i}{ \tilde\alpha_i}\right)^{-1},
\end{align*}
which implies the upper bound for $t_1^i$ as
\begin{align*}
    t_1^i 
    &\le \frac{1}{\lambda_id_i \omega_0 \tilde\alpha_i} \cdot \log \left(\frac{L \phi_i^{-1} \omega_0}{\muistar(0) \left(1 - 2 \varsigma_i \tilde\alpha_i^{-1}\right)}\right) \\
    &\le \frac{1}{\lambda_id_i \omega_0 \tilde\alpha_i} \cdot \log \left(\frac{2L  \omega_0}{\phi_i\muistar(0)}\right)
    \deleq t_{1+}^{i}.
\end{align*}

\paragraph{Only \Cref{itm:warmup-step1} is violated at $t_1^i$.}
Here, we aim to show that for a sufficiently small time beyond $t_1^i$, only \Cref{itm:warmup-step1} is violated.
To do so, it suffices to show that the remaining conditions in \Cref{def:warmup stage} are satisfied for $t\in(0, t_1^i]$ and also has some marginal at time $t_{1}^i$ such that they will not be violated even if we go beyond $t_{1}^i$ for a sufficiently small amount of time. 
In particular, when appealing to \Cref{fact:critical conditions}, we only need to show that \Cref{itm:warmup-lbdaU separation} and \Cref{itm:warmup-lbdaW ub} are satisfied for $t \in (0, t_1^i]$ and also have some marginal gap at $t_1^i$.

Using the upper bound for $\baromegaistar$ \eqref{eq:lbdaW* upper bound-warmup-stage1} and the fact that $\baromegaistar$ monotonically increases, we have 
\Cref{itm:warmup-lbdaW ub} satisfied with sufficient marginal gap for $t \in (0, t_{1}^i]$. 
The last thing to verify is \Cref{itm:warmup-lbdaU separation}.
As we do not have the property that this ratio is monotonically increasing, we need to invoke the upper bound on the duration of Step 1 to show that this ratio must not decrease too much.
The following proposition rigorously shows this fact. 
\begin{lemma}[Lower Bounding the Separation for \Cref{itm:warmup-lbdaU separation}]
    \label[lemma]{prop:lbdaU separation lower bound}
    Let
    \begin{align*}
        T_\Uthres \deleq \frac{\epsilon}{18 e (\tilde \xi+\tilde\zeta) \lambda_id_i \omega_0}.
    \end{align*}
    For any $t\in [0, T_\Uthres \pmin T_\warmup^i]$, we have for the \TIFwarmupi dynamics that
    \[
    \frac{\muistar(t) - \muih(t)}{\muistar(t)} \ge \epsilon \cdot \left(1  -  \frac{t}{T_\Uthres} \right).
\]
In particular, if $t \ll (1-c^{-1}) T_\Uthres$, we have \Cref{itm:warmup-lbdaU separation} strictly satisfied.
\end{lemma}
\begin{proof}
To provide a lower bound for the separation $(\muistar - \muih)/\muistar$, we first derive a lower bound for $\partial_t \log \muistar - \partial_t \log \muih$ at any $\tau \in (0, t)$ as
\begin{align}
    \left.\partial_{t}\left(\log \frac{\muistar}{\muih}\right)\right|_{\tau} 
    & = \left. \lambda_i d_i\left(\baromegaistar-\baromegaih \pm \tilde{\xi}\left(\baromegaistar +\baromegaih\right) \pm 2 \tilde{\zeta} \phi_i \sum_{h''=1}^{H} \frac{\muihprime}{L}\right)\right|_{\tau}\notag\\
    &\ge \left. -\tilde\xi \lambda_i d_i \left(\baromegaistar + \baromegaih\right) - 2\tilde\zeta \lambda_i d_i \phi_i \rho_i \baromegaistar\right|_{\tau} \notag\\
    &\ge \left.- 2(\tilde \xi+\tilde\zeta) \lambda_id_i \cdot \baromegaistar (1 + \phi_i\rho_i)\right|_{\tau}, \notag
\end{align}
where the first and the second inequality hold by \Cref{itm:warmup-lbdaW* grows faster} that $\baromegaistar \ge \baromegaih$. 
Using the upper bound $\phi_i\rho_i \le 1 + \tilde\xi + 2\tilde\zeta < 1.25$ from \Cref{fact:critical conditions} together with the upper bound $\baromegaistar \le 4\omega_0$ by \Cref{itm:warmup-lbdaW ub}, we conclude that 
\begin{align*}
    \left.\partial_{t}\left(\log \frac{\muistar}{\muih}\right)\right|_{\tau} \ge - 18 (\tilde \xi+\tilde\zeta) \lambda_id_i \omega_0, 
\end{align*}
Thus, for all $t \in (0, T_\Uthres \pmin T_\warmup^i]$, we integrate the above inequality from $0$ to $t$ to obtain
\begin{align*}
    &\log\left(\frac{\muistar(t)}{\muih(t)}\right) - \log\left(\frac{\muistar(0)}{\muih(0)}\right) \ge - 18 (\tilde \xi+\tilde\zeta) \lambda_id_i \omega_0 \cdot t  \ge -1, 
\end{align*}
where the last inequality holds by definition of $T_\Uthres$.
Given the initial condition $\muih(0) / \muistar(0) \le (1-\epsilon)$, we have for $t \in (0, T_\Uthres \pmin T_\warmup^i]$ that
\begin{align*}
    \frac{\muih(t)}{\muistar(t)}
    &\le \frac{\muih(0)}{\muistar(0)} \cdot \exp\left( 18 (\tilde \xi+\tilde\zeta) \lambda_id_i \omega_0 \cdot t\right) \notag\\
    &\le (1-\epsilon) \cdot \left( 1 + 18 e (\tilde \xi+\tilde\zeta) \lambda_id_i \omega_0 \cdot t\right) \notag\\
    &\le 1 - \epsilon \cdot \left(1  -  18 e (\tilde \xi+\tilde\zeta) \lambda_id_i \omega_0 \cdot t \epsilon^{-1} \right), 
\end{align*}
which implies the result
\[
    \frac{\muistar(t) - \muih(t)}{\muistar(t)} \ge \epsilon \cdot \left(1  -  18 e (\tilde \xi+\tilde\zeta) \lambda_id_i \omega_0 \cdot t \epsilon^{-1} \right) = \epsilon \cdot \left(1  -  \frac{t}{T_\Uthres} \right).
\]
as we claimed in \Cref{prop:lbdaU separation lower bound}.
\end{proof}

Note that $t_1^i \le T_\warmup^i$ and the upper bound $t_{1+}^i$ satisfies
\begin{align*}
    t_{1+}^i 
    &= \frac{1}{\lambda_id_i \omega_0 \tilde\alpha_i} \cdot \log \left(\frac{2L  \omega_0}{\phi_i\muistar(0)}\right) \ll \frac{\epsilon (1-c^{-1})}{18 e (\tilde \xi+\tilde\zeta) \lambda_id_i \omega_0} = (1 - c^{-1}) T_\Uthres.
\end{align*}
given that $\brown \alpha_i \gg \frac{18e  (\tilde \xi+\tilde\zeta)}{(1-c^{-1})\epsilon} \cdot \log \left(\frac{2CL  \omega_0}{\phi_i\muistar(0)}\right)$. 
Hence, we conclude from \Cref{prop:lbdaU separation lower bound} that for $t \in (0, t_{1}^i]$, \Cref{itm:warmup-lbdaU separation} holds with sufficient margin gap.
The following table summarizes the dynamics for $\muistar$ and $\baromegaistar$ in Step 1 of the warm-up stage.
\begin{table}[ht]
    \centering
    \begin{tabular}{l@{\hspace{1cm}} l}
        \hline
        \hline
        \textbf{Properties} & \textbf{Typical Values} \\
        \hline
        Duration & $t_1^i \le t_{1+}^i = \frac{1}{\lambda_id_i \omega_0 \tilde\alpha_i} \cdot \log \left(\frac{2L  \omega_0}{\phi_i\muistar(0)}\right)$ \\
        $\muistar \nearrow$ & $\muistar(0) \le \muistar (t) \le L \phi_i^{-1} \omega_0 \cdot \left(1 - \frac{2 \varsigma_i}{ \tilde\alpha_i}\right)^{-1}$ \\
        $\baromegaistar \nearrow$ & $\omega_0 \le \baromegaistar(t) \le \omega_0 \cdot \left(1 - \frac{2 \varsigma_i}{ \tilde\alpha_i}\right)^{-1}$ \\
        $(\muistar - \muih)/ \muistar$ & lower bounded by $\epsilon \cdot \left(1  - \frac{4e (\tilde \xi+\tilde\zeta)  }{\epsilon \tilde\alpha_i} \cdot \log \left(\frac{2CL  \omega_0}{\phi_i\muistar(0)}\right)\right)$ \\
        \hline
    \end{tabular}
    \caption{Summary of Step 1 of the Warm-up Stage}
    \label{tab:summary-warmup-step 1}
\end{table}

\paragraph{Analysis of $\rhoistar$ and $\rho_i$ after Step 1.}
Before we dive into the analysis of Step 2, we first understand two key features of the dynamics: (i) $\rhoistar\phi_i\ge (1-\alpha_i)/H$; (ii) $\rho_i\phi_i \le 1 + \tilde\xi + 2\tilde\zeta$.
These properties will preserve in the following steps and will be frequently used for our analysis throughout the warm-up stage.
Recall that we have for $\partial_t \rhoistar$ that
\begin{align}
    \partial_t \rhoistar
    &\ge (1 + \tilde \zeta)^{-1} L^{-1} \lambda_id_i \muistar  \cdot\Bigl(
        -\phi_i H \rhoistar +  \bigl(1 - \underbrace{(\tilde \xi + 4\varsigma_i + 2\tilde\zeta)}_{\ds \ll \alpha_i}\bigr)\Bigr), 
    \label{eq:dot eta* lower bound}
\end{align}
as we have shown in the proof of \Cref{fact:critical conditions}.
At time $t_1^i$, we already have $\rhoistar(t_1^i)\phi_i \le (1-\alpha_i)/H$, and we have $\rhoistar$ still growing for a while after $t_1^i$ by \eqref{eq:dot eta* lower bound} since $\brown \tilde\xi + 4\varsigma_i + 2\tilde\zeta \ll \alpha_i$.
That is, as long as $\rhoistar\phi_i$ lies between $(1-\alpha_i)/H$ and $(1 - \tilde \xi - 4\varsigma_i - 2\tilde\zeta)/H$, we will have $\partial_t \rhoistar > 0$.
Therefore, the first condition $\rhoistar\phi_i \ge (1-\alpha_i)/H$ is preserved in the following steps for $t \in [ t_1^i, T_\warmup^i ]$.
A direct conclusion from this property is that 
\begin{align}
    \check{\xi}_i(t) \given_{t\ge t_1^i}
    &= \tilde{\xi} + \frac{\langle \vecd, \lambda \odot \phi \rangle}{\lambda_i} \cdot 50 H L \omega_0^{2} \cdot \frac{\baromegaistar}{\muistar} \bigggiven_{t\ge t_1^i} \nend 
    &= \tilde{\xi} + \frac{\langle \vecd, \lambda \odot \phi \rangle}{\lambda_i} \cdot 50 H L \omega_0^{2} \phi_i \cdot L^{-1} (\rhoistar(t)\given_{t\ge t_1^i})^{-1} \phi_i^{-1} \nend
    &\le \tilde{\xi} + \frac{\langle \vecd, \lambda \odot \phi \rangle}{\lambda_i} \cdot 50 H^2 \omega_0^{2} \phi_i (1-\alpha_i)^{-1} \nend 
    &\le \tilde{\xi} + \frac{\langle \vecd, \lambda \odot \phi \rangle}{\lambda_i} \cdot 100 H^2 \omega_0^{2} \phi_i \deleq \checkxiiprime. 
    \label{eq:check xi after step 1}
\end{align}
And also we have new upper bound 
\begin{align}
\hat\xi_i(t) \le 2c\check\xi_i(t) \epsilon^{-1} \le 2c\epsilon^{-1} \checkxiiprime \deleq \hatxiiprime. 
\label{eq:hat xi after step 1}
\end{align} 
for $t \in [t_1^i, T_\warmup^i]$.
Using these $\checkxiiprime$ and $\hatxiiprime$, we have the following upper bound for $\baromegaistar$. 
\begin{lemma}
    \label{fact: after step 1}
    For any $(t_0, t)\in [t_1^i, T_\warmup^i]^2$ such that $t_0 < t$, we have the following upper bounds for $\baromegaistar$ in the \TIFwarmupi dynamics as 
\[
    \baromegaistar(t) \le \frac{1}{\baromegaistar(t_0)^{-1} - \sigma_3^i (t - t_0)}, \where \sigma_3^i \deleq \frac{(1 + \checkxiiprime + \tilde\xi + 2\tilde\zeta)L\lambda_i}{\sqrt{d_e}\phi_i}.
\]
Moreover, we have $\rhoistar\phi_i \ge (1-\alpha_i)/H$ for any $t \in [t_1^i, T_\warmup^i]$, which implies a lower bound for $\baromegaistar$ as
\[
    \baromegaistar(t) \ge \frac{1}{\baromegaistar(t_0)^{-1} - \sigma_2^i (t - t_0)}, \where \sigma_2^i \deleq \frac{ (1-\alpha_i - \checkxiiprime)\cdot L \lambda_i}{H\phi_i\sqrt{d_e}}.
\]
\end{lemma}
\begin{proof}
we write down the dynamics for $\baromegaistar$ as
\begin{align*}
    \partial_t \baromegaistar 
    &= (1\pm \checkxiiprime) \cdot  \sqrt{d_e}^{-1}\lambda_i \muistar \baromegaistar 
    \mytag{\Cref{eq:check xi after step 1}}\\
    &\le (1 + \checkxiiprime) \cdot  \sqrt{d_e}^{-1}\lambda_i L \rho_i (\baromegaistar)^2 
    \mytag{$\muistar\le L\rho_i\baromegaistar$ by definition} \\
    &\le \underbrace{\frac{(1 + \checkxiiprime + \tilde\xi + 2\tilde\zeta)L\lambda_i}{\sqrt{d_e}\phi_i}}_{\ds \sigma_3^i} \cdot (\baromegaistar)^2.
    \mytag{\Cref{fact:critical conditions}}
\end{align*}
Since $\baromegaistar$ is monotonically increasing, we have that
\[
    \baromegaistar(t_0 + \Delta t) \le \frac{1}{\baromegaistar(t_0)^{-1} - \sigma_3^i \Delta t} 
\]
as we claimed. The second condition $\rhoistar\phi_i \ge (1-\alpha_i)/H$ is just a direct conclusion of \eqref{eq:dot eta* lower bound} and using the fact that $\rhoistar\phi_i $ already reaches $(1-\alpha_i)/H$ at $t_1^i$.
For the lower bound, 
using the gradient of $\baromegaistar$, we have 
\begin{align*}
    \partial_t \baromegaistar &= (1\pm \checkxiiprime) \cdot  \sqrt{d_e}^{-1}\lambda_i \muistar \baromegaistar \mytag{\eqref{eq:check xi after step 1} for $\check\xi_i(t)$} \\
    & = (1\pm \checkxiiprime) \cdot  L\sqrt{d_e}^{-1} \lambda_i \rhoistar (\baromegaistar)^2 \\
    & \ge \underbrace{\frac{ (1-\alpha_i - \checkxiiprime)\cdot L \lambda_i}{H\phi_i\sqrt{d_e}}}_{\ds \deleq \sigma_2^i}\cdot (\baromegaistar)^2,\mytag{lower bound for $\rhoistar$} 
\end{align*}
which implies that 
\[
    \baromegaistar(t_0 + \Delta t) \ge \frac{1}{\baromegaistar(t_0)^{-1} - \sigma_2^i \Delta t}. 
\]
We complete the proof.
\end{proof}

\subsubsection{Step 2: Growth of $\baromegaistar/\baromegaih$}
In Step 2 of the warm-up stage, we aim to show the Growth of marginal difference $(\baromegaistar - \baromegaih)/\baromegaistar$ for all $h\neq h_i^*$.
\paragraph{Dynamics Either Enter or Skip Step 2  at $t_1^i$.}
Using the inequality
\begin{align*}
    \rhoistar(t_1^i)\phi_i \le H^{-1}\cdot \rho_i(t_1^i) \phi_i \le (1-\alpha_i) /H, 
\end{align*}
where the first inequality holds since \Cref{itm:warmup-lbdaU separation} holds at $t=t_1^i$ and the second inequality holds by \Cref{itm:warmup-step1}, we have \Cref{itm:warmup-step2-eta} satisfied.
If there exists at least one $h \in [H]\backslash\{h_i^*\}$ such that $(\baromegaistar - \baromegaih) / \baromegaistar\given_{t_1^i} < (3\tilde\zeta + 2\tilde\xi) \beta_i^{-1}$, the dynamics will enter Step 2 at $t_1^i$.
Otherwise, the dynamics skip Step 2 and we have $t_2^i = t_1^i$.
Next, we consider the case where the dynamics enter Step 2 at $t_1^i$.

\paragraph{Upper Bounding the Duration of Step 2.}
We study the following quantity for $h\neq h_i^*$:
\begin{align}
    \frac{\partial_t(\log\baromegaistar - \log\baromegaih)}{\partial_t \log \baromegaistar} 
    &= \frac{(1 \pm \hat\xi_i(t) ) \cdot \sqrt{d_e}^{-1} \lambda_i \left(\muistar-\muih\right)}{(1\pm \check{\xi}_i(t)) \cdot  \sqrt{d_e}^{-1}\lambda_i \muistar} \nend
    &= \frac{(1 \pm \hat\xi_i(t) ) \cdot \left(\muistar-\muih\right)}{(1\pm \check{\xi}_i(t)) \cdot \muistar} \ge (1-\hatxiiprime-\checkxiiprime) c^{-1}\epsilon, 
    \label{eq:dot log lbdaW diff-log lbdaW-ratio-stage 2}
\end{align}
where the last inequality holds by \Cref{itm:warmup-lbdaU separation} and \eqref{eq:check xi after step 1} for both $\hat\xi_i(t)$ and $\check\xi_i(t)$.
Note that for any $t\in [t_1^i, t_2^i]$, there always exists a $h \in [H]\backslash\{h_i^*\}$ such that $(\baromegaistar - \baromegaih)/\baromegaistar \given_{t} < (3\tilde\zeta +2\tilde\xi)\beta_i^{-1}$ by \Cref{itm:warmup-step2-lbdaW-diff}, which implies that $\baromegaistar/\baromegaih\given_{t} \le (1 - (3\tilde\zeta +2\tilde\xi)\beta_i^{-1})^{-1}$.
Consequently, we have by the gradient ratio argument in \eqref{eq:dot log lbdaW diff-log lbdaW-ratio-stage 2} that for this pair of $(t, h)$, 
\begin{align}
    \log \left(\frac{\baromegaistar(t)}{\baromegaistar(t_1^i)}\right) 
    &\le \frac{\log \left(\frac{\baromegaistar(t)}{\baromegaih(t)}\right) - \log \left(\frac{\baromegaistar(t_1^i)}{\baromegaih(t_1^i)}\right)}{(1-\hatxiiprime-\checkxiiprime) c^{-1}\epsilon} \le \frac{-\log \left(1 - (3\tilde\zeta +2\tilde\xi)\beta_i^{-1}\right) }{(1-\hatxiiprime-\checkxiiprime) c^{-1}\epsilon} \nend
    &\le -\log \left(1 - \frac{(3\tilde\zeta +2\tilde\xi)\beta_i^{-1}}{(1-\hatxiiprime-\checkxiiprime) c^{-1}\epsilon}\right)
    \le  -\log \left(1 - \frac{2c(3\tilde\zeta +2\tilde\xi )}{\beta_i\epsilon}\right), 
    \label{eq:lbdaW* upper bound-warmup-stage2}
\end{align}
where the second inequality comes from the fact that $\baromegaistar \ge \baromegaih$ for all $h\neq h_i^*$ by \Cref{itm:warmup-lbdaW* grows faster}, with which we drop the second term in the numerator.
Here, the third inequality holds by noting that $k\log(1-x) \ge \log(1-kx)$ for $k>1$ and $0< kx < 1$.
We can also lower bound the gradient of $\baromegaistar$ as
\begin{align*}
    \partial_t \log \baromegaistar &= (1\pm \check{\xi}_i(t)) \cdot  \sqrt{d_e}^{-1}\lambda_i \muistar \\
    &\ge (1-\checkxiiprime) \cdot  \sqrt{d_e}^{-1}\lambda_i L \rhoistar \baromegaistar \\
    &\ge (1-\checkxiiprime -\alpha_i) \cdot H^{-1} L \sqrt{d_e}^{-1}\lambda_i  \phi_i^{-1} \omega_0,
\end{align*}
where in the last inequality we use the fact that $\rhoistar\phi_i \ge (1-\alpha_i)/H$ by \Cref{fact: after step 1} and that $\baromegaistar \ge \omega_0$ by the monotonicity of $\baromegaistar$ according to \Cref{fact:critical conditions}.
It follows that
\begin{align}
    \log\left(\frac{\baromegaistar(t)}{\baromegaistar(t_1^i)}\right)
    &\ge (1-\checkxiiprime -\alpha_i) \cdot H^{-1} L \sqrt{d_e}^{-1}\lambda_i  \phi_i^{-1} \omega_0 \cdot (t - t_1^i).
    \label{eq:lbdaW* lower bound-warmup-stage2}
\end{align}
Combining the lower \eqref{eq:lbdaW* lower bound-warmup-stage2} and upper \eqref{eq:lbdaW* upper bound-warmup-stage2} bounds with $t = t_2^i$, we have
\begin{align*}
    t_2^i - t_1^i \le \frac{-\log \left(1 - \frac{2c(3\tilde\zeta +2\tilde\xi )}{\beta_i\epsilon}\right) H \sqrt{d_e} \phi_i }
    {(1-\checkxiiprime-\alpha_i) \cdot L \lambda_i   \omega_0} 
    \le \frac{4c(3\tilde\zeta +2\tilde\xi ) \cdot H \sqrt{d_e} \phi_i}{\beta_i\epsilon L \lambda_i\omega_0}
    \deleq \Delta t_{2+}^i, 
\end{align*}
where in the last inequality we use the scale $\brown \beta_i \gg c\epsilon^{-1} \log (2c(3\tilde\zeta +2\tilde\xi )^{-1})$ and the fact that $\brown \checkxiiprime + \alpha_i \ll 1$.

\paragraph{Only \Cref{itm:warmup-step2-lbdaW-diff} is violated at $t_2^i$.}
Here, we use the same argument as in the analysis of Step 1 to show that only \Cref{itm:warmup-step2-lbdaW-diff} is violated at $t_2^i$.
Note that \Cref{itm:warmup-step2-eta} is already justified by \eqref{eq:dot eta* lower bound}.
To show that the critical condition \Cref{itm:warmup-lbdaW ub} holds for $t \in (t_1^i, t_2^i]$ with some marginal gap, we first note that $\baromegaistar$ is monotonically increasing.
We next calculate the upper bound for $\baromegaistar(t_2^i)$. 
Invoking \eqref{eq:lbdaW* upper bound-warmup-stage2}, we have
\begin{align*}
    \baromegaistar(t_2^i) 
    &\le \baromegaistar(t_1^i) \cdot \exp\left(-\log \left(1 - \frac{2c(3\tilde\zeta +2\tilde\xi )}{\beta_i\epsilon}\right)\right) \nend 
    &\le \omega_0 \cdot \left(1 - \frac{2 \varsigma_i}{ \tilde\alpha_i}\right)^{-1} \cdot \left(1 - \frac{2c(3\tilde\zeta +2\tilde\xi )}{\beta_i\epsilon}\right)^{-1} = (1+o(1))\omega_0, 
\end{align*}
where the second inequality holds by the upper bound for $\baromegaistar(t_1^i)$ in \eqref{eq:lbdaW* upper bound-warmup-stage1} from step 1 and the fact that $\brown 2c(3\tilde\zeta +2\tilde\xi )\epsilon^{-1} \ll \beta_i$.
Hence, \Cref{itm:warmup-lbdaW ub} holds for $t \in (t_1^i, t_2^i]$ with sufficient margin gap.

The next thing is to show that the other one \Cref{itm:warmup-lbdaU separation} holds with some marginal gap. 
In order to invoke \Cref{prop:lbdaU separation lower bound} again, we check the condition that 
\begin{align*}
    \Delta t_{2}^i
    \le \frac{4c(3\tilde\zeta +2\tilde\xi ) \cdot H \sqrt{d_e} \phi_i}{\beta_i\epsilon L \lambda_i\omega_0} \ll \frac{\epsilon (1-c^{-1})}{18 e (\tilde \xi+\tilde\zeta) \lambda_id_i \omega_0} = (1 - c^{-1}) T_\Uthres,
\end{align*}
which holds by noting that $\brown \beta_i \gg 
\frac{c(\tilde\zeta + \tilde\xi)^2 H }{(1-c^{-1})\epsilon^2 \varsigma_i}
$ in \Cref{def:warmup stage-split}.
Together with the previous verified condition $t_{1}^i \ll (1-c^{-1})T_\Uthres$, we conclude that \Cref{itm:warmup-lbdaU separation} holds with sufficient margin gap for $t = t_2^i$.
To summarize, we have the following table for the dynamics in Step 2.
\begin{table}[ht]
    \centering
    \begin{tabular}{l@{\hspace{1cm}} l}
        \hline
        \hline
        \textbf{Properties} & \textbf{Typical Values} \\
        \hline
        Duration & $\Delta t_{2}^i \le \Delta t_{2+}^i = \frac{4c(3\tilde\zeta +2\tilde\xi ) \cdot H \sqrt{d_e} \phi_i}{\beta_i\epsilon L \lambda_i\omega_0}$ \\
        $\muistar \nearrow$ & $\muistar(t_2^i) \le L \phi_i^{-1} (1 + \tilde\xi + 2\tilde\zeta) \baromegaistar(t_2^i)$ \\
        $\baromegaistar \nearrow$ & $\baromegaistar(t_2^i) \le \omega_0 \cdot \left(1 - \frac{2 \varsigma_i}{ \tilde\alpha_i}\right)^{-1} \cdot \left(1 - \frac{2c(3\tilde\zeta +2\tilde\xi )}{\beta_i\epsilon}\right)^{-1}$ \\
        $(\muistar - \muih)/ \muistar$ & lower bounded by $\epsilon \cdot \left(1  -  (t_{1+}^i + \Delta t_{2+}^i)/T_\Uthres\right)$ \\
        \hline
    \end{tabular}
    \caption{Summary of Step 2 of the Warm-up Stage}
    \label{tab:summary-warmup-step 2}
\end{table}

\paragraph{Growth of $\muistar/ \muih$ is guaranteed after Step 2.}
As long as the dynamics exist Step 2, we will have for any $h \in [H]\backslash\{h_i^*\}$ that $(\baromegaistar - \baromegaih)/\baromegaih \ge (3\tilde\zeta + 2\tilde\xi)\beta_i^{-1}$.
Integrating with the fact that $\baromegaistar / \baromegaih$ is monotonically increasing by \Cref{fact:critical conditions}, the very same condition will always hold thereafter.
We have the following fact for the growth of $\muistar/ \muih$.
\begin{lemma}
    \label[lemma]{fact:after step 2}
    Suppose that $(\baromegaistar - \baromegaih)/\baromegaih \ge (3\tilde\zeta + 2\tilde\xi)\beta_i^{-1}$ holds for any $h \in [H]\backslash\{h_i^*\}$ at some $\tau$.
    Then, we have for $t \in [\tau, T_\warmup^i]$ that $\muistar(\tau) / \muih(\tau)$ is monotonically increasing for any $h \in [H]\backslash\{h_i^*\}$.
\end{lemma}
\begin{proof}
We thus have for the gradient of $\log \muistar - \log \muih$ that
\begin{align}
    &\partial_t (\log \muistar - \log \muih) \nend
    &\quad = \lambda_i d_i\left(\baromegaistar-\baromegaih \pm \tilde{\xi}\left(\baromegaistar+\baromegaih\right) \pm 2 \tilde{\zeta} \phi_i \sum_{h'=1}^{H} \frac{\muihprime}{L}\right) \notag\\
    &\quad = \lambda_i d_i\left(\baromegaistar-\baromegaih\right)\left(1 \pm \tilde{\xi} \cdot \frac{\baromegaistar+\baromegaih}{\baromegaistar-\baromegaih} \pm 2 \tilde{\zeta} \phi_i \rho_i \cdot \frac{\baromegaistar}{\baromegaistar-\baromegaih}\right) \nend 
    &\quad = (1 \pm \beta_i) \lambda_i d_i\left(\baromegaistar-\baromegaih\right), 
    \label{eq:dot log lbdaU diff-after stage 2}
\end{align}
where the last equality holds by using the fact that $\frac{\baromegaistar+\baromegaih}{\baromegaistar-\baromegaih} \le 2 \frac{\baromegaistar}{\baromegaistar-\baromegaih} \le 2 \beta_i / (3\tilde\zeta + 2\tilde\xi)$ and that $\phi_i\rho_i \le (1 + \tilde\xi + 2\tilde\zeta) \le 1.5$ by \Cref{fact:critical conditions}.
It is then obvious from \eqref{eq:dot log lbdaU diff-after stage 2} that $\log \muistar - \log \muih$ is monotonically increasing for $t \in [\tau, T_\warmup^i]$ given that $\brown \beta_i \ll 1$ and $\baromegaistar \ge \baromegaih$ for all $h\neq h_i^*$ under \Cref{itm:warmup-lbdaW* grows faster}.
\end{proof}


\subsubsection{Step 3: Growth of $\muistar/\muih$}
In this step, we aim to show the Growth of marginal difference $\muistar/\muih$ to $(H-1)/\gamma_i$ for all $h\neq h_i^*$.

\paragraph{Dynamics Either Enter or Skip Step 3 at $t_2^i$.}
After Step 2, we have and \Cref{itm:warmup-step3-lbdaW-diff} satisfied.
If there exists at least one $h \in [H]\backslash\{h_i^*\}$ such that $\muistar / \muih\given_{t_2^i} < (H - 1) \gamma_i^{-1}$, the dynamics will enter Step 3 at $t_2^i$.
Otherwise the dynamics skip Step 3 and we have $t_3^i = t_2^i$.
Next, we consider the case where the dynamics enter Step 3 at $t_2^i$.

\paragraph{Upper Bounding the Duration of Step 3.}
Suppose the dynamics enter Step 3 at $t_2^i$.
In the previous steps, we establish the upper bounds by controlling the maximal value of $\baromegaistar$.
In this step, we will establish the upper bound in a more direct way---by lower bounding the growth of $\muistar/\muih$.

We first derive a tighter lower bound for $\baromegaistar$ by incorporating \Cref{fact: after step 1} with $t_0$ replaced by $t_2^i$.
\[
    \baromegaistar(t_2^i + \Delta t) \ge \frac{1}{\baromegaistar(t_2^i)^{-1} - \sigma_2^i \Delta t}, \where \sigma_2^i \deleq \frac{ (1-\alpha_i - \checkxiiprime)\cdot L \lambda_i}{H\phi_i\sqrt{d_e}}.
\]
On the other hand, we have for the gradient of $\baromegaistar - \baromegaih$ that
\begin{align*}
    \partial_t\left(
        \baromegaistar - \baromegaih
    \right)
    &\ge (1 - \hatxiiprime) \cdot \sqrt{d_e}^{-1} \lambda_i \baromegaistar\left(\muistar-\muih\right) 
    \mytag{\eqref{eq:hat xi after step 1} for $\hat\xi_i(t)$} \\
    &\ge (1 - \hatxiiprime) \cdot \sqrt{d_e}^{-1} \lambda_i \baromegaistar \cdot \frac{\epsilon}{c}  \cdot \muistar 
    \mytag{\Cref{itm:warmup-lbdaU separation} } \\
    &\ge \underbrace{\frac{(1 - \hatxiiprime-\alpha_i) \cdot L \lambda_i }{H \phi_i \sqrt{d_e}}  \cdot  \frac{\epsilon}{c}}_{\ds \deleq \sigma_1^i} \cdot (\baromegaistar)^2.  
    \mytag{\Cref{fact: after step 1}} 
\end{align*}
We plug in the lower bound for $\baromegaistar$ and conclude that $\partial_t(\baromegaistar - \baromegaih)\ge \sigma_1^i \cdot \left(\baromegaistar(t_2^i)^{-1} - \sigma_2^i \Delta t\right)^{-2}$.
Integrating the above inequality, we have
\begin{align*}
    (\baromegaistar - \baromegaih)\biggiven_{t_2^i + \Delta t} 
    &\ge \frac{\sigma_1^i \cdot \baromegaistar(t_2^i)\cdot \Delta t}{\baromegaistar(t_2^i)^{-1} - \sigma_2^i \Delta t} + (\baromegaistar - \baromegaih)\biggiven_{t_2^i} \ge \frac{\sigma_1^i \cdot \baromegaistar(t_2^i)\cdot \Delta t}{\baromegaistar(t_2^i)^{-1} - \sigma_2^i \Delta t}.
\end{align*}
where the last inequality holds by \Cref{itm:warmup-lbdaW* grows faster} that $\baromegaistar(t_2^i) \ge \baromegaih(t_2^i)$.
Moreover, we have for the gradient of $\log \muistar - \log \muih$ following \eqref{eq:dot log lbdaU diff-after stage 2} that
\begin{align*}
    \partial_t (\log \muistar - \log \muih) = (1 \pm \beta_i) \lambda_i d_i\left(\baromegaistar-\baromegaih\right), 
\end{align*}
Now, we plug in the lower bound for $\baromegaistar - \baromegaih$ and obtain 
\begin{align*}
    \partial_t (\log \muistar - \log \muih) \biggiven_{t_2^i + \Delta t} \ge (1-\beta_i) \lambda_i d_i \cdot \frac{\sigma_1^i \cdot \baromegaistar(t_2^i)\cdot \Delta t}{\baromegaistar(t_2^i)^{-1} - \sigma_2^i \Delta t}. 
\end{align*}
Integrating both sides with respect to $\Delta t$, we have 
\begin{align}
 \left.\frac{\muistar}{\muih}\right|_{t_2^i+\Delta t} 
 &\ge \left.\frac{\muistar}{\muih}\right|_{t_2^i} \cdot \exp \Biggl\{(1-\beta_i) \lambda_i d_i \baromegaistar\left(t_2^i\right) \notag\\
& \qquad \cdot \left(-\frac{\sigma_1^i \Delta t}{\sigma_2^i}-\frac{\sigma_1^i}{(\sigma_2^i)^2 \baromegaistar\left(t_2^i\right)} \log \left(1-\sigma_2^i \baromegaistar^*\left(t_2^i\right) \Delta t\right)\right)\Biggr\} 
\notag
\\
&\ge \left.\frac{\muistar}{\muih}\right|_{t_2^i} \cdot \exp \left((1-\beta_i) \lambda_i d_i\sigma_1^i \cdot \left(\baromegaistar(t_2^i) \Delta t\right)^2 /2\right) \label{eq:lbdaU*/lbdaU^h lower bound-warmup-stage3}\\
&\ge \exp \left((1-\beta_i) \lambda_i d_i\sigma_1^i \cdot \left(\baromegaistar(t_2^i) \Delta t\right)^2 /2\right), \notag
\end{align} 
where in the second inequality, we use the fact that $\log(1-x) \ge -x - x^2/2$ for $x\in [0, 1/2]$ by Taylor expansion and in the last inequality we use the fact that $\muistar/\muih \ge 1$ by \Cref{itm:warmup-lbdaU separation}.
For any $t\in [t_2^i, t_3^i]$, there always exists a $h \in [H]\backslash\{h_i^*\}$ such that $\muistar / \muih\given_{t} < (H - 1) \gamma_i^{-1}$ by \Cref{itm:warmup-step3-lbdaU-dominance}.
Let $\Delta t = t - t_2^i$.
For this pair of $(\Delta t, h)$, we incorporate the lower and upper bounds and obtain
\begin{align*}
    (H - 1) \gamma_i^{-1} \ge \exp \left((1-\beta) \lambda_i d_i\sigma_1^i \cdot \left(\baromegaistar(t_2^i) \Delta t\right)^2 /2\right).
\end{align*}
We can solve for the maximal $\Delta t$ and obtain
\begin{align*}
    \Delta t \le \sqrt{\frac{2\log \left((H - 1) \gamma_i^{-1}\right)}{(1-\beta_i) \lambda_i d_i\sigma_1^i \baromegaistar(t_2^i)^2}} \le \sqrt{\frac{4c\log \left((H - 1) \gamma_i^{-1}\right) \cdot H \phi_i \sqrt{d_e}}{L \lambda_i^2 d_i \omega_0^2 \cdot  {\epsilon}}} \deleq \Delta t_{3+}^i, 
\end{align*}
where the inequality holds by the monotonicity of $\baromegaistar$.

\paragraph{Only \Cref{itm:warmup-step3-lbdaU-dominance} is violated at $t_3^i$.}
Note that we no longer need to check \Cref{itm:warmup-lbdaU separation} as we have explicitly characterize the growth of $\muistar/\muih$ in \eqref{eq:lbdaU*/lbdaU^h lower bound-warmup-stage3}.
The growth of $\baromegaistar / \baromegaih$ is already established by \Cref{fact:critical conditions}, which implies that \Cref{itm:warmup-step3-lbdaW-diff} is satisfied with some marginal gap.
Therefore, we only need to verify that \Cref{itm:warmup-lbdaW ub} holds for $t = t_3^i$ with some marginal gap according to \Cref{fact:critical conditions}.

To control the upper bound for $\baromegaistar(t_3^i)$, we invoke \Cref{fact: after step 1} that for any $t_1^i \le t_0 < t\le T_\warmup^i$, it holds that
\[
    \baromegaistar(t_0 + \Delta t) \le \frac{1}{\baromegaistar(t_0)^{-1} - \sigma_3^i \Delta t}, \where \sigma_3^i \deleq \frac{(1 + \checkxiiprime + \tilde\xi + 2\tilde\zeta)L\lambda_i}{\sqrt{d_e}\phi_i}. 
\]
Here, we just take $t_0 = t_2^i$ and $\Delta t = t_3^i - t_2^i \le \Delta t_{3+}^i$ and obtain
\begin{align*}
    \baromegaistar(t_3^i) &\le \frac{1}{\omega_0^{-1} \cdot \left(1 - \frac{2 \varsigma_i}{ \tilde\alpha_i}\right) \cdot \left(1 - \frac{2c(3\tilde\zeta +2\tilde\xi )}{\beta_i\epsilon}\right) - \frac{(1 + \checkxiiprime + \tilde\xi + 2\tilde\zeta)L\lambda_i}{\sqrt{d_e}\phi_i} \sqrt{\frac{4c\log \left((H - 1) \gamma_i^{-1}\right) \cdot H \phi_i \sqrt{d_e}}{L \lambda_i^2 d_i \omega_0^2 \cdot  {\epsilon}}} } \\
    &\le \frac{1}{\omega_0^{-1} \cdot \left(\left(1 - \frac{2 \varsigma_i}{ \tilde\alpha_i}\right) \cdot \left(1 - \frac{2c(3\tilde\zeta +2\tilde\xi )}{\beta_i\epsilon}\right) - \sqrt{\frac{16 c\log \left((H - 1) \gamma_i^{-1}\right) \cdot H \varsigma_i}{\epsilon }}\right) }\\
    &\le \frac{1}{\omega_0^{-1} \cdot \left(1 - \frac{2 \varsigma_i}{ \tilde\alpha_i}\right) \cdot \left(1 - \frac{2c(3\tilde\zeta +2\tilde\xi )}{\beta_i\epsilon}\right) \cdot \left(1 - \sqrt{\frac{32 c\log \left((H - 1) \gamma_i^{-1}\right) \cdot H \varsigma_i}{\epsilon }}\right)  } = (1+o(1)) \omega_0,
\end{align*}
where the last inequality holds by the condition that $\brown \gamma_i \gg (H-1)^{-1} \exp\left(-\frac{\epsilon}{32c H \varsigma_i}\right)$ in \Cref{def:warmup stage-split}.
Hence, we verify \Cref{itm:warmup-lbdaW ub} as well.
The following table summarizes the dynamics in Step 3.
\begin{table}[ht]
    \centering
    \begin{tabular}{p{4cm}@{\hspace{1cm}} p{9cm}}
        \hline
        \hline
        \textbf{Properties} & \textbf{Typical Values} \\
        \hline
        Duration & $\Delta t_{3}^i \le \Delta t_{3+}^i = \sqrt{\frac{4c\log \left((H - 1) \gamma_i^{-1}\right) \cdot H \phi_i \sqrt{d_e}}{L \lambda_i^2 d_i \omega_0^2 \cdot  {\epsilon}}}$ \\
        $\muistar$ & $\muistar(t_3^i) \le L \phi_i^{-1} (1 + \tilde\xi + 2\tilde\zeta) \baromegaistar(t_3^i)$ \\
        $\baromegaistar \nearrow$ & $\baromegaistar(t_3^i) \le \omega_0 \cdot \left(1 - \frac{2 \varsigma_i}{ \tilde\alpha_i}\right)^{-1} \cdot \left(1 - \frac{2c(3\tilde\zeta +2\tilde\xi )}{\beta_i\epsilon}\right)^{-1} \cdot \left(1 - \sqrt{\frac{32 c\log \left((H - 1) \gamma_i^{-1}\right) \cdot H \varsigma_i}{\epsilon }}\right)^{-1}$ \\
        $(\muistar - \muih)/ \muistar \nearrow$ & $(1 - \gamma_i/(H-1))$ \\
        \hline
    \end{tabular}
    \caption{Summary of Step 3 of the Warm-up Stage}
    \label{tab:summary-warmup-step 3}
\end{table}

\subsubsection{Step 4: Growth of $\rhoistar$}
The moment the dynamics exist Step 3, we have \Cref{itm:warmup-step4 lbdaU ratio} satisfied. 
Suppose that at $t_3^i$, we already have $\rhoistar \ge (1-\alpha_i)/\phi_i$.
Then, we just set $t_4^i = t_3^i$ and skip Step 4.
Next, we consider the case where the dynamics enter Step 4 at $t_3^i$.

\paragraph{Upper Bounding the Duration of Step 4.}
Using \Cref{itm:warmup-step4 lbdaU ratio}, we have for $t >t_3^i$ that $\rhoistar < \rho_i \le (1 + \gamma_i)\rhoistar$.
Therefore, we have for the gradient of $\rhoistar$ that 
\begin{align*}
    \partial_t \rhoistar &= L^{-1} \lambda_id_i \muistar  \cdot\left(
        -(1 \pm \tilde \zeta \pm 2\gamma_i) \phi_i \rhoistar +  (1 \pm \tilde \xi \pm 4\varsigma_i)
    \right).
\end{align*}
We plug in the upper bound $\rhoistar \le (1 - \alpha_i)\phi_i^{-1}$ in \Cref{itm:warmup-step4 eta*} and obtain
\begin{align*}
    \partial_t \rhoistar 
    &\ge L^{-1} \lambda_id_i \muistar  \cdot\left(
        -(1 \pm \tilde \zeta \pm 2\gamma_i) (1 - \alpha_i) +  (1 \pm \tilde \xi \pm 4\varsigma_i)
    \right) \\
    & \ge L^{-1} \lambda_id_i \muistar \cdot (\alpha_i - \tilde\xi - 4\varsigma_i - 2\gamma_i - \tilde\zeta).
\end{align*}
By looking at the ratio between $\partial_t\rhoistar$ and $\partial_t \log\baromegaistar$, we have
\begin{align*}
    \frac{\partial_t \rhoistar}{\partial_t \log\baromegaistar} &\ge \frac{L^{-1} \lambda_id_i \muistar \cdot (\alpha_i - \tilde\xi - 4\varsigma_i - 2\gamma_i - \tilde\zeta)}{(1\pm \checkxiiprime) \cdot  \sqrt{d_e}^{-1}\lambda_i \muistar} \\
    &\ge \frac{ (\alpha_i - \tilde\xi - 4\varsigma_i - 2\gamma_i - \tilde\zeta)}{(1 + \checkxiiprime) \varsigma_i\phi_i} \ge \frac{\alpha_i}{2 \varsigma_i\phi_i}, 
\end{align*}
where the last inequality holds by the condition that $\brown \gamma_i\ll \alpha_i$. The maximal increment of $\rhoistar$ is bounded by the maximal value $(1-\alpha_i)/\phi_i$. As a result, the maximal increment of $\log \baromegaistar$ is bounded by $\frac{(1-\alpha_i)2 \varsigma_i}{\alpha_i} \le \frac{2 \varsigma_i}{\alpha_i}$.
Hence, we have
\begin{align}
    \frac{\baromegaistar(t_4^i)}{\baromegaistar(t_3^i)} 
    &\le \exp\left(\frac{2 \varsigma_i}{\alpha_i}\right) \le 1 + \frac{2 e\varsigma_i}{\alpha_i} = 1 + o(1), 
    \label{eq:lbdaW* ratio upper bound-warmup-stage4}
\end{align}
where the last inequality holds by the condition $\brown \alpha \gg \varsigma_i$ in \Cref{def:warmup stage-split}. 
To find a lower bound for $\baromegaistar(t_4^i)$, we invoke \Cref{fact: after step 1} with $t_0$ replaced by $t_3^i$ and obtain
\[
    \baromegaistar(t_4^i) \ge \frac{1}{\baromegaistar(t_3^i)^{-1} - \sigma_2^i \Delta t_4^i}, \where \sigma_2^i \deleq \frac{ (1-\alpha_i - \checkxiiprime)\cdot L \lambda_i}{H\phi_i\sqrt{d_e}}, \quad \Delta t_4^i \deleq t_4^i - t_3^i.
\]
Comparing the upper and lower bounds, we have
\[
    \Delta t_4^i \le \frac{ 2 e H }{(1-\alpha_i - \checkxiiprime)\cdot d_i \lambda_i\alpha_i  \cdot \baromegaistar(t_3^i)} \le \frac{ 4 e H }{ d_i \lambda_i\alpha_i \omega_0 } \deleq \Delta t_{4+}^i.
\]
As discussed previously in \Cref{fact:after step 2}, we have the ratio $\lbdaU^*/\lbdaU^h$ monotonically increasing after Step 2. 
Thus there is no need to check \Cref{itm:warmup-lbdaU separation} again.
Also, we have from \Cref{eq:lbdaW* ratio upper bound-warmup-stage4} that $\baromegaistar(t_4^i) = (1 + o(1)) \baromegaistar(t_3^i) < 2\omega_0 L$, which implies that \Cref{itm:warmup-lbdaW ub} is satisfied with some marginal gap.
In the following table, we summarize the dynamics in Step 4.
\begin{table}[ht]
    \centering
    \begin{tabular}{p{2cm}@{\hspace{1cm}} p{11cm}}
        \hline
        \hline
        \textbf{Properties} & \textbf{Typical Values} \\
        \hline
        Duration & $\Delta t_{4}^i \le \Delta t_{4+}^i = \frac{ 4 e H }{ d_i \lambda_i\alpha_i \omega_0 }$ \\
        $\muistar$ & $\muistar(t_4^i) \le L \phi_i^{-1} (1 + \tilde\xi + 2\tilde\zeta) \baromegaistar(t_4^i)$ \\
        $\baromegaistar \nearrow$ & $\baromegaistar(t_4^i) \le \omega_0 \cdot \left(1 - \frac{2 \varsigma_i}{ \tilde\alpha_i}\right)^{-1} \cdot \left(1 - \frac{2c(3\tilde\zeta +2\tilde\xi )}{\beta_i\epsilon}\right)^{-1} \cdot \left(1 - \sqrt{\frac{32 c\log \left((H - 1) \gamma_i^{-1}\right) \cdot H \varsigma_i}{  \epsilon }}\right)^{-1} \cdot \left(1 + \frac{2 e\varsigma_i}{\alpha_i}\right)$ \\
        \hline
    \end{tabular}
    \caption{Summary of Step 4 of the Warm-up Stage}
    \label{tab:summary-warmup-step 4}
\end{table}

\subsubsection{Step 5: Lottery Winner Dominates}
After step 4, we summarize all the good properties we have established so far for any $t \in [t_4^i, T_\warmup^i]$:
\begin{factenum}[
    label={(\textbf{5-\arabic*})},
    ref  =(5-\arabic*), 
]
    \item \label{fact:warmup-step4-lbdaU-ratio} 
    $\muistar / \muih \ge (H-1)/\gamma_i$ for all $h \in [H]\backslash\{h_i^*\}$.
    \item \label{fact:warmup-step4-lbdaW-ratio}
    $(\baromegaistar - \baromegaih) / \baromegaistar \ge (3\tilde\zeta + 2\tilde\xi) \beta_i^{-1}$ for any $h \in [H]\backslash\{h_i^*\}$.
    \item \label{fact:warmup-step4-eta*}
    $1-\alpha_i \le \rhoistar\phi_i < \rho_i\phi_i \le 1 + \tilde\xi + 2\tilde\zeta$.
    \item \label{fact:warmup-step4-growing}
    $\baromegaistar$, $\baromegaistar / \baromegaih$ and $\muistar / \muih$ are monotonically increasing for any $h \in [H]\backslash\{h_i^*\}$.
\end{factenum}
In the sequel, we will appeal to these facts and refrain from checking them again.

\paragraph{Lower Bounding the Duration of Step 5.}
We first understand what's happening in Step 5.
We have for $\baromegaistar$ that 
\begin{align*}
    \partial_t \baromegaistar &= (1\pm \checkxiiprime) \cdot  \sqrt{d_e}^{-1}\lambda_i \muistar \baromegaistar \mytag{\eqref{eq:check xi after step 1} for $\check\xi_i(t)$} \\
    & = (1\pm \checkxiiprime) \cdot  L\sqrt{d_e}^{-1} \lambda_i \rhoistar (\baromegaistar)^2 \\
    & = (1\pm \checkxiiprime \pm 2\alpha_i) \cdot  L\sqrt{d_e}^{-1} \lambda_i \phi_i^{-1} (\baromegaistar)^2, \mytag{\Cref{fact:warmup-step4-eta*}}
\end{align*}
where in the last equality we also invoke the fact that $\alpha \gg \tilde\zeta + \tilde \xi$ when comparing the lower bound with the upper bound.
As a result, we have for $t = t_4^i + \Delta t$ that
\begin{align}
    \baromegaistar^{\pm} (t) = \frac{1}{\baromegaistar (t_4^i)^{-1} - {\sigma_4^i}^{\pm} \Delta t}, \where {\sigma_4^i}^{\pm} \deleq \frac{(1\pm \checkxiiprime \pm 2\alpha_i) \cdot  L \lambda_i }{\sqrt{d_e}\phi_i }, \label{eq:lbdaW* ub-warmup-stage5}
\end{align}
where \say{$+$} and \say{$-$} correspond to the upper and lower bounds, respectively.
Therefore, the duration of Step 5 is given by
\begin{align*}
    \Delta t_5^i = \frac{\baromegaistar(t_4^i)^{-1} - \baromegaistar(t_5^i)^{-1}}{{\sigma_4^i}^{\pm}} = \frac{\frac{\omega_0}{\baromegaistar(t_4^i)} - 1/4}{\omega_0 {\sigma_4^i}^{\pm}} = \left(\frac{3}{4} \pm o(1)\right)\frac{\sqrt{d_e} \phi_i}{\omega_0 L \lambda_i}, 
\end{align*}
where in the last inequality  we use the fact that $\baromegaistar(t_4^i) = (1+o(1))\omega_0$.

\paragraph{Growth of $\baromegaistar / \baromegaih$ and $\muistar / \muih$.}
On the other hand, we have for the gradient of $\baromegaistar - \baromegaih$ that
\begin{align*}
    \partial_t\left(
        \baromegaistar - \baromegaih
    \right)
    &= (1 \pm \hatxiiprime) \cdot \sqrt{d_e}^{-1} \lambda_i \baromegaistar\left(\muistar-\muih\right) 
    \mytag{\eqref{eq:hat xi after step 1} for $\hat\xi_i(t)$} \\
    &\ge (1 - \hatxiiprime) \cdot \sqrt{d_e}^{-1} \lambda_i \baromegaistar \cdot \left(1-\frac{\gamma_i}{H-1}\right) \cdot \muistar 
    \mytag{\Cref{fact:warmup-step4-lbdaU-ratio} } \\
    &\ge \underbrace{\frac{\left(1 - \hatxiiprime-\alpha_i - \frac{\gamma_i}{H-1}\right) \cdot L \lambda_i }{\phi_i \sqrt{d_e}} }_{\ds \deleq \sigma_5^i} \cdot (\baromegaistar)^2.  
    \mytag{\Cref{fact:warmup-step4-eta*}} 
\end{align*}
Integrating the above inequality, with the lower bound for $\baromegaistar$, we conclude that 
\begin{align*}
    \left.\left(\baromegaistar - \baromegaih\right)\right|_{t} 
    &\ge \sigma_5^i \int_{0}^{\Delta t} \frac{1}{\left(\baromegaistar (t_4^i)^{-1} - {\sigma_4^{i}}^{-} \tau\right)^2} \rd \tau + \left.\left(\baromegaistar - \baromegaih\right)\right|_{t_4^i} \\
    &\ge \frac{\baromegaistar(t_4^i) \cdot \sigma_5^i \Delta t}{\baromegaistar(t_4^i)^{-1} - {\sigma_4^i}^{-} \Delta t}, 
\end{align*}
where the last inequality holds by noting that $\baromegaistar(t_4^i) \ge \baromegaih(t_4^i)$.
Now, similar to the analysis in Step 3, with $t_2^i$ replaced by $t_5^i$, $\sigma_1^i$ replaced by $\sigma_5^i$, and $\sigma_2^i$ replaced by ${\sigma_4^i}^{-}$, we adapt \eqref{eq:lbdaU*/lbdaU^h lower bound-warmup-stage3} to have that
\begin{align*}
    \left.\frac{\muistar}{\muih}\right|_{t}
    &\ge \left.\frac{\muistar}{\muih}\right|_{t_4^i} \cdot \exp \left((1-\beta) \lambda_i d_i\sigma_5^i \cdot \left(\baromegaistar(t_4^i) \Delta t\right)^2 /2\right) \\
    &\ge \frac{H-1}{\gamma_i} \cdot \exp \left(\frac 1 2 \cdot (1-\beta) \lambda_i d_i\sigma_5^i \cdot \left(\frac{1-\baromegaistar(t_4^i)/\baromegaistar(t)}{{\sigma_4^{i}}^{+}}\right)^2 \right)\\
    &= \frac{H-1}{\gamma_i} \cdot \exp \left(\frac{\sigma_5^i (1-\beta) \lambda_i d_i (1 - \kappa_i(t))^2 }{2({\sigma_4^{i}}^{+})^2} \right).
\end{align*}
where the second inequality holds by the upper bound for $\baromegaistar$ in \eqref{eq:lbdaW* ub-warmup-stage5}, and in the last inequality we define $\kappa_i(t) \deleq \baromegaistar(t_4^i) / \baromegaistar(t) $.
Not that $\kappa_i(t) > 1/4 + o(1)$ still makes the dynamics satisfying the condition $\baromegaistar(t_5^i) \le 4\omega_0$ given that $\baromegaistar(t_4^i) = (1+o(1))\omega_0$.
By \Cref{fact:critical conditions} and the monotonicity of $\muistar/\muih$, we can guarantee that we are still within the warm-up stage for the \TIFwarmupi dynamics.
Furthermore, by plugging in the definition of $\sigma_5^i$ and $\omega_4^+$, we have
\begin{align*}
    \frac{\sigma_5^i}{(\omega_4^+)^2} = \frac{\left(1 - \hatxiiprime-\alpha_i - \frac{\gamma_i}{H-1}\right)\cdot \sqrt{d_e}\phi_i}{(1 + \checkxiiprime + 2\alpha_i)^2 \cdot L \lambda_i} = \left(1 - \hatxiiprime - 4\alpha_i - \frac{\gamma_i}{H-1} - 2\checkxiiprime\right) \cdot \frac{\sqrt{d_e} \phi_i}{L \lambda_i} \ge \frac{\sqrt{d_e} \phi_i}{2L \lambda_i}.
\end{align*}
In conclusion, we have 
\begin{align}
    \left.\frac{\muistar}{\muih}\right|_{t}
    &\ge \frac{H-1}{\gamma_i} \cdot \exp \left(\frac{ (1-\beta) (1-\kappa_i(t))^2 }{4 \varsigma_i} \right) \ge \frac{H-1}{\gamma_i } \cdot \exp \left(\frac{(1-\kappa_i(t))^2}{8 \varsigma_i} \right).
    \label{eq:lbdaU*/lbdaU^h lb-warmup-stage5}
\end{align}
We can roughly estimate the scale of $\varsigma_i = L / (\sqrt{d_e} d_i\phi_i) \approx L / L^{1.5} = 1/\sqrt L$ given that both $d_e=\cO(L) $ and $d_i =\cO(L)$.
As a result, we have the ratio growing exponentially large, which means those non-optimal head will decay exponentially fast.
Similar idea also applies to the ratio $\baromegaistar / \baromegaih$ where we do not wish $\baromegaih$ for the nonoptimal head to grow large.
For our purpose, we just compare the ratio between $\partial_t(\log\baromegaistar - \log\baromegaih)$ and $\partial_t(\log\baromegaistar)$ as 
\begin{align*}
    \frac{\partial_t(\log\baromegaistar - \log\baromegaih)}{\partial_t(\log\baromegaistar)} 
    &= \frac{(1 \pm \hatxiiprime ) \cdot \sqrt{d_e}^{-1} \lambda_i \left(\muistar-\muih\right)}{(1\pm \check{\xi}_i(t)) \cdot  \sqrt{d_e}^{-1}\lambda_i \muistar} \nend 
    &= \frac{(1 \pm \hatxiiprime ) \cdot \left(1\pm \gamma_i/(H-1)\right)}{(1\pm \check{\xi}_i(t)) } \mytag{\Cref{fact:warmup-step4-lbdaU-ratio}}\\
    &= 1 \pm \hatxiiprime \pm 2\gamma_i/(H-1) \pm 2\checkxiiprime.
\end{align*}
The message is that the growth of $\baromegaistar / \baromegaih$ is roughly the same as $\baromegaistar$.
As a result, we conclude that for any $t \in [t_4^i, t_5^i]$, it holds that
\begin{align}
    \log\left(\frac{\baromegaih(t)}{\baromegaih(t_4^i)}\right) = \pm \left(\hatxiiprime + 2\gamma_i/(H-1) + 2\checkxiiprime\right) \cdot \log\left(\frac{\baromegaistar(t)}{\baromegaistar(t_4^i)}\right).
    \label{eq:lbdaW^h/lbdaW^* ratio-warmup-stage5}
\end{align}
Therefore, $\baromegaih$ is almost \say{fixed} at initialization.
In the following table, we summarize the dynamics in Step 5.
\begin{table}[ht]
    \centering
    \begin{tabular}{l@{\hspace{.5cm}} l}
        \hline
        \hline
        \textbf{Properties} & \textbf{Typical Values} \\
        \hline
        Duration & $\Delta t_{5}^i \ge \frac{\sqrt{d_e} \phi_i}{4\omega_0 L \lambda_i}$ \\
        $\muistar$ & $\muistar \le (1+\tilde\xi + 2\tilde\zeta) L \phi_i^{-1} \baromegaistar$ \\
        $\baromegaistar$ & $\baromegaistar(t_5^i) = 4\omega_0$ \\
        $\baromegaistar / \baromegaih$ & $\log\left(\frac{\baromegaih(t_5^i)}{\baromegaih(t_4^i)}\right) = \pm \left(\hatxiiprime + 2\gamma_i/(H-1) + 2\checkxiiprime\right) \cdot \log\left(\frac{\baromegaistar(t_5^i)}{\baromegaistar(t_4^i)}\right)$ \\
        $\muistar/\muih$ & $\ge \frac{H-1}{\gamma_i } \cdot \exp \left(\frac{(1-\kappa_i(t))^2}{8 \varsigma_i} \right)$ \\
        \hline
    \end{tabular}
    \caption{Summary of Step 5 of the Warm-up Stage}
    \label{tab:summary-warmup-step 5}
\end{table}

\paragraph{The Last Step Takes Much Longer.}
We first summarize the durations for each of the five steps in the warm-up stage in the following table.
\begin{table}[ht]
    \centering
    \begin{tabular}{l@{\hspace{.5cm}} l}
        \hline
        \hline
        \textbf{Step} & \textbf{Duration} \\
        \hline
        Step 1 & $t_1^i \le t_{1+}^i = \frac{1}{\lambda_id_i \omega_0 \tilde\alpha_i} \cdot \log \left(\frac{2L  \omega_0}{\phi_i\muistar(0)}\right)$ \\
        Step 2 & $\Delta t_{2}^i \le \Delta t_{2+}^i = \frac{4c(3\tilde\zeta +2\tilde\xi ) \cdot H \sqrt{d_e} \phi_i}{\beta_i\epsilon L \lambda_i\omega_0}$ \\
        Step 3 & $\Delta t_{3}^i \le \Delta t_{3+}^i = \sqrt{\frac{4c\log \left((H - 1) \gamma_i^{-1}\right) \cdot H \phi_i \sqrt{d_e}}{L \lambda_i^2 d_i \omega_0^2 \cdot  {\epsilon}}}$ \\
        Step 4 & $\Delta t_{4}^i \le \Delta t_{4+}^i = \frac{ 4 e H }{d_i \lambda_i\alpha_i \omega_0}$ \\
        Step 5 & $\Delta t_{5}^i \ge \frac{\sqrt{d_e} \phi_i}{4\omega_0 L \lambda_i}$ \\
        \hline
    \end{tabular}
    \caption{Summary of the Durations for Each Step in the Warm-up Stage}
    \label{tab:summary-warmup-durations}
\end{table}
We calculate the ratio between the duration of the first four steps against the last step as
\begin{gather*}
    \frac{t_{1+}^i}{\Delta t_{5-}^i} 
    \le \frac{8 \varsigma_i }{ \alpha_i  } \cdot \log \left(\frac{2L  \omega_0}{\phi_i\muistar(0)}\right), \quad
    \frac{\Delta t_{2+}^i}{\Delta t_{5-}^i} \le \frac{16 c H (3\tilde\zeta +2\tilde\xi ) }{\beta_i\epsilon}, \\
    \frac{\Delta t_{3+}^i}{\Delta t_{5-}^i} \le \sqrt{\frac{ 
        64c\log \left((H - 1) \gamma_i^{-1}\right) \cdot \varsigma_i  H}{ \epsilon } },\quad
    \frac{\Delta t_{4+}^i}{\Delta t_{5-}^i} 
    \le  \frac{16 e \varsigma_i H }{  \alpha_i  }.
\end{gather*}
By our conditions for $\alpha_i$, $\beta_i$, $\gamma_i$, we can ensure that all these ratios are $o(1)$.
To ensure $o(1)$, we need 
\begin{align*}
    \alpha_i\gg \varsigma_i \cdot \left(\log \left(\frac{2L \omega_0}{\phi_i\muistar(0)}\right) \pmax H\right), \quad \beta_i\gg \frac{16 c H (3\tilde\zeta +2\tilde\xi) }{\epsilon}, \quad
     \gamma_i\gg H \cdot \exp\left(-\frac{\epsilon}{64c \varsigma_i H}\right).
\end{align*}

\paragraph{End of the Warm-up Stage.}
Let us first deal with the nominal task. Our goal is to show that \Cref{itm:warmup-lbdaW ub} and \Cref{itm:warmup-lbdaU ub} for the nominal task will not be violated even for the upper bound of the duration of the warm-up stage, which is given by $\min_{i\in\cI} T_\warmup^i \le \min_{i\in\cI}\frac{\sqrt{d_e} \phi_i}{\omega_0 L \lambda_i}$.
By the gradient of $\partial_t \mujh$ of a nominal task $j \in \cI_c$, we directly conclude that $\mujh$ cannot increase, which verifies \Cref{itm:warmup-lbdaU ub}. 
For the $\baromegajh$'s side, we have by the gradient
$
    \partial_t \baromegajh \le \nu \cdot (\baromegajh)^2
$
that 
\begin{align*}
    \baromegaih &\le \frac{1}{\omega_0^{-1} - 26CL\zeta \sqrt{d_e}^{-1}\cdot \langle \vecd, \lambda \odot \phi \rangle\cdot  \omega_0^2 \cdot \min_{i\in\cI}\frac{\sqrt{d_e} \phi_i}{\omega_0 L \lambda_i}} = \omega_0 (1+o(1)), 
\end{align*}
where the last equality holds by our conditions for $\omega_0$ that $
    \color{teal}\omega_0 \ll \max_{i\in\cI} \frac{\lambda_i}{ \phi_i \langle \vecd, \lambda \odot \phi \rangle \zeta }.
$
Hence, we also have \Cref{itm:warmup-lbdaW ub} satisfied for the nominal task.

Previously, we have shown that the \TIFwarmupi dynamics for $i\in\cI$ in the warm-up stage can be divided into five steps for individual task $i$. 
However, the real dynamics does not necessarily roll out in this way, since if one of these tasks finishes the warm-up stage, at least one of the conditions \Cref{itm:warmup-lbdaU separation}-\Cref{itm:warmup-lbdaU*/lbdaW* ub} will be violated. 
In particular, the violated condition will be one of the critical conditions---\Cref{itm:warmup-lbdaW ub} for a certain effective task since the other critical condition \Cref{itm:warmup-lbdaU separation} will not be violated by the monotonicity of $\muistar/\muih$.

We have $T_\warmup = \min_{i\in [d_y]} T_\warmup^i=\min_{i\in \cI} T_\warmup^i$.
Let $i^*$ be the task that finishes the warm-up stage first.
We first argue that at $T_\warmup$, the other \TIFwarmupj dynamics are all in Step 5 for any $j\in\cI\backslash\{i^*\}$.
\begin{lemma}
    \label[lemma]{fact:all in step 5 at T_warmup}
    All \TIFwarmupj dynamics are in Step 5 at $T_\warmup$ for any $j\in\cI\backslash\{i^*\}$. 
    In particular, we have 
    \[
        \max_{j\in\cI} t_{4+}^j / \Delta t_{5-}^j \ll \vartheta.
    \]
\end{lemma}
\begin{proof}
    We just need to check that $T_\warmup = T_\warmup^{i^*} \ge t_4^{j}$ for any $j\in\cI\backslash\{i^*\}$.
    By the fact that $\Delta t_{l+}^j \ll \Delta t_{5-}^j$ for any $l\in\{1, 2, 3, 4\}$, it suffices to show that $4 \max_{l\in\{1,2,3,4\}}\Delta t_{l+}^{j}/\Delta t_{5-}^j \le \Delta t_{5-}^{i^*} / \Delta t_{5-}^{j}$.
    The left hand side is bounded by 
    \begin{align*}
        \max_{l\in\{1, 2, 3, 4\}} \frac{4\Delta t_{l+}^{j}}{\Delta t_{5-}^j} &\le \max_{j\in\cI}\left\{
    \frac{64e\varsigma_j }{ \alpha_j } \cdot \left(\log \left(\frac{2L  \omega_0}{\phi_j\mu_j^\star(0)}\right) \pmax H\right), \quad
    \frac{64 c H (3\tilde\zeta +2\tilde\xi ) }{\beta_j\epsilon}, \right.\\
    &\autoquad{5}\left.
    32\sqrt{c\log \left((H - 1) \gamma_j^{-1}\right) \cdot  \frac{ 
        \varsigma_j  H}{ \epsilon } },
        \right\}, 
    \end{align*}
    while the right hand side is lower bounded by 
    \begin{align*}
        \frac{\Delta t_{5-}^{i^*}}{ \Delta t_{5-}^{j}} = \frac{\phi_{i^*}\lambda_{i^*}^{-1}}{ \phi_j\lambda_j^{-1}} \ge \frac{\min_{i\in\cI}\phi_i\lambda_i^{-1}}{\max_{j\in\cI}\phi_j\lambda_j^{-1}} \deleq \vartheta.
    \end{align*}
    Hence it suffices to have $\vartheta \gg \max_{l\in\{1, 2, 3, 4\}} \frac{4\Delta t_{l+}^{j}}{\Delta t_{5-}^j}$, which is already guaranteed by our conditions for $\alpha_j$, $\beta_j$, $\gamma_j$ in \Cref{def:warmup stage-split}.
\end{proof}
Since all the \TIFwarmupj dynamics are in Step 5, we can easily conclude from \eqref{eq:lbdaW^h/lbdaW^* ratio-warmup-stage5} that for any $i\in\cI$ and $h\in[H]\backslash\{h_i^*\}$, it holds that
\begin{align*}
    \baromegaih(T_\warmup) &= \baromegaih(t_4^i) \cdot \exp\left(\pm \left(\hatxiiprime + 2\gamma_i/(H-1) + 2\checkxiiprime\right) \cdot \log\left(\frac{\baromegaistar(t)}{\baromegaistar(t_4^i)}\right)\right) \\
    &\le \baromegaistar(t_4^i) \cdot \left(1 + \left(\hatxiiprime + \frac{2\gamma_i}{H-1} + 2\checkxiiprime\right) \cdot e\log 4\right) = (1+o(1)) \omega_0,
\end{align*}
where the inquality holds by the fact that $\baromegaistar(t_5^i) \le 4\omega_0$ and $\baromegaistar(t_4^i) = (1+o(1))\omega_0$.
Therefore, $\baromegaih(T_\warmup)$ is \say{fixed} at initialization as $\baromegaistar(t_4^i) = (1+o(1))\omega_0$.
Meanwhile, we have from \eqref{eq:lbdaW* ub-warmup-stage5} for $\baromegaistar(T_\warmup)$ that
\begin{align*}
    \baromegaistar (T_\warmup) 
    &\ge \frac{1}{\baromegaistar (t_4^i)^{-1} - \frac{(1- \checkxiiprime - 2\alpha_i) \cdot  L \lambda_i }{\sqrt{d_e}\phi_i } \Delta t_{5-}^{i^*}} \nend
    &\ge \left(\baromegaistar (t_4^i)^{-1} - \frac{(1- \checkxiiprime - 2\alpha_i)(3/4 - o(1))}{ \omega_0 } \cdot \frac{\lambda_i\phi_i^{-1}}{\max_{k\in\cI} \lambda_k\phi_k^{-1}}\right)^{-1} \nend
    &\ge \baromegaistar (t_4^i) \cdot \left(1 - \frac{\lambda_i\phi_i^{-1}}{1.5\cdot \max_{k\in\cI} \lambda_k\phi_k^{-1}}\right)^{-1}.
\end{align*}
Thus, we conclude that 
\begin{align*}
    &\left.\frac{\baromegaistar  - \baromegaih }{\baromegaistar }\right|_{T_\warmup} \nend
    &\quad \ge 1 - \left(1 + \left(\hatxiiprime + \frac{2\gamma_i}{H-1} + 2\checkxiiprime\right) \cdot e\log 4\right) \cdot \left(1 - \frac{\lambda_i\phi_i^{-1}}{1.5\cdot \max_{k\in\cI} \lambda_k\phi_k^{-1}}\right) \nend
    &\quad \ge \frac{\lambda_i\phi_i^{-1}}{1.5\cdot \max_{k\in\cI} \lambda_k\phi_k^{-1}} - \left(\hatxiiprime + \frac{2\gamma_i}{H-1} + 2\checkxiiprime\right) \cdot e\log 4
    \ge \frac{\lambda_i\phi_i^{-1}}{2\cdot \max_{k\in\cI} \lambda_k\phi_k^{-1}}.
\end{align*}
Under the condition that $\color{teal} \frac{\lambda_i\phi_i^{-1}}{\max_{k\in\cI} \lambda_k\phi_k^{-1}} \gg \max\{\hatxiiprime, \gamma_i/H, \checkxiiprime\}$, 
At the same time, we have for $\muih$ with $h\in[H]\backslash\{h_i^*\}$ that
\begin{align}
    \left.\frac{\muih}{\muistar}\right|_{T_\warmup} \le \frac{\gamma_i}{ H-1} \cdot \exp \left(-\frac{(1-\kappa_i(T_\warmup))^2}{8 \varsigma_i} \right).
    \label{eq:lbdaU^h/lbdaU^* ratio-warmup-T_warmup-1}
\end{align}
Given that $\kappa_i(t) < 1$, it suffices to upper bound $\kappa_i(T_\warmup)$ for upper bounding the ratio. 
Recall by definition, 
\begin{align*}
    \kappa_i(T_\warmup) = \frac{\baromegaistar(t_4^i)}{\baromegaistar(T_\warmup)} &\le \baromegaistar(t_4^i) \cdot \left(\baromegaistar (t_4^i)^{-1} - {\sigma_4^{i}}^{-} (T_\warmup - t_4^i)\right) \\
    &\le  1 - \omega_0\cdot {\sigma_4^{i}}^{-} (\Delta t_{5-}^{i^*} - t_{4+}^i), 
\end{align*}
where in the first inequality we invoke \Cref{eq:lbdaW* ub-warmup-stage5} and in the second inequality we invoke the lower bound for $T_\warmup$ since $T_\warmup = T_\warmup^{i^*} \ge \Delta t_{5-}^{i^*}$ and lower bound $\baromegaistar(t_4^i)$ by its initial value $\omega_0$.
To further lower bound $\Delta t_{5-}^{i^*} - t_{4+}^i$, we invoke \Cref{fact:all in step 5 at T_warmup} that $\max_{j\in\cI} t_{4+}^j / \Delta t_{5-}^j \ll \vartheta$ and conclude that 
\[
    \Delta t_{5-}^{i^*} - t_{4+}^i = \Delta t_{5-}^{i^*} \cdot \left(1 - \frac{t_{4+}^i}{\Delta t_{5-}^i} \cdot \frac{\Delta t_{5-}^i}{\Delta t_{5-}^{i^*}}\right) \ge \Delta t_{5-}^{i^*} \cdot \left(1 - \frac{t_{4+}^i}{\Delta t_{5-}^i} \cdot \vartheta^{-1}\right) \ge \frac{\Delta t_{5-}^{i^*}}{\sqrt 2}.
\]
Therefore, we have for \eqref{eq:lbdaU^h/lbdaU^* ratio-warmup-T_warmup-1} that 
\begin{align}
    \left.\frac{\muih}{\muistar}\right|_{T_\warmup} 
    &\le \frac{\gamma_i}{ H-1} \cdot \exp \left(-\frac{\left(\omega_0\cdot {\sigma_4^{i}}^{-} \Delta t_{5-}^{i^*}\right)^2}{16 \varsigma_i} \right) \notag\\
    &\le \frac{\gamma_i}{ H-1} \cdot \exp \left(-\frac{\left(\omega_0\cdot \frac{(1-\checkxiiprime-2\alpha_i) \cdot  L \lambda_i }{\sqrt{d_e}\phi_i } \cdot \frac{\sqrt{d_e} \phi_{i^*}}{4\omega_0 L \lambda_{i^*}}\right)^2}{16 \varsigma_i} \right) \notag\\
    &\le \frac{\gamma_i}{ H-1} \cdot \exp \left(-\frac{\left( {\lambda_{i^*}^{-1} \phi_{i^*}} / {(\lambda_i^{-1} \phi_i)}\right)^2}{256 \varsigma_i} \right) \le \frac{\gamma_i}{ H-1} \cdot \exp \left(-\frac{\vartheta^2}{256 \varsigma_i} \right) \deleq \varrho_i\ll 1.
    \label{eq:varrho def}
\end{align}
In conclusion, as long as $\color{teal}\varsigma_i\ll \vartheta^2$, we have the ratio $\muih/\muistar$ being exponentially small.
In the following table, we summarize the dynamics at $T_\warmup$.
\begin{table}[h]
    \centering
    \begin{tabular}{l@{\hspace{.5cm}} l}
        \hline
        \hline
        \textbf{Properties} & \textbf{Typical Values} \\
        \hline
        Duration & $T_\warmup = \min_{i\in [d_y]} T_\warmup^i$ \\
        $\muistar$ & $(1-\alpha_i)\phi_i^{-1} L \baromegaistar\le \muistar \le (1+\tilde\xi + 2\tilde\zeta) L \phi_i^{-1} \baromegaistar$ \\
        $\baromegaistar$ & $\baromegaistar\le 4\omega_0$ with $\baromega_{i^*}^\star = 4\omega_0$ \\
        $\baromegaistar / \baromegaih$ & $\log\left(\frac{\baromegaih(T_\warmup)}{\baromegaih(t_4^i)}\right) = \pm \left(\hatxiiprime + 2\gamma_i/(H-1) + 2\checkxiiprime\right) \cdot \log\left(\frac{\baromegaistar(T_\warmup)}{\baromegaistar(t_4^i)}\right)$ \\
        $\muistar/\muih$ & $\left.\frac{\muih}{\muistar}\right|_{T_\warmup} \le \varrho_i \ll 1$ \\
        \hline
    \end{tabular}
    \caption{Summary of the Dynamics at $T_\warmup$}
    \label{tab:summary-warmup-T_warmup}
\end{table}

\subsection{Growing Stage}
We give the following definition for the growing stage of the dynamics.
\begin{definition}[Growing Stage]
    We say that the dynamics is in the growing stage if all the conditions in \Cref{cond:bounded weights} and the following conditions hold:
    \begin{myenumi}
        \item \label{cond:S2-growth-stage2} $\muistar \ge  L \phi_i^{-1}  \omega_0 / \sqrt 2$ for all $i\in\cI$.
        \item \label{cond:S2-dominate}
        $\baromegaistar > \baromegaih$, $\baromegaistar/\baromegaih \ge \baromegaistar/\baromegaih\biggiven_{T_\warmup}$ and $\muistar/\muih \ge \muistar/\muih\biggiven_{T_\warmup}$
        for any $(i, h)\in(\cI\otimes[H])\backslash\{(i, h_i^*)\given i\in\cI\}$.
        \item \label{cond:S2-nonopt-lbdaW-ub}
        $\baromegaih \le \sqrt{2}\omega_0$ for all $(i, h)\in([d_y]\otimes[H]) \backslash \{(i, h_i^*)\given i\in\cI\}$.
        \item \label{cond:S2-lbdaU*/lbdaW* ratio lb&ub}
        $\left(d_i \pmin Le^{-1}\phi_i^{-1}\right)/4\le \muistar/\baromegaistar \le \sqrt 2 L\phi_i^{-1}$ for all $i\in\cI$.
        \item \label{cond:S2-lbdaU* lbdaW^* ub} $\muistar \baromegaistar \le \frac{2}{1+d_i\phi_iL^{-1}}$ for all $i\in\cI$.
    \end{myenumi}
\end{definition}

Based on these conditions, we can simplify the dynamics based on \eqref{eq:dot baromegah} and \eqref{eq:dot muh} in the growing stage as follows.
For completeness, we copy \eqref{eq:dot baromegah} here: 
\begin{align}
    \sqrt{d_e} \cdot \partial_t \baromegaih
    &= - (1 \pm \zeta)\cdot \sum_{h'=1}^H 
        \frac{\exp\bigl(\langle\omega^\h, \omega^\hprime\rangle\bigr)}{L} 
        \cdot \bigl\langle\vecd\odot \lambda \odot \phi, \muh \odot \muhprime \bigr\rangle 
        \cdot\baromegaihprime \baromegaih 
        \label{eq:S2-dot lbdaW-1}
        \\
    & \qquad - (1\pm \xi)\lambda_i \cdot \sum_{h'=1}^H  \muihprime \muih\baromegaihprime \baromegaih 
    + (1\pm \xi)\lambda_i \cdot \muih \baromegaih 
    \label{eq:S2-dot lbdaW-2}
    \\
    &\qquad (\pm1) \cdot \frac{\exp\bigl(\langle\omega^\h, \omega^\h\rangle\bigr)}{L} 
        \cdot \bigl\langle \vecd \odot \lambda, \baromegah \odot \muh \bigr\rangle \cdot (\baromegaih)^2 
        \label{eq:S2-dot lbdaW-3}
        \\
    &\qquad (\pm\eta) \cdot  \sum_{h'=1}^H \bigl\langle\vecd\odot \lambda \odot \phi, \muh \odot \muhprime \bigr\rangle 
    \cdot (\baromegaih)^2.
    \label{eq:S2-dot lbdaW-4}
\end{align}
We notice that the cross-task interference happens when dealing with $\langle \omega^\h, \omega^\hprime\rangle$, $\langle\vecd\odot \lambda \odot \phi, \muh \odot \muhprime \rangle$ and $\langle \vecd \odot \lambda, \baromega^\h \odot \muh \rangle$.
For the first term, we have 
\begin{align}
    \langle \omega^\h, \omega^\hprime\rangle 
    &= \sum_{i=1}^{d_y} \baromegaih \baromegaihprime d_i \nend
    &\overset{(a)}{\le} \sum_{i=1}^{d_y} d_i\cdot \left(\baromegaistar \chi_i^\h + \sqrt 2 \omega_0\right) \cdot \left(\baromegaistar \chi_i^\hprime + \sqrt 2 \omega_0\right) \nend
    &\overset{(b)}{=} d_{i_h^*} \cdot (\baromega_{i_h^*}^\star)^2 \cdot \ind(h=h'\in \cB) + 2\omega_0 \left( \sqrt{d_{i_{h^\prime}^*} } \ind(h'\in \cB) + \sqrt{d_{i_{h}^*}}  \ind(h\in \cB)\right) + 2\omega_0^2 d \nend
    &\le d_{i_h^*} \cdot (\baromega_{i_h^*}^\star)^2 \cdot \ind(h=h'\in \cB) + \underbrace{\left(4\omega_0 \sqrt{d} + 2\omega_0^2 d\right)}_{\ds \color{teal} \ll 1}
    \le 1+o(1).
    \label{eq:S2-product-1-ub}
\end{align}
where we define $\chi_i^\h$ as an indicator function for which takes value $1$ if $h$ is the unique optimal head for task $i\in\cI$ and $0$ otherwise (also $0$ if $i\in\cI_c$ belongs to the nominal tasks), and $\cB=\{h\in[H]\given \exists i\in\cI \st h=h_i^*\}$ as the set for the heads that are optimal for some tasks, $\ind(\cdot)$ as the indicator function, and $i_h^*$ as the unique task that takes head $h$ as the optimal head if $h\in\cB$.

Here, $(a)$ holds by noting that each task $i$ has a unique optimal head $h_i^*$ and $\baromegaih \le \sqrt 2 \omega_0$ for all nonoptimal head $h\in[H]\backslash\{h_i^*\}$ according to \Cref{cond:S2-nonopt-lbdaW-ub}.
$(b)$ holds by noting that $\chi_i^\h \chi_i^\hprime= \ind(h=h_i^*\in\cB)$ by the uniqueness of the optimal head and $\sqrt{d_{i_{h^\prime}^*}}$ and $\sqrt{d_{i_{h}^*}}$ comes from the upper bound $\baromegaih\le \sqrt{2 d_i^{-1}}$ according to \Cref{cond:bounded weights}.
The last inequality holds by invoking the upper bound for $\baromegaistar$ in \Cref{cond:bounded weights}.
A naive lower bound for $\langle \omega^\h, \omega^\hprime\rangle$ is 
\[
    \langle \omega^\h, \omega^\hprime\rangle \ge d_{i_h^*} \cdot (\baromega_{i_h^*}^\star)^2 \cdot \ind(h=h'\in \cB)
\] 
by the nonnegativity of $\baromegaih$ and $\baromegaihprime$.

For the second term $\langle\vecd\odot \lambda \odot \phi, \muh \odot \muhprime \rangle$, we have
\begin{align*}
    &\langle\vecd\odot \lambda \odot \phi, \muh \odot \muhprime \rangle  
    = \sum_{i\in\cI} d_i \lambda_i \phi_i \muih \muihprime + \sum_{j\in\cI_c} \sigma^2 d  \mujh\mu_j^{\hprime}\\
    &\quad \overset{(a)}{\le} \sum_{i\in\cI} d_i \lambda_i \phi_i (\muistar)^2\cdot \left(\chi_i^\h + \varrho_i\right) \cdot \left(\chi_i^\hprime + \varrho_i\right)  + \sum_{j\in\cI_c} \sigma^2 d  \mujh(0)\mu_j^{\hprime}(0)\\
    &\quad = E_{i_h^*}(\mu_{i_h^*}^\star)^2\cdot \ind(h=h'\in\cB) + E_{i_h^*} \varrho_{i_h^*}(\mu_{i_h^*}^\star)^2 \cdot \ind(h\in\cB)\\
    &\autoquad{3} + E_{i_{h^\prime}^*} \varrho_{i_{h^\prime}^*}(\mu_{i_{h^\prime}^*}^\star)^2 \cdot \ind(h'\in\cB) + \sum_{i\in \cI} E_i (\muistar)^2 (\varrho_i)^2 \\
    &\autoquad{3} + E_{i_h^*} (\mu_{i_{h^\prime}^*}^\star)^2 \cdot  \underbrace{\max_{i\in\cI, j\in\cI_c,h\in[H]}\frac{\sigma^2 d  |\cI_c|}{E_i} \cdot \left(\frac{\mujh(0)}{\frac 1 4 \left(d_i \pmin \frac{L}{e\phi_i}\right)\omega_0}\right)^2}_{\ds \color{teal} \deleq \nu_1 \ll 1} , 
\end{align*}
where we define $E_i = d_i \lambda_i \phi_i = (d_i\lambda_i / (d ) + \sigma^2)\cdot d $ as the rescaled \emph{total energy} for task $i$.
Here, the first term in $(a)$ holds by noting that $\muih/\muistar \le \varrho_i$ for all nonoptimal head $h\in[H]\backslash\{h_i^*\}$ and $i\in \cI$ according to \eqref{eq:varrho def} and \Cref{cond:S2-dominate} which says that $\muistar/\muih\ge\varrho_i$. 
The second term in $(a)$ holds by noting that $\mujh$ is non-increasing for all $j\in\cI_c$.
Next, we invoke the upper bound $\muih \le \sqrt{2 L\phi_i^{-1}} \le \sqrt{2 L}$ in \Cref{cond:bounded weights} to get
\begin{align}
    \langle\vecd\odot \lambda \odot \phi, \muh \odot \muhprime \rangle  
    &\le E_{i_h^*} (\mu_{i_h^*}^\star)^2 \ind(h=h'\in\cB) + (4+2|\cI|\bar\varrho)\bar E\bar\varrho L  + \sigma^2 d |\cI_c| \bar\mu^2,  \nend
    \langle\vecd\odot \lambda \odot \phi, \muh \odot \muhprime \rangle&\le E_{i_h^*}\cdot (\mu_{i_h^*}^\star)^2\cdot (\ind(h=h'\in\cB) +\nu_1) + (4 + 2 |\cI| \bar\varrho)\bar E \bar \varrho L,
    \label{eq:S2-product-2-ub}
\end{align}
where we define $\bar E = \max_{i\in\cI} E_i$, $\bar \varrho = \max_{i\in\cI} \varrho_i$, and $\bar\mu\deleq \max_{h\in[H], j\in\cI_c}\mujh(0)$.
A naive lower bound for $\langle\vecd\odot \lambda \odot \phi, \muh \odot \muhprime \rangle$ is
\begin{align}
    \langle\vecd\odot \lambda \odot \phi, \muh \odot \muhprime \rangle \ge E_{i_h^*}\cdot (\mu_{i_h^*}^\star)^2\cdot \ind(h=h'\in\cB).
    \label{eq:S2-product-2-lb}
\end{align}
Lastly, for $\langle \vecd \odot \lambda, \baromega^\h \odot \muh \rangle$, we have
\begin{align}
    \langle \vecd \odot \lambda, \baromegah \odot \muh \rangle
    & = \sum_{i\in\cI} d_i \lambda_i \muih \baromegaih \nend
    & \le \sum_{i\in\cI} d_i \lambda_i \muistar \cdot \left(\chi_i^\h + \varrho_i(\chi_i^\h)^c\right) \cdot \left(\tilde\lambda_W^*[i] \chi_i^\h + \sqrt 2 \omega_0 (\chi_i^\h)^c\right)\nend
    & \le d_{i_h^*} \lambda_{i_h^*} \mu_{i_h^*}^\star \baromega_{i_h^*}^\star \ind(h\in\cB) + 2 \omega_0 |\cI| \bar E \sqrt{L} \bar\varrho, 
    \label{eq:S2-product-3-ub}
\end{align}
where we additionally define $(\chi_i^\h)^c = 1 - \chi_i^\h$ as the complement of $\chi_i^\h$.
Here, the first inequality holds for the same argument as before,  and the second inequality holds by invoking the upper bound for $\muistar$ in \Cref{cond:bounded weights}.

Now, we are ready to simplify the dynamics in the growing stage.
Consider $i\in\cI$ as a specific effective task.
For \eqref{eq:S2-dot lbdaW-1} with $\color{violet} h=h_i^*$, we have the following ratio bound:
\begin{align}
    &\frac{\sum_{h'=1}^H \exp\bigl(\langle\omega^\star, \omega^\hprime\rangle\bigr) L^{-1}
    \cdot \bigl\langle\vecd\odot \lambda \odot \phi, \lambda_{U}^* \odot \muhprime \bigr\rangle 
    \cdot\baromegaihprime \baromegaistar}{ \exp\bigl(\langle\omega^\star, \omega^\star\rangle\bigr) L^{-1} E_i (\muistar)^2 \cdot (\baromegaistar)^2} \nend
    &\quad \le \sum_{h'=1}^H \exp\left( \langle \baromega^\star, \baromega^\star \rangle\cdot \ind(h'=h_i^*) + \left(4\omega_0\sqrt{d} + 2\omega_0^2 d\right) \cdot \ind(h'\neq h_i^*) - \langle \baromega^\star, \baromega^\star \rangle\right)  \nend
    &\autoquad{3} \cdot \baromegaihprime \frac{E_i\cdot (\muistar)^2\cdot (\ind(h'=h_i^*) + \nu_1) + (4 + 2|\cI|\bar\varrho)\bar E \bar \varrho L}{E_i (\muistar)^2 \baromegaistar} \mytag{\eqref{eq:S2-product-1-ub}\text{ and }\eqref{eq:S2-product-2-ub} }\\
    &\quad \color{teal}\le 
    1+ \underbrace{eH\nu_1 + \max_{i\in\cI} \frac{4eH (4 + 2|\cI|\bar\varrho)\bar E \bar \varrho L}{E_i \left(L\omega_0\phi_i^{-1}\right)^2 }}_{\ds \deleq \nu_2} = \deleq 1+ \nu_2 = 1+o(1),
    \label{eq:S2-dot lbdaW-1-dominant}
\end{align}
where the second inequality holds by also noting that $\baromegaihprime/\baromegaistar \le 1$ for all $i\in\cI$ and $h'\in[H]\backslash\{h_i^*\}$ according to \Cref{cond:S2-dominate}, and we also invoke the following lower bound on $\muistar$ for all $\color{violet}i\in\cI$ by \Cref{cond:S2-growth-stage2}:
\begin{align}
    \muistar &\ge (1-\alpha_i) \phi_i^{-1} L \omega_0 \ge \frac{L\omega_0}{\sqrt 2\phi_i}
    \label{eq:S2-lbdaU* lb}
\end{align}
The message from \eqref{eq:S2-dot lbdaW-1-dominant} is that $ \exp\bigl(\langle\omega^\star, \omega^\star\rangle\bigr) L^{-1} E_i (\muistar)^2 \cdot (\baromegaistar)^2$ is the dominant term in \eqref{eq:S2-dot lbdaW-1} for any $i\in\cI$.
Again, we need to deal with the error terms \eqref{eq:S2-dot lbdaW-3} and \eqref{eq:S2-dot lbdaW-4}.
For \eqref{eq:S2-dot lbdaW-3}, we observe for any $\color{violet}h\in\cB$ that
\begin{align}
    &\frac{(\pm1)\cdot \exp\bigl(\langle\omega^\h, \omega^\h\rangle\bigr) L^{-1}
    \cdot \bigl\langle \vecd \odot \lambda, \baromegah \odot \muh \bigr\rangle \cdot (\baromegaih)^2 }{ \exp\bigl(\langle\omega^\h, \omega^\h\rangle\bigr) L^{-1}
    \cdot \bigl\langle\vecd\odot \lambda \odot \phi, (\muh)^{\odot 2}\bigr\rangle 
    \cdot (\baromegaih)^2 } 
    \nend
    &\quad \overset{(a)}{\le} C \cdot \frac{d_{i_h^*} \lambda_{i_h^*} \mu_{i_h^*}^\star \baromega_{i_h^*}^\star \ind(h\in\cB) + \sqrt 2 \omega_0 |\cI| \bar E \sqrt{2L} \bar\varrho}{E_{i_h^*}\cdot (\mu_{i_h^*}^\star)^2\cdot \ind(h\in\cB)} \nend
    &\quad = C \cdot \left(\frac{\baromega_{i_h^*}^\star}{\mu_{i_h^*}^\star} + \frac{2 \omega_0 |\cI| \bar\varrho \bar E \sqrt{L}}{E_{i_h^*}\cdot (\mu_{i_h^*}^\star)^2}\right) \nend
    &\quad \color{teal}\overset{(b)}{\le} \max_{i\in\cI} \left\{
        4C \left(d_k^{-1} \pmax \frac{e\phi_k}{L}\right) + \frac{12C \omega_0 |\cI| \bar\varrho \bar E \sqrt{L}}{E_k\cdot \left(\phi_k^{-1} L \omega_0\right)^2} \right\}
     \deleq \nu_3 \ll 1,
    \label{eq:S2-dot lbdaW-3-dominated by 1}
\end{align}
where in $(a)$ we use the upper bound and lower bound in \eqref{eq:S2-product-3-ub} and \eqref{eq:S2-product-2-lb} respectively, and in $(b)$ we invoke the lower bound for $\muistar/\baromegaistar$ in \Cref{cond:S2-lbdaU*/lbdaW* ratio lb&ub} and also the lower bound for $\muistar$ in \eqref{eq:S2-lbdaU* lb}.
The message conveyed here is that \eqref{eq:S2-dot lbdaW-3} is negligible compared to \eqref{eq:S2-dot lbdaW-1} for any $i\in\cI$ and $h\in\cB$.
On the other hand, when $\color{violet}h\notin\cB$, we have following the upper bounds for the above ratio following from \eqref{eq:S2-product-1-ub} and \eqref{eq:S2-product-3-ub} that
\begin{align}
    (\pm1)\cdot\frac{\exp\bigl(\langle\omega^\h, \omega^\h\rangle\bigr)}{L} 
    \cdot \bigl\langle \vecd \odot \lambda, \baromegah \odot \muh \bigr\rangle \cdot (\baromegaih)^2 
    \le \frac{2Ce^{1.1} \omega_0 |\cI| \bar E \sqrt{L} \bar\varrho }{L} \cdot (\baromegaih)^2.
    \label{eq:S2-dot lbdaW-3-ub}
\end{align}
Similarly, we have for the error term \eqref{eq:S2-dot lbdaW-4} that when $\color{violet}h\in\cB$,
\begin{align}
    &\frac{(\pm \eta) \cdot \sum_{h'=1}^H \bigl\langle\vecd\odot \lambda \odot \phi, \muh \odot \muhprime \bigr\rangle 
    \cdot (\baromegaih)^2 }{ \exp\bigl(\langle\omega^\h, \omega^\h\rangle\bigr) L^{-1}
    \cdot \bigl\langle\vecd\odot \lambda \odot \phi, (\muh)^{\odot 2}\bigr\rangle 
    \cdot (\baromegaih)^2} \nend
    &\quad \le \frac{\sum_{h'=1}^H \bigl\langle\vecd\odot \lambda \odot \phi, \muh \odot \muhprime \bigr\rangle}{\bigl\langle\vecd\odot \lambda \odot \phi, (\muh)^{\odot 2}\bigr\rangle} \cdot C L \eta \nend
    &\quad \le \frac{E_{i_h^*}\cdot (\mu_{i_h^*}^\star)^2\cdot (\ind(h\in\cB) + H\nu_1) + H (4 + 2|\cI|\bar\varrho)\bar E \bar \varrho L}{E_{i_h^*}\cdot (\mu_{i_h^*}^\star)^2\cdot \ind(h\in\cB)} \cdot C L \eta 
    \mytag{\eqref{eq:S2-product-2-ub} \& \eqref{eq:S2-product-2-lb}}\nend
    &\quad \le \biggl((1+H\nu_1) + \underbrace{\frac{4H (4 + 2|\cI|\bar\varrho)\bar E \bar \varrho L}{E_{i_h^*}\cdot \left(\phi_{i_h^*}^{-1} L \omega_0\right)^2}}_{\ds \color{teal} \ll 1}\biggr) \cdot C \zeta \le 2C\zeta \ll 1, 
    \label{eq:S2-dot lbdaW-4-dominated by 1}
\end{align}
where the last inequality holds by \eqref{eq:S2-lbdaU* lb} on the lower bound for $\mu_{i_h^*}^\star$ and by condition $\zeta \ge L \eta$.
When $\color{violet}h\notin\cB$,
\begin{align}
    &\sum_{h'=1}^H \bigl\langle\vecd\odot \lambda \odot \phi, \muh \odot \muhprime \bigr\rangle 
    \cdot (\baromegaih)^2 \cdot  \eta \nend
    &\quad \le \left((4 + 2|\cI|\bar\varrho)\bar E \bar \varrho L + \sigma^2 d H |\cI_c|\bar\mu^2\right)  \cdot C\zeta  \cdot (\baromegaih)^2.
    \label{eq:S2-dot lbdaW-4-ub}
\end{align}
Lastly, for the first term in \eqref{eq:S2-dot lbdaW-2}, we have for any $\color{violet} i\in\cI$ that
\begin{align}
    \frac{\sum_{h'=1}^H  \muihprime \muih\baromegaihprime \baromegaih}{ \muih\baromegaih \muistar \baromegaistar} 
    & = 1 + \sum_{h'\neq h_i^*} \frac{\muihprime  \baromegaihprime}{\muistar \baromegaistar} \le 1 + H \bar\varrho, 
    \label{eq:S2-dot lbdaW-2-dominant}
\end{align}
where we invoke the upper bound $\muih/\muistar \le \bar\varrho$ by \Cref{cond:S2-dominate}.

\paragraph{Simplifying $\partial_t \baromegaih$ in the Growing Stage.}
We have by \eqref{eq:dot baromegah} that when $\color{violet}h\in\cB$ and $\color{violet}i\in\cI$, 
\begin{align}
    \sqrt{d_e} \cdot  \partial_t \baromegaih
    &= - (1 \pm \zeta)\cdot \sum_{h'=1}^H 
        \frac{\exp\bigl(\langle\omega^\h, \omega^\hprime\rangle\bigr)}{L} 
        \cdot \bigl\langle\vecd\odot \lambda \odot \phi, \muh \odot \muhprime \bigr\rangle 
        \cdot\baromegaihprime \baromegaih 
        \notag
        \\
    & \qquad - (1 \pm 2\xi \pm H\bar\varrho)\cdot \lambda_i \muistar \baromegaistar \muih\baromegaih 
    + (1\pm \xi)\cdot \lambda_i \muih \baromegaih 
    \mytag{\eqref{eq:S2-dot lbdaW-2-dominant} \& $i\in\cI$}
    \notag\\
    &\qquad   \left(\pm
        (\nu_3 + 2C\zeta) 
    \right) \cdot \frac{\exp\bigl(\langle\omega^\h, \omega^\h\rangle\bigr)}{L} 
    \cdot \bigl\langle\vecd\odot \lambda \odot \phi, (\muh)^{\odot 2}\bigr\rangle 
    \cdot (\baromegaih)^2  \mytag{\eqref{eq:S2-dot lbdaW-3-dominated by 1} \& \eqref{eq:S2-dot lbdaW-4-dominated by 1} \& $h\in\cB$}
    \notag\\
    &= - (1 \pm {(2\zeta + \nu_3 + 2C\zeta)})\cdot \sum_{h'=1}^H 
        \frac{\exp\bigl(\langle\omega^\h, \omega^\hprime\rangle\bigr)}{L} \cdot \bigl\langle\vecd\odot \lambda \odot \phi, \muh \odot \muhprime \bigr\rangle  \nend
    & \qquad \cdot\baromegaihprime \baromegaih + \lambda_i \muih \baromegaih\left(- \muistar \baromegaistar 
    + (1\pm {(4 \xi + 2H\bar\varrho + \xi)})\right), \notag
\end{align}
where in the last equality we merge the error terms into the first term and 
invoke \Cref{cond:S2-lbdaU* lbdaW^* ub} on the upper bound of $\muistar\baromegaistar \le 2$ to get the error term $4\xi + 2H\bar\varrho + \xi$ out in the last line. 
Define $\tilde\xi_1 = 4\xi + 2H\bar\varrho + \xi$.
For the case $i\in\cI$ with $h\in\cB\backslash\{h_i^*\}$, we have a naive upper bound for this gradient
\begin{align}
    \partial_t\baromegaih 
    &\le \sqrt{d_e}^{-1} \lambda_i \muih \baromegaih\left(- \muistar \baromegaistar 
    + (1\pm {\tilde\xi_1})\right) \nend 
    &\le \bar\varrho\sqrt{d_e}^{-1} \lambda_i \muistar \baromegaih\left(- \muistar \baromegaistar 
    + (1\pm {\tilde\xi_1})\right) \nend
    &\le \bar\varrho\sqrt{d_e}^{-1} \lambda_i \sqrt{2L\phi_i^{-1}}\baromegaih\left(- \muistar \baromegaistar 
    + (1\pm {\tilde\xi_1})\right)
    , 
    \label{eq:S2-dot lbdaW-nonoptimal}
\end{align}
where the second inequality holds by invoking the upper bound $\muih/\muistar \le \bar\varrho$ by \Cref{cond:S2-dominate} and the last inequality holds by invoking \Cref{cond:bounded weights}.
For $\baromegaistar$ where $i\in\cI$, we can further simplify the dynamics as
\begin{align}
    \partial_t \baromegaistar 
    & = - (1 \pm \underbrace{(2\zeta + \nu_3+2C\zeta + 2\nu_2)}_{\ds \deleq\tilde\zeta_1}) \cdot \sqrt{d_e}^{-1} \frac{\exp\bigl(\langle\omega^\star, \omega^\star\rangle\bigr)}{L} E_i (\muistar)^2 \cdot (\baromegaistar)^2 
    \mytag{\eqref{eq:S2-dot lbdaW-1-dominant}}
    \nend 
    &\qquad + \sqrt{d_e}^{-1} \lambda_i \muistar \baromegaistar\Bigl(- \muistar \baromegaistar 
    + (1\pm \tilde\xi_1)\Bigr)\nend
    & =   \left(1 - \left(1 + (1 \pm \tilde\zeta_1) \cdot \frac{\exp\bigl(\langle\omega^\star, \omega^\star\rangle\bigr)}{L} d_i\phi_i \right) \muistar  \baromegaistar\pm \tilde\xi_1\right) \sqrt{d_e}^{-1} \lambda_i\baromegaistar\muistar \nend 
    & = \left(1 - \left(1 + \frac{\exp\bigl(d_i(\baromegaistar)^2\bigr)}{L} d_i\phi_i \right) \muistar  \baromegaistar\pm \tau_1\right) \sqrt{d_e}^{-1} \lambda_i\baromegaistar\muistar,
    \label{eq:S2-dot lbdaW-optimal}
\end{align}
where we define $\color{teal} \tau_1 \triangleq \tilde\xi_1 + 2e \tilde\zeta_1 + e^{1.1} 2\omega_0^2 d = 4(C+1)e\zeta + 4e\nu_2 + 2e\nu_3 + 5\xi+2H\bar\varrho + 2e^{1.1}\omega_0^2 d\ll 1$. The scale of $\tau_1$ comes from the upper bound for $\baromegaistar \muistar \le 2/(1+d_i\phi_iL^{-1})$ in \Cref{cond:S2-lbdaU* lbdaW^* ub} and also the upper bound $\langle \omega^{(h_i^*)},  \omega^{(h_i^*)}\rangle - d_i (\baromegaistar)^2 \le 2\omega_0^2 d$ by \Cref{cond:S2-nonopt-lbdaW-ub}.
For $\omega_j^\h$ where $\color{violet}h\in\cB$ and $\color{violet}j\in\cI_c$, we have
\begin{align}
    \partial_t \log \omega_j^\h 
    = - \frac{1 \pm \tilde\zeta_1}{\sqrt{d_e}} \cdot \sum_{h'=1}^H 
        \frac{\exp\bigl(\langle\omega^\h, \omega^\hprime\rangle\bigr)}{L} 
        \cdot \bigl\langle\vecd\odot \lambda \odot \phi, \muh \odot \muhprime \bigr\rangle 
        \cdot\baromega_j^\hprime  < 0, 
    \label{eq:S2-dot lbdaW-Ic}
\end{align}
since $\lambda_j = 0$ and we can also combine \eqref{eq:S2-dot lbdaW-3} and \eqref{eq:S2-dot lbdaW-4} into \eqref{eq:S2-dot lbdaW-1} with error $\tilde\zeta_1$ since $h\in\cB$.
Therefore, these $\omega_j^\h$ monotonically decrease. 
For $\color{violet}h\in\cB_c$ and $\color{violet}k\in[d_y]$, we can upper bound $\baromega_k^\h$ as
\begin{align}
    \partial_t \log\baromega_k^\h
    &\le 2 \ind(k\in\cI) \sqrt{d_e}^{-1}\lambda_k\sqrt{L\phi_k^{-1}} \bar\varrho + \frac{2Ce^{1.1} \omega_0 |\cI| \bar E \sqrt{L} \bar\varrho }{L \sqrt{d_e}} \cdot \baromegaih \nend
    &\qquad + \left((4 + 2|\cI|\bar\varrho)\bar E \bar \varrho L + \sigma^2 d  H |\cI_c|\bar\mu^2\right)  \cdot \sqrt{d_e}^{-1} C\zeta  \cdot \baromegaih,
    \label{eq:S2-dot lbdaW-Bc}
\end{align}
where we define $\bar\mu\deleq \max_{h\in[H], j\in\cI_c}\mujh(0)$.
Here, the first term comes from upper bounding $\sqrt{d_e}^{-1} \lambda_i \muih \baromegaih$ by the upper bound for $\muih$ in \Cref{cond:bounded weights}, the second and the third term come from the upper bounds in \eqref{eq:S2-dot lbdaW-3-ub} and \eqref{eq:S2-dot lbdaW-4-ub} respectively.

\paragraph{Simplification of $\partial_t \muih$.}
We copy the dynamics for $\muih$ as
\begin{align}
    \partial_t \muih &= - \lambda_i d_i \phi_i (1 \pm \zeta)\cdot \sum_{h'=1}^H \frac{\exp\bigl(\langle \omega^\h,  \omega^\hprime\rangle\bigr)}{L} \cdot \muihprime \muih \nend
    &\qquad - \lambda_i d_i (1 \pm \xi)\cdot \sum_{h'=1}^H \baromegaihprime \baromegaih\muihprime \muih  + \lambda_i d_i (1 \pm \xi) \cdot \baromegaih \muih. \notag
\end{align}
For the first term, we have for $\color{violet}i\in\cI$ that
\begin{align*}
    & \frac{\lambda_i d_i \phi_i\cdot \sum_{h'=1}^H  \exp\bigl(\langle \omega^\h,  \omega^\hprime\rangle\bigr) L^{-1} \cdot  \muihprime \muih}{\lambda_i d_i \phi_i\cdot \exp\bigl(\langle \omega^\h,  \omega^{(h_i^*)}\rangle\bigr) L^{-1} \cdot  \muistar \muih} \le 1 + \frac{e^{1.1} \sum_{h'\neq h_i^*} \muihprime}{ \muistar} \le 1 + e^{1.1} H \bar\varrho = 1 + o(1).
\end{align*}
As a result, for $\log\muih$ with $\color{violet}i\in\cI$, it holds that
\begin{align}
    \partial_t \log\muih &= \lambda_i d_i \biggl(-(1 \pm (2\zeta + e^{1.1}H\bar\varrho))  \frac{\exp\bigl(\langle \omega^\h,  \omega^{(h_i^*)}\rangle\bigr)}{L} \phi_i\muistar + \left(1-\baromegaistar\muistar \pm \tilde\xi_1 \right) \baromegaih \biggr).
    \label{eq:S2-dot lbdaU-nonoptimal}
\end{align}
In particular, for $h=h_i^*$, we have
\begin{align*}
    \partial_t \log\muistar 
    &= \lambda_i d_i \biggl(-(1 \pm (2\zeta + e^{1.1}H\bar\varrho))  \frac{\exp\bigl(\langle \omega^{(h_i^*)},  \omega^{(h_i^*)}\rangle\bigr)}{L} \phi_i\muistar + \left(1-\baromegaistar\muistar \pm \tilde\xi_1 \right) \baromegaistar \biggr) \nend 
    & = \lambda_i d_i \biggl(- \frac{\exp\bigl(d_i(\baromegaistar)^2\bigr)}{L} \phi_i\muistar + \Bigl(1-\baromegaistar\muistar \pm \tau_2 \Bigr) \baromegaistar \biggr), 
\end{align*}
where we define $\color{teal}\tau_2 \triangleq \tilde\xi_1 + 2e^{1.1}(2\omega_0^2 d + 2\zeta + e^{1.1}H\bar\varrho) = 5 \xi + 2H\bar\varrho + 2e^{1.1}(2\omega_0^2 d + 2\zeta + e^{1.1}H\bar\varrho) \ll 1$ and the scale of $\tau_2$ comes from the upper bound $\muistar/\baromegaistar \le 2 L \phi_i^{-1}$ in \Cref{cond:S2-lbdaU*/lbdaW* ratio lb&ub} together with the upper bound $\langle \omega^{(h_i^*)},  \omega^{(h_i^*)}\rangle - d_i (\baromegaistar)^2 \le 2\omega_0^2 d$ by \Cref{cond:S2-nonopt-lbdaW-ub}.
For $\color{violet} j\in \cI_c$, we simply have 
\begin{align*}
    \partial_t \log\mujh 
    & = - \sigma^2 d  (1 \pm \zeta)\cdot \sum_{h'=1}^H \frac{\exp\bigl(\langle \omega^\h,  \omega^\hprime\rangle\bigr)}{L} \cdot \muihprime
    \le - \sigma^2 d  (1 - \zeta) L^{-1} \sum_{h'=1}^H \mu_j^{\hprime}.
\end{align*}
Another gradient we keep track of is $\log\muistar - \log\muih$ with $\color{violet}i\in\cI$ and $\color{violet}h\neq h_i^*$, which satisfies
\begin{align*}
    \partial_t\log\left(\frac{\muistar}{\muih}\right)
    &= \lambda_id_i \left(
        \bigl(1- \baromegaistar\muistar\bigr)\cdot\bigl(\baromegaistar - \baromegaih\bigr)\pm\tilde\xi_1\bigl(1 + \baromegaistar\muistar\bigr)\cdot\bigl(\baromegaistar + \baromegaih\bigr) \right) \nend
        &\autoquad{2}  + \lambda_id_i L^{-1} \phi_i \muistar \left(\exp\bigl(\langle \omega^\h,  \omega^{(h_i^*)}\rangle\bigr)- \exp\bigl(\langle \omega^{(h_i^*)},  \omega^{(h_i^*)}\rangle\bigr)\right)
        \nend 
        &\autoquad{2} \pm (2\zeta + e^{1.1}H\bar\varrho) \lambda_id_i L^{-1} \phi_i \muistar \left(\exp\bigl(\langle \omega^\h,  \omega^{(h_i^*)}\rangle\bigr) + \exp\bigl(\langle \omega^{(h_i^*)},  \omega^{(h_i^*)}\rangle\bigr)\right).
\end{align*}
Here, for the second terms, we follow a similar argument as in the warm-up stage and obtain 
\begin{align*}
    \lambda_id_i \tilde\xi_1\bigl(1 + \baromegaistar\muistar\bigr)\cdot\bigl(\baromegaistar + \baromegaih\bigr) 
    &\le 2 \lambda_id_i \tilde\xi_1 \cdot\bigl(\baromegaistar - \baromegaih\bigr) \cdot \frac{\baromegaistar}{\baromegaistar - \baromegaih} \nend 
    &\le 4 \lambda_id_i \tilde\xi_1 \cdot\bigl(\baromegaistar - \baromegaih\bigr) \cdot \frac{\max_{k\in\cI} \lambda_k\phi_k^{-1}}{\lambda_i\phi_i^{-1}}, 
\end{align*}
where in the last inequality we invoke \Cref{cond:S2-dominate} for the upper bound of $\baromegaih/\baromegaistar$ by its corresponding value at the end of the warm-up stage.
For the third term, we observe that 
\begin{align*}
    &\lambda_id_i L^{-1} \phi_i \muistar \left(\exp\bigl(\langle \omega^\h,  \omega^{(h_i^*)}\rangle\bigr)- \exp\bigl(\langle \omega^{(h_i^*)},  \omega^{(h_i^*)}\rangle\bigr)\right) \nend 
    &\quad \ge \lambda_id_i L^{-1} \phi_i \muistar \left(1 - \exp\bigl(d_i(\baromegaistar)^2 + 2\omega_0^2d\bigr)\right) \nend 
    &\quad \ge - e^{1.1} \lambda_id_i L^{-1} \phi_i \muistar \left(d_i(\baromegaistar)^2 + 2\omega_0^2d\right). 
\end{align*}
Combining the third and the fourth term, we have
\begin{align*}
    &\lambda_id_i L^{-1} \phi_i \muistar \left(\exp\bigl(\langle \omega^\h,  \omega^{(h_i^*)}\rangle\bigr)- \exp\bigl(\langle \omega^{(h_i^*)},  \omega^{(h_i^*)}\rangle\bigr)\right)
        \nend 
    &\autoquad{2} \pm (2\zeta + e^{1.1}H\bar\varrho) \lambda_id_i L^{-1} \phi_i \muistar \left(\exp\bigl(\langle \omega^\h,  \omega^{(h_i^*)}\rangle\bigr) + \exp\bigl(\langle \omega^{(h_i^*)},  \omega^{(h_i^*)}\rangle\bigr)\right) \nend 
    &\quad \ge -\lambda_id_iL^{-1}\phi_i\muistar\left(d_i(\baromegaistar)^2 + 2\omega_0^2d + 4\zeta + 2e^{1.1}H\bar\varrho\right) \nend 
    &\quad \ge -\lambda_id_iL^{-1}\phi_i\left(d_i(\baromegaistar)^2 + 2\omega_0^2d + 4\zeta + 2e^{1.1}H\bar\varrho\right) \cdot \frac{\muistar}{\baromegaistar} \cdot \frac{\baromegaistar}{\baromegaistar - \baromegaih} \cdot \left(\baromegaistar - \baromegaih\right).
\end{align*}
Using the upper bound $\muistar/\baromegaistar\le 2L\phi_i^{-1}$ in \Cref{cond:S2-lbdaU* lbdaW^* ub} and the upper bound for $\baromegaih/\baromegaistar$ by \Cref{cond:S2-dominate}, we have for $\log\muistar - \log\muih$ with $\color{violet}i\in\cI$ and $\color{violet}h\neq h_i^*$ that 
\begin{align}
    \partial_t\log\left(\frac{\muistar}{\muih}\right) 
    &\ge \lambda_id_i \left(- 4 \left(\tilde\xi_1 + d_i(\baromegaistar)^2 + 2\omega_0^2d + 4\zeta + 2e^{1.1}H\bar\varrho\right) \frac{\max_{k\in\cI} \lambda_k\phi_k^{-1}}{\lambda_i\phi_i^{-1}} \right.\nend 
    &\autoquad{5} \left. +1- \baromegaistar\muistar \right)\cdot\bigl(\baromegaistar - \baromegaih\bigr) \nend 
    &\ge \lambda_id_i \left(1- \baromegaistar\muistar- 4 \left(\tau_2 + d_i(\baromegaistar)^2\right) \varPhi_i^{-1} \right)\cdot\bigl(\baromegaistar - \baromegaih\bigr), 
    \label{eq:S2-dot lbdaU^*/lbdaU^h lb}
\end{align}
where we define $\varPhi_i = {\lambda_i\phi_i^{-1}}/({\max_{k\in\cI} \lambda_k\phi_k^{-1}})$.

\subsubsection{Coupled Growth of $\muistar$ and $\baromegaistar$.}
We aim to show that $\muistar$ is roughly a function of $\baromegaistar$ during the growing stage.
To this end, we define the following quantity
\begin{align}
    \pi_i \triangleq \muistar \cdot \left(
        \baromegaistar + \frac{\phi_i\exp\bigl(d_i(\baromegaistar)^2\bigr)}{L\baromegaistar}
    \right).
    \label{eq:S2-def pi}
\end{align}
We always have nonnegative $\pi_i$.
In the following, we only consider $\pi_i \in [1/2, 2]$. 
We characterize the dynamics for $\log\pi_i$ as
\begin{align*}
    \partial_t \log\pi_i = \partial_t \log \muistar + \frac{\partial}{\partial \baromegaistar} \left(
        \baromegaistar + \frac{\phi_i\exp\bigl(d_i(\baromegaistar)^2\bigr)}{L\baromegaistar}
    \right) \cdot \partial_t \baromegaistar.
\end{align*}
For the second part in the derivative, we have
\begin{align*}
    &\frac{\partial}{\partial \baromegaistar} \left(
        \baromegaistar + \frac{\phi_i\exp\bigl(d_i(\baromegaistar)^2\bigr)}{L\baromegaistar}
    \right) \nend
    &\quad = \frac{\muistar}{\pi_i} \cdot \left(
        1 + \frac{\phi_i\exp\bigl(d_i(\baromegaistar)^2\bigr)}{L} \left(- \frac{1}{(\baromegaistar)^2} + 2 d_i\right)
    \right) 
\end{align*}
Let us look at the scale of the following expression. 
\begin{align*}
    &\frac{(\muistar)^2}{d_i}\left|1+ \frac{\phi_i\exp\bigl(d_i(\baromegaistar)^2\bigr)}{L} \left(- \frac{1}{(\baromegaistar)^2} + 2 d_i\right)\right| \nend
    &\quad \le \frac{(\muistar)^2}{d_i} + \frac{\phi_i\exp\bigl(d_i(\baromegaistar)^2\bigr)}{L d_i} \left( \frac{(\muistar)^2}{(\baromegaistar)^2} + 2d_i(\muistar)^2\right) 
    \le 4e^{1.1}\left(\frac{ L}{\phi_i d_i} + 1 \right), 
\end{align*}
where we invoke the upper bound $\muistar/\baromegaistar \le \sqrt 2 L \phi_i^{-1}$ by \Cref{cond:S2-lbdaU*/lbdaW* ratio lb&ub} and the upper bound $
(\muistar)^2 \le 2 L \phi_i^{-1}$ by \Cref{cond:bounded weights}.
We can also rewrite $\partial_t \log \muistar$ and $\partial_t \baromegaistar$ in terms of $\pi_i$ as 
\begin{equation}
    \label{eq:S2-dominate dynamics-with pi}
\begin{gathered}
    \partial_t \log\muistar = \lambda_id_i\baromegaistar \left(1-\pi_i \pm \tau_2\right),\\
    \partial_t \baromegaistar = \sqrt{d_e}^{-1} \lambda_i\baromegaistar\muistar \left(1 - \pi_i + \frac{\exp\bigl(d_i(\baromegaistar)^2\bigr)\phi_i}{L} \cdot\frac{\muistar}{\baromegaistar}\cdot\left(1 - d_i(\baromegaistar)^2 \right)\pm \tau_1\right).
\end{gathered}
\end{equation}
Here, for $\partial_t \baromegaistar$ we also have the following upper bound
\begin{align*}
    \left|\partial_t \baromegaistar\right| \le \frac{\lambda_i\baromegaistar\muistar}{\sqrt{d_e} } \cdot \left(|1 -\pi_i| +  \sqrt 2  e^{1.1} + \tau_1\right),  
\end{align*}
where we invoke the upper bound $\muistar/\baromegaistar \le \sqrt 2 L \phi_i^{-1}$ by \Cref{cond:S2-lbdaU*/lbdaW* ratio lb&ub}.
Therefore, we have for $\partial_t \log\pi_i$ that
\begin{align*}
    \partial_t \log\pi_i 
    &= \lambda_id_i\baromegaistar \left(1-\pi_i \pm \tau_2 \pm \frac{1}{\sqrt{d_e} \pi_i} \cdot \left(|1 -\pi_i| +  \sqrt 2  e^{1.1} + \tau_1\right) \cdot 4e^{1.1}\left(\frac{ L}{\phi_i d_i} + 1 \right)\right)\nend
    &= \lambda_id_i\baromegaistar \Bigl(1-\pi_i \pm \underbrace{(\tau_2 + 120  \varsigma_i + 120\sqrt{d_e}^{-1} )}_{\ds \le \iota_i}\Bigr).
\end{align*}
Define $\color{brown} \iota_i \deleq \max\{\tau_2 + 120  \varsigma_i + 120\sqrt{d_e}^{-1}, \alpha_i + 17d_i\omega_0^2 + 32 L\omega_0^2 \phi_i^{-1} \}$, we automatically have from $1-\alpha_i \le \rho_i^*\phi_i \le 1+ \tilde\xi + 2\tilde\zeta$ that at the start of the growing phase, $\pi$ satisfies the following upper bound 
\begin{align*}
    |\pi_i(T_\warmup)-1| &= \left|\left. L \rho_i^* \cdot \left(
        (\baromegaistar)^2 + \frac{\phi_i\exp\bigl(d_i(\baromegaistar)^2\bigr)}{L}
    \right)  - 1 \right|_{T_\warmup}\right| \nend
    &\le (1 + \alpha_i) \cdot (1 + d_i16\omega_0^2) - 1 + 32L \omega_0^2 \phi_i^{-1} \nend 
    &\le \alpha_i + 17d_i\omega_0^2 + 32 L\omega_0^2 \phi_i^{-1}.
\end{align*}
As a result, during the growing stage, we always have $\pi_i\in[1- \iota_i, 1+ \iota_i]\subset [0.5, 1.5]$.

\subsubsection{Duration of the Growing Stage.}
Under the definition $\pi_i = \muistar \cdot \left(
    \baromegaistar + \frac{\phi_i\exp\bigl(d_i(\baromegaistar)^2\bigr)}{L\baromegaistar}
\right)$ and condition $\pi_i\in[1- \iota_i, 1+ \iota_i]$, we have for $\partial_t \baromegaistar$ following \Cref{eq:S2-dominate dynamics-with pi} that
\begin{align}
    \partial_t \baromegaistar 
    &= \left(1-d_i(\baromegaistar)^2 + (1-\pi_i)\left(\frac{L(\baromegaistar)^2}{\phi_i\exp\bigl(d_i(\baromegaistar)^2\bigr)} + d_i(\baromegaistar)^2\right)\right) \pi_i \sqrt{d_e}^{-1}\lambda_i \nend
    &\qquad \cdot
    {
        \left(1 + \frac{L(\baromegaistar)^2}{\phi_i\exp\bigl(d_i(\baromegaistar)^2\bigr)}\right)^{-1} \cdot
        \left(1 + \frac{\phi_i\exp\bigl(d_i(\baromegaistar)^2\bigr)}{L(\baromegaistar)^2}\right)^{-1}} \nend 
    &=\left(1-d_i(\baromegaistar)^2 \pm 2 \iota_i\left(Ld_i^{-1}\phi_i^{-1} + 1\right)\right) \cdot (1\pm \iota_i) \sqrt{d_e}^{-1}\lambda_i \nend 
    &\qquad \cdot \left(2 
    + \frac{\phi_i\exp\bigl(d_i(\baromegaistar)^2\bigr)}{L(\baromegaistar)^2}
    + \frac{L(\baromegaistar)^2}{\phi_i\exp\bigl(d_i(\baromegaistar)^2\bigr)}\right)^{-1}, \label{eq:S2-dot lbdaW*-with pi}
\end{align}
where in the second equality we also invoke the upper bound  $(\baromegaistar)^2 \le 2 d_i^{-1}$ by \Cref{cond:bounded weights} to simplify the numerator.
Define $\brown\tilde \iota_i = 2 \iota_i(Ld_i^{-1}\phi_i^{-1}+1) \ll 1$ and 
consider the regime where $d_i (\baromegaistar)^2 \le 1 - 2\tilde \iota_i$.
Equivalently, we have
\begin{align}
    \frac{\rd t}{\rd (\baromegaistar)} 
    &= \frac{2 
    + \frac{\phi_i\exp\bigl(d_i(\baromegaistar)^2\bigr)}{L(\baromegaistar)^2}
    + \frac{L(\baromegaistar)^2}{\phi_i\exp\bigl(d_i(\baromegaistar)^2\bigr)}}{\left(1-d_i(\baromegaistar)^2 \pm \tilde \iota_i\right) \cdot (1\pm \iota_i) \sqrt{d_e}^{-1}\lambda_i}. 
    \label{eq:dt/domega}
\end{align}
We first pick $\upsilon < 1 - 2\tilde \iota_i$ such that $1-\upsilon \ge 3\tilde \iota_i$ and consider the regime where $d_i (\baromegaistar)^2 \le \upsilon$.
In this regime, we can upper bound the above expression by
\begin{align*}
    \frac{\rd t}{\rd (\baromegaistar)} 
    &\le (1 + 2\tilde \iota_i)\cdot \frac{2 
    + \frac{\phi_i e^{\upsilon}}{L(\baromegaistar)^2}
    + \frac{L(\baromegaistar)^2}{\phi_i}}{\left(1- \upsilon\right) \sqrt{d_e}^{-1}\lambda_i} = \frac{(1 + 2\tilde \iota_i)\cdot\sqrt{d_e}}{\left(1- \upsilon\right) \lambda_i} \cdot\frac{\rd}{\rd (\baromegaistar)} \left( 2 (\baromegaistar)
    - \frac{\phi_i e^{\upsilon}}{L(\baromegaistar)}
    + \frac{L(\baromegaistar)^3}{3\phi_i}\right). 
\end{align*}
Therefore, the duration of this regime $\Delta t_{6+}^i$ where $d_i(\baromega_i^\star)^2$ is raised from the value at the end of the warm-up stage to $\upsilon$ is upper bounded by  
\begin{align}
    \Delta t_{6}^i 
    &\le \left. \frac{(1 + 2\tilde \iota_i)\cdot\sqrt{d_e}}{\left(1- \upsilon\right) \lambda_i} \left(2 (\baromegaistar)
    - \frac{\phi_i e^{\upsilon}}{L(\baromegaistar)}
    + \frac{L(\baromegaistar)^3}{3\phi_i}\right) \right|_{\omega_0}^{\sqrt{\upsilon/d_i}} \nend 
    &\le \frac{(1 + 2\tilde \iota_i)\cdot\sqrt{d_e}}{\left(1- \upsilon\right) \lambda_i} \left(2 \sqrt{\upsilon d_i}
    + \frac{\phi_i e^{\upsilon}}{L\omega_0}
    + \frac{L\upsilon^{3/2} }{3\phi_i d_i^{3/2}}\right).
    \label{eq:Delta t_6-ub}
\end{align}
One should notice that the second term in the above expression dominates as $\omega_0$ is small in the sense that $\color{teal}\omega_0 \ll (\phi_i L^{-1} \sqrt{d_i}^{-1}) \pmin (\phi_i^2 L^{-2} d_i^{3/2})$. 
For the other regime where $\upsilon \le d_i (\baromegaistar)^2 \le 1- 2\tilde \iota_i$, we have
\begin{align}
    \frac{\rd t}{\rd (\baromegaistar)} 
    &\le \underbrace{{\left(2 
    + \max_{x\in\{\sqrt{\upsilon/d_i}, \sqrt{1/d_i}\}}
    \left\{\frac{\phi_i\exp\bigl(d_ix^2\bigr)}{Lx^2}
    + \frac{Lx^2}{\phi_i\exp\bigl(d_ix^2\bigr)}\right\}
    \right)} }_{\ds \deleq \varGamma_i} \cdot \frac{\sqrt{d_e}\lambda_i^{-1}}{(1- \tilde \iota_i -d_i(\baromegaistar)^2 ) (1- \iota_i)}.
    \label{eq:dt/domega-ub}
\end{align}
Here, a typical choice of $\upsilon$ is $1/2$ for the definition of $\varGamma_i$.
By integrating the above expression for $\baromegaistar$ between the above mentioned range, we have
\begin{align*}
    \Delta t_7^i 
    &\le  \frac{\varGamma_i\sqrt{d_e}\lambda_i^{-1}}{\sqrt{d_i (1 - \tilde \iota_i)^3}} \cdot \left(
        \tanh^{-1}\left(
            \sqrt{\frac{1-2\tilde \iota_i}{1-\tilde \iota_i}}
        \right) - \tanh^{-1}\left(
            \sqrt{\frac{\upsilon}{1-\tilde \iota_i}}
        \right)
    \right)\le \frac{\varGamma_i\sqrt{d_e}\lambda_i^{-1}}{\sqrt{d_i (1 - \tilde \iota_i)^3}} \cdot \frac{1}{2} \log\frac{4}{\tilde \iota_i}, 
\end{align*}
where the last inequality follows from upper bounding the positive term in the inverse hyperbolic tangent function.

\subsubsection{Verifying the Conditions for the Growing Stage.}
\paragraph{Verification for \Cref{cond:S2-growth-stage2}.} 
We first verify that $\muistar \ge  L \phi_i^{-1}  \omega_0 / \sqrt 2$ in \Cref{cond:S2-growth-stage2} for $i\in\cI$ during the growing stage.
Using the fact that $\pi_i \in [1- \iota_i, 1+ \iota_i]$, we have
\begin{align*}
    \muistar \ge \left(
        (\baromegaistar) + \frac{\phi_i\exp\bigl(d_i(\baromegaistar)^2\bigr)}{L(\baromegaistar)}
    \right)^{-1} \cdot (1- \iota_i).
\end{align*}
It remains to check that $(\baromegaistar) \ge \omega_0$. Invoking \eqref{eq:S2-dot lbdaW*-with pi}, we can easily observe that for $(\baromegaistar) < 1/\sqrt{2d_i}$, we have $\partial_t (\baromegaistar) > 0$.
Since at the end of the warm-up stage, we already have $(\baromegaistar) \ge \omega_0$, we can conclude that $(\baromegaistar) \ge \omega_0$ for all $i\in\cI$ during the growing stage. 

\paragraph{Verification for \Cref{cond:S2-dominate}.}
We next check the conditions in \Cref{cond:S2-dominate}. 
We first show that $(\baromegaistar) \ge \baromegaih$ and $(\baromegaistar)/\baromegaih \ge (\baromegaistar)/\baromegaih\biggiven_{T_\warmup}$ hold for all $i\in\cI$ and $h\in\cB\backslash\{h_i^*\}$.
Invoking \eqref{eq:S2-dot lbdaW-nonoptimal} and \eqref{eq:S2-dot lbdaW-optimal}, we have for $\partial_t (\log(\baromegaistar) -  \log\baromegaih)$ that
\begin{align*}
    \partial_t \left(
        \log\frac{(\baromegaistar)}{\baromegaih}
    \right)
    & \ge \left(1 - \left(1 + \frac{\exp\bigl(d_i(\baromegaistar)^2\bigr)}{L} d_i\phi_i \right) \muistar  (\baromegaistar)\pm \tau_1\right) \sqrt{d_e}^{-1} \lambda_i\muistar \nend 
    &\qquad - \left(- \muistar \baromegaistar 
    + (1\pm \tilde\xi_1)\right)\sqrt{d_e}^{-1} \lambda_i \muih.
\end{align*}
We split the discussion into two cases based on whether $(\baromegaistar) \muistar < 1/(1 + 2e^{1.1}d_i\phi_iL^{-1})$ or not.
If the condition holds, we have following the upper bound $(\baromegaistar) \le (1+o(1))\sqrt{d_i^{-1}}$ by \Cref{cond:bounded weights} that
\begin{align*}
    \partial_t \left(
        \log\frac{(\baromegaistar)}{\baromegaih}
    \right) \ge \left(\frac{e^{1.1}d_i\phi_iL^{-1}}{1+2e^{1.1} d_i\phi_iL^{-1}} - \tau_1\right) \sqrt{d_e}^{-1} \lambda_i\muistar - (1 + \tilde\xi_1) \sqrt{d_e}^{-1} \lambda_i \muih.
\end{align*}
At this point, we invoke the ratio argument $\muistar/\muih \ge \muistar/\muih\biggiven_{T_\warmup} \ge \bar\varrho^{-1}$ for any $i\in\cI$ in \Cref{cond:S2-dominate} that
\begin{align*}
    \partial_t \left(
        \log\frac{(\baromegaistar)}{\baromegaih}
    \right) \ge \left(\frac{e^{1.1}d_i\phi_iL^{-1}}{1+2e^{1.1} d_i\phi_iL^{-1}} - \tau_1 - 2\bar\varrho\right) \sqrt{d_e}^{-1} \lambda_i\muistar >0, 
\end{align*}
where we invoke the relationship $\color{teal} \min\{e^{1.1}d_i\phi_iL^{-1},1\} \gg \tau_1 \pmax \bar\varrho$. 
Therefore, we have $(\baromegaistar)/\baromegaih$ growing throughout the first case
(Also, note that $(\baromegaistar)$ is also growing under this condition).
Hence, both conditions are satisfied for the first case.
For the other case, since we already have $(\baromegaistar) \muistar \ge 1/(1 + 2e^{1.1}d_i\phi_iL^{-1})$, by incorporating the relationship between $(\baromegaistar)$ and $\muistar$ in \eqref{eq:S2-def pi} we conclude that $(\baromegaistar)\ge \sqrt{1/2d_i}$.
Therefore, we just need to check that $\baromegaih$ is still somewhere near the initialization. 
To do so, we incorporate the upper bound for the duration of the growing stage 
\begin{align*}
    \Delta t_6^i + \Delta t_7^i \le \max_{k\in\cI}\left(2 \frac{(1 + 2\tilde\iota_k)\cdot\sqrt{d_e}}{\left(1- \upsilon\right) \lambda_k} \cdot \frac{\phi_ke^{\upsilon}}{L\omega_0} + \frac{\varGamma_k\sqrt{d_e}\lambda_k^{-1}}{\sqrt{d_k (1 - \tilde\iota_k)^3}} \cdot \frac{1}{2} \log\frac{4}{\tilde\iota_k}\right), 
\end{align*}
where $\upsilon \in(0, 1)$ is a absolute constant and $\tilde \iota_i \ll 1$ is a chosen small constant. 
We can now upper bound the maximal value of $\baromegaih$ for $i\in\cI$ and $h\neq h_i^*$ during the growing stage as
\begin{align}
    \log\frac{\baromegaih}{\baromegaih(T_\warmup)} 
    &\le \bar\varrho \lambda_i \sqrt{2L\phi_i^{-1}}\left(- \muistar \baromegaistar 
    + (1\pm {{\tilde\xi_1}})\right)\nend 
    &\qquad \cdot \max_{k\in\cI}\left(2 \frac{(1 + 2\tilde\iota_k) }{\left(1- \upsilon\right) \lambda_k} \cdot \frac{\phi_ke^{\upsilon}}{L\omega_0} + \frac{\varGamma_k \lambda_k^{-1}}{\sqrt{d_k (1 - \tilde\iota_k)^3}} \cdot \frac{1}{2} \log\frac{4}{\tilde\iota_k}\right) = o(1).
    \label{eq:s2-lbdaW-I-h-drift}
\end{align}
As we have $\color{teal} \omega_0 \gg \bar\varrho \poly(L)$ thanks to the exponential decay in $\bar \varrho$, the above expression is $o(1)$ for the scale of $\tilde \iota_i$ that $\tilde \iota_i \ge \varsigma_i$. 
Therefore, for the second case, we still have $(\baromegaistar) > \baromegaih$ and $(\baromegaistar)/\baromegaih \ge (\baromegaistar)/\baromegaih\biggiven_{T_\warmup}$ for all $i\in\cI$ and $h\neq h_i^*$ during the growing stage. 

For the case $h\in\cB_c$, we have the upper bound for $\partial_t \log \baromegaih$ in \eqref{eq:S2-dot lbdaW-Bc}.
Note that for $(\baromegaistar)\muistar < 1/(1 + 2e^{1.1}d_i\phi_iL^{-1})$, we have
\[
    \partial_t \log(\baromegaistar) \ge \left(\frac{e^{1.1}d_i\phi_iL^{-1}}{1+2e^{1.1} d_i\phi_iL^{-1}} - \tau_1\right) \sqrt{d_e}^{-1} \lambda_i\muistar.
\]
For the terms in \eqref{eq:S2-dot lbdaW-Bc} with $\bar\varrho$, the above expression automatically dominates thanks to the exponentially small scale of $\bar\varrho$ and the fact that $\baromegaih \le \sqrt 2 \omega_0$ when \Cref{cond:S2-nonopt-lbdaW-ub} holds.
The only term we need to consider is $\sigma^2 d  H |\cI_c|\bar\mu^2  \cdot \sqrt{d_e}^{-1} C\zeta  \cdot \baromegaih \le \sigma^2 d  H |\cI_c|\bar\mu^2  \cdot \sqrt{d_e}^{-1} C\zeta  \cdot \sqrt 2\omega_0$. 
Therefore, we have for the time-derivative of $\log ((\baromegaistar) / \baromegaih)$ that
\begin{align*}
    \partial_t \log \frac{(\baromegaistar)}{\baromegaih} 
    &\ge \left(\frac{e^{1.1}d_i\phi_iL^{-1}}{1+2e^{1.1} d_i\phi_iL^{-1}} - \tau_1 - \bar\varrho \poly(L)\right) \frac{\lambda_i\muistar}{\sqrt{d_e}} - \frac{\sigma^2 d H |\cI_c|\bar\mu^2  \cdot  C\zeta  \cdot \sqrt 2\omega_0}{\sqrt{d_e}} \nend 
    &\ge \left(\frac{d_i\phi_i}{L}\cdot \left(\frac{e^{1.1}}{1+2e^{1.1} d_i\phi_iL^{-1}} 
    - \frac{2CH |\cI_c|\bar\mu^2 \zeta  }{\snr[i]}\right) - \tau_1 - \bar\varrho\poly(L) \right) \frac{\lambda_i\muistar}{\sqrt{d_e}}.
\end{align*}
By our initialization condition that $\color{brown} \bar\mu^2 \ll \snr[i] \cdot (1 \pmin L/(d_i\phi_i)) /(H |\cI_c| \zeta)$,  we have the above expression being positive. 
On the other hand, as $(\baromegaistar)\muistar \ge 1/(1 + 2e^{1.1}d_i\phi_iL^{-1})$, we have $(\baromegaistar) \ge \sqrt{1/2d_i}$.
Similar to the previous case for $h\in\cB$, we can upper bound the total drift of $\log \baromegaih$ for $h\in\cB_c$ as
\begin{align*}
    \log\frac{\baromegaih}{\baromegaih(T_\warmup)} 
    &\le (\sqrt 2\sigma^2 d H |\cI_c|\bar\mu^2  \cdot C\zeta   + \bar\varrho \poly(L)\omega_0^{-1})\nend 
    &\qquad \cdot \max_{k\in\cI}\left(2 \frac{(1 + 2\tilde\iota_k)}{\left(1- \upsilon\right) \lambda_k} \cdot \frac{\phi_ke^{\upsilon}}{L} + \frac{\varGamma_k \lambda_k^{-1}\omega_0}{\sqrt{d_k (1 - \tilde\iota_k)^3}} \cdot \frac{1}{2} \log\frac{4}{\tilde\iota_k}\right).
\end{align*}
Note that the above terms are $o(1)$ given the scale of $\zeta \bar\mu^2 d \ll 1$ and $\tilde \iota_i \ge \varsigma_i$. 

Next, we have the argument for $\muistar/\muih$ that the ratio is non-decreasing. Recall the lower bound for $\partial_t \log(\muistar/\muih)$ in \eqref{eq:S2-dot lbdaU^*/lbdaU^h lb}. 
Before $(\baromegaistar)$ reaches $10\omega_0$, we have obviously that the gradient is non-negative for $\log(\muistar/\muih)$.
On the other hand, after $(\baromegaistar)$ reaches $10\omega_0$, we have following from \eqref{eq:S2-dot lbdaU-nonoptimal} and the definition of $\pi_i$ that
\begin{align*}
    \partial_t \log \muih 
    &\le \lambda_i d_i \left(- (1 - (2\zeta + e^{1.1}H\bar\varrho))\phi_iL^{-1}\muistar + \left(1 - (\baromegaistar)\muistar + \tilde\xi_1\right) \baromegaih\right) \nend 
    &\le \lambda_id_i (\baromegaistar)\Biggl(
        (1 - (\baromegaistar)\muistar + \tilde\xi_1) \frac{\baromegaih}{(\baromegaistar)}  
        \nend
    &\qquad 
        - (1 - (2\zeta + e^{1.1}H\bar\varrho)) (\pi_i - \muistar (\baromegaistar)) \exp(-d_i(\baromegaistar)^2)  
    \Biggr) \nend 
    &\le - \lambda_id_i (\baromegaistar) (1 - (2\zeta + e^{1.1}H\bar\varrho)) \exp(-d_i(\baromegaistar)^2) \nend 
    &\qquad \cdot 
        (1 - \muistar(\baromegaistar)-2 \iota_i - 2\tilde\xi_1) \cdot \left( 1 - e^{1.1} (1 + 4\zeta + e^{1.1} H\bar\varrho) \frac{1+o(1)}{10}\right).
\end{align*}
We note that 
\begin{align*}
    1 - \muistar (\baromegaistar) 
    &= 1 -  \pi_i^{-1} \left(1 + \frac{\phi_i\exp\bigl(d_i(\baromegaistar)^2\bigr)}{L (\baromegaistar)^2}\right)^{-1} \nend 
    &\ge \frac{{\phi_i\exp\bigl(d_i(\baromegaistar)^2\bigr)}/{(L (\baromegaistar)^2)}}{1 +  {\phi_i\exp\bigl(d_i(\baromegaistar)^2\bigr)}/{(L (\baromegaistar)^2)}} - \iota_i \nend 
    &\ge \min\left\{\frac 1 2, \frac{\phi_i\exp\bigl(d_i(\baromegaistar)^2\bigr)}{2L (\baromegaistar)^2}\right\} - \iota_i \ge \min\left\{\frac 1 2, \frac{d_i\phi_ie}{2L }\right\} - \iota_i.
\end{align*}
Under the condition that $\color{teal} \min\left\{1, d_i\phi_i/L\right\} \gg 3 \iota_i + 2\tilde\xi_1$, we have that $\partial_t \log \muih \le 0$.
As $\muistar$ is no smaller than the corresponding value when $(\baromegaistar)=10\omega_0$, 
the ratio $\muistar/\muih$ is no smaller than $1/\bar\varrho$. 

\paragraph{Verification for \Cref{cond:S2-nonopt-lbdaW-ub}.}
We next check the conditions in \Cref{cond:S2-nonopt-lbdaW-ub}.
In the previous paragraph, we have already verified that $\baromegaih \le (1+o(1))\omega_0$ for all $i\in\cI$ and $h\neq h_i^*$ during the growing stage.
It remains to characterize the condition for other $j\in\cI_c$.
However, this is trivial as we have $\partial_t \log \omega_j^\h \le 0$ for all $j\in\cI_c$ due to \eqref{eq:S2-dot lbdaW-Ic}.

\paragraph{Verification for \Cref{cond:S2-lbdaU*/lbdaW* ratio lb&ub}.}
We just use the fact that $\pi_i\in [1- \iota_i, 1+ \iota_i]$ to conclude for the ratio $\muistar/(\baromegaistar)$ that
\begin{align*}
    \frac{\muistar}{(\baromegaistar)} = \frac{\pi_i}{(\baromegaistar)^2 + \phi_i\exp\bigl(d_i(\baromegaistar)^2\bigr) L^{-1}}.
\end{align*}
An upper bound is given by
\[
    \frac{\muistar}{(\baromegaistar)} \le \frac{1+ \iota_i}{\phi_i L^{-1}} \le \sqrt 2 L \phi_i^{-1}, 
\]
and an lower bound is given by 
\[
    \frac{\muistar}{(\baromegaistar)} \ge \frac{1- \iota_i}{2d_i + \phi_ie^{1.1} L^{-1}} \ge \frac{d_i \pmin Le^{-1}\phi_i^{-1}}{4}, 
\]
where we invoke the upper bound for $(\baromegaistar)$ in \Cref{cond:bounded weights} to simplify the denominator.

\paragraph{Verification for \Cref{cond:S2-lbdaU* lbdaW^* ub}.}
Following the definition of $\pi_i$ in \eqref{eq:S2-def pi}, we have
\begin{align*}
    \muistar (\baromegaistar) 
    &= \frac{\pi_i}{1 + \phi_i\exp\bigl(d_i(\baromegaistar)^2\bigr) L^{-1}/(\baromegaistar)^2} \le \frac{1+ \iota_i}{1 + \phi_i L^{-1} d_i / 2} \le \frac{2}{1+d_i\phi_iL^{-1}}, 
\end{align*}
where the last inequality holds as we have 
$\brown L\phi_i^{-1} d_i^{-1} \gg \iota_i$.

\subsubsection{Emergence and Convergence of the Growing Stage}\label[appendix]{sec:emergence_and_convergence}
Recall that we rescale the time as $t \leftarrow 2 d t$.
\paragraph{Convergence}
We first check if the dynamics converge or not after the growing stage. 
By \eqref{eq:S2-dot lbdaW*-with pi}, after $d_i \baromegaistar^2\ge 1 - 2\tilde\iota_i$, we still have for $d_i \baromegaistar^2 $ around $1$ that
\begin{align*}
    \partial_t \baromegaistar& = \left(1-d_i(\baromegaistar)^2 \pm 2 \iota_i\left(Ld_i^{-1}\phi_i^{-1} + 1\right)\right) \cdot (1\pm \iota_i) \sqrt{d_e}^{-1}\lambda_i \nend 
    &\qquad \cdot \left(2 
    + \frac{\phi_i\exp\bigl(d_i(\baromegaistar)^2\bigr)}{L(\baromegaistar)^2}
    + \frac{L(\baromegaistar)^2}{\phi_i\exp\bigl(d_i(\baromegaistar)^2\bigr)}\right)^{-1}. 
\end{align*}
Notably, if $d_i \baromegaistar^2$ grows larger than $1 + 2\tilde \iota_i$, the dynamics will \say{draw} the values back. 
Thus, we prove that $d_i \baromegaistar$ will be roughly $1 \pm 2\tilde \iota_i$, which shows the convergence for $\baromegaistar$. 
Since $\pi_i$ is \say{fixed} around $1\pm \iota_i$, we also have 
$\muistar$ \say{fixed} around 
\[  
    \frac{1 \pm \iota_i }{1 + \phi_i\exp\bigl(d_i(\baromegaistar)^2\bigr) L^{-1}/(\baromegaistar)^2} = \frac{1 \pm \iota_i}{1 + d_i \phi_i e L^{-1} (1\pm \tilde\iota_i)}. 
\]
For other nonoptimal head, we have following the same argument as \eqref{eq:s2-lbdaW-I-h-drift} that 
\begin{align*}
    \log\frac{\baromegaih(t)}{\baromegaih(T_\warmup)} 
    &\le \bar\varrho \lambda_i \sqrt{2L\phi_i^{-1}}\left(- \muistar \baromegaistar 
    + (1\pm {{\tilde\xi_1}})\right) \cdot T = o(1)
\end{align*}
thanks to the scale $\bar\varrho = \max_{i\in[I]} (H-1)^{-1}\gamma_i \exp(-\vartheta^2/(256 \varsigma_i))$ where $\varsigma_i= {L }/{(\sqrt{d_e}d_i \phi_i)} \approx 1/\sqrt{d_e}$ as long as $T \ll \exp(O(\sqrt{d_e}))$. 
This is a fairly long time given that $d_e \ge d$. 
Moreover, as we have discussed, for the nonoptimal head, $\mu_i^\h$ dies down as $\partial_t \log \mu_i^\h < 0$. 
Thus, we have shown the convergence. 

\paragraph{Convergence Rate.}
We integrate \eqref{eq:dt/domega-ub} for $\baromegaistar$ between $\sqrt{\upsilon /d_i}$ and $\sqrt{(1-\delta)/ d_i}$ where $\delta \in (2\tilde\iota_i, 1-\upsilon)$ can be arbitrarily picked, we have the following upper bound on the required $t$ to raise $\baromegaistar$ from $\sqrt{\upsilon /d_i}$ to $\sqrt{(1-\delta)/ d_i}$:
\begin{align*}
    t 
    &\le  \frac{\varGamma_i\sqrt{d_e}\lambda_i^{-1}}{\sqrt{d_i (1 - \tilde \iota_i)^3}} \cdot \left(
        \tanh^{-1}\left(
            \sqrt{\frac{1-\delta}{1-\tilde \iota_i}}
        \right) - \tanh^{-1}\left(
            \sqrt{\frac{\upsilon}{1-\tilde \iota_i}}
        \right)
    \right)\le \frac{\varGamma_i\sqrt{d_e}\lambda_i^{-1}}{\sqrt{d_i (1 - \tilde \iota_i)^3}} \cdot \frac{1}{2} \log\frac{4}{\delta}, 
\end{align*}
This clearly shows the convergence rate. 

\paragraph{Emergence.}
The next theorem characterizes the emergence behavior of the dynamics. 
\begin{theorem}[Emergence Behavior]
\label[theorem]{thm:emergence}
Define 
\[  
    T_0 = 2 d \cdot \frac{\phi_i\sqrt{d_e}}{\lambda_i L \omega_0}.
\]
Pick an absolute constant $\upsilon\in(0, 1/4)$.
Then, the following holds:
\begin{itemize}
    \item For the total time less than $(1 - \upsilon - o(1))T_0$, we must have $\baromegaistar \le \upsilon^{-1} \omega_0$. The loss for task $i$ is a least 
    \[  
        \Omega\left(\frac{\lambda_i d_i }{d} \right).
        \]
    \item For the total time larger than $(1+e\upsilon)T_0$, we must have $d_i(\baromega_i^\star)^2\ge \upsilon$. 
    The loss for task $i$ is at most 
    \[  
        O\left(\frac{\lambda_i d_i }{d} \cdot \left(1 - \frac{1}{1 + (d_i\baromegaistar^2)^{-1}\exp(d_i\baromegaistar^2) \phi_i d_i L^{-1}}\right) \right).
        \]
\end{itemize}
\end{theorem}
\begin{proof}{\emph(Proof of \Cref{thm:emergence})}
\textbf{In the following discussion, we also rescale the time by $t \leftarrow 2d t$.}
At the end of the warm-up phase, we just have by \eqref{eq:lbdaW* ub-warmup-stage5} in the fifth step of the warm-up phase that
\begin{align}
    \baromegaistar (T_\warmup) = \frac{1 \pm o(1)}{\omega_0^{-1} - \frac{  L \lambda_i }{\sqrt{d_e}\phi_i } (T_\warmup - t_4^i)} \le 4\omega_0,
    \label{eq:emergence 1}
\end{align}
given that $\baromegaistar = (1+o(1))\omega_0$ as we have shown for the end of the step 4 of the warm-up stage. 
Here, $t_4^i$ is the ending time of step 4 of the warm-up stage for task $i$. 
Here, the $4\omega_0$ is ensured by the definition of the warm-up stage.
Consider a fixed small constant $\upsilon\in(0, 1/4)$. 
Consider the regime when $\baromegaistar$ is raised from the value at the end of the warm-up stage to $\upsilon^{-1} \omega_0$.
Recall the derivative of time with respect to $\baromegaistar$ in \eqref{eq:dt/domega} that 
\begin{align*}
    \frac{\rd t}{\rd (\baromegaistar)} 
    &= \frac{2 
    + \frac{\phi_i\exp\bigl(d_i(\baromegaistar)^2\bigr)}{L(\baromegaistar)^2}
    + \frac{L(\baromegaistar)^2}{\phi_i\exp\bigl(d_i(\baromegaistar)^2\bigr)}}{\left(1-d_i(\baromegaistar)^2 \pm \tilde \iota_i\right) \cdot (1\pm \iota_i) \sqrt{d_e}^{-1}\lambda_i}
    \ge  \frac{2 
    + \frac{\phi_i}{L(\baromegaistar)^2}
    + \frac{L(\baromegaistar)^2}{\phi_i e}}{ \sqrt{d_e}^{-1}\lambda_i} \nend 
    &= \frac{\sqrt{d_e}}{\lambda_i} \cdot \frac{\rd}{\rd \baromegaistar} \left( 2\baromegaistar - \frac{\phi_i}{L \baromegaistar} + \frac{L \baromegaistar^3}{3\phi_i e}  \right). 
\end{align*}
It is clear that the time used for $\baromegaistar$ to be raised to $\upsilon^{-1} \omega_0$, which we define as $\Delta T_1$, is lower bounded by 
\begin{align*}
    \Delta T_1  \ge \left. \frac{ \sqrt{d_e}}{ \lambda_i} \left(2 (\baromegaistar)
    - \frac{\phi_i}{L(\baromegaistar)}
    + \frac{L(\baromegaistar)^3}{3\phi_i e}\right) \right|_{\baromegaistar (T_\warmup)}^{\upsilon^{-1} \omega_0} \ge \frac{\phi_i\sqrt{d_e}}{\lambda_i L} \left( 
        \baromegaistar (T_\warmup)^{-1} - \upsilon  \omega_0^{-1}
    \right). 
\end{align*}
Using the result in \eqref{eq:emergence 1}, we also have 
\begin{align}
    \baromegaistar (T_\warmup)^{-1} = \left(\omega_0^{-1} - \frac{  L \lambda_i }{\sqrt{d_e}\phi_i } (T_\warmup - t_4^i)\right) \cdot (1\pm o(1)).
    \label{eq:emergence 2}
\end{align}
Plugging \eqref{eq:emergence 2} into the lower bound for $\Delta T_1$, we conclude that
\begin{align*}
    \Delta T_1 \ge \frac{\phi_i\sqrt{d_e}}{\lambda_i L} \left( 
        \left(\omega_0^{-1} - \frac{  L \lambda_i }{\sqrt{d_e}\phi_i } (T_\warmup - t_4^i)\right) \cdot (1- o(1)) - \upsilon \omega_0^{-1}
    \right), 
\end{align*}
which shows that 
\begin{align*}
    \Delta T_1 +  (1- o(1)) (T_\warmup - t_4^i) \ge \frac{\phi_i\sqrt{d_e}}{\lambda_i L} \cdot (1-o(1)- \upsilon  )  \omega_0^{-1}. 
\end{align*}
Note that $t_4^i$ is much smaller than $T_\warmup$ as we have discussed previously.
Define 
\[  
    T_0 = \frac{\phi_i\sqrt{d_e}}{\lambda_i L \omega_0}.
\]
As a result, we conclude that for the total time 
less than $T_0 \cdot (1 - \upsilon - o(1))$, we must have $\baromegaistar \le \upsilon^{-1} \omega_0$. 
For the other side $T_0 (1 + \upsilon)$, we have by \eqref{eq:Delta t_6-ub} that the time to raise $d_i(\baromega_i^\star)^2$ from the value at the end of the warm-up stage to $\upsilon$, which we define as $T_2$, is upper bounded by 
\begin{align*}
    T_2 \le \frac{(1 + 2\tilde \iota_i)\cdot\sqrt{d_e}}{\left(1- \upsilon\right) \lambda_i} \left(2 \sqrt{\upsilon d_i}
    + \frac{\phi_i e^{\upsilon}}{L\omega_0}
    + \frac{L\upsilon^{3/2} }{3\phi_i d_i^{3/2}}\right). 
\end{align*}
Note that the second term dominates as we discussed before. 
Thus, $T_2$ is no larger than $(1+o(1))e^\upsilon/(1-\upsilon) T_0 \le (1 + e\upsilon) T_0$. 

For the loss before $(1 - \upsilon -o(1)) T_0$, we have by noting that both $\mu_i$ and $\baromega_i$ are small that the output for position $i$ is much smaller in scale compared to $(y_q)_0$. 
Hence, the loss is at least $\Omega\left(\frac{\lambda_i d_i }{d} \right)$.
Lastly, for the training loss at $T_2$, we notice that $\muistar$ is around the optimal value for the corresponding $\baromegaistar$ as we have 
\[
    \pi = \muistar \baromegaistar\cdot \left(
    1 + \frac{\phi_i d_i \exp\bigl(d_i(\baromegaistar)^2\bigr)}{L (d_i \baromegaistar^2)}
\right) = 1 \pm o(1), 
\]
compared to the optimality condition in \eqref{eq:uistar(b)} of \Cref{lem:loss perturbation} with $b_i$ replaced by $d_i \baromegaistar^2$. 
Therefore, given that the task is only carried out by its optimal head, the loss for each task is well approximated by
\[
    \frac{\lambda_i d_i }{d} \cdot \left(1 - \frac{1}{1 + (d_i\baromegaistar^2)^{-1}\exp(d_i\baromegaistar^2) \phi_i d_i L^{-1}}\right).
\]
Plugging in the value of $\baromegaistar$, we have the desired result.
\end{proof}

\subsection{Dynamics Path}\label[appendix]{sec:dynamics_path}
We give a simplified version of the dynamics path for ease of understanding what's happening during each phase. 
We remind the readers that the following is not rigorous proof.
The \ac{gf} dynamics of $\baromegah$
is approximated by
\begin{align*}
    \sqrt d_e \cdot \partial_t \baromegaih
    &= - \sum_{h'=1}^H 
        \fracexproductL[\omega^\h][\omega^\hprime]
        \cdot \bigl\langle \vecd\odot \lambda \odot \phi, \mu^\h \odot \mu^\hprime \bigr\rangle 
        \cdot\baromegaihprime \baromegaih
        \notag
        \\
    & \qquad - \lambda_i \cdot \sum_{h'=1}^H  \muihprime \muih\baromegaihprime \baromegaih
    + \lambda_i \cdot \muih \baromegaih,
\end{align*}
and the GF dynamics of $\muh$ is approximated by  
\begin{align*}
    \partial_t \muih 
    &=  - \sum_{h'=1}^H \fracexproductL[\omega^\h][\omega^\hprime] d_i  \lambda_i \phi_i \mu_i^\hprime \mu_i^\h 
    -  \sum_{h'=1}^H  d_i \lambda_i    \baromega_i^\h  \baromega_i^\hprime \mu_i^\hprime  \mu_i^\h
      + d_i \lambda_i  \baromega_i^\h \mu_i^\h. 
\end{align*}

\paragraph{Dynamics Path for the Warm-up Phase.}
Note that during the warmup phase $\baromegaih\approx \omega_0$.
Hence, any term of 2 or higher order in $\baromega$ is ignorable. 
Hence, we have only the signal term for $\partial_t \baromega^\h$:
\begin{align*}
    \sqrt d_e \cdot \partial_t \baromegaih
    &\approx \lambda_i \cdot \muih \baromegaih,
\end{align*}
and the last signal term and the first interference term for $\partial_t \mu^\h$, where it is more clear to write in the logarithm form:
\begin{align}
    &\partial_t \log \muih
    \approx  d_i  \lambda_i \baromega_i^\h  - \sum_{h'=1}^H L^{-1} d_i  \lambda_i \phi_i \mu_i^\hprime \cdot \left(1 + \langle \omega^\h, \omega^\hprime \rangle \right).
    \label{eq:mu-warmup-approx}
\end{align}

\paragraph{Mean-Field Interference Makes the Optimal Head Dominate:}  To understand how the optimal head gradually dominates the other heads, we consider the time-derivative of the ratios: 
\begin{align*}
    \partial_t \left(\log \frac{\muistar}{\muih}\right) \approx \lambda_i d_i \left(\baromegaistar - \baromegaih\right),  \quad 
    \partial_t \left(\log \frac{\baromegaistar}{\baromegaih}\right) \approx \lambda_i (\muistar - \muih)
\end{align*}
where we actually cancel out the interference term since the interference term is almost the same for all heads as long as $\baromega^\h$ is small enough!
This is the underlying reason why the optimal head can gradually dominates the other heads using the advantage of initialization. 
In particular, the advantage in $\muistar$ will be converted to the advantage in $\baromegaistar$ through the second equation and then back to the further advantage in $\muistar$.

For more insight, note that the dynamics of $\mu^\h$ runs much faster than $\baromega^\h$ in scale. 
Therefore, the first step should be $\mu^\h$ trying to reaching the equilibrium given by \eqref{eq:mu-warmup-approx}:
\begin{align}
    \sum_{h'=1}^H \phi_i \mu_i^\hprime \approx  L   \baromega_i^\h.
    \label{eq:mu-equilibrium}
\end{align}
However, the left-hand side of \eqref{eq:mu-equilibrium} is independent of $h$ but the right-hand side is not.
Notably, $\sum_{h\in[H]} \phi_i \muih$ should be larger than the smallest $L \baromegaih$ while it should be smaller than the largest $L \baromegaih$, since otherwise we can always find all $\{\muih\}_{h\in[H]}$ to simultaneously increase or decrease according to \eqref{eq:mu-warmup-approx}.
As a result, we will have a clear separation between heads: For heads with $L \baromegaih$ larger than the interference level $\sum_{h'\in[H]} \phi_i \mu_i^\hprime$, the corresponding $\muih$ will increase, and vice versa.
Only the optimal head will keep going all the time.  

\paragraph{Dynamics Path for the Emergence and Convergence Phase (Growing Phase).}
As the optimal head takes over the other heads,
the dynamics of $\baromega^\h$ and $\mu^\h$ can be simplified to only the optimal one and neglect the interference terms, which gives us 
\begin{align*}
    \partial_t \log\muistar 
    &\approx \lambda_i d_i \biggl(- \frac{\exp\bigl(d_i(\baromegaistar)^2\bigr)}{L} \phi_i\muistar + \Bigl(1-\baromegaistar\muistar \Bigr) \baromegaistar \biggr), \nend 
    \partial_t \baromegaistar &\approx \left(1 - \left(1 + \frac{\exp\bigl(d_i(\baromegaistar)^2\bigr)}{L} d_i\phi_i \right) \muistar  \baromegaistar\pm \tau_1\right) \sqrt{d_e}^{-1} \lambda_i\baromegaistar\muistar.
\end{align*}
The equilibrium of $\mu^\h$ is given by
\begin{align*}
    \muistar \approx \frac{\baromegaistar}{\baromegaistar^2 + \phi_i\exp\bigl(d_i(\baromegaistar)^2\bigr) L^{-1}}. 
\end{align*}
Plugging this into the dynamics of $\baromegaistar$, we have
\begin{align*}
    &\partial_t \baromegaistar \approx \left(1-d_i(\baromegaistar)^2 \right) \cdot \sqrt{d_e}^{-1}\lambda_i  \cdot \left(2 
    + \frac{\phi_i\exp\bigl(d_i(\baromegaistar)^2\bigr)}{L(\baromegaistar)^2}
    + \frac{L(\baromegaistar)^2}{\phi_i\exp\bigl(d_i(\baromegaistar)^2\bigr)}\right)^{-1}. 
\end{align*}
It is clear that the stationary point is given by $d_i(\baromegaistar)^2 = 1$ and before that, $\baromegaistar$ will keep increasing.
To understand the sudden emergence, we just look at the 
    \[
        2 
    + \frac{\phi_i\exp\bigl(d_i(\baromegaistar)^2\bigr)}{L(\baromegaistar)^2}
    + \frac{L(\baromegaistar)^2}{\phi_i\exp\bigl(d_i(\baromegaistar)^2\bigr)}
\]
term. For small $\baromegaistar$, the second term dominates and we have a super large denominator, thus making the growth of $\baromegaistar$ slow. 
However, as $\baromegaistar$ goes somewhere near the valley  of the function (not necessarily achieving the minimum) while still remaining away from $\sqrt{1/d_i}$, the value of this term suddenly decreases and makes the growth of $\baromegaistar$ fast.
This is the reason why we have the sudden emergence.
After $d_i (\baromegaistar)^2$ reaches $1$, we have the convergence result. 
To characterize the emergence and convergence behavior, we also consider time $t$ as a function of $\baromegaistar$:
\begin{align*}
    \frac{\rd t}{\rd (\baromegaistar)} 
    &= \frac{2 
    + \frac{\phi_i\exp\bigl(d_i(\baromegaistar)^2\bigr)}{L(\baromegaistar)^2}
    + \frac{L(\baromegaistar)^2}{\phi_i\exp\bigl(d_i(\baromegaistar)^2\bigr)}}{\left(1-d_i(\baromegaistar)^2 \right) \cdot \sqrt{d_e}^{-1}\lambda_i}. 
\end{align*}
By integrating this function with respect to $\baromegaistar$, we obtain some hyperbolic function for the time function upon convergence. 
\newpage
\section{Optimality}\label[appendix]{sec:optimality_proof}
In this section, we provide optimality results for both the single-head case and the convergence point of the multi-head attention's training dynamics by providing a lower bound. 

\paragraph{Notations.}
We denote by $[A; B]$ the concatenation of two matrices $A$ and $B$ along the row direction.

\subsection{Optimality of Single-Head Attention}
\label{sec:global-optimality-single-head}
In this section, we provide a characterization of the global optimality of the single-head case. 
Under the condition $W_Y = 0$, we have the training loss given by
\begin{align*}
    \cL(U, W) \deleq \EE\left[
        \bigl\| G^\top q - (U_X X + U_Y (G^\top X + \varepsilon)) p\bigr\|_2^2
    \right], \:
    \where p = \softmax(X^\top W_X q). 
\end{align*}
Here, $G\in \RR^{d \times d_y}$ is the random coefficient matrix. 
Recall that $G = d^{-1/2} \varPhi \diag(g_1, \dots, g_I) \varPsi^\top$. 
Define $G' = \diag(g_1, \dots, g_I)$.
Since $X$, $\varepsilon$ and $q$ are rotationally invariant in distribution, one can equivalently consider the rotated input $X' = \varPhi^\top X$, $\varepsilon' = \varPsi^\top \varepsilon$ and $q' = \varPhi^\top q$, and also rotate the weights by
\[
    W_X' = \varPhi^\top W_X \varPhi,  \quad 
    U_Y' = \varPsi^\top U_Y \varPsi, \quad
    U_X' = \varPsi^\top U_X \varPhi.
\]
Using the rotated input and weights, we have
\begin{equation}\label{eq:rotation-optimality}
\begin{gathered}
    p = \softmax(X'^\top W_X' q'), \quad 
    y_q = G^\top q = \varPsi {G'}^\top q', \\
    \hat y_q = (U_X X + U_Y (G^\top X + \varepsilon)) p = \varPsi (U_X' X' + U_Y' (G'^\top X' + \varepsilon')) p, \\
    \cL(U', W') = \EE\left[
        \bigl\| {G'}^\top q' - (U_X' X' + U_Y' (G'^\top X' + \varepsilon')) p\bigr\|_2^2
    \right].
\end{gathered}
\end{equation}
Therefore, it suffices to consider the case where $\varPhi$ and $\varPsi$ are identity matrices.
We aim to prove the following theorem.
\begin{theorem}[Approximate Optimality of Single-Head Attention]
\label[theorem]{thm:global-optimality-single-head-multitask}
For \ac{ssa} with $I$ tasks, assume that $W_Y = 0$, and the noise level $\sigma > 0$ is a constant independent of $d$ and $L$. 
Suppose $ L =  o(\exp(d))$, then for some $b^\star = (b_i^\star)_{i\in[I]}$ and $u^\star = (u_i^\star)_{i\in[I]}$, the following solution 
\begin{align*}
        W^\star = \begin{bNiceArray}{cccc|c}
            \sqrt{b_1^\star/d_i}\cdot I_{d_1} & 0 & \cdots & 0 & \Block{4-1}<\large>{\star}\\
            0 & \sqrt{b_2^\star/d_i}\cdot I_{d_2} & \cdots & 0 & \\
            \vdots & \vdots & \ddots & \vdots & \\
            0 & 0 & \cdots & \sqrt{b_I^\star/d_I} \cdot I_{d_I} & \\
            \hline 
            \Block{1-4}<\large>{0} & & & & \star
        \end{bNiceArray}, \quad 
        U^\star =\begin{bNiceArray}{c|cccc}
            0 & \Block{1-4}<\large>{0} & & & \\
            \hline 
            \Block{4-1}<\large>{0} & u_1^\star & 0 & \cdots & 0\\
             & 0 & u_2^\star & \cdots & 0\\
             & \vdots & \vdots & \ddots & \vdots\\
             & 0 & 0 & \cdots & u_I^\star
        \end{bNiceArray}
\end{align*}
    approximately achieves the minimal loss value in the sense that for any other $(W, U)$ such that $W_Y = 0$, we have
    \begin{align*}
        \cL(U^\star, W^\star) \le \cL(U, W) + \sum_{i=1}^I \frac{\lambda_i d_i}{d} \cdot O(L^{-(1-\epsilon)/2} + L^{-1 + 2(c^{-2}\pmax \epsilon)}) , 
    \end{align*}
    where $\epsilon$ is a small constant.
In particular, $b^\star$ and $u^\star$ can be obtained by first finding the solution $(B^\star, b^\star)$ to the following optimization problem for an absolute constant $C\in(0,1)$:
\begin{gather*}
    b^\star = \argmin_{C\cdot\log L\ge B \ge 0,\atop b\in \RR_+^{I},  \vone^\top b = B}  \sum_{i=1}^I \frac{d_i \lambda_i}{d} \left( 1 - \frac{b_i }{ b_i  +  d_i\phi_i \exp(B) / L} \right), \quad \text{and}\quad u_i^\star = \frac{\sqrt{b_i^\star d_i}}{ b_i^\star + d_i \phi_i  \exp(B^\star)/L}.
\end{gather*}
\end{theorem}

Before we prove \Cref{thm:global-optimality-single-head-multitask}, we first rewrite the loss function in a more convenient form and discuss some related properties. 
\subsubsection{Rewriting and Lower Bounding the Loss Function} 
Note that $p$ is only a function of $(X, q)$ by $W_Y = 0$.
We use the mutual independence between $(X, q)$, $G$ and $\varepsilon$ to rewrite the loss function as follows:
\begin{align}
    \cL(U, W) 
    & = \EE 
        \bigl\| G^\top q - U_Y G^\top X p\bigr\|_2^2
     + \EE \bigl\| U_X X p \bigr\|_2^2 + \EE \bigl\| U_Y \varepsilon p \bigr\|_2^2 \nend
    & = \EE\bigl[ \trace(G G^\top) - 2 \trace(
        U_Y G^\top X p q^\top G
    ) \nend
    &\qquad + \trace(U_X X p p^\top X^\top U_X^\top + U_Y G^\top X p p^\top X^\top G U_Y^\top + U_Y \varepsilon p p^\top \varepsilon^\top U_Y^\top) \bigr] \nend 
    &\ge \EE\bigl[ \trace(G G^\top) - 2 \trace(
        U_Y G^\top X p q^\top G
    )  + \trace(U_Y (G^\top X p p^\top X^\top G + \varepsilon p p^\top \varepsilon^\top) U_Y^\top) \bigr]. 
    \label{eq:cL(U, W)}
\end{align}
Here, the inequality follows from the fact that the trace of a positive semi-definite matrix is non-negative and by setting $U_X = 0$, the inequality becomes an equality.
Thus, it is sufficient to consider the case $U_X = 0$ for optimality. 
Note that \eqref{eq:cL(U, W)} is a quadratic function in terms of $U_Y$. 
By optimizing over $U_Y$, we obtain $\cL(W)$ as a lower bound of the loss value:
\begin{align}
    \cL(W) \defeq 
    \EE \trace(G G^\top) - \trace \bigl(\EE[G^\top q p^\top X^\top G] \cdot \EE[G^\top X p p^\top X^\top G + \varepsilon p p^\top \varepsilon^\top]^{-1} \cdot \EE[G^\top X p q^\top G] \bigr). 
    \label{eq:cL(W)}
\end{align}
Here, the optimal $U_Y$ is given by 
\begin{align}
    U_Y^\star(W) = \EE[G^\top q p^\top X^\top G] \cdot \EE[G^\top X p p^\top X^\top G + \varepsilon p p^\top \varepsilon^\top]^{-1}. 
    \label{eq:optimal U_Y}
\end{align}
When plugging in the concrete form of $G$, we notice that $\EE[G^\top M G]$ always gives a diagonal matrix for any conformable matrix $M$. In particular, 
\begin{align*}
    \EE[G^\top M G] = d^{-1} \cdot \diag\left(\lambda_1 \sum_{i\in\cJ_1} M_{ii}, \dots, \lambda_I \sum_{i\in\cJ_I} M_{ii}\right).
\end{align*}
In the sequel, we denote by $\trace_i(M) = \sum_{j\in\cJ_i} M_{jj}$ the sliced trace of the $i$-th block of $M$.
Note that $\EE[\varepsilon p p^\top \varepsilon^\top] = \sigma^2 I_{d_y} \EE[\norm{p}_2^2]$ is also a diagonal matrix.
Consequently, \eqref{eq:cL(W)} can be rewritten as 
\begin{align}
    \cL(W) = \sum_{i=1}^I \frac{\lambda_i d_i}{d} \left(
        1 - \frac{d_i^{-1} \bigl(\EE\trace_i(\EE[X p \given q] q^\top)\bigr)^2}{\EE\trace_i(\EE[X p p^\top X^\top \given q]) + \sigma^2 d \lambda_i^{-1}  \EE[\norm{p}_2^2]}
    \right).
    \label{eq:cL(W)-2}
\end{align}
Moreover, $U_Y^\star(W)$ is also a diagonal matrix:
\begin{align}
    U_Y^\star(W) = \diag(u_1^\star(W), \dots, u_I^\star(W)), \quad
    u_i^\star(W) \defeq \frac{\EE \trace_i(\EE[X p\given q] q^\top )}{\EE\trace_i(\EE[X p p^\top X^\top \given q]) + \sigma^2 d \lambda_i^{-1}  \EE[\norm{p}_2^2]}.
    \label{eq:optimal U_Y-2}
\end{align}
To lower bound \eqref{eq:cL(W)-2}, we invoke inequality $(\EE\trace_i(M N))^2 \le \EE\trace_i(M M^\top) \cdot \EE\trace_i(N^\top N)$ for any two conformable random matrices $M$ and $N$ to obtain 
\begin{align}
    \cL(W) & \ge \sum_{i=1}^I \frac{\lambda_i d_i}{d} \left(
        1 - \frac{d_i^{-1}\cdot \EE[\trace_i(\EE[X p\given q]\EE[X p\given q]^\top) ] \cdot \EE[\trace_i(q q^\top)]}{\EE\trace_i(\EE[X p p^\top X^\top \given q]) + \sigma^2 d \lambda_i^{-1}  \EE[\norm{p}_2^2]}
    \right) \nend 
    & = \sum_{i=1}^I \frac{\lambda_i d_i}{d} \left(
        1 - \frac{\EE[\trace_i(\EE[X p\given q]\EE[X p\given q]^\top) ] }{\EE\trace_i(\EE[X p p^\top X^\top \given q]) + \sigma^2 d \lambda_i^{-1}  \EE[\norm{p}_2^2]}
    \right).
    \label{eq:cL(W)-3}
\end{align}
We apply Stein's lemma to deal with the expectation terms with respect to $X$ and $p$. The result is summarized by  \Cref{fact:derivatives for B terms} to and we defer readers to 
, we conclude that
\begin{equation}
    \label{eq:expectation of ZpqpZ}
\begin{aligned}
    \EE[X p \given q] &=  W_X q (1 - f_1(\norm{W_X q}_2^2)), \\
    \EE[X p p^\top X^\top\given q] &= W_X q q^\top W_X^\top f_2(\norm{W_X q}_2^2) + I_{d} f_1(\norm{W_X q}_2^2),
\end{aligned}
\end{equation}
where 
\[
    f_1(\norm{W_X q}_2^2) = \EE[\norm{p}_2^2\given q], \quad 
    f_2(\norm{W_X q}_2^2) = \EE[1 - \norm{p}_2^2 - 6 \norm{p}_3^3 + 6 \norm{p}_2^4 \given q]. 
\]
We remark that both $f_1$ and $f_2$ are just functions of $\norm{W_X q}_2^2$ as we have proved in \Cref{fact:p is a function of Wq's 2-norm}.
In the sequel, we drop the dependence of $f_1, f_2$ on $\norm{W_X q}_2^2$ for simplicity.
Plugging \eqref{eq:expectation of ZpqpZ} into \eqref{eq:cL(W)-2}, we obtain
\begin{align}
    \cL(W) = \sum_{i=1}^I \frac{\lambda_i d_i}{d} \left(
        1 - \frac{d_i^{-1} \bigl(\EE[\trace_i(W_X q q^\top (1 - f_1))]\bigr)^2}{\EE[\trace_i(W_X q q^\top W_X^\top f_2) + d_i \phi_i  f_1]}
    \right), 
    \label{eq:cL(W)-4}
\end{align}
where we follow the convention to denote $\phi_i = 1 + \snr_i^{-1} = 1 + \sigma^2 d \lambda_i^{-1} d_i^{-1}$.
Also, $U_Y^\star(W)$ has diagonal entries given by plugging \eqref{eq:expectation of ZpqpZ} into \eqref{eq:optimal U_Y-2}:
\begin{align}
    u_i^\star(W) = \frac{\EE[\trace_i(W_X q q^\top (1 - f_1))]}{\EE[\trace_i(W_X q q^\top W_X^\top f_2) + d_i \phi_i  f_1]}.
    \label{eq:optimal U_Y-3}
\end{align}
Plugging \eqref{eq:expectation of ZpqpZ} into \eqref{eq:cL(W)-3}, we obtain
\begin{align}
    \cL(W) & \ge \sum_{i=1}^I \frac{\lambda_i d_i}{d} \left(
        1 - \frac{\EE[\trace_i(W_X q q^\top W_X^\top) (1 - f_1) ] }{\EE[\trace_i(W_X q q^\top W_X^\top f_2) + d_i \phi_i  f_1]}
    \right).
    \label{eq:cL(W)-5}
\end{align}
We remind the readers that we do not make any change to the denominator of \eqref{eq:cL(W)-5}, which is still equal to $\EE\trace_i(\EE[X p p^\top X^\top \given q]) + \sigma^2 d \lambda_i^{-1}  \EE[\norm{p}_2^2] > 0$. 
Invoking the fact that both $f_1$ and $f_2$ are just functions of $\norm{W_X q}_2^2 = \trace(W_X q q^\top W_X^\top)$, and that $\EE[f(x)]/\EE[g(x)] \le \sup_{x} f(x)/g(x)$ for non-negative function $g$, we have \eqref{eq:cL(W)-5} further lower bounded by
\begin{align}
    \cL(W) \geq \inf_{B\ge 0, b\ge 0 \atop \vone^\top b = B} \cL_{\lowerbound}(B, b), \where \cL_{\lowerbound}(B, b) \defeq  \sum_{i=1}^I \frac{d_i \lambda_i}{d} \left( 1 - \frac{b_i (1-f_1(B))^2 }{ b_i f_2(B)  +  d_i\phi_i f_1(B)}\right). 
    \label{eq:SS-Attn cL lb-4}
\end{align}
Comparing \eqref{eq:SS-Attn cL lb-4} to \eqref{eq:cL(W)-5}, we are just replacing the sliced trace by $b_i$ and the full trace by $B$ and imposing the constraint $\vone^\top b = B$.
We give a physical interpretation of \eqref{eq:SS-Attn cL lb-4}.
If we view $B$ as the \say{\emph{attention budget}} and $b_i$ as the \say{\emph{attention allocation for task $i$}}, then \eqref{eq:SS-Attn cL lb-4} is a bilevel optimization problem where: 
\begin{itemize}
    \item The lower-level problem is to find the optimal allocation of the attention budget to each task for a given attention budget $B$.
    \item The upper-level problem is then to decide the optimal attention budget $B$.
\end{itemize}
As a concluding remark for this lower bound, we have the following lemma that shows the achievability of the lower bound.

\subsubsection{Approximations of Nonlinear Functions $f_1$ and $f_2$}
The next thing is to understand the behavior of $f_1$ and $f_2$.
Note that as we have characterized in \Cref{sec: p-moment}, we have under bounded parameter norm condition that 
\[
    f_1(B) = \EE[\norm{p}_2^2 \given \norm{W_X q}_2^2 = B] \approx \frac{\exp(B)}{L}, 
    \quad 
    f_2(B) = \EE[1 - \norm{p}_2^2 - 6 \norm{p}_3^3 + 6 \norm{p}_2^4 \given \norm{W_X q}_2^2 = B] \approx 1.
\]
Therefore, it suffices to characterize the following dominant part of the loss's lower bound 
\begin{align*}
    \cL_{\simple}(B, b)\defeq \sum_{i=1}^I \frac{d_i \lambda_i}{d} \left( 1 - \frac{b_i }{ b_i  +  d_i\phi_i \cdot \exp(B) \cdot L^{-1} } \right). 
\end{align*}
The following lemma gives an approximation of $\cL_{\lowerbound}$ with $\cL_{\simple}$ when $B$ is bounded from above.
\begin{lemma}[Simple Approximation of $\cL_{\lowerbound}$]
    \label[lemma]{lemma: cL simple approx}
For any $\epsilon>0$,
suppose $\sqrt B \le c^{-1} \sqrt{2 \log L}$, where $c$ satisfies 
    \[\frac{1}{c} + \frac{3}{1 + \sqrt{1+c^2/2}} \le  \epsilon. \]
    Then for any $b$ with non-negative elements such that $\vone^\top b = B$, it holds that
    \begin{align*}
        \left|\cL_{\lowerbound}(B, b) - \cL_{\simple}(B, b)\right| \le \sum_{i=1}^I \frac{d_i \lambda_i}{d} \cdot O(L^{-1 + 2(c^{-2}\pmax\epsilon)}).
    \end{align*}
\end{lemma}
\begin{proof}(\emph{Proof of \Cref{lemma: cL simple approx}})
    Note that for $\sqrt B \le c^{-1} \sqrt{2 \log L}$, where by our choice $c$ satisfies the conditions in \Cref{lemma: pseudo-dynamics}, we have
\[
    \left|f_1(B) - \frac{\exp(B)}{L}\right| = \left|\EE[\norm{p}_2^2 \given \norm{W_X q}_2^2 = B] -  \frac{\exp(B)}{L}\right| \le O(L^{-2(1-\epsilon)}). 
\]
In additions, \Cref{lemma: pseudo-dynamics} also implies that
\[
    \left|f_2(B) - 1\right| = \left|\EE[ - \norm{p}_2^2 - 6 \norm{p}_3^3 + 6 \norm{p}_2^4\given \norm{W_X q}_2^2 = B]\right| \le \frac{\exp(B)}{L} + O(L^{-2(1-\epsilon)}) \le O(L^{-1 + 2 c^{-2}}).
\]
Therefore, we have for the target function in \eqref{eq:SS-Attn cL lb-4} that
\begin{align*}
    \cL_{\lowerbound}(B, b)& = \sum_{i=1}^I \frac{d_i \lambda_i}{d} \left( 1 - \frac{b_i \cdot (1 \pm O(L^{-1 + 2 c^{-2}}))^2 }{ b_i \cdot (1 \pm O(L^{-1 + 2 c^{-2}}))  +  d_i\phi_i \cdot \exp(B) \cdot L^{-1} \cdot (1 \pm O(L^{-1+ 2\epsilon}))}\right) \nend
    &= \sum_{i=1}^I \frac{d_i \lambda_i}{d} \left( 1 - \frac{b_i }{ b_i  +  d_i\phi_i \cdot \exp(B) \cdot L^{-1} } \cdot \left(1 \pm O(L^{-1 + 2(c^{-2}\pmax\epsilon)})\right)\right) \nend 
    &=  \sum_{i=1}^I \frac{d_i \lambda_i}{d} \left( 1 - \frac{b_i }{ b_i  +  d_i\phi_i \cdot \exp(B) \cdot L^{-1} } \right) \pm \sum_{i=1}^I \frac{d_i \lambda_i}{d} \cdot O(L^{-1 + 2(c^{-2}\pmax\epsilon)}) , 
\end{align*}
where the last equality is due to the fact that $\frac{b_i }{ b_i  +  d_i\phi_i \cdot \exp(B) \cdot L^{-1} } \le 1$.
\end{proof}

\subsubsection{The Optimal Attention Budget is Bounded}
Following the result of \Cref{lemma: cL simple approx}, we are able to characterize $\cL_{\simple}$ in order to understand the behavior of $\cL_{\lowerbound}$ with bounded attention budget $B$.
We first give a naive upper bound of $\cL_{\lowerbound}$ by plugging in $b_i = B d_i / d$ for all $i\in[I]$ and $B = 1$. 
\begin{lemma}[Upper Bound of $\cL_{\lowerbound}$]
    \label[lemma]{fact:upper bound of cL simple}
    For $\cL_{\lowerbound}(B, b)$, we have upper bound on the minimal value as
    \begin{align*}
        \inf_{ B \ge 0,  b\ge 0, \atop\vone^\top b =B} \cL_{\lowerbound}(B, b) \le  \sum_{i=1}^I \frac{( d_i \lambda_i + \sigma^2 d) e \cdot L^{-1}}{ 1  +  d\phi_i e \cdot L^{-1} } + \sum_{i=1}^I \frac{d_i \lambda_i}{d} \cdot O(L^{-1 + 2(c^{-2}\pmax\epsilon)}) = O(d/L). 
    \end{align*}
\end{lemma}
\begin{proof}[Proof of \Cref{fact:upper bound of cL simple}]
    We choose $B=1$ and plug $b_i = B d_i / d$ into $\cL_{\simple}(B, b)$ and have
\begin{align*}
    \left.\cL_{\simple}(B, b) \right|_{B=1, b_i = B d_i/d} &= \left.\sum_{i=1}^I \frac{d_i \lambda_i}{d} \left( 1 - \frac{1}{ 1  +  d\phi_i\exp(B) B^{-1} \cdot L^{-1} } \right) \right|_{B=1} \nend 
    &= \sum_{i=1}^I \frac{d_i \lambda_i}{d} \cdot  \frac{d\phi_i e \cdot L^{-1}}{ 1  +  d\phi_i e \cdot L^{-1} } = \sum_{i=1}^I \frac{( d_i \lambda_i + \sigma^2 d) e \cdot L^{-1}}{ 1  +  d\phi_i e \cdot L^{-1} }, 
\end{align*}
which is clearly of the order $O(d/L)$.  
In addition, we note that $B = 1 \le c^{-2} \cdot 2\log L$ for sufficiently large $L$ and any choice of $\epsilon \in (0, 1)$ and corresponding $c$.
Hence, the condition in \Cref{lemma: cL simple approx} is satisfied for this choice of $B, b$. 
Therefore, we conclude the proof by also including the error of approximating $\cL_{\lowerbound}$ with $\cL_{\simple}$ in \Cref{lemma: cL simple approx}.
\end{proof}
Given the upper bound of the optimal value of $\cL_{\lowerbound}$, we next show that it suffices to optimize over $B$ in a small range, i.e., $0\le B\le c^{-2} 2\log L$ for some small constant $c$.
We first consider the large scale regime, i.e.,
$\sqrt B \ge O(d^{1-\epsilon})$ with a constant scale of noise $\sigma^2 = \Omega(1)$.
For our convenience, we define a \emph{normalized} version of the query token as 
\begin{align*}
    \tilde q = \begin{bNiceMatrix}
        \tilde q_{(1)} \\ \vdots \\ \tilde q_{(I)}
    \end{bNiceMatrix}, 
    \where \tilde q_{(i)} = d_i \norm{q_{(i)}}_2^{-1}\cdot q_{(i)}.
\end{align*}
Here, we let $q_{(i)}$ be the $i$-th block of $q$ corresponding to task $i$'s position $\cJ_i$. 
Notably, each $\tilde q_{(i)}$ is uniformly distributed on the $\sqrt{d_i}$-sphere in $\RR^{d_i}$.
Under this model with normalized query token, we have the optimal $U_Y$ still given by \eqref{eq:optimal U_Y-2} and this time each diagonal element is given by
\begin{align}
    \tilde u_i^\star(W) \defeq \frac{\EE[\trace_i(W_X \tilde q \tilde q^\top (1 - f_1))]}{\EE[\trace_i(W_X \tilde q \tilde q^\top W_X^\top f_2) + d_i \phi_i  f_1]}.
    \label{eq:optimal U_Y-tilde q}
\end{align}
In addition, we define the corresponding loss function as 
\begin{align}
    \tilde \cL(W) \defeq \sum_{i=1}^I \frac{\lambda_i d_i}{d} \left(
        1 - \frac{d_i^{-1} \bigl(\EE[\trace_i(W_X \tilde q \tilde q^\top (1 - f_1))]\bigr)^2}{\EE[\trace_i(W_X \tilde q \tilde q^\top W_X^\top f_2) + d_i \phi_i  f_1]}
    \right), 
    \label{eq:tilde cL(W)}
\end{align}
\begin{lemma}(Large $B$ is suboptimal for constant noise level)
    \label[lemma]{lemma:SS-Attn lb-2}
    Suppose $B = \Omega(d^{2-2\epsilon})$ for some small $\epsilon\in(0, 1)$, the noise level $\sigma^2 = \Omega(1)$, and $L=o(\exp(d))$.
    Then for any $b$ with non-negative elements such that $\vone^\top b = B$,
    \begin{align*}
        \cL_{\lowerbound}(B, b) \ge \sum_{i=1}^I \frac{d_i \lambda_i}{d} \frac{\sigma^2 / \lambda_i  (1 - O(d^{-(1-2\epsilon)}))}{ 1 + \sigma^2 / \lambda_i  (1 - O(d^{-(1-2\epsilon)}))} = \Omega(1).
    \end{align*}
\end{lemma}
\begin{proof}[Proof of \Cref{lemma:SS-Attn lb-2}]
    Note that under the condition $B \ge \Omega(d^{2-2\epsilon})$, it follows from \Cref{lem:p-moment-tail} that 
    \begin{align}
        f_1(B ) = \EE[\norm{p}_2^2 \given \norm{W_X q}_2^2 = B] \ge 1 - O(\sqrt B^{-(1-\epsilon)}) \ge 1 - O(d^{-(1-2\epsilon)}),
        \label{eq:f1 lower bound-large B}
    \end{align}
    which means that the attention probability vector is approximately one-hot.
    Our goal is thus using this \say{almost deterministic} property to show that as long as we have a constant level of noise, the \ac{icl} rate is suboptimal.
    Our proof is constructive in nature, i.e., relating the loss $\cL_\lowerbound(B, b)$ to the \ac{icl} rate of another model with a constructed attention weights.
    Consider the following choice of weights:
    \begin{align}
        W = \begin{bNiceArray}{ccc|c}
            \sqrt{b_1/d_1}\cdot I_{d_1} & \cdots & 0 & \Block{3-1}<\large>{\star}\\
            \vdots & \ddots & \vdots & \\
            0 & \cdots & \sqrt{b_I/d_I} \cdot I_{d_I} & \\
            \hline 
            \Block{1-3}<\large>{0} & & & \star
        \end{bNiceArray}. 
        \label{eq:W,U-construction-1}
    \end{align}
    Recall the definition of $\cL_{\lowerbound}(B, b)$ in \eqref{eq:SS-Attn cL lb-4} that
    \begin{align*}
        \cL_{\lowerbound}(B, b) &=\sum_{i=1}^I \frac{d_i \lambda_i}{d} \left( 1 - \frac{b_i (1-f_1(B))^2 }{ b_i f_2(B)  +  d_i\phi_i f_1(B)}\right).
    \end{align*}
    Note that the denominator $b_i f_2(B)  +  d_i\phi_i f_1(B) $ is always positive. 
    In the sequel, we aim to lower bound the value of $\cL_{\lowerbound}(B, b)$, thus equivalent to lower bounding the denominator. 
    We have 
    \begin{align*}
        b_i f_2(B)  +  d_i\phi_i f_1(B) &= {\trace}_i \bigl(W_X \tilde q \tilde q^\top W_X \cdot f_2(\norm{W_X \tilde q}_2^2)\bigr)    + d_i\phi_i  f_1(\norm{W_X \tilde q}_2^2) \nend 
        &= {\trace}_i \bigl(W_X \tilde q \tilde q^\top W_X \cdot f_2(\norm{W_X \tilde q}_2^2) + I_d f_1(\norm{W_X \tilde q}_2^2)\bigr)    +  \sigma^2 d \lambda_i^{-1} f_1(B) \nend 
        &= {\trace}_i\bigl(\EE[X p p^\top X^\top \given \tilde q]\bigr) + \sigma^2 d \lambda_i^{-1} f_1(B). 
    \end{align*}
    Here, $\tilde q$ can be any normalized query token. 
    The first equality holds by noting that $\norm{W_X \tilde q}_2^2 \equiv B$ and $\trace_i(W_X \tilde q \tilde q^\top W_X^\top) \equiv b_i$ by our construction. 
    The second equality holds by definition $\phi_i = 1 + d \sigma^2/(\lambda_i d_i)$.
    The last equality holds by \eqref{eq:expectation of ZpqpZ} where $p = \softmax(X^\top W_X \tilde q)$.
    Meanwhile, we also have for the numerator that 
    \begin{align*}
        b_i (1-f_1(B))^2 &= {\trace}_i \bigl(W_X \tilde q \tilde q^\top W_X^\top (1 - f_1(B))^2 \bigr) = {\trace}_i \bigl(\EE[X p \given \tilde q] \EE[ X p \given \tilde q]^\top \bigr). 
    \end{align*}
    For $\cL_{\lowerbound}(B, b)$, we thus have
    \begin{align*}
        \cL_{\lowerbound}(B, b) &= \sum_{i=1}^I \frac{d_i \lambda_i}{d} \left( 1 - \frac{{\trace}_i \bigl(\EE[X p \given \tilde q] \EE[ X p \given \tilde q]^\top \bigr)}{ {\trace}_i\bigl(\EE[X p p^\top X^\top \given \tilde q]\bigr) + \sigma^2 d \lambda_i^{-1} f_1(B)}\right) \nend 
        &\ge \sum_{i=1}^I \frac{d_i \lambda_i}{d} \cdot \frac{\sigma^2 d \lambda_i^{-1} f_1(B)}{{\trace}_i \bigl(\EE[X p \given \tilde q] \EE[ X p \given \tilde q]^\top \bigr)  +  \sigma^2 d \lambda_i^{-1} f_1(B)} \nend 
        &\ge \sum_{i=1}^I \frac{d_i \lambda_i}{d} \cdot \frac{\sigma^2 \lambda_i^{-1} f_1(B)}{d^{-1} \cdot \bigl\|\EE[X p \given \tilde q] \bigr\|_2^2  +  \sigma^2 \lambda_i^{-1} f_1(B)}
    \end{align*}
    Here, the first inequality is Cauchy-Schwarz.
    Therefore, it suffices to upper bound the term $\bigl\|\EE[X p \given \tilde q] \bigr\|_2^2$.
    Consider $v=W_X \tilde q / \norm{W_X \tilde q}_2$ as the \emph{direction of the attention}.
    For each $x_l$, we decompose it into the part that is parallel to $v$ and the part that is orthogonal to $v$, i.e., $x_l = z_l v + x_l^\perp$ where in distribution $z_l \iidfrom \cN(0, 1)$ and are independent of $x_l^\perp$.
    Define $z=(z_1, \dots, z_L)^\top$ and $X^\perp = (x_1^\perp, \dots, x_L^\perp)$.
    By the definition of the softmax probability, we have
    \[
        p = \softmax(x_l^\top W_X \tilde q) = \softmax(\sqrt B z),
    \] 
    which is just a function of $z$.
    Therefore, in expectation, we have
    \begin{align*}
        &\EE[X p\given \tilde q] = \EE[ X^\perp p + v \cdot z p \given \tilde q] = v\cdot \EE[z p \given \tilde q] \nend 
        &\quad \Rightarrow \bigl\|\EE[X p\given \tilde q] \bigr\|_2^2 = \EE[z p \given \tilde q]^2 = \EE[z \softmax(\sqrt B z)]^2 \le \bigl(\EE\max\{z_1, \dots, z_L\}\bigr)^2  = O(2\log L). 
    \end{align*}
    Here we use the fact that hard max is larger than softmax. 
    Therefore, we have for $\cL_{\lowerbound}(B, b)$ that
    \begin{align*}
        \cL_{\lowerbound}(B, b) 
        &\ge \sum_{i=1}^I \frac{d_i \lambda_i}{d} \cdot \frac{\sigma^2 \lambda_i^{-1} f_1(B)}{d^{-1} \cdot \bigl\|\EE[X p \given \tilde q] \bigr\|_2^2  +  \sigma^2 \lambda_i^{-1} f_1(B)} \nend 
        &\ge \sum_{i=1}^I \frac{d_i \lambda_i}{d} \cdot \frac{\sigma^2 \lambda_i^{-1} (1 - O(d^{-(1-2\epsilon)}))}{O(2d^{-1}\log L)  +  \sigma^2 \lambda_i^{-1} (1 - O(d^{-(1-2\epsilon)}))}.
    \end{align*}
    Hence, for $L = o(\exp(d))$, we always have $\cL_{\lowerbound}(B, b) = \Omega(1)$ with constant noise level.
\end{proof}
Thus, it suffices to focus on the case $0\le \sqrt B \le O(d^{1-\epsilon})$ for some small $\epsilon\in(0, 1)$ when doing the optimization of the target function in \eqref{eq:SS-Attn cL lb-4}.
The next lemma further shows we can further restrict the attention budget to a much smaller range $0\le B \le O(\log L)$.
\begin{lemma}
    \label[lemma]{lemma:SS-Attn lb}
    Consider $0\le \sqrt{B} \le O(d^{1-\epsilon})$ for some small $\epsilon\in(0, 1)$ when doing the optimization of the target function in \eqref{eq:SS-Attn cL lb-4}. 
    Let $c$ be a constant satisfying 
    \[\frac{1}{c} + \frac{3}{1 + \sqrt{1+c^2/2}} \le  \epsilon. \]
    In addition, we take sufficiently large $L$ such that $c^{-2} \cdot 2\log L = \Omega(1)$. 
    We have for any $b\ge 0$ such that $\vone^\top b = B$,
    \begin{align*}
        \cL_{\lowerbound}(B, b) \ge \min_{c^{-2}2\log L\ge B \ge 0 \atop \vone^\top b= B, b\ge 0}  \cL_{\simple}(B, b) - \sum_{i=1}^I \frac{d_i \lambda_i}{d} \cdot O(L^{-1 + 2(c^{-2}\pmax \epsilon)}).
    \end{align*}
\end{lemma}
\begin{proof}[Proof of \Cref{lemma:SS-Attn lb}]
We discuss the behavior of $\cL_{\lowerbound}$ in three cases. 
\paragraph{Case 1.}
For the case $0\le \sqrt{B} \le c^{-1} \sqrt{2\log L}$, we have following \Cref{lemma: cL simple approx} that $\cL_\lowerbound(B, b)$ is well approximated by $\cL_{\simple}(B, b)$.
Combining the discussion in \Cref{lemma: cL simple approx} and also the upper bound in \Cref{fact:upper bound of cL simple}, we have
\begin{align*}
    \min_{c^{-2}2\log L\ge B \ge 0 \atop \vone^\top b= B, b\ge 0} \cL_{\lowerbound}(B, b) = \min_{c^{-2}2\log L\ge B \ge 0 \atop \vone^\top b= B, b\ge 0}  \cL_{\simple}(B, b) \pm \sum_{i=1}^I \frac{d_i \lambda_i}{d} \cdot O(L^{-1 + 2(c^{-2}\pmax \epsilon)}) \le O(d/L).
\end{align*}
The remaining thing is to show for the other case the optimal value of $\cL_{\lowerbound}$ is much larger than $O(d/L)$.
\paragraph{Case 2.}
For $O(d^{1-\epsilon}) \ge \sqrt{B} \ge \Omega(\sqrt{2\log L}^{3/\epsilon})$, we have $f_1(b) = \EE[\norm{p}_2^2 \given \norm{W_X q}_2^2 = B] \ge 1 - O(\sqrt{B}^{-1+\epsilon})$ for some small $\epsilon$ by \Cref{lem:p-moment-tail}.
Therefore, 
\begin{align*}
    \cL_{\lowerbound}(B, b) 
    & =  \sum_{i=1}^I \frac{d_i \lambda_i}{d} \left( 1 - \frac{b_i (1-f_1(B))^2 }{ b_i f_2(B)  +  d_i\phi_i f_1(B)}\right)\nend
    & \ge \sum_{i=1}^I \frac{d_i \lambda_i}{d} \bigg( 1 - \frac{b_i \cdot  O(B^{-1+\epsilon}) }{ - O(\sqrt B^{-1 + \epsilon}) \cdot b_i +  d_i\phi_i \cdot (1 - O(\sqrt B^{-1 +\epsilon }))}\bigg) \nend
    & \ge \sum_{i=1}^I \frac{d_i \lambda_i}{d} \bigg( 1 - \frac{ O(\sqrt{B}^{2\epsilon}) }{ - O(\sqrt{B}^{1+\epsilon}) +  d_i\phi_i \cdot (1 - O(\sqrt{B}^{-1 + \epsilon}))}\bigg)  = \sum_{i=1}^I \frac{d_i \lambda_i}{d} (1 - o(1)).
\end{align*}
Here, in the first inequality, we have $f_2(B)$ lower bounded as
\begin{align*}
    f_2(B) 
    &= \EE[1 - \norm{p}_2^2 - 6 \norm{p}_3^3 + 6 \norm{p}_2^4\given \norm{W_X q}_2^2 = B] \ge -6 \EE[\norm{p}_3^3 - \norm{p}_2^4 \given \norm{W_X q}_2^2 = B] \nend 
    &= -\sum_{l=1}^L 6 \EE[ p_l^2 (p_l - \norm{p_l}_2^2) \given \norm{W_X q}_2^2 = B] \ge -\sum_{l=1}^L 6 \EE[p_l^2 (\norm{p_l}_2 - \norm{p_l}_2^2) \given \norm{W_X q}_2^2 = B] \nend 
    & = - 6 \EE[\norm{p}_2^3 (1 - \norm{p}_2) \given \norm{W_X q}_2^2 = B] \ge - 3 \EE[\norm{p}_2^2 (1 + \norm{p}_2)(1 - \norm{p}_2) \given \norm{W_X q}_2^2 = B]\nend 
    & = -3 \EE[1 - \norm{p}_2^2 \given \norm{W_X q}_2^2 = B] \ge - O(\sqrt B^{-1 + \epsilon}).
\end{align*}
The second inequality follows by $0\le b_i \le B$.
And for the last inequality, we notice that
\[
    d_i \phi_i (1 - O(\sqrt{B}^{-1 + \epsilon})) = O(d), \quad O(\sqrt B^{1 + \epsilon}) \le  O(d^{1 - \epsilon^2}) = o(d), \quad O(\sqrt{B}^{2\epsilon}) = o(d).
\]
Hence, there is no effective \ac{icl} learning for this case.
\paragraph{Case 3.}
For the last case, i.e., 
$c^{-1}\sqrt{2\log L}\le \sqrt B \le O(\sqrt{2\log L}^{3/\epsilon})$, 
we apply the monotonicity of $\EE[\norm{p}_2^2 \given \norm{W_X q}_2^2 = B]$ with respect to $B$ as is shown in \Cref{lem: p-moment-monotone}, which gives us 
$f_1(B) \ge f_1(c^{-2} 2\log L ) = \exp(c^{-2} 2\log L) / L - O(L^{-2(1-\epsilon)})\ge O(L^{2 c^{-2} - 1})$.
Therefore, it follows that
\begin{align*}
    \cL_{\lowerbound}(B, b)  \ge \sum_{i=1}^I \frac{d_i \lambda_i}{d} \left( 1 - \frac{b_i}{-6 b_i + d_i\phi_i L^{2 c^{-2} - 1}}\right), 
\end{align*}
where we use the fact that $f_2(B) \ge - 6$. 
Note that $d_i/L = \Theta(1)$ and $b_i \le B\le O(\sqrt{2\log L}^{3/\epsilon}) \ll d_i L^{2 c^{-2} -1}$. 
Thus, we conclude that the above term is lower bounded by 
$
    \cL_{\lowerbound}(B, b)  \ge \sum_{i=1}^I \frac{d_i \lambda_i}{d} \left( 1- o(1)\right).
$
Combining the above results, we conclude that the optimality could only be achieved in the first case, i.e., $0\le \sqrt{B} \le c^{-1} \sqrt{2\log L}$.
\end{proof}
Let $(B^\star, b^\star)$ be the solution to the following optimization problem:
\begin{align}
    \min_{c^{-2}2\log L\ge B \ge 0, b\ge 0,\atop \vone^\top b = B}  \cL_\simple (B, b) = \sum_{i=1}^I \frac{d_i \lambda_i}{d} \left( 1 - \frac{b_i }{ b_i  +  d_i\phi_i \cdot \exp(B) \cdot L^{-1} } \right). 
    \label{eq:SS-Attn cL ub-1}
\end{align}
By \Cref{lemma:SS-Attn lb}, it suffices to consider the range $0\le B\le c^{-2} 2\log L$ for some constant $c$.
However, for our purpose, we need a more refined characterization of the optimal attention budget $B$. 
The following  result shows that the optimal $B$ 
should be of order $o(\log L)$. 
\begin{lemma}
    \label[lemma]{fact:order of attention budget}
Let $(B^\star, b^\star)$ be any optimal solution to \eqref{eq:SS-Attn cL ub-1}. Then it holds that $B^\star = o(\log L)$.
\end{lemma}
\begin{proof}[Proof of \Cref{fact:order of attention budget}]
    Suppose $B^* = \beta\log L$ for some  constant $\beta$ such that $c^{-2} \ge \beta>0$ and we have $\exp(B)/L = L^{-1 + \beta}$. As a result, the value to the optimization target of \eqref{eq:SS-Attn cL ub-1} is lower bounded by 
    \begin{align*}
        \sum_{i=1}^I \frac{d_i \lambda_i}{d} \left( 1 - \frac{\beta \log L }{ \beta \log L  +  d_i\phi_i \cdot L^{-1 +\beta} } \right) = \sum_{i=1}^I \frac{d_i \lambda_i}{d} \left( 1 - o(1) \right),
    \end{align*}
    which, given $d_i/L = \Theta(1)$, is clearly suboptimal.
\end{proof}

\subsubsection{Perturbation Analysis}
In this part, we understand the effect of perturbation of the optimal weights on the \ac{icl} rate. 
Consider the following weights: 
\begin{align}
    W = \begin{bNiceArray}{cccc|c}
        \sqrt{b_1/d_i}\cdot I_{d_1} & 0 & \cdots & 0 & \Block{4-1}<\large>{\star}\\
        0 & \sqrt{b_2/d_i}\cdot I_{d_2} & \cdots & 0 & \\
        \vdots & \vdots & \ddots & \vdots & \\
        0 & 0 & \cdots & \sqrt{b_I/d_I} \cdot I_{d_I} & \\
        \hline 
        \Block{1-4}<\large>{0} & & & & \star
    \end{bNiceArray}, \quad 
    U =\begin{bNiceArray}{c|cccc}
        0 & \Block{1-4}<\large>{0} & & & \\
        \hline 
        \Block{4-1}<\large>{0} & u_1 & 0 & \cdots & 0\\
         & 0 & u_2 & \cdots & 0\\
         & \vdots & \vdots & \ddots & \vdots\\
         & 0 & 0 & \cdots & u_I
    \end{bNiceArray}.
    \label{eq:Decomposable W,U} 
\end{align}
We have the following result for the above weights. 
\begin{lemma}[Perturbation of attention weights]
    \label[lemma]{lem:loss perturbation}
We define $\baromega_i = \sqrt{b_i/d_i}$ and
$\uistar(b)$ as
\begin{align}
    u_i^\star(b) = \frac{\sqrt{b_i d_i}}{ b_i + d_i \phi_i  \exp(B)/L}, 
    \label{eq:uistar(b)}
\end{align}
and consider $u_i$ to be the perturbation of $u_i^\star(b)$. 
We consider bounded attention budget $0\le B = o(\log L)$ and bounded perturbation such that 
\begin{align}
    0 < \baromega_i u_i = O(1), \quad u_i^2 = O(d_i).
    \label{eq:u-perturbation bound}
\end{align}
Suppose $d/L = \Theta(1)$ and $d_i/d = \Theta(1)$. It then holds that
\begin{align*}
    \cL(U, W) = \cL_\simple(B, b) + \sum_{i=1}^I \frac{d_i}{d} \cdot \left(\left(\lambda_i  b_i +  \frac{(d_i \lambda_i  + \sigma^2 d ) e^B}{L} \right) d_i^{-1}\cdot \left(u_i - \uistar(b) \right)^2 \pm \lambda_i \cdot O(L^{-(1-\epsilon)/2})\right).
\end{align*}
If in addition, we have $|B - B^\star| = o(B^\star)$, then we have
\begin{align*}
    \cL(U, W) &= \cL_\simple(B^\star, b^\star) + \sum_{i=1}^I \frac{\lambda_i }{d} \cdot \left(\left(b_i^\star +  \frac{d_i \phi_i e^{B^\star}}{L} \right) \cdot \left(u_i - u_i^\star\right)^2 \right) \nend
    &\qquad + \sum_{i=1}^I \frac{\lambda_i d_i }{d} \cdot O\left(u_i |\baromega_i - \baromega_i^\star| + \frac{\phi_i}{L}|B - B^\star| + \xi\right). 
\end{align*}
\end{lemma}
\begin{proof}[Proof of \Cref{lem:loss perturbation}]
By our previous discussion in \Cref{fact:order of attention budget}, we restrict our attention budget to the range $0\le B = o(\log L)$.
Our goal is to understand the \ac{icl} loss 
\begin{align*}
    \cL(U, W) &= \EE\left[
        \bigl\| G^\top q - U_Y (G^\top X + \varepsilon) p\bigr\|_2^2 
    \right] \nend
    &= \EE\bigl[ \trace(G G^\top) - 2 \trace(
        U_Y G^\top X p q^\top G
    )  + \trace(U_Y (G^\top X p p^\top X^\top G + \varepsilon p p^\top \varepsilon^\top) U_Y^\top) \bigr] \nend 
    &= \sum_{i=1}^I \frac{\lambda_i d_i}{d} \cdot \left(
        1 - 2 u_i d_i^{-1} \EE{\trace}_i (X p q^\top) + u_i^2 d_i^{-1}\bigl( \EE{\trace}_i(X p p^\top X^\top) + d \sigma^2 \lambda_i^{-1} \EE[\norm{p}_2^2]\bigr)
    \right). 
\end{align*}
Using \eqref{eq:expectation of ZpqpZ}, we have 
\begin{align*}
    \cL(U, W) &= \sum_{i=1}^I \frac{\lambda_i d_i}{d} \cdot \left(
        1 - u_i d_i^{-1} \EE\bigl[{\trace}_i (W_X q q^\top) (1 - f_1) \bigr] + u_i^2 d_i^{-1}\EE\bigl[ {\trace}_i(W_X q q^\top W_X^\top f_2) + d_i\phi_i f_1\bigr]
    \right).
\end{align*}
With $0\le B = o(\log L)$, and noting that $b_i\le B$, we have that all the conditions in \Cref{lemma: E[pq]-moment} on the eigenvalues $\omega$ of $W_X$ are satisfied given that $d/L = \Theta(1)$ and $d_i/d = \Theta(1)$, i.e., 
\begin{gather*}
    \norm{\omega}_\infty \le   L^{-1/4} \cdot (\log L)^{-1/2} , \quad \norm{\omega}_2^2 \le \frac{2 \log L}{3 c^2},  \quad
    \norm{\omega}_4^4 \le L^{-(1-\epsilon)}\cdot (\log L)^{-1}.
\end{gather*}
We thus conclude that 
\begin{align*}
    \EE\left[\left(f_1(\norm{W_X q}_2^2) - \frac{\exp(B)}{L}\right)^2\right] \le O(L^{-(3-\epsilon)}). 
\end{align*}
This together with the fact that $\norm{p}_3^3 \le \norm{p}_2^2$ and $\norm{p}_2^4 \le \norm{p}_2^2$ implies that
\begin{align*}
    \EE\left[\left(f_2(\norm{W_X q}_2^2)- 1 \right)^2 \right] \le O \left(\EE[f_1(\norm{W_X q}_2^2)^2]\right) \le O(L^{-2(1-\epsilon)}). 
\end{align*}
Here, the last inequality holds since $B \le \epsilon \log L$ for fixed small constant $\epsilon>0$.
Therefore, we have each term in the \ac{icl} loss to be 
\begin{align*}
    \left|\EE[{\trace}_i (W_X q q^\top) (1 - f_1)] - \baromega_i d_i \right|
    &= \left|\EE[{\trace}_i (W_X q q^\top) (1 - f_1)] - \EE[{\trace}_i (W_X q q^\top)] \right| \nend 
    &\le  \left|{\trace}_i \left(\sqrt{\EE[\diag(W_X q q^\top)^{\odot 2} ] \EE[f_1^2]} \right) \right| \nend 
    &=\baromega_i d_i \sqrt{\EE[\chi^4 ] \EE[f_1^2]} = \baromega_i d_i \cdot O(L^{-(1-\epsilon)}). 
\end{align*}
\begin{align*}
    \left|\EE{\trace}_i(W_X q q^\top W_X^\top f_2) - \baromega_i^2 d_i\right| 
    &= \left|\EE{\trace}_i(W_X q q^\top W_X^\top f_2) - \EE{\trace}_i(W_X q q^\top W_X^\top)\right| \nend
    &\le \left|{\trace}_i\left(\sqrt{\EE[\diag(W_X q q^\top W_X^\top)^{\odot 2} ] \EE[(f_2 - 1)^2]} \right) \right| \nend
    &= \baromega_i^2 d_i \sqrt{\EE[\chi^4] \EE[(f_2 - 1)^2]} = \baromega_i^2 d_i \cdot O(L^{-(1-\epsilon)}).
\end{align*}
\begin{align*}
    \left|\EE[f_1] - \frac{\exp(B)}{L}\right|
    \le \sqrt{\EE\left[\left(f_1 - \frac{\exp(B)}{L}\right)^2\right]} \le O(L^{-(3- \epsilon)/2}) \le \frac{\exp(B)}{L} \cdot O(L^{-(1-\epsilon)/2}). 
\end{align*}
Therefore, the \ac{icl} loss can be controlled as 
\begin{align*}
    \cL(U, W) & = \sum_{i=1}^I \frac{\lambda_i d_i}{d} \left(
        1 - 2 u_i \baromega_i (1 \pm \xi^2) + u_i^2 \bigl( \baromega_i^2 (1 \pm \xi^2) +  \frac{\phi_i e^B}{L} (1 \pm \xi)\bigr)
    \right) \nend 
    & = \sum_{i=1}^I \frac{\lambda_i d_i}{d} \left(
        1 - 2 u_i \baromega_i + u_i^2 \bigl( \baromega_i^2 +  \frac{\phi_i e^B}{L} \bigr)
    \right) \pm \sum_{i=1}^I \frac{\lambda_i d_i}{d} \cdot \left(\xi^2 (2 u_i \baromega_i + u_i^2 \baromega_i^2) + \xi u_i^2 \phi_i e^B /L\right) \nend 
    & = \sum_{i=1}^I \frac{\lambda_i d_i}{d} \cdot 
        \frac{d_i\phi_i e^B / L}{b_i + d_i\phi_i e^B / L} 
    + \sum_{i=1}^I \frac{\lambda_i d_i}{d} \cdot \left(\left(b_i +  \frac{d_i \phi_i e^B}{L} \right) \cdot \left(\frac{u_i}{\sqrt d_i} - \frac{\sqrt{b_i}}{ b_i +  {d_i \phi_i e^B} / {L} }\right)^2 \pm O(\xi)\right), 
\end{align*}
where $\xi = O(L^{-(1-\epsilon)/2})$. 
Here, the last upper and lower bound follows from \eqref{eq:u-perturbation bound}.
Moreover, if we also consider the perturbation of $\baromega$ from the optimal value, we have with $|B - B^\star| = o(B^\star)$ that
\begin{align*}
    \cL(U, W) &= \sum_{i=1}^I \frac{\lambda_i d_i}{d} \cdot \left(
        1 - 2 u_i \baromega_i^\star + (u_i)^2 \left((\baromega_i^\star)^2 + \frac{\phi_i e^{B^\star}}{L}\right)
        \pm O\left(u_i |\baromega_i - \baromega_i^\star| + \frac{\phi_i}{L} |B^\star - B|\right)
    \right) \nend 
    & \qquad \pm \sum_{i=1}^I \frac{\lambda_i d_i }{d} \cdot O(\xi) \nend 
    &= \sum_{i=1}^I \frac{\lambda_i d_i}{d} \cdot \frac{d_i\phi_i e^{B^\star} / L}{b_i + d_i\phi_i e^{B^\star} / L} + \sum_{i=1}^I \frac{\lambda_i }{d} \cdot \left(\left(b_i^\star +  \frac{d_i \phi_i e^{B^\star}}{L} \right) \cdot \left(u_i - u_i^\star\right)^2 \right) \nend
    &\qquad + \sum_{i=1}^I \frac{\lambda_i d_i }{d} \cdot O\left(u_i |\baromega_i - \baromega_i^\star| + \frac{\phi_i}{L}|B - B^\star| + \xi\right). 
\end{align*}
Hence, we prove the desired result.
\end{proof}

\subsubsection{Proof of The Main Theorem and Its Consequences}
Now we are ready to present the proof of \Cref{thm:global-optimality-single-head-multitask}.
\begin{proof}[Proof of \Cref{thm:global-optimality-single-head-multitask}]
Now, we construct our $W^\star$ and $U^\star$ matrix for an upper bound:
\begin{align}
    W^\star = \begin{bNiceArray}{cccc|c}
        \sqrt{b_1^\star/d_i}\cdot I_{d_1} & 0 & \cdots & 0 & \Block{4-1}<\large>{\star}\\
        0 & \sqrt{b_2^\star/d_i}\cdot I_{d_2} & \cdots & 0 & \\
        \vdots & \vdots & \ddots & \vdots & \\
        0 & 0 & \cdots & \sqrt{b_I^\star/d_I} \cdot I_{d_I} & \\
        \hline 
        \Block{1-4}<\large>{0} & & & & \star
    \end{bNiceArray}, \:
    U^\star =\begin{bNiceArray}{c|cccc}
        0 & \Block{1-4}<\large>{0} & & & \\
        \hline 
        \Block{4-1}<\large>{0} & u_1^\star & 0 & \cdots & 0\\
         & 0 & u_2^\star & \cdots & 0\\
         & \vdots & \vdots & \ddots & \vdots\\
         & 0 & 0 & \cdots & u_I^\star
    \end{bNiceArray},
    \label{eq:SS-Attn cL-construction-W U} 
\end{align}
where $B^\star$ and $b^\star = (b_1^\star, \dots, b_I^\star)$ are the optimal solution to \eqref{eq:SS-Attn cL ub-1} and
each $u_i^\star$ is given by
\begin{align*}
    u_i^\star = \frac{\sqrt{b_i^\star d_i}}{ b_i^\star + d_i \phi_i  \exp(B^\star)/L}
\end{align*}
according to \eqref{eq:uistar(b)}. 
Combining our previous discussions, we already have any group of weights $W, U$ such that $W_Y=0$ satisfying the lower bound
\begin{align*}
    \inf_{U, W \atop W_Y=0} \cL(U, W) 
    &\ge \inf_{B \ge 0 \atop \vone^\top b= B, b\ge 0}\cL_{\lowerbound}(B, b) \mytag{\eqref{eq:SS-Attn cL lb-4}}\nend 
    &\ge \min_{O(d^{2-2\epsilon})\ge B \ge 0 \atop \vone^\top b= B, b\ge 0}\cL_{\lowerbound}(B, b) 
    \mytag{\Cref{lemma:SS-Attn lb-2}}
    \nend 
    &\ge \min_{c^{-2}2\log L\ge B \ge 0 \atop \vone^\top b= B, b\ge 0}  \cL_{\simple}(B, b) - \sum_{i=1}^I \frac{d_i \lambda_i}{d} \cdot O(L^{-1 + 2(c^{-2}\pmax \epsilon)}). 
    \mytag{\Cref{lemma:SS-Attn lb}}
\end{align*}
We also note that $B^\star = o(\log L)$ by \Cref{fact:order of attention budget}.
In addition, \eqref{eq:u-perturbation bound} is satisfied by our construction. 
Hence, it holds by \Cref{lem:loss perturbation} that
\begin{align*}
    \cL(U^\star, W^\star)\le \min_{c^{-2}2\log L\ge B \ge 0 \atop \vone^\top b= B, b\ge 0}  \cL_{\simple}(B, b) + \sum_{i=1}^I \frac{d_i \lambda_i}{d} \cdot O(L^{-(1 - \epsilon) / 2}).
\end{align*}
Hence, we conclude that
\begin{align*}
    \cL(U^\star, W^\star) \le \inf_{U,W \atop W_Y=0} \cL(U, W) + \sum_{i=1}^I \frac{\lambda_i d_i}{d} \cdot O(L^{-(1-\epsilon)/2} + L^{-1 + 2(c^{-2}\pmax \epsilon)}) , 
\end{align*}
which completes our proof of \Cref{thm:global-optimality-single-head-multitask}.
\end{proof}

In the following, we consider a special case where we have only one task $i$ having nonzero signal strength $\lambda_i = \Theta(1)$.
\begin{lemma}
\label[lemma]{cor:single task icl loss}
    Suppose that we have only one task $i$ with signal strength $\lambda_i=\Theta(1)$ and $\lambda_j = 0$ for $j\neq i$. Then 
    \begin{myenumi} 
    \item The solution to \eqref{eq:SS-Attn cL ub-1} is given by $B^\star = b_i^\star = 1$ (which implies $\baromega_i^\star = \sqrt{d_i}^{-1}$ and $\baromega_j^\star = 0$ for $j\neq i$). In addition, we have optimal value $u_i^\star = \sqrt{d_i}/(1 + d_i \phi_i e L^{-1})$ and $u_j^\star = 0$ for $j\neq i$.
    Let $U$, $W$ be constructed as in \eqref{eq:SS-Attn cL-construction-W U} with $b^\star$ and $u^\star$ given as above.
    This gives the optimal \ac{icl} loss in the sense that 
    \[
        \cL(U^\star, W^\star) = \cL_\simple (B^\star, b^\star) \pm \frac{d_i \lambda_i}{d}\cdot O(L^{-(1-\epsilon) / 2})  \le \inf_{U, W \atop W_Y=0} \cL(U, W) + \frac{d_i \lambda_i}{d}\cdot O(L^{-(1-\epsilon) / 2}), 
    \]
    where $\cL_\simple(B^\star, b^\star)$ is given by
    $
        \cL_\simple(B^\star, b^\star) = \frac{e d_i \phi_i/L}{1 + e d_i \phi_i/L}.
    $
    \item Let $\baromega^\star$ and $u^\star$ be defined as in (i).
    Suppose we have a construction of $\baromega_i$ and $u_i$ such that 
    \begin{align}
        \left|\frac{\baromega_i}{\baromega_i^\star} - 1\right| \le \delta_\omega,\quad \left|\frac{u_i}{u_i^\star} - 1\right| \le \delta_u, \quad \left|\baromega_j\right| \le \xi_\omega, \quad \left|u_j\right| \le \xi_u, \quad \forall j\in [I]\setminus\{i\}, 
        \label{eq:WU-perturbation condition}
    \end{align}
    where $\delta_\omega = o(1), \delta_u = o(1)$,  $\xi_\omega=o(d^{-1/2})$ and $\xi_u = o(d^{1/2})$.
    Then the \ac{icl} loss for such a construction is upper bounded by 
    \[  
       \cL(U, W) \le \cL_\simple(B^\star, b^\star) + \frac{\lambda_i d_i}{d} \cdot O(\delta_\omega + \xi_\omega^2 d + \delta_u^2) + \sum_{k=1}^I \frac{d_k \lambda_k }{d} \cdot O(L^{-(1-\epsilon)/2}) + \frac{\sigma^2 I}{L} \cdot O(\xi_u^2).
    \]
    \end{myenumi}
\end{lemma}
\begin{proof}[Proof of \Cref{cor:single task icl loss}]
    The first argument is straightforward by noting that the inner problem of \eqref{eq:SS-Attn cL ub-1} is convex with fixed total attention budget $B$, and we have monotonicity that the more attention budget is allocated to a task, the smaller the loss for that task is. 
    With $\lambda_j=0$, the nominal task $j$ has no effect on the \ac{icl} loss and thus all attention budget should be allocated to the only task $i$, which gives the simplified loss 
    \begin{align*}
        \min_{b: \vone^\top b = B} \cL_\simple(B, b) = \frac{d_i \lambda_i}{d} \cdot \frac{1}{B e^{-B} \cdot d_i^{-1} \phi_i^{-1} L + 1}. 
    \end{align*}
    Obviously, $B=1$ is the optimal solution to the above problem since $x e^{-x}$ is maximized at $x=1$ for $x\ge 0$.
    The optimal $u_i^\star$ is given by \eqref{eq:uistar(b)}.
    Invoking \Cref{lem:loss perturbation}, we have the desired result for the first argument.

    For the second argument, we define $b_i = \baromega_i^2 d_i$ and $B= \sum_{i=1}^I b_i$.
    Under the perturbation of $\baromega$, we have for task $i$ that
    \begin{align*}
        \left|\frac{d_i\phi_i e^B/L}{b_i + d_i\phi_i e^B/L} - \frac{d_i\phi_i e^{B^\star}/L}{b_i^\star + d_i\phi_i e^{B^\star}/L} \right| 
        &\le \frac{d_i\phi_i e^{B^\star}/L}{b_i^\star + d_i\phi_i e^{B^\star}/L} \cdot \left| \frac{b_i^\star e^{B-B^\star}}{b_i} - 1\right| \le \frac{d_i\phi_i e^{B^\star}/L}{b_i^\star + d_i\phi_i e^{B^\star}/L} \cdot O(\delta_\omega + \xi_\omega^2 d). 
    \end{align*}
    and for task $j\neq i$ that 
    \[
        0 \le \frac{\lambda_j d_j }{d} \cdot \frac{d_j\phi_j e^B/L}{b_j + d_j\phi_j e^B/L} \le \frac{\lambda_j d_j }{d} = 0. 
    \]
    In addition, we have 
    \begin{align*}
        |u_i^\star(b) - u_i^\star| &= \left| \frac{\sqrt{b_i d_i}}{b_i + d_i\phi_i e^B/L} - \frac{\sqrt{b_i^\star d_i}}{b_i^\star + d_i\phi_i e^{B^\star}/L}\right| \nend
        &\le u_i^\star \cdot \left|\frac{\sqrt{b_i} \left(\sqrt{b_i^\star} - \sqrt{b_i} + d_i \phi_i L^{-1} \cdot \left(e^{B^\star}/\sqrt{b_i^\star} - e^B / \sqrt{b_i}\right)\right)}{b_i + d_i\phi_i e^B/L} \right| \le u_i^\star \cdot O(\delta_\omega + \xi_\omega^2 d).
    \end{align*}
    For the perturbation of $u$, we have for task $i$ that
    \begin{align*}
        \left| u_i - u_i^\star(b) \right| \le \left| u_i - u_i^\star \right| + \left| u_i^\star(b) - u_i^\star \right| \le u_i^\star \cdot O(\delta_\omega + \delta_u + \xi_\omega^2 d), 
    \end{align*}
    and for task $j\neq i$ that
    \begin{align*}
        \left| u_j - u_j^\star(b) \right| \le \xi_u \pmax \frac{\sqrt{b_j d_j}}{b_j + d_j\phi_j e^B/L} = \xi_u \pmax 0 = \xi_u.
    \end{align*}
    Plugging these error terms into \Cref{lem:loss perturbation}, we have 
    \begin{align*}
        \cL(U, W) 
        &\le \cL_\simple(B^\star, b^\star) + \frac{\lambda_i d_i}{d} \cdot \frac{d_i\phi_i e^{B^\star}/L}{b_i^\star + d_i\phi_i e^{B^\star}/L} \cdot O(\delta_\omega + \xi_\omega^2 d) + \sum_{k=1}^I \frac{d_k \lambda_k }{d} \cdot O(L^{-(1-\epsilon)/2}) \nend 
        &\qquad + \left(\lambda_i  b_i +  \frac{(d_i \lambda_i  + \sigma^2 d ) e^B}{L} \right) d^{-1} (u_i^\star)^2 \cdot O(\delta_\omega^2 + \delta_u^2 + \xi_\omega^4 d^2)  + \sum_{j\neq i} \frac{\sigma^2 e^B}{L} \xi_u^2 \nend 
        &\le \cL_\simple(B^\star, b^\star) + \frac{\lambda_i d_i}{d} \cdot O(\delta_\omega + \xi_\omega^2 d + \delta_u^2) + \sum_{k=1}^I \frac{d_k \lambda_k }{d} \cdot O(L^{-(1-\epsilon)/2}) + \frac{\sigma^2 I}{L} \cdot O(\xi_u^2).
    \end{align*}
    This completes the proof.
\end{proof}

\begin{lemma}[Restatement of \Cref{lem:msa icl loss-mainbody}]
\label[lemma]{lem:msa icl loss}
    The \ac{icl} loss under the convergence point described for the \ac{msa} in \Cref{thm:convergence-multi-head-symmetric}, which we define as $\cL_\convergence$,  is upper bounded by \[
    \cL_\convergence \le \sum_{i=1}^I \frac{\lambda_i d_i}{d}\cdot \left(\frac{e d_i \phi_i L^{-1}}{1 + e d_i \phi_i L^{-1}} +  O(L^{-(1-\epsilon)/2} + \omega_0^2 d + \delta)\right). 
    \]
\end{lemma}
\begin{proof}[Proof of \Cref{lem:msa icl loss}]
    Same as before, we consider $\varPhi$ and $\varPsi$ to be identity matrices. 
    Define 
    \[
        G_i = \diag(\underbrace{0, \: \dots, \: 0}_{i-1}, \: g_i, \: \underbrace{0, \: \dots, \: 0}_{I - i})
        \]
    as the coefficient matrix with only the $i$-th block of $G$.
    Note that for Decomposable weights we always have $U^\h$ and $W^\h$ satisfying the form in \eqref{eq:Decomposable W,U} with 
    \begin{align*}
        W^\h = \begin{bNiceArray}{cccc|c}
            \baromega_1^\h\cdot I_{d_1} & 0 & \cdots & 0 & \Block{4-1}<\large>{\star}\\
            0 & \baromega_2^\h\cdot I_{d_2} & \cdots & 0 & \\
            \vdots & \vdots & \ddots & \vdots & \\
            0 & 0 & \cdots & \baromega_I^\h \cdot I_{d_I} & \\
            \hline 
            \Block{1-4}<\large>{0} & & & & \star
        \end{bNiceArray}, \quad 
        U^\h =\begin{bNiceArray}{c|cccc}
            0 & \Block{1-4}<\large>{0} & & & \\
            \hline 
            \Block{4-1}<\large>{0} & \mu_1^\h & 0 & \cdots & 0\\
             & 0 & \mu_2^\h & \cdots & 0\\
             & \vdots & \vdots & \ddots & \vdots\\
             & 0 & 0 & \cdots & \mu_I^\h
        \end{bNiceArray}.
    \end{align*}
    With $W_Y^\h=0$ and $U_X^\h=0$, the \ac{icl} loss for \ac{msa} is given by
    \begin{align*}
        \cL(\{U^\h, W^\h\}_{h\in[H]}) &= \EE\left[
            \left\| G^\top q - \sum_{h=1}^H U_Y^\h (G^\top X + \varepsilon^\top) p^\h \right\|_2^2
        \right] \nend 
        & = \sum_{i=1}^I \EE\left[
            \left(g_i^\top q - \mu_i^\star (g_i^\top X_{(i)} + \varepsilon_i^\top) p_i^\star + \sum_{h\neq h_i^\star} \mu_i^\h (g_i^\top X_{(i)} + \varepsilon_i^\top) p^\h\right)^2
        \right] \nend 
        & \le \sum_{i=1}^I \left(\cL_i^{(1)} + 2 \sqrt{\cL_i^{(1)} \cdot I^2\cL_i^{(2)}} + I^2 \cL_i^{(2)}\right)
    \end{align*}
    Here, we have  kinds of errors $\cL_i^{(1)}$ and $\cL_i^{(2)}$ in the above loss and the last inequality is due to the Cauchy-Schwarz inequality.
    The first is the error of the optimal head on each task, which is given by 
    \begin{align*}
        \cL_i^{(1)} &\defeq \EE\left[\left( g_i^\top q_{(i)} - \mu_i^\star (g_i^\top X_{(i)} + \varepsilon_i^\top ) p_i^\star\right)^2\right]. 
    \end{align*}
    In this case, we can think as having only one head and with only one nonzero task $i$.
    In addition, we have $u_j = 0$ for $j\neq i$ for this only head. 
    Invoking the second argument in \Cref{cor:single task icl loss} with $\delta_\omega = \delta_u = O(\delta)$ (convergence of the optimal heads), $\xi_u = 0$ ($u_j = 0$), and $\xi_\omega = O(\omega_0)$ (nonoptimal head $h\neq h_i^\star$ has $\baromega_i^\h$ fixed at initialization), we have that the error for each task is upper bounded by
    \begin{align}
        \cL_i^{(1)} \le \frac{\lambda_i d_i}{d}\cdot \left(\frac{e d_i \phi_i L^{-1}}{1 + e d_i \phi_i L^{-1}} + O(\delta + \omega_0^2 d + L^{-(1-\epsilon)/2})\right) \le O(1).
        \label{eq:msa convergence icl loss-1}
    \end{align}
    The second kind of error is the error coming from the \say{leakage} of the nonoptimal heads, which is given by
    \begin{align*}
        \cL_i^{(2)} &\defeq \max_{h\neq h_i^\star} \cL_i^{(2, h)}, \where \cL_i^{(2, h)}\defeq \EE\left[\left(\mu_i^\h (g_i^\top X_{(i)} + \varepsilon_i^\top) p^\h\right)^2\right] 
        .
    \end{align*}
    We can upper bound this error for each $h\neq h_i^\star$ by
    \begin{align*}
        \cL_i^{(2, h)} &= (\mu_i^\h)^2 \cdot \left(\lambda_i d^{-1} \trace(\EE[X_{(i)} p^\h {p^\h}^\top X_{(i)}^\top]) + \sigma^2 \EE[\norm{p^\h}_2^2] \right) \nend
        &\le (\mu_i^\h)^2 \cdot \EE\left[\lambda_i d^{-1} \norm{W_X^\h q}_2^2  f_2(\norm{W_X^\h q}_2^2) + \sigma^2 f_1(\norm{W_X^\h q}_2^2)\right] \nend 
        &\le (\mu_i^\h)^2 \cdot \left(\lambda_i d^{-1} \cdot O(1)  + \sigma^2 \cdot O(L^{-(1-\epsilon)}) \right), 
    \end{align*}
    where in the second inequality we replace the trace on the $i$-th block to the trace on the whole matrix as an upper bound, and in the last inequality, we use the fact that $\trace(W_X^\h {W_X^\h}^\top) = 1 + o(1)$ at the convergence point. 
    Combining the above results, we conclude that 
    \begin{align*}
        \cL(\{U^\h, W^\h\}_{h\in[H]}) &\le \sum_{i=1}^I \left(\cL_i^{(1)} + 2 \sqrt{\cL_i^{(1)} \cdot I^2\cL_i^{(2)}} + I^2 \cL_i^{(2)}\right) \nend 
        &\le \sum_{i=1}^I \frac{\lambda_i d_i}{d}\cdot \left(\frac{e d_i \phi_i L^{-1}}{1 + e d_i \phi_i L^{-1}} + O(\delta + \omega_0^2 d + L^{-(1-\epsilon)/2})\right) \nend 
        &\qquad +  O(I^2) \cdot \max_{i, h\neq h_i^\star}\mu_i^\h  \cdot  \sqrt{\lambda_i d^{-1} \cdot O(1)  + \sigma^2 \cdot O(L^{-(1-\epsilon)}) }. 
    \end{align*}
    Note that $\mu_i^\h = \mu_i^\star \exp(-O(\sqrt{d} d_i \phi_i L^{-1}))$,  which is sufficiently small in scale. 
    Hence, we just need to consider the first term, which gives the desired result.
\end{proof}

\subsection{Lower Bound for Multi-Head Attention}
\label{sec:lowerbound_multihead}
In this section, we aim to prove the following optimality result for multihead attention within the class of equiangular weights.
We consider \ac{msa} with $H$ heads and $I$ tasks. Let $\baromega=\{\baromega_i^\h\}_{h\in[H]}$ and $\mu = \{\mu_i^\h\}_{h\in[H]}$, and the loss is defined as
\begin{align*}
    \cL(\mu, \baromega) = \EE\left[\left\|G^\top q - \sum_{h=1}^H U_Y^\h (G^\top X + \varepsilon) p^\h\right\|_2^2\right], \where p^\h = \softmax\left(X^\top W_X^\h q\right).
\end{align*}
\begin{theorem}[Optimality of Multi-Head Attention]
\label[theorem]{thm:optimality of msa}
Suppose we have $H\ge 2$ heads and $I$ tasks and consider \ac{msa} with equiangular weights.
Assume that all the tasks are homogeneous, i.e., $d_i = \bard \defeq d/H$ and $\lambda_i = \lambda$ for all $i\in[H]$. 
Consider the regime
 where 
 \begin{align*}
    \norm{\baromega^\h}_\infty \le \sqrt{2 \log L / 3 d c^2}, \quad \norm{\mu^\h}_\infty \le L^{3/4 - \epsilon/2}
    \quad \forall h\in[H],
\end{align*}
where $c$ is a constant satisfying 
\begin{align*}
    \frac{1}{c} + \frac{3}{1 + \sqrt{1+c^2/2}} \le \epsilon.
\end{align*}
The \ac{icl} loss for \ac{msa} is lower bounded by the maximum of
\begin{align*}
    \frac{\lambda}{ \phi^{-1} d^{-1} L \cdot (H-1) +1 } - O(L^{-\epsilon/2}), 
\end{align*}
and the Bayesian risk 
\begin{equation}
    \begin{aligned}
        \mathrm{Bayesian~Risk} & = \mathrm{Variance} + \mathrm{Bias}, \where  \\
        \mathrm{Variance} &= I \sigma^2 \cdot \frac{br + (1 + r) - \sqrt{(br - 1 + r)^2 + 4 b}}{2 \sqrt{(br- 1 + r)^2 + 4 b}}, \\
        \mathrm{Bias} &= \lambda \cdot \left(\frac{br (1+r) + (1 - r)^2 - |1-r|\sqrt{(br - 1 + r)^2 + 4 b r}}{2 \sqrt{(br- 1 + r)^2 + 4 b}} + \left(1 - \frac 1 r\right) \ind (r > 1)\right). 
    \end{aligned}
    \notag
\end{equation}
\end{theorem}

\begin{proof}[Proof of \Cref{thm:optimality of msa}]
Under the decomposability assumption, we have the loss function as 
\begin{align*}
    \cL(\mu, \baromega) &= \EE\left[
        q^\top G G^\top q - 2 q^\top G \sum_{h\in[H]}U_Y^\h (G^\top X + \varepsilon) p^\h \right.\nend 
        &\qqquad \left.+ \sum_{h, h'\in[H]} {p^\h}^\top (X^\top G + \varepsilon^\top) {U_Y^\h}^\top U_Y^\hprime (G^\top X + \varepsilon) p^\hprime
    \right] \nend 
    &=\lambda - \frac{2 \lambda}{d} \sum_{h\in[H]} \trace\left((\diag(\mu^\h)\kron I_{\bard} ) \EE[X p^\h q^\top]\right) \nend 
    &\qqquad + \frac{\lambda}{d} \sum_{h, h'\in[H]} \trace\left((\diag(\mu^\h\odot \mu^\hprime)\kron I_{\bard} ) \EE[X p^\h {p^\hprime}^\top X^\top ]\right) \nend
    &\qqquad + \sum_{h, h'\in[H]} \sigma^2 \langle \mu^\h, \mu^\hprime\rangle \EE[{p^\h}^\top p^\hprime] \nend 
    &= \lambda - \frac{2 \lambda}{I} \sum_{h\in[H]} \EE\left[\langle \mu^\h, \baromega^\h \odot  v^2 \rangle (1 - f_3^\handh)\right] \nend 
    &\qqquad + \frac{\lambda}{I} \sum_{h, h'\in[H]} \EE\left[\bigl\langle \mu^\h \odot \mu^\hprime, \baromega^\h\odot \baromega^\h \odot v^2  f_1^\handhprime + \baromega^\hprime \odot \baromega^\hprime \odot v^2 f_1^\hprimeandh \bigr\rangle\right]  \nend
    &\qqquad + \frac{\lambda}{I} \sum_{h, h'\in[H]} \EE\left[\bigl\langle \mu^\h \odot \mu^\hprime, \baromega^\h\odot \baromega^\hprime \odot v^2  f_2^\handhprime \bigr\rangle + \left(1+\frac{\sigma^2 I}{\lambda}\right)\langle \mu^\h, \mu^\hprime\rangle f_3^\handhprime\right]
\end{align*}
where $\kron$ denotes the Kronecker product, and $\odot$ denotes the element-wise product.
Here, we denote by $v^2$ a $d$-dimensional random vector where each entry is drawn $\iid$ from the chi-square distribution with degree $1$. 
We use the following definition for the last equality, 
\begin{align*}
    f_1^\handhprime & \defeq \EE\left[(p^\hprime)^\top (p^\h)^{\odot 2} - (p^\hprime)^\top p^\h \norm{p^\h}_2^2 \given q \right]\nend 
    f_2^\handhprime & \defeq \EE\left[(1 - \norm{p^\hprime}_2^2) (1 - \norm{p^\h}_2^2) + (p^\h)^\top p^\hprime - (p^\hprime)^\top (p^\h)^{\odot 2} - (p^\hprime)^\top (p^\h)^{\odot 2} + ({p^\h}^\top p^\hprime)^2\given q \right] \nend 
    f_3^\handhprime & \defeq \EE\left[(p^\h)^\top p^\hprime\right],
\end{align*}
where the relationship between these terms to 
$\EE[X p]$ and 
$\EE[X p^\h {p^\hprime}^\top X^\top ]$ is given by \Cref{fact:derivatives for B terms}. 
As a matter of fact, both $f_1$ and $f_2$ are functions of $q$. 
Note that $f_1^\handhprime$ is actually a higher order term, and according to our regime where $\norm{\baromega^\h}_\infty \le \sqrt{2 \log L / 3 d c^2}$, we have all the conditions in \Cref{lemma: higher-order-moment} and \Cref{lemma: E[pq]-moment} satisfied and thus 
\begin{gather*}
    |\EE[v_i^2 \cdot f_1^\handhprime]| \le \sqrt{\EE[v_i^4] \cdot \EE[(f_1^\handhprime)^2]} = O(L^{-2(1-\epsilon)}), \\
    |\EE[v_i^2 \cdot f_2^\handhprime] - 1| \le \sqrt{\EE[v_i^4] \cdot \EE[(f_2^\handhprime -1)^2]} = O(L^{-1}), \\
    \left|\EE[f_3^\handhprime] - \frac{\exp(\bard\langle \baromega^\h, \baromega^\hprime \rangle)}{L}\right| \le \sqrt{\EE\left[\left(f_3^\handhprime -\frac{\exp(\bard\langle \baromega^\h, \baromega^\hprime \rangle)}{L}\right)^2\right]} = O(L^{-(3-\epsilon)/2}). 
\end{gather*}
Note that it is hard to optimize for $\mu$ for each individual head. 
However, we note that by the \DC, we can instead optimize for each task separately.
In particular, the individual loss of task $i$ is given by
\begin{align*}
    \cL_i(\mu, \baromega) 
    &\defeq 1 -  2 \mu_i^\top \baromega_i   +  \mu_i^\top (\baromega_i \baromega_i^\top + B) \mu_i .
\end{align*}
The $B\in\RR^{H\times H}$ matrix in the above equation is given element-wise by
\begin{align*}
    (B_i)_{h h'} \defeq \phi_i L^{-1} \exp(\bard \langle \baromega^\h, \baromega^\hprime \rangle).
\end{align*}
Here, we denote by $\mu_i = (\mu_i^\h)_{h\in[H]}$ and $\baromega_i = (\baromega_i^\h)_{h\in[H]}$ the eigenvalues of the attention weights for task $i$.
Here, we drop the energy term $ {\lambda}/{I}$  in the definition. 
Note that we also ignore the higher order terms by suffering from an error of $O(L^{-\epsilon/2})$ given the scale of $\mu$ and $\baromega$.
As a result, the relationship between $\cL$ and $\cL_i$ is given by
\begin{align*}
    \cL(\mu, \baromega) = \frac{\lambda}{I}\sum_{i=1}^I \cL_i(\mu, \baromega) + O(L^{-\epsilon/2}).
\end{align*}

\paragraph{Lower Bound for the \ac{icl} Loss.}
We define an operator $\cR$ that applies the following elementwise transformation to a matrix $A$: 
\begin{align*}
    (\cR\circ A)_{m n} =  \phi_i L^{-1} \exp(\bard A_{m n}). 
\end{align*}
Let $M = (\baromega^{(1)}, \dots, \baromega^{(H)}) \in \RR^{I\times H}$. 
To this end, we can rewrite $B$ as 
\begin{align*}
    B = \cR\circ ( M^\top M).
\end{align*}
Solving for this quadratic problem, we obtain 
\begin{align}
    \mu_i^\star = (\baromega_i \baromega_i^\top + B)^{-1} \baromega_i,
    \label{eq:optimal-mu}
\end{align}
and the corresponding loss
\begin{align*}
    \cL_i(\baromega) \defeq 1 - \baromega_i^\top (\baromega_i \baromega_i^\top + B)^{-1} \baromega_i. 
\end{align*}
Note that although we restrict $\norm{\mu^\h}_\infty \le L^{3/4 - \epsilon/2}$, the optimal $\mu_i^\star$ in \Cref{eq:optimal-mu} can be much larger than $L^{3/4 - \epsilon/2}$, which may lead to a looser lower bound. 
However, as we will show later, this only leads to an additional constant multiplicative factor in the lower bound.
By property of the equiangular weights, $M^\top M$ has the same value for the diagonal elements, which we denote by $a$, and also the same value for off-diagonal elements, which we denote by $b$.
The same also holds for $B$, and we denote the value for the diagonal elements by $\tilde a$ and the value for the off-diagonal elements by $\tilde b$.
The quadratic term can be simplified as
\begin{align*}
    A_i \defeq \baromega_i \baromega_i^\top + (\tilde a - \tilde b) I + \tilde b E . 
\end{align*}
We note that 
\[
    \varUpsilon_i  = \left(\frac{\vone_H}{\sqrt H}, \:
    \baromega_i - \frac{\vone_H^\top \baromega_i}{H} \cdot \vone_H,\: \dots \right)
\]
forms an orthogonal basis. Let $q_i = \vone_H^\top \baromega_i /\sqrt H$ be the inner product of $\vone_H/\sqrt H$ and $\baromega_i$ 
and  $\tilde\omega_i = \baromega_i - q_i \cdot \vone_H/\sqrt H$ be the orthogonal projection of $\baromega_i$ onto the orthogonal complement of $\vone_H/\sqrt H$.
We thus have 
\begin{align*}
    \varUpsilon_i^\top A_i \varUpsilon_i = \begin{bNiceArray}{ccc}
        {q_i^2 + \tilde a + (H-1)\tilde b} & q_i \tomegainorm & 0\\
        q_i \tomegainorm & {\tomegainorm^2 + \tilde a -\tilde b} & 0\\ 
        0 & 0 & (\tilde a - \tilde b) I_{H-2}
    \end{bNiceArray}. 
\end{align*}
The inverse of the above matrix is given by 
\begin{align*}
    (\varUpsilon_i^\top A_i \varUpsilon_i)^{-1} = 
    \begin{bNiceArray}{ccc}
        \frac{{\tomegainorm^2 + \tilde a -\tilde b}}{\Delta} & -\frac{q_i \tomegainorm}{\Delta } & 0\\
        -\frac{q_i \tomegainorm}{\Delta} & \frac{ {q_i^2 + \tilde a + (H-1)\tilde b}}{\Delta} & 0\\ 
        0 & 0 & \frac{I_{H-1}}{\tilde a - \tilde b}
    \end{bNiceArray}.
\end{align*}
where 
\begin{align*}
    \Delta &= (q_i^2 + \tilde a + (H-1)\tilde b)(\tomegainorm^2 + \tilde a -\tilde b) - q_i^2 \tomegainorm^2 \nend 
    & = (\tilde a+ (H-1) \tilde b) \tomegainorm^2 + q_i^2 (\tilde a - \tilde b) + (\tilde a + (H-1) \tilde b)(\tilde a - \tilde b)
\end{align*}
is the determinant of top left $2\times 2$ submatrix of $\varUpsilon_i^\top A_i \varUpsilon_i$.
Therefore, 
\begin{align*}
    \baromega_i^\top A_i^{-1} \baromega_i &= (\varUpsilon_i^\top \baromega_i)^\top (\varUpsilon_i^\top A_i \varUpsilon_i)^{-1} (\varUpsilon_i^\top \baromega_i) \nend 
    & = \begin{bNiceMatrix}
        q_i & \tomegainorm & 0
    \end{bNiceMatrix} 
    \cdot 
    \begin{bNiceArray}{ccc}
        \frac{{\tomegainorm^2 + \tilde a -\tilde b}}{\Delta} & -\frac{q_i \tomegainorm}{\Delta } & 0\\
        -\frac{q_i \tomegainorm}{\Delta} & \frac{ {q_i^2 + \tilde a + (H-1)\tilde b}}{\Delta} & 0\\ 
        0 & 0 & \frac{I_{H-1}}{\tilde a - \tilde b}
    \end{bNiceArray}
    \cdot 
    \begin{bNiceMatrix}
        q_i \\ \tomegainorm \\ 0
    \end{bNiceMatrix} \nend 
    & = \frac{q_i^2 (\tomegainorm^2 + \tilde a - \tilde b) + \tomegainorm^2 (q_i^2 + \tilde a + (H-1)\tilde b) - 2 q_i^2 \tomegainorm^2 }{\Delta} \nend
    & = 1 - \frac{(\tilde a - \tilde b)(\tilde a + (H-1)\tilde b)}{\Delta} .
\end{align*}
As a result, we have
\begin{align}
    &\frac{1}{I}\sum_{i=1}^I \cL_i(\baromega) \nend
    &\quad= \frac{1}{I}\sum_{i=1}^I \frac{(\tilde a - \tilde b)(\tilde a + (H-1)\tilde b)}{(\tilde a+ (H-1) \tilde b) \tomegainorm^2 + q_i^2 (\tilde a - \tilde b) + (\tilde a + (H-1) \tilde b)(\tilde a - \tilde b)} \nend 
    &\quad\ge \frac{(\tilde a - \tilde b)(\tilde a + (H-1)\tilde b)}{(\tilde a+ (H-1) \tilde b)  \cdot \frac{1}{I}\sum_{i=1}^I \tomegainorm^2 + \frac{1}{I}\sum_{i=1}^I q_i^2 (\tilde a - \tilde b) + (\tilde a + (H-1) \tilde b)(\tilde a - \tilde b)} 
    \label{eq: lower bound for L_i}
    \\
    &\quad = \frac{(\tilde a - \tilde b)(\tilde a + (H-1)\tilde b)}{(\tilde a+ (H-1) \tilde b)  \cdot I^{-1} (H-1)(  a -  b) + I^{-1}(a + (H-1)b)(\tilde a - \tilde b) + (\tilde a + (H-1) \tilde b)(\tilde a - \tilde b)} \notag,
\end{align}
where in the inequality we use the Jensen's inequality for the convex function $x\mapsto 1/(1+x)$.
In the last equality, we note that 
\[
    \sum_{i=1}^I \norm{\baromega_i}_2^2 = \trace(M^\top M) = H a, \quad \sum_{i=1}^I q_i^2 = \sum_{i=1}^I \norm{\vone_H^\top \baromega_i /\sqrt H}_2^2 = H^{-1} \trace(E M^\top M) = a + (H-1) b. 
\]
As a result, $I^{-1}\sum_{i=1}^I \tomegainorm^2 = I^{-1} \sum_{i=1}^I \norm{\baromega_i}_2^2 - I^{-1}\sum_{i=1}^I q_i^2 = I^{-1}(H-1)(a-b)$ and $I^{-1}\sum_{i=1}^I q_i^2 = I^{-1}(a + (H-1) b)$.
By a direction calculation, we also have
\begin{align*}
    \frac{1}{I} \sum_{i=1}^I \cL_i(\baromega) 
    &\ge \frac{1}{ \frac{(H-1)( a- b)}{I (\tilde a -\tilde b)} + \frac{a+(H-1)b}{I(\tilde a + (H-1)\tilde b)} + 1} \nend 
    &= \frac{1}{ \tilde b^{-1}\cdot \frac{(H-1)( a- b)}{I (\tilde a /\tilde b - 1)} + \tilde b^{-1} \cdot \frac{a+(H-1)b}{I(\tilde a/\tilde b + (H-1))} + 1}.
\end{align*}
Now, we plug in the definition $\tilde a = \phi L^{-1} \exp(\bard a)$ and $\tilde b = \phi L^{-1} \exp(\bard b)$, and we have
\begin{align*}
    \frac{1}{I} \sum_{i=1}^I \cL_i(\baromega)
    &\ge \frac{1}{ I^{-1}\phi^{-1} L \exp(-\bard b) \cdot \left(\frac{(H-1)( a- b)}{\exp(\bard (a-b))- 1} +  \frac{a+(H-1)b}{\exp(\bard (a-b)) + H-1} \right) + 1}. 
\end{align*}
To this end, we define $x= \bard (a-b)$ and $y = \bard b$.
Our target is thus to optimize: 
\begin{align*}
    \max_{x, y} g(x, y) \defeq \exp(-y) \cdot \left(\frac{(H-1)x}{\exp(x)- 1} +  \frac{x + Hy}{\exp(x) + H-1} \right).
\end{align*}
However, there are constraints we need to consider, given the nonnegativity of the energy: 
\begin{align*}
    \sum_{i=1}^I \tomegainorm^2 = (H-1)(a-b)\ge 0 \quad &\Rightarrow  & x \ge 0.\\
    \sum_{i=1}^I q_i^2 = a + (H-1)b \ge 0 \quad &\Rightarrow & x + H y \ge 0. 
\end{align*}
We first optimize $y$ given $x$.
The partial derivative of $g(x, y)$ with respect to $y$ is given by
\begin{align*}
    \frac{\partial g(x, y)}{\partial y} = C(x,y) \cdot \left(1 + (H-2) x - y e^x (x + y -1)\right), 
\end{align*}
where $C:\RR\times\RR\mapsto \RR_+$ is a positive function. 
Let 
\[
    y_0 = - \frac{e^x(x-1) + 1 + (H-2) x}{e^x -1}. 
\]
Note that the above function increases when $y\le y_0$ and decreases when $y\ge y_0$.
Our constraint also requires that $y\ge -x/H$.
We just need to decide which one is larger between $-x/H$ and $y_0$. 
Suppose $y_0 > - x/H$. 
With a little bit of algebra, we have this condition equivalent to
\[
    \frac{H}{H-1} > \left(1 + \frac{H}{e^x -1}\right) x 
\]
given the nonnegativity of $x$.
Note that the left hand side decreases to $1$ as $H$ increases, and the right hand side increases. Thus, we just need to check the case when $H=2$, and we consider another function 
\[
    g_1(x) = x + \frac{2x}{e^x -1} -2.
\]
Notably, this function is increasing as $x \ge 0$, and the minimal value is achieved at $x=0$ which gives $g_1(0) = 0$.
Thus, we have $y_0 > - x/H$ impossible for $H\ge 2$.
As a result, we conclude that the optimal $y$ should always be 
$
    y^\star(x) \defeq -\frac{x}{H}.
$
The remaining is just a simple plug in, which gives 
\begin{align*}
    g(x, y^\star) = \exp\left(\frac{x}{H}\right) \cdot \frac{(H-1)x}{\exp(x)- 1}, 
\end{align*}
where the optimal $x$ is given by finding the maximum of the above function when $x \ge 0$. 
Let the optimal $x$ be $x^\star$.
Clearly $0 < x^\star \le 1$ as for $x=0$, $\exp(x/H)x$ has a faster growing rate than $\exp(x)$ in a sufficiently small region. For the other side, it is not hard to see that function decreases for $x \ge 1$. 
We have the loss lower bounded by 
\begin{align}
    \cL(\mu, \baromega) &\ge  \frac{\lambda}{I}\sum_{i=1}^I \cL_i(\mu, \baromega) - O(L^{-\epsilon/2}) \nend 
    &\ge \frac{\lambda}{ I^{-1}\phi^{-1} \bard^{-1} L \cdot \max_{x\ge 0}\frac{\exp(x/H) (H-1)x}{\exp(x)-1} + 1} - O(L^{-\epsilon/2}) \nend
    &= \frac{\lambda}{ \phi^{-1} d^{-1} L \cdot \max_{x\ge 0}\frac{\exp(x/H) (H-1)x}{\exp(x)-1} + 1} - O(L^{-\epsilon/2}). 
    \label{eq:lower-bound-multihead-asi}
\end{align}
Note that the loss in \eqref{eq:lower-bound-multihead-asi} cannot be better than the loss achieved by the Bayesian posterior mean estimator. 
Let $b = \sigma^2 /\lambda $. 
The Bayesian risk in our case is just the \ac{mmse}. 
Asymptotically, i.e., $d\rightarrow \infty$ and $L \rightarrow \infty$ with $\bard/L\rightarrow r$,  the \ac{mmse} is given by
\begin{equation}
\begin{aligned}
    \mathrm{Bayesian~Risk} & = \mathrm{Variance} + \mathrm{Bias}, \where  \\
    \mathrm{Variance} &= I \sigma^2 \cdot \frac{br + (1 + r) - \sqrt{(br - 1 + r)^2 + 4 br}}{2 \sqrt{(br- 1 + r)^2 + 4 br}}, \\
    \mathrm{Bias} &= \lambda \cdot \left(\frac{br (1+r) + (1 - r)^2 - |1-r|\sqrt{(br - 1 + r)^2 + 4 b r}}{2 \sqrt{(br- 1 + r)^2 + 4 br}} + \left(1 - \frac 1 r\right) \ind (r > 1)\right). 
\end{aligned}
\owntag{MMSE \& Ridge}
\label{eq:asymptotic-bayesian-risk}
\end{equation}
where $r = \bard/L = d/(I L)$.
Here, characterizing the asymptotical behavior of the \ac{mmse} is just a matter of calculating the following variance and bias term in the limit $\bard/L \rightarrow r$:
\begin{align*}
    \mathrm{Variance} &\defeq \sum_{i=1}^I \sigma^2 \EE\left[\trace\left(
        \left(X_{(i)} X_{(i)}^\top + b d I_{\bard}\right)^{-1} X_{(i)} X_{(i)}^\top \left(X_{(i)} X_{(i)}^\top + b d I_{\bard}\right)^{-1}\right)
    \right], \nend 
    \mathrm{Bias} &\defeq \sum_{i=1}^I \frac{\lambda}{d} \cdot \EE\left[\trace\left(
        \left(\left(X_{(i)} X_{(i)}^\top + b d I_{\bard}\right)^{-1} X_{(i)} - I_{\bard}\right) \left(\left(X_{(i)} X_{(i)}^\top + b d I_{\bard}\right)^{-1} X_{(i)} - I_{\bard}\right)^\top\right)
    \right].
\end{align*}
Let $\nu$ be the Marchenko-Pastur distribution \citep{marchenko1967distribution} for the eigenvalues of $X_{(i)} X_{(i)}^\top / L$, which is the same for different $i$ since we assume the homogeneity of the tasks. 
Therefore, we have
\begin{align*}
    \mathrm{Variance} = \sigma^2 d \cdot \EE_{s\sim \nu}\left[\frac{L s}{(L s + b d)^2}\right], 
    \mathrm{Bias} = \lambda \cdot \EE_{s\sim \nu}\left[\frac{(bd)^2}{(L s + b d)^2}\right], 
\end{align*}
which gives the result in \eqref{eq:asymptotic-bayesian-risk}. 
In particular, if $r \rightarrow 0$, i.e., large sequence length, the \ac{mmse} is approximately given by 
This completes the proof.
\end{proof}

\paragraph{Discussion of the \ac{icl} Loss Lower Bound.}
Note that throughout the proof of the lower bound, the first inequality we are using is \eqref{eq: lower bound for L_i}, where equality holds if
\[
    \norm{\tilde\omega_i}_2^2 = \norm{\tilde\omega_j}_2^2 \quad \text{for all } i, j\in[I] \quad \text{and} \quad q_i^2 = q_j^2 \quad \text{for all } i, j\in[I], 
\]
and the second optimality condition is $y^\star = -x^\star/H$ with $x^\star$ solved by $\argmax_{x\ge 0}\frac{\exp(x/H) (H-1)x}{\exp(x)-1}$.
Note that $x +  Hy = 0 \Rightarrow a + (H-1)b = 0 \Rightarrow \sum_{i=1}^I q_i^2 = 0$, which means that the projection of $\baromega_i$ onto $\vone_H/\sqrt H$ is zero for all $i\in[I]$. 
Therefore, we just need to ensure $\norm{\baromega_i}_2^2 = \norm{\baromega_j}_2^2$ for all $i, j\in[I]$.
Consider matrix $M = (\baromega^{(1)}, \dots, \baromega^{(H)}) \in \RR^{I\times H}$.
Our condition is just saying that 
\begin{myenumi}
    \item the row sum of $M$ is $0$. 
    \item the row norm of $M$ is the same for all rows. 
    \item the columns of $M$ form an equiangular system.
    \item $\trace(M^\top M) = (H-1) x^\star$. 
\end{myenumi}
However, observant readers might have noticed that the optimal $x^\star$ is always $0$ for $H\ge 2$. 
This is not realistic as $\trace(M^\top M) =0$ just means $W=0$, which is not learning anything.
We remark that this issue is due to our approximation of the loss function's landscape and allowing violation of the $\norm{\mu^\h}_\infty \le L^{3/4 - \epsilon/2}$ condition in \eqref{eq:optimal-mu}.
However, we remark that \eqref{eq:lower-bound-multihead-asi} still gives us a correct lower bound on the \ac{icl} rate, where the difference between letting $x\rightarrow 0$ and $x = \Theta(1)$ is just a loss of a constant multiplicative factor as we can clearly see from \eqref{eq:lower-bound-multihead-asi}. 
Although condition (iv) is not realistic, the remaining conditions are indeed realistic and achievable by the gradient flow of the \ac{msa} as we will show in the experiment.


\subsection{Comparison}
\label{sec:icl loss comparison}
In the following discussion, we neglect some error terms for simplicity and only consider the regime where $L/d$ beats the $\snr$, i.e., $r = \bard/L =o(\lambda/\sigma^2)$. 
Note that when we are having a single head to solve the multitask problem, the optimal rate given by \Cref{thm:global-optimality-single-head-multitask} is
\begin{equation}
    \frac{\lambda}{e^{-1} \phi^{-1} d^{-1} L + 1} \approx  e (\lambda + \sigma^2 I) \cdot \frac{d}{L}, 
    \owntag{\SingleheadICLOpt}
    \label{OPT-single head}
\end{equation}
while the lower bound in \eqref{eq:lower-bound-multihead-asi} is given by 
\begin{equation}
    \frac{\lambda}{ \phi^{-1} d^{-1} L \cdot (H-1) + 1} \approx (H-1)^{-1} \cdot (\lambda + \sigma^2 I) \cdot \frac{d}{L}. 
    \owntag{\MultiheadICLLB}
    \label{LB-multihead}
\end{equation}
The \ac{icl} loss achieved by the convergence point of the gradient flow of the \ac{msa} with $H\ge I$ characterized by \Cref{lem:msa icl loss-mainbody}, 
\begin{equation}
    \frac{\lambda}{I e^{-1}\phi^{-1} d^{-1} L + 1} \approx e I^{-1} \cdot (\lambda + \sigma^2 I) \cdot \frac{d}{L}.
    \owntag{\ConvergenceICL}
    \label{GF-multihead}
\end{equation}
Moreover, the Bayesian risk is given by \eqref{eq:asymptotic-bayesian-risk}, which for $r = \bard/L \le 1$ is approximately given by the variance term 
\begin{equation}
    \frac{I\sigma^2 r}{1 + (b-1)r} = \frac{I \sigma^2}{I \cdot d^{-1} L + \sigma^2/\lambda - 1} \approx \sigma^2 \cdot \frac{d}{L}.
    \tag{MMSE}
    \label{Bayesian risk}
\end{equation}
In addition, when using a ridge regression with regularization parameter $bd$ for each task, i.e., 
\[
    \hat g_i = \argmin_{g\in \RR^{\bard}} \left\|Y_i - g^\top X_{(i)}\right\|_2^2 + b d \cdot \norm{g}_2^2, 
\]
and use the minimizer as the estimator for the query,
we still have the rate given by \eqref{eq:asymptotic-bayesian-risk} ($b$ functions as the regularization parameter).
Notably, as the number of tasks $I$ increases, the \ac{icl} loss achieved by the gradient flow of the \ac{msa} is getting closer to the \ac{mmse} but only with a constant multiplicative factor gap.
There is clear gap between the single head and the multihead \ac{icl} loss. 
\begin{figure}[ht]
    \centering
    \includegraphics[width=0.4\textwidth]{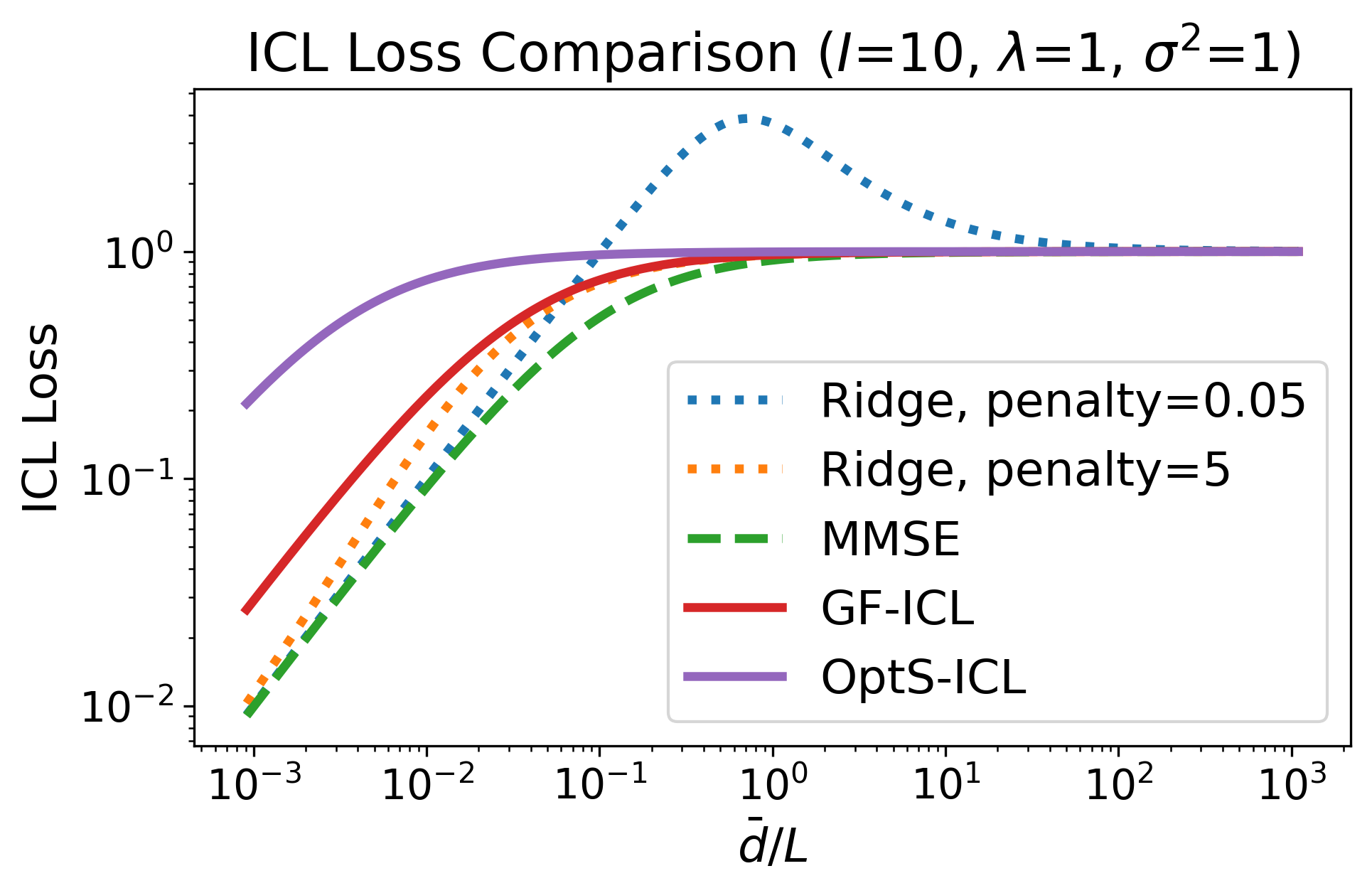}
    \caption{Comparison of loss functions}
\end{figure}


\newpage
\section{Generalization}
\label{sec:generalization-proof}
In this section, we provide generalization results both to a new sequence length and nonlinear tasks. 
Before we proceed, we first introduce the basic definition of multivariate Hermite polynomials and some additional notations.
\subsection{Length Generalization}
\label{sec:length-generalization}
We first consider generalization of the optimal single-head softmax attention to a new sequence length $\tilde L$.
\begin{proof}[Proof of \Cref{cor:length generalization}]
    The proof is a combination of \Cref{cor:single task icl loss} and the proof of \Cref{lem:msa icl loss}.
    Note that for a new sequence length $\tilde L$, the optimal $W_X$ remains unchanged. Also, we can equivalently view $U_Y$ as a perturbation of the optimal $U_Y$ under $\tilde L$.
    Let $\baromegaistar$ and $\muistar$ be the optimal (nonzero) eigenvalues for $W_X^{\histar}$ and $U_Y^\histar$ at task $i$'s position under the new sequence length $\tilde L$.
    Let $\baromega_i^\h$ and $\baromega_i^\h$ be the eigenvalues at the convergence point of the gradient flow trained under the original sequence length $L$.
    In particular, 
    \begin{gather*}
        \baromegaistar = d_i^{-1/2}, \quad \muistar = \frac{\sqrt{d_i}}{1 + e \phi_i d_i/\tilde L}, \quad 
        \left|\frac{\baromega_i^\histar}{d_i^{-1/2}} - 1 \right| \le O(\delta), \quad \left|\frac{\mu_i^\histar}{\sqrt{d_i}/(1 + e \phi_i d_i L^{-1})} - 1 \right| \le O(\delta). \\
        \baromega_i^\h = O(\omega_0), \quad \mu_i^\h = \frac{\sqrt{d_i}}{1 + e \phi_i d_i L^{-1}} \cdot \exp(- O(\sqrt d d_i \phi_i L^{-1})), 
        \quad \forall h\in [H]\setminus\{h_i^\star\}.
    \end{gather*}
    Therefore, we have for the convergence point of the gradient flow,
    \begin{gather*}
        \left|\frac{\baromega_i^\histar}{\baromegaistar} - 1\right| \le O(\delta),\quad  
        \left|\frac{\mu_i^\histar}{\muistar} - 1\right| \le O(\delta) \cdot \frac{1 + e \phi_i d_i L^{-1}}{1 + e\phi_i d_i \tilde L^{-1}} + e \phi_i d_i |L^{-1} - \tilde L^{-1}|.
    \end{gather*}
    We observe that the only difference from the previous proof is the last error in $\mu_i^\histar$.
    Therefore, the only difference is in the first kind of error concerning the optimal head previously described by \eqref{eq:msa convergence icl loss-1}. 
    In this case, we have following \Cref{cor:single task icl loss} to get \eqref{eq:msa convergence icl loss-1} replaced by
    \begin{align*}
        \cL_i^{(1)} \le \frac{\lambda_i d_i}{d}\cdot \left(\frac{e d_i \phi_i \tilde L^{-1}}{1 + e d_i \phi_i \tilde L^{-1}} + O(\delta + \omega_0^2 d + L^{-(1-\epsilon)/2}) + O ( e \phi_i d_i |L^{-1} - \tilde L^{-1}|)^2 \right) \le O(1). 
    \end{align*}
    Therefore, we have the result in \Cref{lem:msa icl loss} replaced by
    \[
        \cL_\convergence \le \sum_{i=1}^I \frac{\lambda_i d_i}{d}\cdot \left(\frac{e d_i \phi_i L^{-1}}{1 + e d_i \phi_i L^{-1}} +  O(L^{-(1-\epsilon)/2} + \omega_0^2 d + \delta) + O (\phi_i d_i |L^{-1} - \tilde L^{-1}|)^2\right). 
    \]
    Hence, we have the result in \Cref{cor:length generalization}.
\end{proof}

\subsection{Generalization to Nonlinear Task}
\label{sec:nonlinear-generalization}
In this section, we provide proof for the nonlinear generalization of the optimal single-head softmax attention for a single task.
\begin{proof}[Proof of \Cref{lem:nonlinear generalization}]
Since at the convergence point of the gradient flow, each optimal head is handling a unique task as the optimal single-head softmax attention, the generalization for the multi-head case is consequently implied by the single-head case.
To this end, we consider a nonlinear task $y = f(x)$, where $f$ is a nonlinear function.
Suppose $f$ has degree at most $D$ in the sense that $f$ is a linear combination of $d$-dimensional multivariate Hermite polynomials $\{\herm_\alpha \given \alpha\in\NN^d, \: |\alpha| = \sum_{i=1}^d \alpha_i \le D\}$ as 
\begin{align*}
    f(x) = \sum_{|\alpha|\le D} \hat f_\alpha \herm_\alpha(x), \where \hat f_\alpha\in\RR.
\end{align*}
We consider only the bias for the softmax attention and it suffices to only consider 
\begin{align*}
    \EE\left[\sum_{l=1}^L f(x_l) p_l \Biggiven q\right], \where p = \softmax(s), \quad s_l = x_l^\top W_X q.
\end{align*}
We consider each term in the decomposition of $f(x)$ separately.
For degree $\alpha$, we have 
\begin{align*}
    \EE\left[
        \sum_{l=1}^L \herm_\alpha(x_l) p_l(s) \Biggiven q
    \right]
    & = \sum_{l=1}^L \EE\left[
        (-1)^{|\alpha|} e^{\norm{x_l}_2^2/2} \frac{\partial^{|\alpha|}}{\partial x^\alpha} e^{-\norm{x_l}_2^2/2} \cdot p_l(s)
    \right] \nend 
    & = \sum_{l=1}^L \int_{x_l} (-1)^{|\alpha|}\frac{\partial^{|\alpha|}}{\partial x_l^\alpha} \frac{e^{-\norm{x_l}_2^2/2}}{\sqrt{2\pi}^d} \cdot p_l(s) \rd x_l \nend 
    & = \sum_{l=1}^L \EE\left[
        \frac{\partial^{|\alpha|}}{\partial x_l^\alpha} p_l(s) \Biggiven q
    \right] = \prod_{i=1}^d (\baromega q_i)^{\alpha_i} \cdot \sum_{l=1}^L \EE\left[
         \frac{\partial^{|\alpha| } p_l(s)}{\partial s_l^{|\alpha|}} \Biggiven q
    \right]
\end{align*}
where the last equality is by integrate by parts.
The derivative of the softmax probability with respect to the attention score will yield a polynomial of $p$ as well since $\partial p_l/\partial s_l  = p_l(1-p_l)$.
Note that since the optimal attention works in the linear operating region and we have low degree $|\alpha|\le D$, we can safely apply the low-effective-order (\Cref{def:graph-polynomial}) approximation since the coefficient before each term is constant and we have at most $D$ terms $\partial^{|\alpha|} p_l/\partial s_l^{|\alpha|}$.
Note that $p_l$ has effective order $0$, $p_l - p_l^2$ also has lowest effective order $0$ and low effective order approximation $p_l - p_l^2 \ldeq[0] p_l$.
As a matter of fact, by \Cref{fact:derivative does not decrease effective order} and \eqref{eq:low-effective-order calculation-2} in \Cref{fact:low-effective-order equivalence}, we can always ensures that 
\begin{align*}
    \frac{\partial^{|\alpha|}}{\partial s_l^{|\alpha|}} p_l \ldeq[0] p_l. 
\end{align*}
We have by the tail bound of chi-square distribution (by \eqref{eq:p-moment-chi^2 right tail}) that for any $\delta\in(0,1)$,
\begin{align}
    \PP\left(
        \norm{W_X q}_2^2  \ge \norm{W_X}_{\fro}^2 + 2 \norm{W_X}_{\fro}^2 \sqrt{\log \delta^{-1}} + 2 \norm{W_X}_{\oper}^2 \log \delta^{-1}
    \right) \le \delta.
    \label{eq:success event}
\end{align}
Since $W_X = d^{-1/2} I_d$, to ensure $\norm{W_X q}_2^2 \le c^{-2} \cdot 2\log L$ for some constant $c$ such that 
\[
    \frac{1}{c} + \frac{3}{1 + \sqrt{1+c^2/2}} = \epsilon 
\]
holds for some small constant $\epsilon$, we just need to set 
\[
    \delta = \exp\left(- (c^{-2} \log L - 1)^2  \right) \pmax \exp\left(-d/2\right).
\]
Notably, we have all terms with higher effective order upper bounded by the only term with effective order $1$, i.e., $\norm{p}_2^2$,  since $p_l \le 1$ for all $l\in[L]$.
Thus, by \Cref{lemma: pseudo-dynamics}, we have with probability at least $1-\delta$ that the terms effective order higher than $0$ in $p$ (e.g., $\EE[\norm{p}_2^2\given q], \EE[\norm{p}_3^3\given q]$) is upper bounded by $O(L^{-2(1-\epsilon)})$ given that $D$ is constant.
Thus, conditioned on the success of the event in \eqref{eq:success event}, we have
\begin{align*}
    \sum_{l=1}^L \EE\left[
        \frac{\partial^{|\alpha| } p_l(s)}{\partial s_l^{|\alpha|}} \Biggiven q
    \right] = \EE\left[\norm{p}_1 \given q\right] \pm O(L^{-2(1-\epsilon)}) = 1 \pm O(L^{-2(1-\epsilon)}).
\end{align*}
Thus, we conclude that 
\begin{align*}
    &\PP\left(\forall |\alpha|\le D, \quad \left|\EE\left[
        \sum_{l=1}^L \herm_\alpha(x_l) p_l(s) \Biggiven q
    \right] - \prod_{i=1}^d (\baromega q_i)^{\alpha_i} \right| \ge O(L^{-2(1-\epsilon)})\right) \nend 
    &\qquad \le \exp\left(- (c^{-2} \log L - 1)^2  \right) \pmax \exp\left(-d/2\right).
\end{align*}
Therefore, 
\begin{align*}
    &\PP\left( \left|\EE\left[
        \mu Y p \given q
    \right] - \mu \cdot \sum_{\alpha:|\alpha|\le D} \hat f_\alpha \baromega^{|\alpha|} \cdot q^\alpha \right| \ge O(L^{-2(1-\epsilon)})\right)  \le \exp\left(- \left(\frac{\log L}{c^2} - 1 \right)^2 \pmin \frac{d}{2}\right), 
\end{align*}
where we define $q^\alpha = \prod_{i=1}^d q_i^{\alpha_i}$.
Thus, we conclude the result in \Cref{lem:nonlinear generalization}.
\end{proof}

\newpage
\section{Auxiliary Results} \label[appendix]{sec:proof_aux_lemma}
We collect the proofs of the auxiliary results used before.

\subsection{Proof of Auxiliary Results in \S\ref{sec:simplify_AB}}

\subsubsection{Proof of \Cref{fact:expectation terms}} \label[proof]{proof:fact:expectation terms}

\begin{proof} 
    We give a proof for the first equality in \Cref{fact:expectation terms} and the remaining two can be calculated using a similar argument.
    The proof is based on the Stein's Lemma, which states that for a Gaussian random variable $x \sim N(0, I_d)$ and a differentiable function $g \colon \RR^{d} \rightarrow \RR$ with  $ \EE[ \| \nabla g(x)  \|_2]  $ finite, we have
    \[
        \EE[x \cdot g(x)] = \EE[\nabla g(X)].
    \]
    We apply the Stein's Lemma to the function $f_{lmn}$ recursively  for three times. 
    Note that $f_{lmn}$ is a function of the attention scores $s_l$ and $\tilde s_l$, and thus is a function of the covariate sequence  $\{x_l \}_{l\in [L]}$ in the ICL.  
    By chain rule, we have 
    \begin{align*}
        \EE\left[
            f_{lmn} \cdot x_l\otimes x_m\otimes x_n
        \right]
        &= \EE\left[
            \nabla_{x_l} \otimes (f_{lmn} \cdot  x_m\otimes x_n) 
        \right] \nend 
        &= \EE\left[
            \nabla_{x_l} f_{lmn} \otimes x_m\otimes x_n + \delta_{lm} f_{lmn} \cdot I_{d}\otimes x_n + \delta_{ln} f_{lmn} \cdot \left( I_{d}\otimes x_m\right)^{\top(132)}
        \right].
    \end{align*}
    Here the first equality is due to the Stein's lemma, and the second equality is due to the chain rule, where we use the fact that $\nabla _{x_l }(x_m) = \delta_{lm} I_d$. 
Moreover, applying the Stein's Lemma again, we have 
\begin{align*}
    \EE\left[
        \nabla_{x_l} f_{lmn} \otimes x_m\otimes x_n
    \right]
    &= \EE\left[
        (\nabla_{x_l} \otimes \nabla_{x_m })  f_{lmn}  \otimes x_n + \delta_{mn}\nabla_{x_l} f_{lmn} \otimes I_{d} 
    \right] \nend 
    &= \EE\left[
        (\nabla_{x_l} \otimes \nabla_{x_m} \otimes \nabla_{x_n}) f_{lmn} + \delta_{mn}\nabla_{x_l} f_{lmn} \otimes I_{d} 
    \right],  \\
    \EE\left[ \delta_{lm} f_{lmn} \cdot I_{d}\otimes x_n  \right] &  = \EE\left[
        \delta_{lm}  \cdot I_{d}\otimes \nabla_{x_n} f_{lmn} \right] , \\
        \EE \left[ \delta_{ln} f_{lmn} \cdot \left( I_{d}\otimes x_m\right)^{\top(132)} \right] & = 
        \EE\left[
            \delta_{ln} f_{lmn} \cdot \left( I_{d}\otimes \nabla_{x_m} f_{lmn}\right)^{\top(132)} \right], 
\end{align*}
where we apply the Stein's Lemma twice in the first equation. 
Therefore, we have 
    \begin{align*}
        \EE\left[
            f_{lmn} \cdot x_l\otimes x_m\otimes x_n
        \right]
        &= \EE\left[
            (\nabla_{x_l} \otimes \nabla_{x_m} \otimes \nabla_{x_n}) f_{lmn} + \delta_{mn}\nabla_{x_l} f_{lmn} \otimes I_{d} 
        \right]\nend 
        &\qquad + \EE\left[
            \delta_{lm}  \cdot I_{d}\otimes \nabla_{x_n} f_{lmn} + \delta_{ln} f_{lmn} \cdot \left( I_{d}\otimes \nabla_{x_m} f_{lmn}\right)^{\top(132)}
        \right]. 
    \end{align*}
    Notice that $x_l$ only appears in the attention scores $s_l$ and $\tilde s_l$.
This is because $s_l = x_l^{\top} W_X q$ and $\tilde s_l = x_l^{\top} \tilde W_X q$. 
As a result, we can write 
$\nabla_{x_l} = (W_X q \partial_l + \tilde W_X q \tilde\partial_l)$. 
Plugging in the definition of $\cT_l$ in \eqref{eq:define_T_l}, we obtain that $\nabla _{x_l} f_{lmn} = \cT_l \circ  f_{lmn}$, i.e., we can replace $\nabla _{x_l} $ by the operator $\cT_l  $ in the above equation.
Finally, taking the summation over $l, m, n \in [H]$, we obtain the desired result.

For the second equality, consider the second-order tensor $\{ f_{lm} \}_{l, m \in [L]}$.  By applying the Stein's Lemma twice, we have 
\begin{align*}
    \EE\left[
        f_{lm} \cdot x_l\otimes x_m \right] = 
        \EE\left[
            \nabla_{x_l} f_{lm} \otimes x_m + \delta_{lm} f_{lm} \cdot I_{d}
        \right] = \EE \left [
            (\nabla_{x_l} \otimes  \nabla_{x_m}) f_{lm} + \delta_{lm} f_{lm} \cdot I_{d}\right ].
\end{align*}
Then replacing $\nabla_{x_l} $ by $\cT_l$ and taking the summation over $l, m \in [H]$, we obtain the desired result.

Finally, the last equality can be proved by directly applying the Stein's Lemma to the function $f_{l}$, i.e., $$
    \EE [
        f_{l } \cdot x_l  ] = 
        \EE [\nabla _{x_{l}} f_{l} ].$$
Therefore, we conclude the proof. 
    \end{proof}

    \subsubsection{Proof of \Cref{fact:p is a function of Wq's 2-norm}} \label{proof:fact:p is a function of Wq's 2-norm}
    \begin{proof}
  We define $\nu = \diag(\omega) \cleverbar q$ and $\tilde \nu = \diag(\tilde\omega) \cleverbar q$. 
            Note that the joint distribution of $(s, \tilde s)\given q$ can be factorized as a tensor product of the distribution of each individual $(s_1, \tilde s_1)\given q$ since $\{ x_l\}_{l \in [L]}$ are independently sampled.
            It suffices to show that the distribution of $(s_1, \tilde s_1)\given q$ is fully determined by $\norm{\nu}_2$, $\norm{\tilde\nu}_2$ and $\langle \nu, \tilde\nu\rangle$.
            For fixed $\norm{\nu}_2$, $\norm{\tilde\nu}_2$ and $\langle \nu, \tilde\nu\rangle$, there exist constants $c_1$ and $c_2$ that depend on $\norm{\nu}_2$, $\norm{\tilde\nu}_2$ and $\langle \nu, \tilde\nu\rangle$ such that
            we can rewrite $\tilde \nu$ as $\tilde\nu = c_1 \nu_{\parallelsum} + c_2 \nu_{\perp}$, 
            where $\nu_{\parallelsum} = \nu / \norm{\nu}_2$ and $\nu_\perp$ is orthogonal to $\nu_{\parallelsum}$.
            By definition, we have 
            \begin{align*}
                s_1 &= \cleverbar x_1^\top \diag(\omega) \cleverbar q = \langle \cleverbar x_1, \nu\rangle =\norm{\nu}_2 \cdot \langle \cleverbar x_1, \nu_{\parallelsum}\rangle, \\ \tilde s_1 &= \cleverbar x_1^\top \diag(\tilde\omega) \cleverbar q = \langle \cleverbar x_1, \tilde \nu\rangle = c_1 \cdot \langle \cleverbar x_1, \nu_{\parallelsum}\rangle + c_2 \cdot \langle \cleverbar x_1, \nu_\perp\rangle,
            \end{align*}
            where $\cleverbar x_1 = \varPhi^\top x_1$ is the rotated covariate of the first context token, whose distribution is also rotationally invariant.
            Therefore, the joint distribution of $\langle \cleverbar x_1, \nu_{\parallelsum}\rangle$ and $\langle \cleverbar x_1, \nu_\perp\rangle$ does not depend on the specific choice of $\nu_{\parallelsum}$ and $\nu_\perp$.
            Thus, the joint distribution of $(s_1, \tilde s_1)\given q$ is fully determined by $(\norm{\nu}_2, c_1, c_2)$, and thus by $(\norm{\nu}_2, \norm{\tilde\nu}_2, \langle \nu, \tilde\nu\rangle)$. Thus, we conclude the proof. 
        \end{proof}

        \subsection{Proof of Auxiliary Results in \S\ref{sec:approximation_dynamics}} \label{sec:proof_approximation_dynamics}

\subsubsection{Proof of \Cref{fact:low-effective-order equivalence}} \label{proof:flow-effective-order equivalence}

\begin{proof} 
    For the last term in \eqref{eq:partial derivative operator}, if $W(R(u)) \ge 1$, then adding $1$ to $a_u$ will not change $\connectedcomponent_{\ge 1}(\cG)$ and only increase the effective order by $1$.
    If $W(R(u)) = 0$, then adding $1$ to $a_u$ will also include $R(u)$ in $\connectedcomponent_{\ge 1}(\cG)$ and increase the cardinality of $\connectedcomponent_{\ge 1}(\cG)$ by $1$, which offsets the effect of the increase in $a_u$. 
    Therefore, the conditions for keeping effective order $k$ is $W(R(u)) =0$ for the last term. 

    We next consider the first $|\cV|$ terms in the summation, where each term only adds a new edge $(u, s)$ for some $s \in \cV$ to the graph $\cG$.
    The requirement $a_s \ge 1$ is obvious, which also suggests that $R(s)\in \connectedcomponent_{\ge 1}(\cG)$.
    To this end, it suffices to see under what conditions will adding an edge $(u, s)$ to the graph $\cG$ not affect $|\connectedcomponent_{\ge 1}(\cG)|$. 
    If $s \in R(u)$, then adding an edge $(u, s)$ within the same connected component will not change anything. 
    If $s \notin R(u)$, consider two cases: (i) $W(R(u)) = 0$ and (ii) $W(R(u)) \ge 1$.
    If $W(R(u)) = 0$, then $R(u)\notin \connectedcomponent_{\ge 1}(\cG)$ and adding an edge $(u, s)$ only replaces $R(s)$ with $R(s)\cup R(u)$ in $\connectedcomponent_{\ge 1}(\cG)$, which does not change the cardinality of $\connectedcomponent_{\ge 1}(\cG)$.
    If $W(R(u)) \ge 1$, then $R(u)\in \connectedcomponent_{\ge 1}(\cG)$ and adding an edge $(u, s)$ will merge $R(u)$ and $R(s)$ into a single connected component, which decreases the cardinality of $\connectedcomponent_{\ge 1}(\cG)$ by $1$. 
    Therefore, the conditions for keeping effective order $k$ is $s\in R(u) \lor W(R(u)) =0$.
\end{proof}

\subsubsection{Proof of \Cref{lem:decoupling error}}\label{sec:proof_decoupling error}
        \begin{proof}
            We denote by $f_{=1}$ the terms $p^\top \tilde p$ and $f_{>1}$ the terms of higher effective order. 
            Let $f= f_{=1} + f_{>1}$.
            We denote by $f^* = L^{-1}\cdot\exproduct[\omega][\tilde\omega]$ as the approximation of $f$.
            Invoking \eqref{eq:E[p tildep]-approx} and \eqref{eq:E[f]-moment ub}, we have
            \begin{align*}
                \frac{\sqrt{\EE[(f - f^*)^2]}}{f^*} \le \frac{\sqrt{2\EE[(f_{>1})^2] + 2\EE[(f_{=1} - f^*)^2]}}{L^{-1}\cdot\exproduct[\omega][\tilde\omega]} \le O(L^{-(1-\epsilon_0)/2}) .
            \end{align*}
            For the first case, we denote by $g = \langle v, \barq^{\odot 2}\rangle\cdot \barq_i^2$ and $g^* = \EE[g]=\langle v \rangle + 2 v_i$.
            For the second case, we denote by $g = \barq_i^2$ and $g^* = \EE[g] = 1$. 
            We use the following error decomposition 
            \begin{align*}
                \left| \frac{\EE[f g]}{f^* g^*} - 1  \right|
                &\le  \frac{\left|\EE[(f-f^*)(g-g^*)]\right| + \left|\EE[(f-f^*)g^*]\right| + \left|f^*\EE[(g-g^*)]\right|}{|f^*g^*|} \nend 
                &\le \frac{\sqrt{\EE[(f-f^*)^2] \cdot \EE[(g-g^*)^2]}}{|f^*g^*|} + \frac{\sqrt{\EE[(f-f^*)^2]}}{|f^*|} \nend 
                &\le \left(2 + \frac{\sqrt{\EE[(g-g^*)^2]}}{|g^*|}\right) \cdot O(L^{-(1-\epsilon_0)/2}). 
            \end{align*}
            where in the second inequality, we use the Cauchy-Schwarz inequality and the fact that $\EE[(g-g^*)] = 0$.
            It suffices to verify that ${\sqrt{\EE[(g-g^*)^2]}}/{|g^*|} = O(1)$ for both cases.
            For the first case, we have the following calculation: 
            \begin{align*}
                \EE[(g-g^*)^2] &= \Var[g] = \Var[\langle v_{-i}, \barq_{-i}^{\odot 2}\rangle \barq_i^2 + v_i \barq_i^4] \le 2 \Var[\langle v_{-i}, \barq_{-i}^{\odot 2}\rangle \barq_i^2] + 2 \Var[v_i \barq_i^4] \nend 
                &= 2 \left(\EE[\langle v_{-i}, \barq_{-i}^{\odot 2}\rangle^2] \cdot \EE[\barq_i^4] - \EE[\langle v_{-i}, \barq_{-i}^{\odot 2}\rangle]^2 \cdot \EE[\barq_i^2]^2\right) + 2 \Var[v_i \barq_i^4] \nend 
                &= 2 \left(
                    3 v_{-i}^\top (\vone \vone^\top + 2 I_{d -1}) v_{-i} - \langle v_{-i} \rangle^2 
                \right) + 192 v_i^2 = 4 \langle v_{-i} \rangle^2 + 12 \langle v_{-i}^{\odot 2} \rangle + 192 v_i^2. 
            \end{align*}
            Therefore, 
            \begin{align*}
                \frac{\sqrt{\EE[(g-g^*)^2]}}{|g^*|} 
                &\le \frac{\sqrt{4 \langle v_{-i} \rangle^2 + 12 \langle v_{-i}^{\odot 2} \rangle + 192 v_i^2}}{\left|\langle v \rangle + 2 v_i \right|} \le \frac{4 \langle v_{-i} \rangle + \sqrt{192}v_i}{\langle v \rangle + 2 v_i} \le O(1), 
            \end{align*}
            where we use the fact that $\langle v_{-i}^{\odot 2} \rangle \le \langle v_{-i} \rangle^2$ and $|\langle v_{-i} \rangle| \le |\langle v \rangle|$ if $v\in\RR_+^{d}$ or $v\in\RR_{-}^{d}$.
            For the second case, we have $\EE[(g-g^*)^2] = \Var[g] = \Var[\barq_i^2] = 2$ and hence ${\sqrt{\EE[(g-g^*)^2]}}/{g^*} = \sqrt 2 = O(1)$. 
            Next, we consider only having terms with higher effective order. Using the same definition for $g$ and $g^*$,  it holds that
            \begin{align*}
                \left|\frac{ {\EE[f_{>1} g]}}{  g^*}\right| 
                &\le \left|\frac{ \EE[f_{>1} (g-g^*)] + \EE[f_{>1} g^*]}{  g^*}\right| 
                \le \frac{\sqrt{\EE[(f_{>1})^2] \cdot \EE[(g-g^*)^2]} + g^*\EE[f_{>1} ]}{ |g^*|} \nend 
                &\le \left(1 + \frac{\sqrt{\EE[(g-g^*)^2]}}{|g^*|}\right) \cdot {\sqrt{\EE[(f_{>1})^2]}}{} = O( L^{-2(1-\epsilon)}), 
            \end{align*}
            where in the last inequality, we have already proved that $\sqrt{\EE[(g-g^*)^2]}/{g^*} = O(1)$ for both cases. 
            We also invoke \eqref{eq:E[f]-moment ub} to conclude that $\sqrt{\EE[(f_{>1})^2]} = O( L^{-2(1-\epsilon)})$.
        \end{proof}

\newpage
\end{document}